\documentclass{article}

\PassOptionsToPackage{numbers, compress}{natbib}
\usepackage{iclr2020_conference,times}
\usepackage[utf8]{inputenc} % allow utf-8 input
\usepackage[T1]{fontenc}    % use 8-bit T1 fonts
\usepackage{hyperref}       % hyperlinks
\usepackage{url}            % simple URL typesetting
\usepackage{booktabs}       % professional-quality tables
\usepackage{amsfonts}       % blackboard math symbols
\usepackage{nicefrac}       % compact symbols for 1/2, etc.
\usepackage{microtype}      % microtypography
\usepackage{algorithm, algorithmic}
\usepackage{amsmath, amssymb}
\usepackage{mkolar_definitions}
\usepackage{graphicx}
\usepackage{overpic}
\usepackage{graphicx}
\usepackage{subcaption}
\usepackage{diagbox}
\usepackage[export]{adjustbox}
\usepackage{xspace}
\usepackage{geometry}
\usepackage{wrapfig}
\usepackage{enumitem,kantlipsum}

\title{Deep Batch Active Learning by \\Diverse, Uncertain Gradient Lower Bounds}

\newcount\Comments  % 0 suppresses notes to selves in text
\definecolor{darkgreen}{rgb}{0,0.5,0}
\definecolor{darkred}{rgb}{0.7,0,0}
\definecolor{teal}{rgb}{0.3,0.8,0.8}
\definecolor{orange}{rgb}{1.0,0.5,0.0}
\definecolor{purple}{rgb}{0.8,0.0,0.8}
\newcommand{\kibitz}[2]{\ifnum\Comments=1{\textcolor{#1}{\textsf{\footnotesize #2}}}\fi}

\newcommand{\hide}[1]{}

\newcommand\ouralg{\ensuremath{\operatorname{\textsc{BADGE}}}\xspace}
\newcommand\kmeansp{\ensuremath{\operatorname{\textsc{$k$-means++}}}\xspace}
\newcommand\ffkc{\ensuremath{\operatorname{\textsc{FF-$k$-center}}}\xspace}
\newcommand\rand{\ensuremath{\operatorname{\textsc{Rand}}}\xspace}
\newcommand\conf{\ensuremath{\operatorname{\textsc{Conf}}}\xspace}
\newcommand\coreset{\ensuremath{\operatorname{\textsc{Coreset}}}\xspace}
\newcommand\albl{\ensuremath{\operatorname{\textsc{ALBL}}}\xspace}
\newcommand\marg{\ensuremath{\operatorname{\textsc{Marg}}}\xspace}
\newcommand\entropy{\ensuremath{\operatorname{\textsc{Entropy}}}\xspace}
\newcommand\ner{\mathrm{ne}}
\newcommand\ce{\mathrm{CE}}
\def\given{\vert}

\author{%
  Jordan T. Ash\\
  Princeton University
  \And
  Chicheng Zhang\\
  University of Arizona
  \And
  Akshay Krishnamurthy\\
  Microsoft Research NYC
  \AND
  John Langford\\
  Microsoft Research NYC
  \And
  Alekh Agarwal\\
  Microsoft Research Redmond
}
\begin{document}
\iclrfinalcopy % Uncomment for camera-ready version, but NOT for submission.

\maketitle

\begin{abstract}
We design a new algorithm for batch active learning with deep neural network models. Our algorithm, Batch Active learning by Diverse Gradient Embeddings (\ouralg), samples groups of points that are disparate and high magnitude when represented in a hallucinated gradient space, a strategy designed to incorporate both predictive uncertainty and sample diversity into every selected batch. Crucially, \ouralg trades off between uncertainty and diversity  without requiring any hand-tuned hyperparameters.
While other approaches sometimes succeed for particular batch sizes or architectures, \ouralg consistently performs as well or better, making it a useful option for real world active learning problems. \looseness=-1
%%embeds data points via the gradient of the penultimate layer of the model, using a hallucinated label, a space

%%points in batches, incorporating predictive uncertainty and sample diversity within each batch

%%\ouralg\xspace is minibatch friendly, as is commonly needed for efficient training of such architectures.  It incorporates both predictive uncertainty and sample diversity within each batch of selected data while retaining the ability to make unbiased estimates of performance on the \emph{original} data distribution through importance weighting.  We show that our algorithm consistently performs as well as or better than other methods empirically even when varying batch size, architecture, and dataset.
\end{abstract}
\vskip -2cm
%\vskip -1cm
\section{Introduction}
%\vskip -0.3cm
In recent years, deep neural networks have produced state-of-the-art results on a variety of important supervised learning tasks. However, many of these successes have been limited to domains where large amounts of labeled data are available.
%\chicheng{Changed in regard to reviewer 1's comments}
A promising approach for minimizing labeling effort is \emph{active learning}, a learning protocol where labels can be requested by the algorithm in a sequential, feedback-driven fashion.
Active learning algorithms aim to identify and label only maximally-informative samples, so that a high-performing classifier can be trained with minimal labeling effort.
As such, a robust active learning algorithm for deep neural networks may considerably expand the domains in which these models are applicable.\looseness=-1

How should we design a practical, general-purpose, label-efficient active learning algorithm for deep neural networks?
Theory for active learning suggests a version-space-based approach~\citep{cohn1994improving,A^2}, which explicitly or implicitly maintains a set of plausible models, and queries examples for which these models make different predictions.
%\chicheng{explained version-spaced algorithm.}
But when using highly expressive models like neural networks, these algorithms degenerate to querying every example.
Further, the computational overhead of training deep neural networks precludes approaches that update the model to best fit data after each label query, as is often done (exactly or approximately) for linear methods~\citep{beygelzimer2010agnostic, cesa2009robust}.
Unfortunately, the theory provides little guidance for these models.

One option is to use the network's uncertainty to inform a query strategy, for example by labeling samples for which the model is least confident. In a batch setting, however, this creates a pathological scenario where data in the batch are nearly identical, a clear inefficiency. Remedying this issue, we could select samples to maximize batch diversity, but this might choose points that provide little new information to the model.\looseness=-1

For these reasons, methods that exploit just uncertainty or diversity do not consistently work well across model architectures, batch sizes, or datasets. An algorithm that performs well when using a ResNet, for example, might perform poorly when using a multilayer perceptron. A diversity-based approach might work well when the batch size is very large, but poorly when the batch size is small. Further, what even constitutes a ``large'' or ``small'' batch size is largely a function of the statistical properties of the data in question. These weaknesses pose a major problem for real, practical batch active learning situations, where data are unfamiliar and potentially unstructured. There is no way to know which active learning algorithm is best to use. \looseness=-1%An ideal active learning approach should be able to perform well regardless of environmental conditions.

Moreover, in a real active learning scenario, every change of hyperparameters typically causes the algorithm to label examples not chosen under other hyperparameters, provoking substantial labeling inefficiency. That is, hyperparameter sweeps in active learning can be label expensive. As a result, active learning algorithms need to ``just work”, given fixed hyperparameters, to a greater extent than is typical for supervised learning.\looseness=-1

Based on these observations, we design an approach which creates diverse batches of examples about which the current model is uncertain. We measure uncertainty as the gradient magnitude with respect to parameters in the final (output) layer, which is computed using the most likely label according to the model. To capture diversity, we collect a batch of examples where these gradients span a diverse set of directions. More specifically, we build up the batch of query points based on these hallucinated gradients using the \kmeansp initialization~\citep{arthur2007k}, which simultaneously captures both the magnitude of a candidate gradient and its distance from previously included points in the batch. We name the resulting approach Batch Active learning by Diverse Gradient Embeddings (\ouralg).

We show that \ouralg is robust to architecture choice, batch size, and dataset, generally performing as well as or better than the best baseline across our experiments, which vary all of the aforementioned environmental conditions. We begin by introducing our notation and setting, followed by a description of the \ouralg algorithm in Section~\ref{sec:alg} and experiments in Section~\ref{sec:experiments}.  We defer our discussion of related work to Section~\ref{sec:related}.
\vskip -0.5cm

% an efficient active learning algorithm that addresses these challenges and show that it works well empirically under varied architectures and batch sizes compared to strong baseline approaches.

%\vskip -0.5cm
\section{Notation and setting}
%\vskip -0.3cm
Define $[K] := \cbr{1,2,\ldots,K}$.
Denote by $\Xcal$ the instance space and by $\Ycal$ the label space. In this work
we consider multiclass classification, so $\Ycal = [K]$.
%(sometimes it will be convenient to specialize to $\Ycal = \cbr{-1,1}$ for binary classification).
Denote by $D$ the distribution from which examples are drawn, by $D_\Xcal$ the unlabeled data distribution, and by $D_{\Ycal \given \Xcal}$ the conditional distribution over labels given examples.
We consider the pool-based active learning setup, where the learner receives an unlabeled dataset $U$ sampled according to $D_\Xcal$ and can request labels sampled according to $D_{\Ycal\given\Xcal}$ for any $x \in U$.
We use $\EE_D$ to denote expectation under the data distribution $D$.
Given a classifier $h~:~\Xcal\to\Ycal$, which maps examples to labels, and a labeled example $(x,y)$, we denote the $0/1$ error of $h$ on $(x,y)$ as $\ell_{01}(h(x),y) = I(h(x) \neq y)$. The performance of a classifier $h$ is measured by its expected $0/1$ error, i.e. $\EE_D[\ell_{01}(h(x),y)] = \text{Pr}_{(x,y)\sim D} (h(x) \ne y)$.
The goal of pool-based active learning is to find a classifier with a small expected $0/1$ error using as few label queries as possible.
Given a set $S$ of labeled examples $(x,y)$, where each $x \in S$ is picked from $U$, followed by a label query, we use $\EE_S$ as the sample averages over $S$.

In this paper, we consider classifiers $h$ parameterized by underlying neural networks $f$ of fixed architecture, with the weights in the network denoted by $\theta$. We abbreviate the classifier with parameters $\theta$ as $h_\theta$ since the architectures are fixed in any given context, and our classifiers take the form $h_\theta(x) = \argmax_{y \in [K]} f(x;\theta)_y$, where $f(x;\theta) \in \RR^K$ is a probability vector of scores assigned to candidate labels, given the example $x$ and parameters $\theta$. We optimize the parameters by minimizing the cross-entropy loss $\EE_S[\ell_{\ce}(f(x;\theta),y)]$ over the labeled examples, where
$\ell_{\ce}(p,y) = \sum_{i=1}^K I(y = i) \ln\nicefrac{1}{p_i} = \ln\nicefrac{1}{p_y}$.
%\vskip -0.5cm
\newpage

\section{Algorithm}
%\vskip -0.4cm
\label{sec:alg}

\begin{algorithm}[t]
\begin{algorithmic}[1]
\REQUIRE Neural network $f(x;\theta)$, unlabeled pool of examples $U$, initial number of examples $M$,
number of iterations $T$, number of examples in a batch $B$.
\STATE Labeled dataset $S \gets$ $M$ examples drawn uniformly at random from $U$ together with queried labels.
\STATE Train an initial model $\theta_1$ on $S$ by minimizing $\EE_S[\ell_{\ce}(f(x;\theta),y)]$.

\FOR{$t = 1,2,\ldots,T$:}

\STATE For all examples $x$ in $U \setminus S$:

\begin{enumerate}
%\arg\max_{i \in [K]} f(x,\theta_t)_i
\item Compute its hypothetical label $\hat{y}(x) = h_{\theta_t}(x)$.
%\item Compute gradient embedding $g_x = \frac{\partial}{\partial W} \ell_{\ce}(f(x;\theta),\hat{y}(x)) \vert_{\theta = \theta_t}$, where $W$ is the weight of the last layer of $\theta$.

\item Compute gradient embedding $g_x = \frac{\partial}{\partial \theta_\text{out}} \ell_{\ce}(f(x;\theta),\hat{y}(x)) \vert_{\theta = \theta_t}$, where $\theta_{\text{out}}$ refers to parameters of the final (output) layer.
\end{enumerate}

\STATE Compute $S_t$, a random subset of $U \setminus S$, using the \kmeansp seeding algorithm on $\cbr{g_x: x \in U \setminus S}$ and query for their labels.

\STATE $S \gets S \cup S_t$.

\STATE Train a model $\theta_{t+1}$ on $S$ by minimizing $\EE_S[\ell_{\ce}(f(x;\theta),y)]$.

\ENDFOR

\RETURN Final model $\theta_{T+1}$.

\end{algorithmic}
\caption{\ouralg: Batch Active learning by Diverse Gradient Embeddings}
\label{alg:main}
\end{algorithm}

\ouralg, described in Algorithm~\ref{alg:main}, starts by
drawing an initial set of $M$ examples uniformly at random from $U$ and asking for their labels.
It then proceeds iteratively, performing two main computations at each step $t$: a \emph{gradient embedding} computation and a \emph{sampling} computation.
Specifically, at each step $t$, for every $x$ in the pool $U$, we compute the label $\hat{y}(x)$ preferred by the current model, and the gradient $g_x$ of the loss on $(x,\hat{y}(x))$ with respect to the parameters of the last layer of the network.
Given these gradient embedding vectors $\cbr{g_x: x \in U}$, \ouralg selects a set of points by sampling via the \kmeansp initialization scheme~\citep{arthur2007k}.
The algorithm queries the labels of these examples, retrains the model, and repeats. %proceeds to the next iteration.\looseness=-1

We now describe the main computations --- the embedding and sampling steps --- in more detail.
%% \paragraph{Gradient lower bound computation.}
%% At round $t$, for every $x$ in the pool $U$, \ouralg computes the label $\hat{y}(x)$ predicted by the current model.
%% Then, it computes the gradient of the loss on $(x, \hat{y}(x))$ with respect to the model parameters $\theta$ in the last layer of the network. We call this gradient $g_x$.
%% It then proceeds
%% iteratively: at each step $t$, for every $x$ in the pool $U$, it computes the label $\hat{y}(x)$ preferred by the current model.  Then, it computes a gradient $g_x$, which is the gradient of the loss on $(x, \hat{y}(x))$ with respect to the model parameters $\theta$ in the last layer of the network.
%% Given all the gradient embedding vectors $\cbr{g_x: x \in U}$, the algorithm selects a set of points using a sampler that
%% favors both length and diversity. The algorithm then queries the
%% labels of these examples, retrains the model, and proceeds to the next
%% iteration. The learning process finishes after the number of iterations reach a
%% pre-specified threshold $T$.

%as measured through a partial derivative of the loss at this prediction, times a vector capturing the direction in which the parameters of the last layer's weights are updated

\paragraph{The gradient embedding.}
Since deep neural networks are optimized using gradient-based methods, we capture uncertainty about an example through the lens of gradients. In particular, we consider the model uncertain about an example if knowing the label induces a large gradient of the loss with respect to the model parameters and hence a large update to the model. A difficulty with this reasoning is that we need to know the label to compute the gradient. As a proxy, we compute the gradient as if the model's current prediction on the example is the true label. We show in Proposition~\ref{prop:grad-lb} that, assuming a common structure satisfied by most natural neural networks, the gradient norm with respect to the last layer using this label provides a lower bound on the gradient norm induced by any other label. In addition, under that assumption, the length of this hypothetical gradient vector captures the uncertainty of the model on the example: if the model is highly certain about the example's label, then the example's gradient embedding will have a small norm, and vice versa for samples where the model is uncertain (see example below).
Thus, the gradient embedding conveys information both about the model's uncertainty and potential update direction upon receiving a label at an example.\looseness=-1

\begin{figure}
\centering
\begin{subfigure}[b]{0.3\linewidth}
  \includegraphics[trim={0cm 0cm 0cm 0cm}, clip, width=0.9\textwidth]{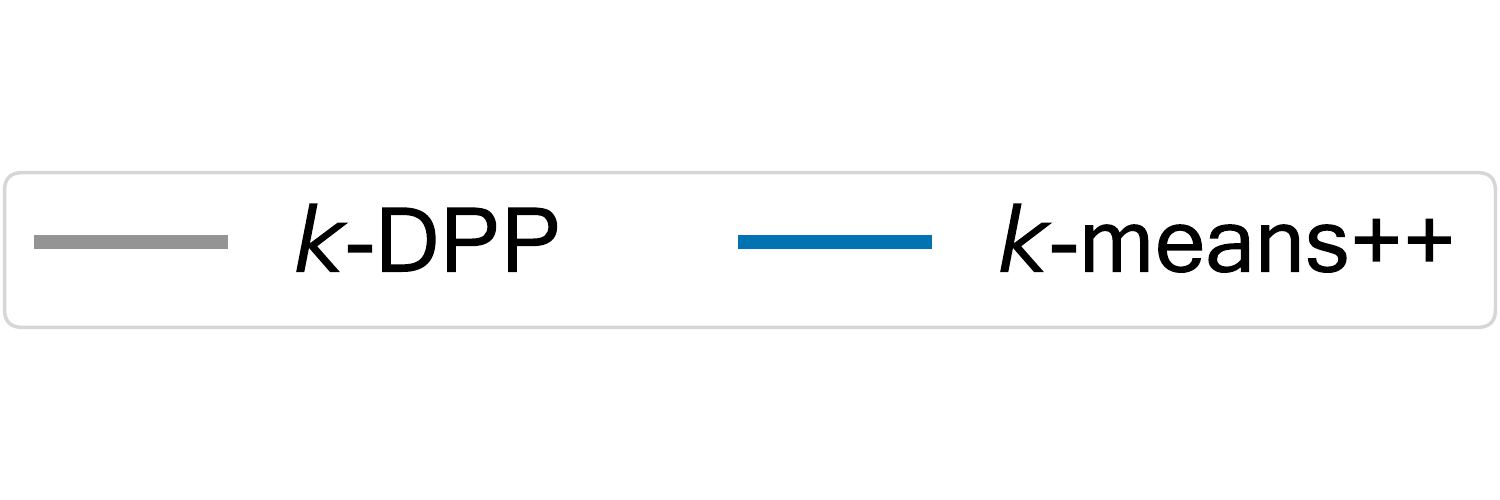}
  \vspace{-0.5cm}
\end{subfigure}
\\
\centering
\begin{subfigure}[b]{0.32\linewidth}
\includegraphics[trim={0.35cm 0cm 2.6cm 0cm}, clip, width=\textwidth]{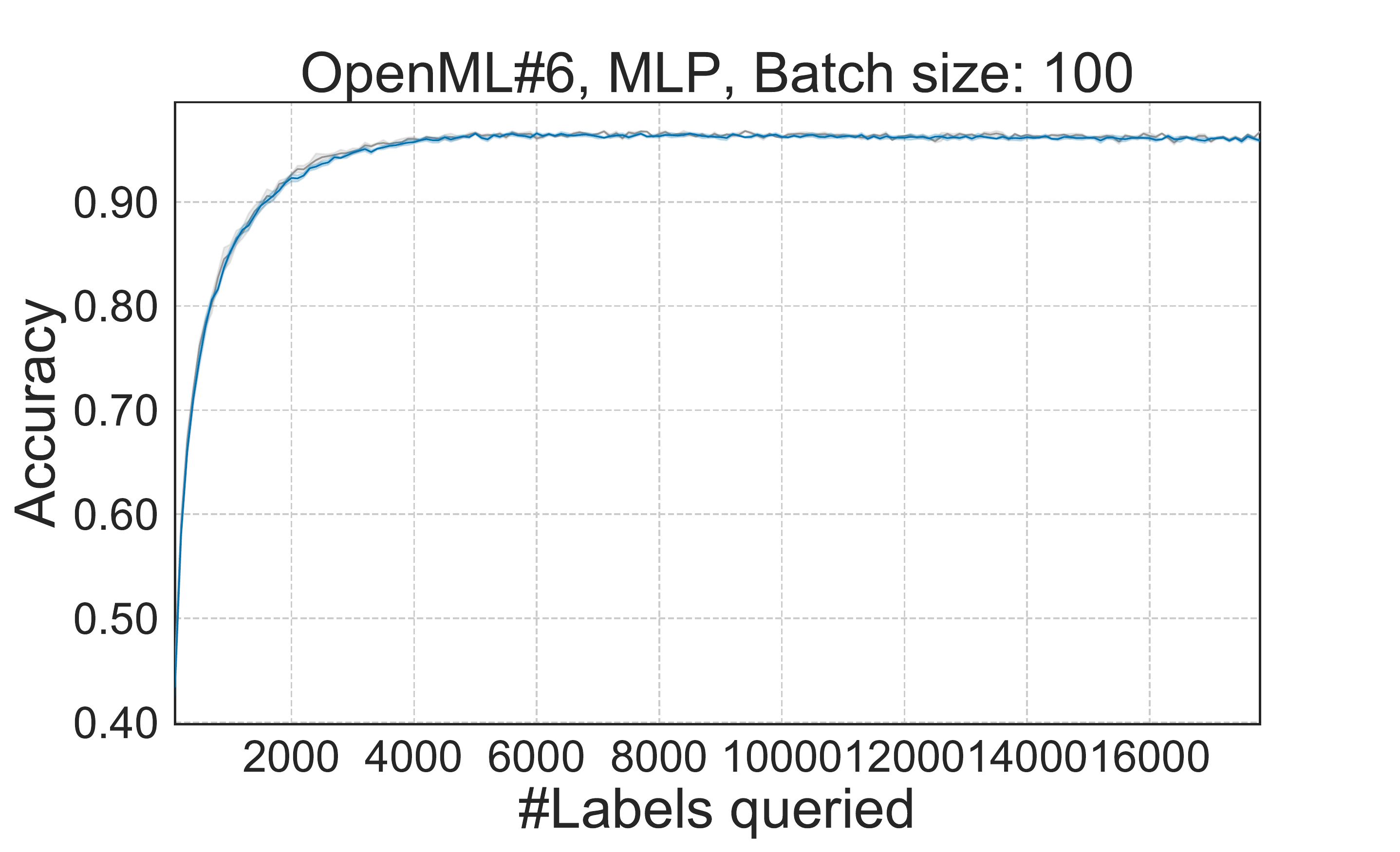}
\end{subfigure}
\hfill
\begin{subfigure}[b]{0.32\linewidth}
  \includegraphics[trim={0.35cm 0cm 2.6cm 0cm}, clip, width=\textwidth]{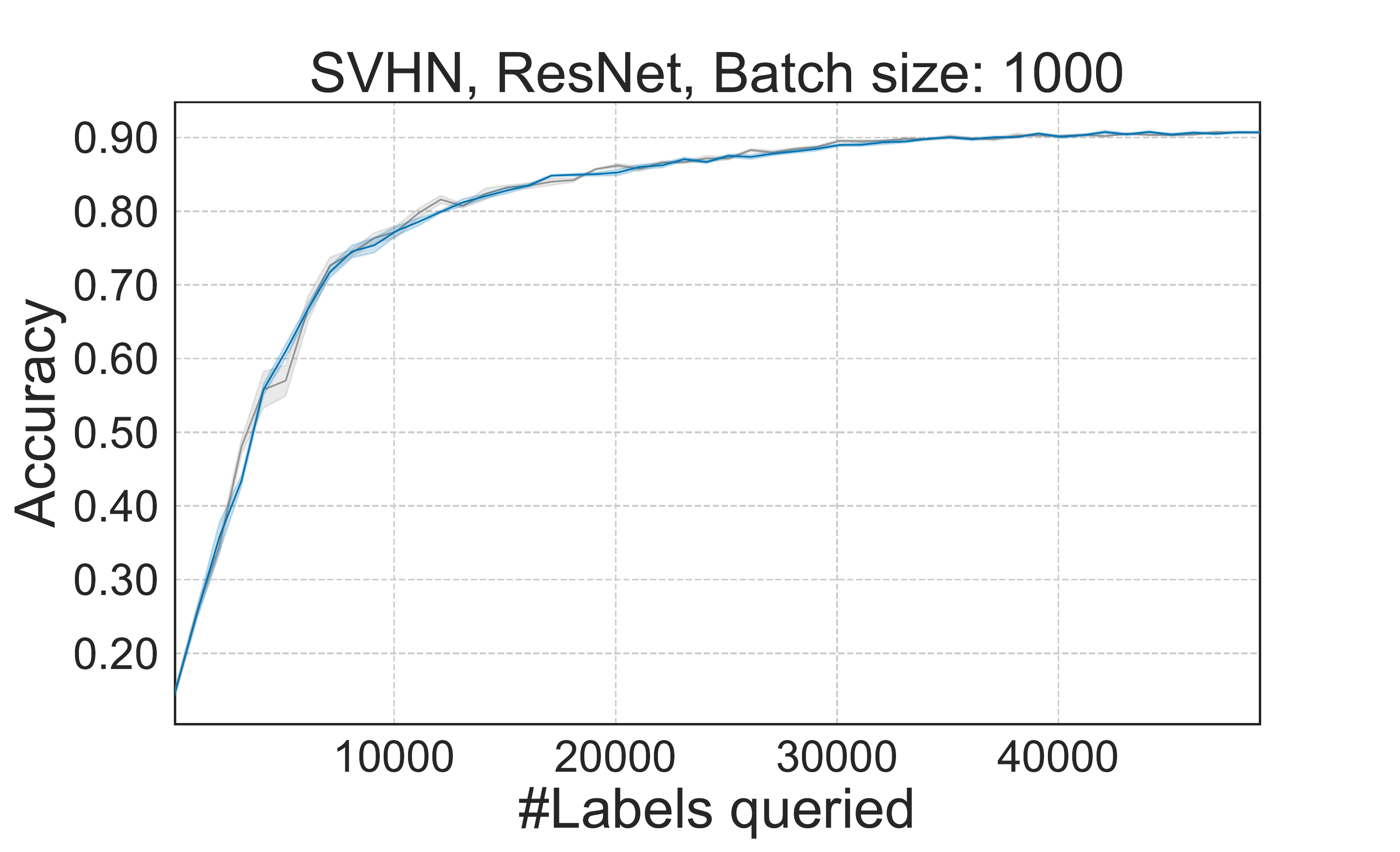}
  %\caption{}
\end{subfigure}
\hfill
\begin{subfigure}[b]{0.32\linewidth}
  \includegraphics[trim={0.35cm 0cm 2.6cm 0cm}, clip, width=\textwidth]{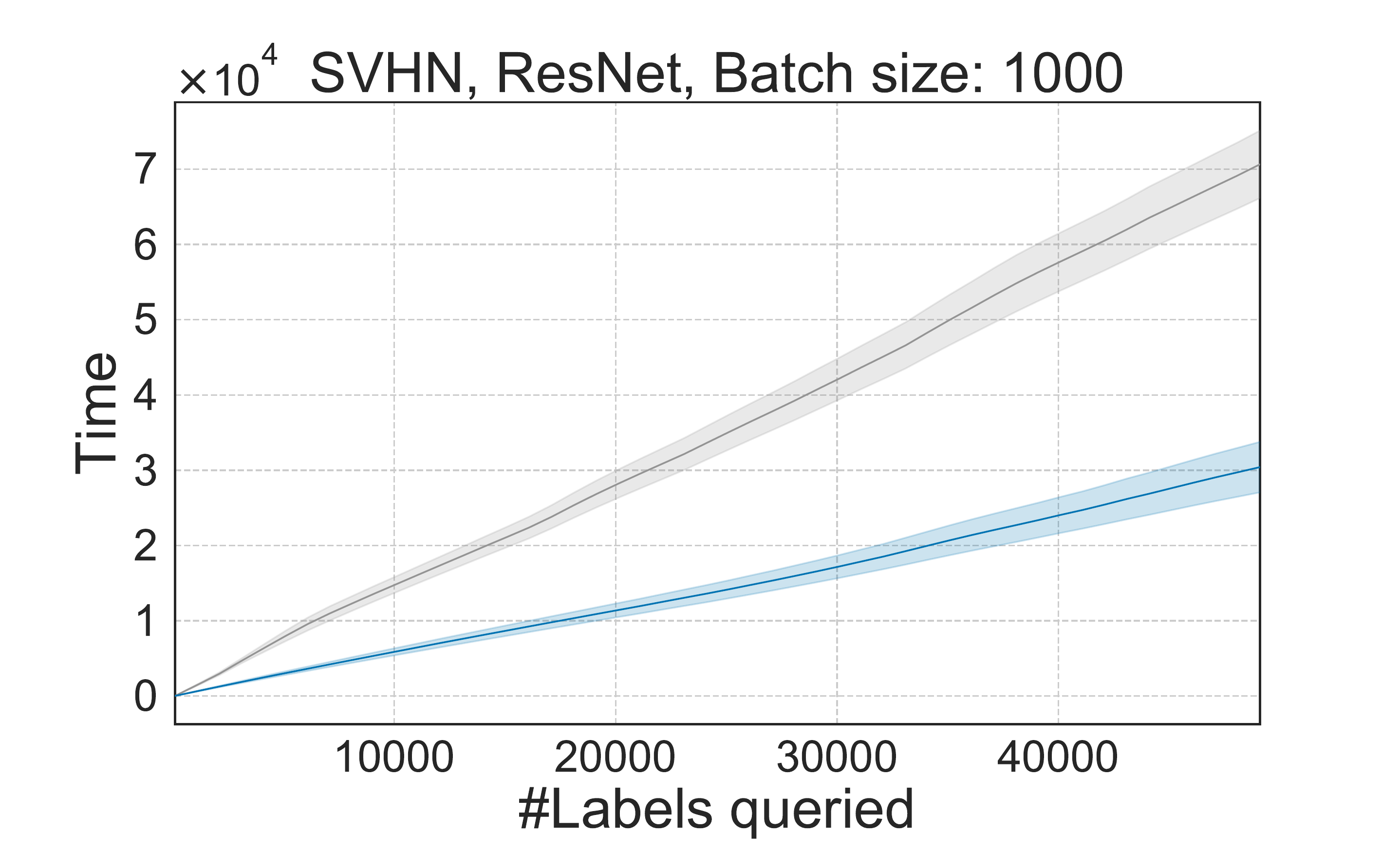}
   %\caption{}
\end{subfigure}
%\vskip -0.25cm
\caption{\textbf{Left and center}: Learning curves for \kmeansp and $k$-DPP sampling with gradient embeddings for different scenarios. The performance of the two sampling approaches nearly perfectly overlaps. \textbf{Right:} A run time comparison (seconds) corresponding to the middle scenario. Each line is the average over five independent experiments. Standard errors are shown by shaded regions.\looseness=-1}
\vskip -0.5cm
\label{fig:k-dpp-k-means}
\end{figure}
\vskip -0.9cm
\paragraph{The sampling step.}
We want the newly-acquired labeled samples to induce large and diverse changes to the model. To this end, we want the selection procedure to favor both sample magnitude and batch diversity. Specifically, we want to avoid the pathology of, for example, selecting a batch of $k$ similar samples where even just a single label could alleviate our uncertainty on all remaining $(k-1)$ samples.

A natural way of making this selection without introducing additional hyperparameters is to sample from a $k$-Determinantal Point Process ($k$-DPP; ~\citep{kulesza2011k}). That is, to select a batch of $k$ points with probability proportional to the determinant of their Gram matrix. Recently,~\cite{derezinski2018reverse} showed that in experimental design for least square linear regression settings, learning from samples drawn from a $k$-DPP can have much smaller mean square prediction error than learning from iid samples.
In this process, when the batch size is very low, the selection will naturally favor points with a large length, which corresponds to uncertainty in our space. When the batch size is large, the sampler focuses more on diversity because linear independence, which is more difficult to achieve for large $k$,  is required to make the Gram determinant non-zero\looseness=-1.

Unfortunately, sampling from a $k$-DPP is not trivial. Many sampling algorithms~\citep{kang2013fast, anari2016monte} rely on MCMC, where mixing time poses a significant computational hurdle. The state-of-the-art algorithm of~\citet{derezinski2018fast} has a high-order polynomial running time in the batch size and the embedding dimension.
%\footnote{Specifically, its running time is $O(n d + d^3 k^2)$, where $n$ is the pool size, and $d$ is the embedding dimension.}
To overcome this computational hurdle, we suggest instead sampling using the \kmeansp seeding algorithm~\citep{arthur2007k}, originally made to produce a good initialization for $k$-means clustering.
\kmeansp seeding selects centroids by iteratively sampling points in proportion to their squared distances from the nearest centroid that has already been chosen, which, like a $k$-DPP, tends to select a diverse batch of high-magnitude samples.
For completeness, we give a formal description of the \kmeansp seeding algorithm in Appendix~\ref{sec:kmeansp}.
\paragraph{Example: multiclass classification with softmax activations.} Consider a neural network $f$ where the last nonlinearity is a softmax, i.e. $\sigma(z)_i = \nicefrac{e^{z_i}}{\sum_{j=1}^K e^{z_j}}$.
Specifically, $f$ is parametrized by $\theta = (W, V)$, where $\theta_{\text{out}} = W = (W_1, \ldots, W_K)^\top \in \RR^{K \times d}$ are the weights of the last layer, and $V$ consists of weights of all previous layers. This means that $f(x;\theta) = \sigma(W \cdot z(x;V))$, where $z$ is the nonlinear function that maps an input $x$ to the output of the network's penultimate layer. Let us fix an unlabeled sample $x$ and define $p_i = f(x;\theta)_i$.
With this notation, we have
\[ \ell_{\ce}(f(x;\theta), y) =  \ln\rbr{\sum_{j=1}^K e^{W_j \cdot z(x;V)}} - W_y \cdot z(x;V). \]

Define $g^y_x = \frac{\partial}{\partial W} \ell_{\ce}(f(x;\theta), y)$ for a label $y$ and $g_x = g_x^{\hat y}$ as the gradient embedding in our algorithm, where $\hat y = \argmax_{i \in [K]} p_i$. Then the $i$-th block of $g_x$ (i.e. the gradients corresponding to label $i$) is
\begin{equation}
   (g_x)_i = \frac{\partial}{\partial W_i} \ell_{\ce}(f(x;\theta), \hat y) = ( p_i - I(\hat{y} = i) ) z(x;V).
\label{eqn:grad-i-blk}
\end{equation}
Based on this expression, we can make the following observations:

\begin{enumerate}
  \item Each block of $g_x$ is a scaling of $z(x;V)$, which is the output of the penultimate layer of the
  network. In this respect, $g_x$ captures $x$'s representation information similar to that of~\cite{sener2018active}.

  \item Proposition~\ref{prop:grad-lb} below shows that the norm of $g_x$ is a lower bound on the norm of the loss gradient induced by the example with true label $y$ with respect to the weights in the last layer, that is $\|g_x\| \leq \|g_x^y\|$. This suggests that the norm of $g_x$ conservatively estimates the example's influence on the current model.

  \item If the current model $\theta$ is highly confident about $x$, i.e. vector $p$ is skewed towards a
  standard basis vector $e_j$, then $\hat{y} = j$, and vector $(p_i - I(\hat{y}=i))_{i=1}^K$ has a small length. Therefore, $g_x$ has a small length as well. Such high-confidence examples tend to have gradient embeddings of small magnitude, which are unlikely to be repeatedly selected by \kmeansp at iteration $t$.
\end{enumerate}

\begin{proposition}
For all $y \in \cbr{1,\ldots,K}$, let $g_x^y = \frac{\partial}{\partial W} \ell_{\ce}(f(x;\theta), y)$. Then
\[ \| g_x^y \|^2 = \Big(\sum_{i=1}^K p_i^2 + 1 - 2p_y \Big) \| z(x;V) \|^2.  \]
Consequently, $\hat{y} = \argmin_{y \in [K]} \| g_x^y \|$.
\label{prop:grad-lb}
\end{proposition}
\begin{proof}
Observe that by Equation~\eqref{eqn:grad-i-blk},
\[ \| g_x^y \|^2 = \sum_{i=1}^K \big(p_i - I(y = i)\big)^2 \| z(x;V) \|^2 = \Big(\sum_{i=1}^K p_i^2 + 1 - 2p_y\Big) \| z(x;V) \|^2. \]
The second claim follows from the fact that $\hat{y} = \argmax_{y \in [K]} p_y$.
\end{proof}
\vspace{-0.3cm}
%Further, many such algorithms rely on a matrix inversion that can sometimes be numerically unstable, and, even worse, if $k$, the batch size, is larger than the embedding dimension of the data, the Gram determinant of any subset is guaranteed to be zero (when using a linear kernel).
%\chicheng{The $k>d$ regime is not a problem I believe - Derenzinski and Warmuth's reverse iterative volume sampling is for this regime.} \jordan{Really? I feel like you can't change this property without using either a kernel or adding a constant to each element on the diagonal. Feel free to remove.} \alekh{Is this resolved?}

This simple sampler tends to produce diverse batches similar to a $k$-DPP.  As shown in Figure~\ref{fig:k-dpp-k-means}, switching between the two samplers does not affect the active learner's statistical performance but greatly improves its computational performance. Appendix~\ref{sec:comp} compares run time and test accuracy for both \kmeansp and $k$-DPP based sampling based on the gradient embeddings of the unlabeled examples.

Figure~\ref{fig:k-dpp-2} illustrates the batch diversity and average gradient magnitude per selected batch for a variety of sampling strategies. As expected, both $k$-DPPs and \kmeansp
%\chicheng{shall we use \ouralg, or \kmeansp?}
tend to select samples that are diverse (as measured by the magnitude of their Gram determinant) and high magnitude. Other samplers, such as furthest-first traversal for $k$-Center clustering (\ffkc), do not seem to have this property.
The \ffkc algorithm is the sampling choice of the \coreset approach to active learning, which we describe in the proceeding section~\citep{sener2018active}. Appendix~\ref{sec:batchdiv} discusses diversity with respect to uncertainty-based approaches. \looseness=-1

%seeks a set-covering sampling of the penultimate layer representation.\chicheng{I think we shouldn't mention the penultimate layer embedding, because the plot actually uses gradient embedding.} %We compare to this method in the proceeding section. In Appendix~\ref{sec:batchdiv}, we also give a comparison of these methods with uncertainty sampling.

%a multiclass classification problem (where $\Ycal = [K]$), and
\begin{figure}
\centering
\begin{subfigure}[b]{0.65\linewidth}
  \includegraphics[trim={0cm 0cm 0cm 0cm}, width=\textwidth]{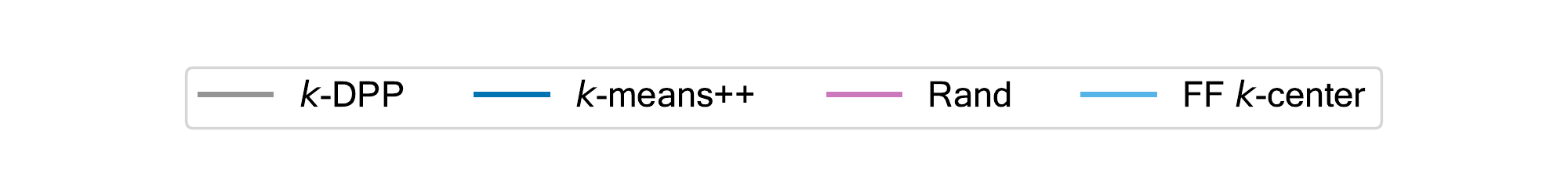}
\end{subfigure}
\\
\begin{subfigure}[b]{0.32\linewidth}
\includegraphics[trim={0cm 0cm 2cm 0cm}, clip, width=\textwidth]{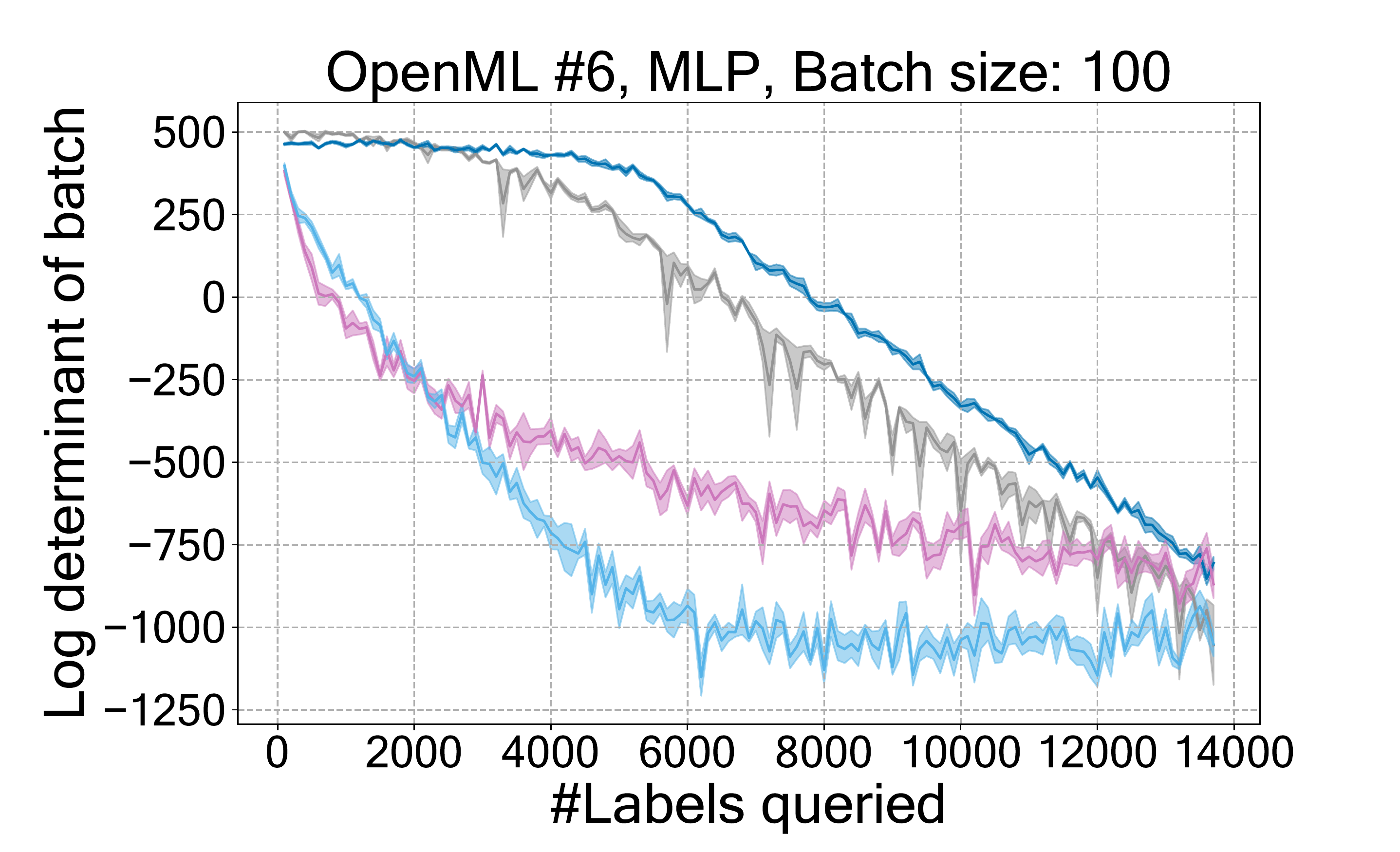}
%\caption{}
\end{subfigure}
\hfill
\begin{subfigure}[b]{0.32\linewidth}
  \includegraphics[trim={0cm 0cm 2cm 0cm}, clip, width=\textwidth]{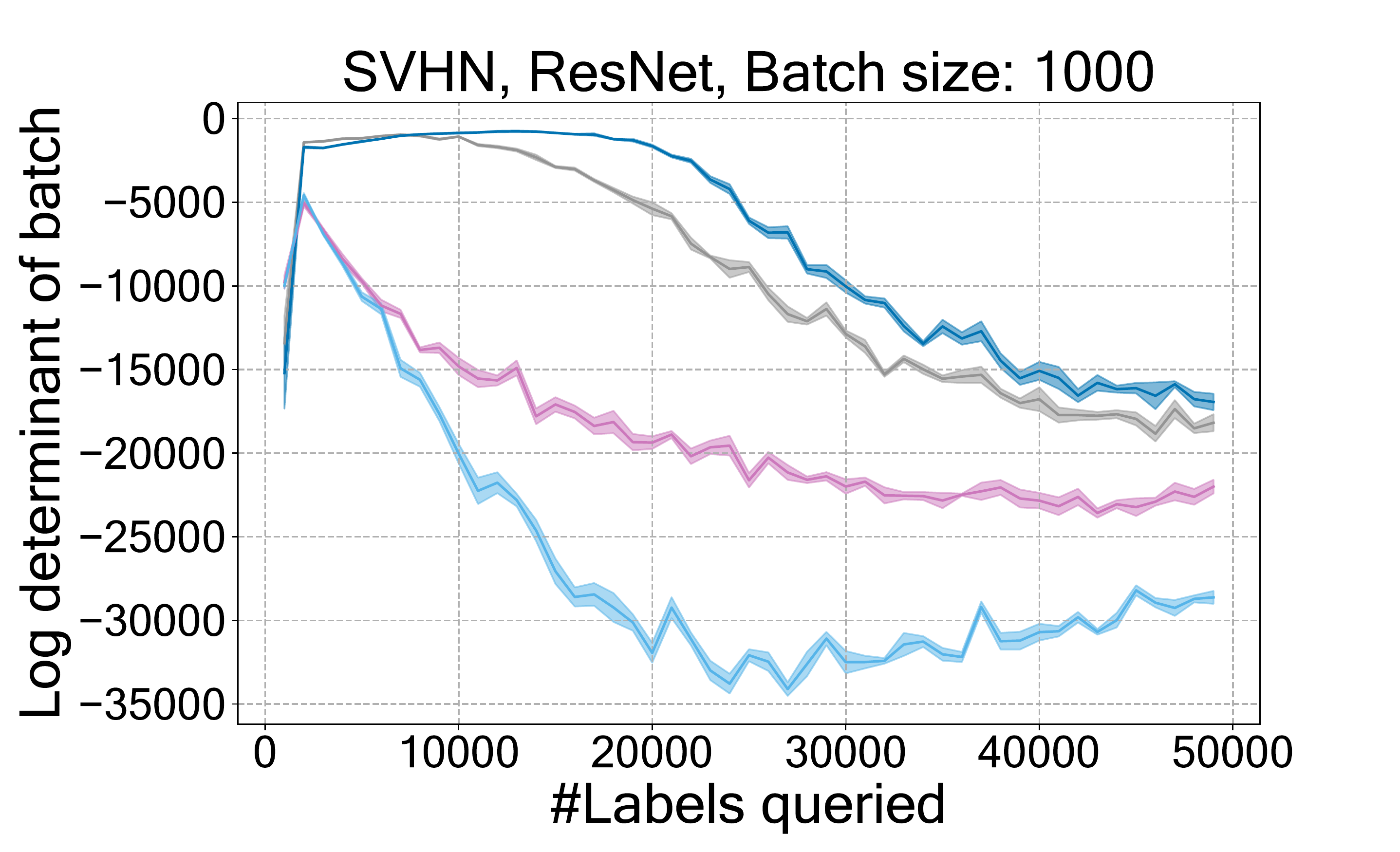}
%\caption{}
\end{subfigure}
\hfill
\begin{subfigure}[b]{0.32\linewidth}
  \includegraphics[trim={0cm 0cm 2cm 0cm}, clip, width=\textwidth]{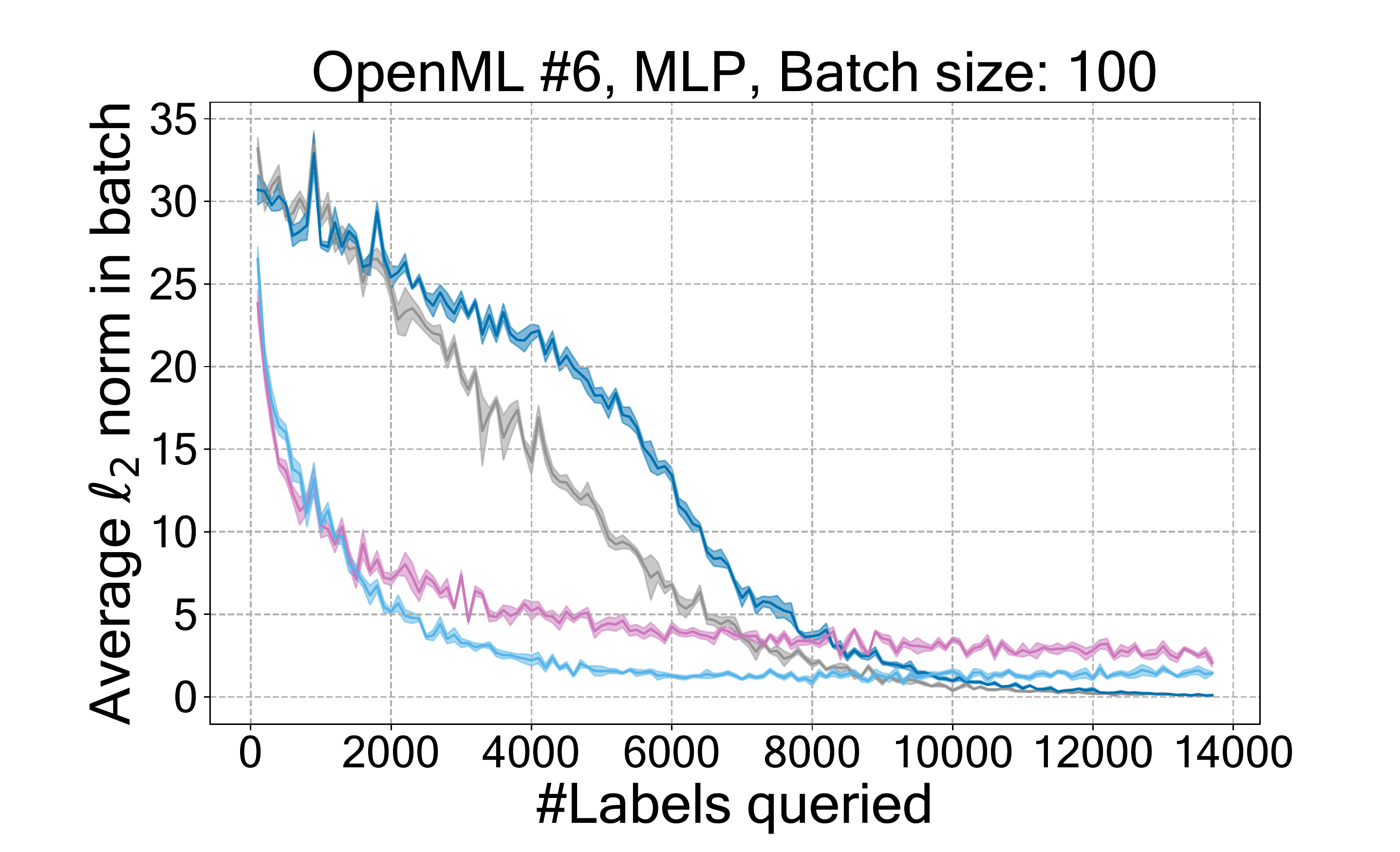}
 % \caption{}
\end{subfigure}
% \hfill
% \begin{subfigure}[b]{0.245\linewidth}
%   \includegraphics[trim={1.3cm 0cm 0.2cm 0cm}, clip, width=\textwidth]{figs/diversity/wo_conf/SVHN_1000_rn_norm.pdf}
%   \caption{}
% \end{subfigure}
\caption{A comparison of batch selection algorithms using our gradient embedding. \textbf{Left and center:} Plots showing the log determinant of the Gram matrix of the selected batch of gradient embeddings as learning progresses. \textbf{Right:} The average embedding magnitude (a measurement of predictive uncertainty) in the selected batch. The \ffkc sampler finds points that are not as diverse or high-magnitude as other samplers. Notice also that \kmeansp tends to actually select samples that are both more diverse and higher-magnitude than a $k$-DPP, a potential pathology of the $k$-DPP's degree of stochastisity. Standard errors are shown by shaded regions. \looseness=-1}
% \vskip -0.25cm
\label{fig:k-dpp-2}
\end{figure}

Appendix~\ref{sec:bin-lr} provides further justification for why \ouralg yields better updates than
vanilla uncertainty sampling in the special case of binary logistic regression ($K=2$ and $z(x;V) = x$).
%\vskip -.3cm

\vspace{-0.2cm}
\section{Experiments}
\vskip -0.3cm
\label{sec:experiments}
We evaluate the performance of \ouralg against several algorithms from the literature. In our experiments, we seek to answer the following question: How robust are the learning algorithms to choices of neural network architecture, batch size, and dataset?

To ensure a comprehensive comparison among all algorithms, we evaluate them in a batch-mode active learning setup with $M = 100$ being the number of initial random labeled examples and batch size $B$ varying from $\cbr{100, 1000, 10000}$.
The following is a list of the baseline algorithms evaluated; the first performs representative sampling, the next three are uncertainty based, the fifth is a hybrid of representative and uncertainty-based approaches, and the last is traditional supervised learning.\looseness=-1
\begin{enumerate}%[wide, labelwidth=!, labelindent=0pt]
\item \coreset: A diversity-based approach using coreset selection. The embedding
of each example is computed by the network's penultimate layer
and the samples at each round are selected using a greedy furthest-first traversal
conditioned on all labeled examples~\citep{sener2018active}.\looseness=-1

\item \conf (Confidence Sampling): An uncertainty-based active learning algorithm that selects $B$
examples with smallest predicted class probability, $\max_{i=1}^K f(x;\theta)_i$~\citep[e.g.][]{wang2014new}. \looseness=-1

\item \marg (Margin Sampling): An uncertainty-based active learning algorithm that selects the bottom $B$
examples sorted according to the example's multiclass margin, defined as
$f(x;\theta)_{\hat{y}} - f(x;\theta)_{y'}$, where $\hat{y}$ and $y'$ are the indices of the
largest and second largest entries of $f(x;\theta)$~\citep{roth2006margin}.\looseness=-1

\item \entropy: An uncertainty-based active learning algorithm that selects the top $B$
examples according to the entropy of the example's predictive class probability distribution,
defined as $H((f(x;\theta)_y)_{y=1}^K)$, where $H(p) = \sum_{i=1}^K p_i \ln\nicefrac{1}{p_i}$~\citep{wang2014new}.

\item \albl (Active Learning by Learning): A bandit-style meta-active learning algorithm that selects between \coreset and \conf at every round~\citep{hsu2015active}.

%\item DFAL~\citep{ducoffe2018adversarial}: uncertainty-based active learning algorithm that approximates the distance
%between an example and the current decision boundary by computing the distance between
%this example and its adversarial example obtained by DeepFool.

\item \rand: The naive baseline of randomly selecting $k$ examples to query at each round.
\end{enumerate}
%?  Second, how robust are the learning algorithms to the choices of
%\vskip -.5cm
We consider three neural network architectures: a two-layer Perceptron with ReLU activations (MLP), an 18-layer convolutional ResNet~\citep{resnet}, and an 11-layer VGG network~\citep{vgg}. We evaluate our algorithms using three image datasets, SVHN~\citep{svhn}, CIFAR10~\citep{cifar} and MNIST~\citep{mnist}
\footnote{Because MNIST is a dataset that is extremely easy to classify, we only use MLPs, rather than convolutional networks, to better study the differences between active learning algorithms.}, and four non-image datasets from the OpenML repository (\#6, \#155, \#156, and \#184).
\footnote{The OpenML datasets are from \url{openml.org} and
are selected on two criteria: first, they have at least 10000 samples; second,
neural networks have a significantly smaller test error rate when compared to
linear models.}
We study each situation with 7 active learning algorithms, including \ouralg, making for 231 total experiments.

For the image datasets, the embedding dimensionality in the MLP is 256.
For the OpenML datasets, the embedding dimensionality of the MLP is 1024, as more capacity helps the model fit training data.
We fit models using cross-entropy loss and the Adam variant of SGD until  training accuracy exceeds 99\%. We use a learning rate of $0.001$ for image data and of $0.0001$ for non-image data. We avoid warm starting and retrain models from scratch every time new samples are queried~\citep{ash2019difficulty}. All experiments are repeated five times. No learning rate schedules or data augmentation are used. Baselines use implementations from the libact library~\citep{yang2017libact}. All models are trained in PyTorch~\citep{pytorch}.\looseness=-1

%, VGG~\citep{vgg}
%\vskip -1cm
% \begin{wrapfigure}{l}{0.5\textwidth}
% \vskip -2cm
%   \begin{center}
%     \includegraphics[trim={4.2cm 1cm 1cm 1cm}, clip, width=0.9\linewidth]{figs/comp_matrices/comp_matrix_overall.pdf}
%   \end{center}
% \vskip -1cm
%   \caption{Pairwise comparisons over all experiments conducted. As described, element $i$, $j$ corresponds roughly to the number of times algorithm $i$ outperforms algorithm $j$. Column-wise averages show average performance (lower is better).
%   \label{fig:penalty}
%  }
% \end{wrapfigure}
%\vskip -.1cm

% \begin{figure}
% \label{fig:penalty}
% \centering
% %\hfill
% %\subfigure{\includegraphics[trim={1cm 0cm 1cm 0cm}, clip, width=.48\textwidth]{figs/comp_matrices/linear.pdf}}
% \includegraphics[trim={1cm 0cm 1cm 0cm}, clip, width=.48\textwidth]{figs/comp_matrices/comp_matrix_overall.pdf}
% \caption{Pairwise comparisons over all experiments conducted. As described below, element $i$, $j$ corresponds roughly to the number of times algorithm $i$ outperforms algorithm $j$. Columnwise averages show aggregate performance (lower is better). We compare model performance on an exponential scale, favoring evaluations when fewer samples are available. Experiment comparisons stop when a random strategy reaches 99\% of its final convergence value.\jordan{change alg names, add column averages}}
% \label{fig:penalty}
% \end{figure}
%\textbf{Left:} comparing model performance every time new samples are acquired. \textbf{Right:}

\vspace{-0.2cm}
\paragraph{Learning curves.} Here we show examples of learning curves that highlight some of the phenomena we observe related to the fragility of active learning algorithms with respect to batch size, architecture, and dataset.\looseness=-1
%The complete set of learning curves can be found in Appendix~\ref{sec:lcs}.
\begin{figure}
%\hfill
%\subfigure{
\begin{subfigure}[b]{\linewidth}
\centering
  \includegraphics[trim={-1cm 1cm 1cm 1.8cm},clip,width=.95\textwidth]{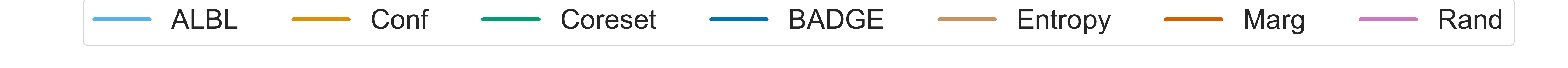}
\end{subfigure}

\begin{subfigure}[b]{0.013\linewidth}
\includegraphics[trim={0.4cm 0cm 27.9cm 0cm}, clip, width=\textwidth]{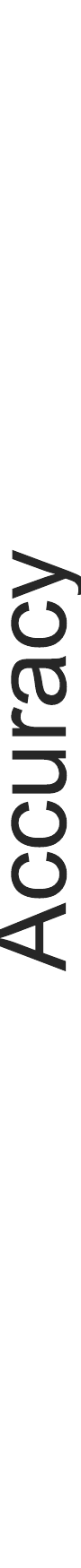}
\label{gp}
\end{subfigure}
\begin{subfigure}[b]{0.31\linewidth}
\includegraphics[trim={1.5cm 0cm 1.6cm 0cm}, clip, width=\textwidth]{figs/learning_curves/all_algs_Accuracy_Data=_SVHN__Model=_rn__nQuery=_100__TrainAug=_0___.pdf}
\caption{}
\label{fp}
\end{subfigure}
%}
%\hfill
\begin{subfigure}[b]{0.31\linewidth}
  \includegraphics[trim={1.5cm 0cm 1.6cm 0cm}, clip, width=\textwidth]{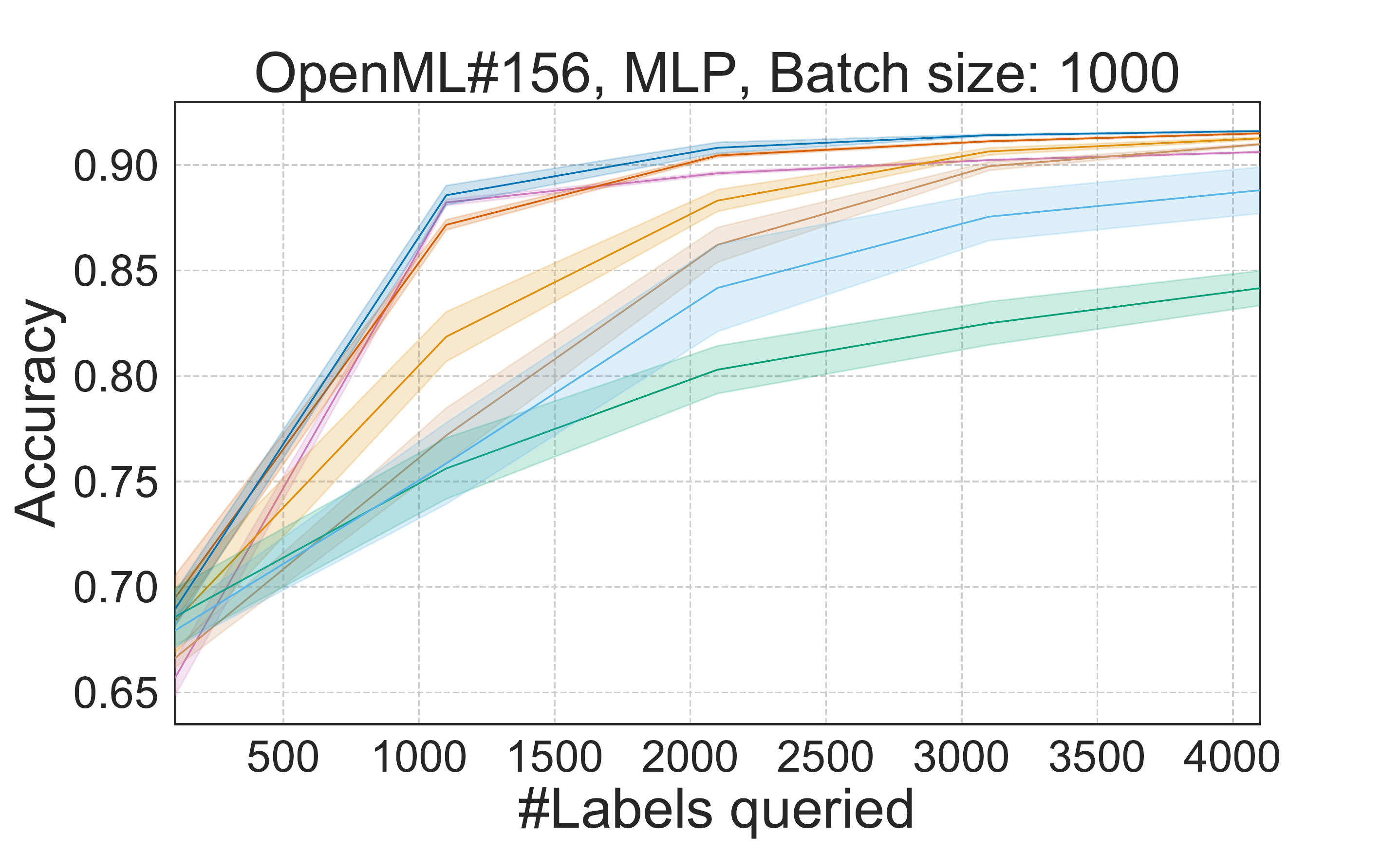}
  \caption{}
  \label{sp}
\end{subfigure}
%\hfill
%\hfill
\begin{subfigure}[b]{0.31\linewidth}
  \includegraphics[trim={1.5cm 0cm 1.6cm 0cm}, clip, width=\textwidth]{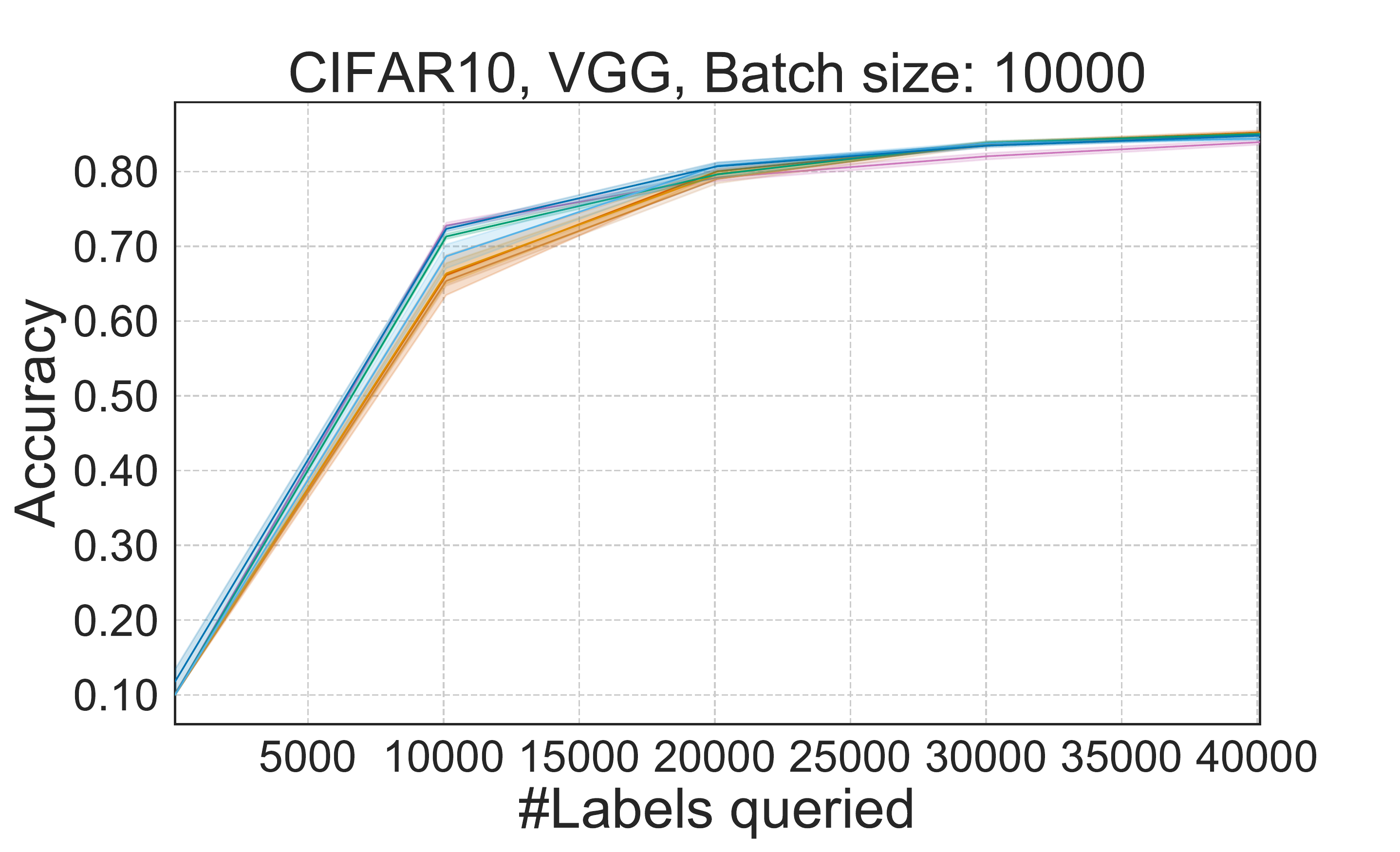}
  \caption{}
  \label{lp}
\end{subfigure}
\\
%\vskip -0.5cm
\centering
% \begin{subfigure}[b]{\linewidth}
%   \includegraphics[trim={1cm 1cm 1cm 1.8cm},clip,width=.95\textwidth]{figs/legends/legend.pdf}
% \end{subfigure}
\vskip -0.3cm
\caption{Active learning test accuracy versus the number of total labeled samples for a range of conditions. Standard errors are shown by shaded regions.\looseness=-1
}
\label{lrplot}
\vskip -0.55cm
\end{figure}

Often, we see that in early rounds of training, it is better to do diversity sampling, and later in training, it is better to do uncertainty sampling. This kind of event is demonstrated in Figure~\ref{fp}, which shows \coreset outperforming confidence-based methods at first, but then doing worse than these methods later on.

\begin{wrapfigure}{r}{0.50\textwidth}
\vspace{-0.55cm}
  \begin{center}
    \centering{\includegraphics[trim={4.2cm 2.0cm 2cm 1cm}, clip, scale=0.75]{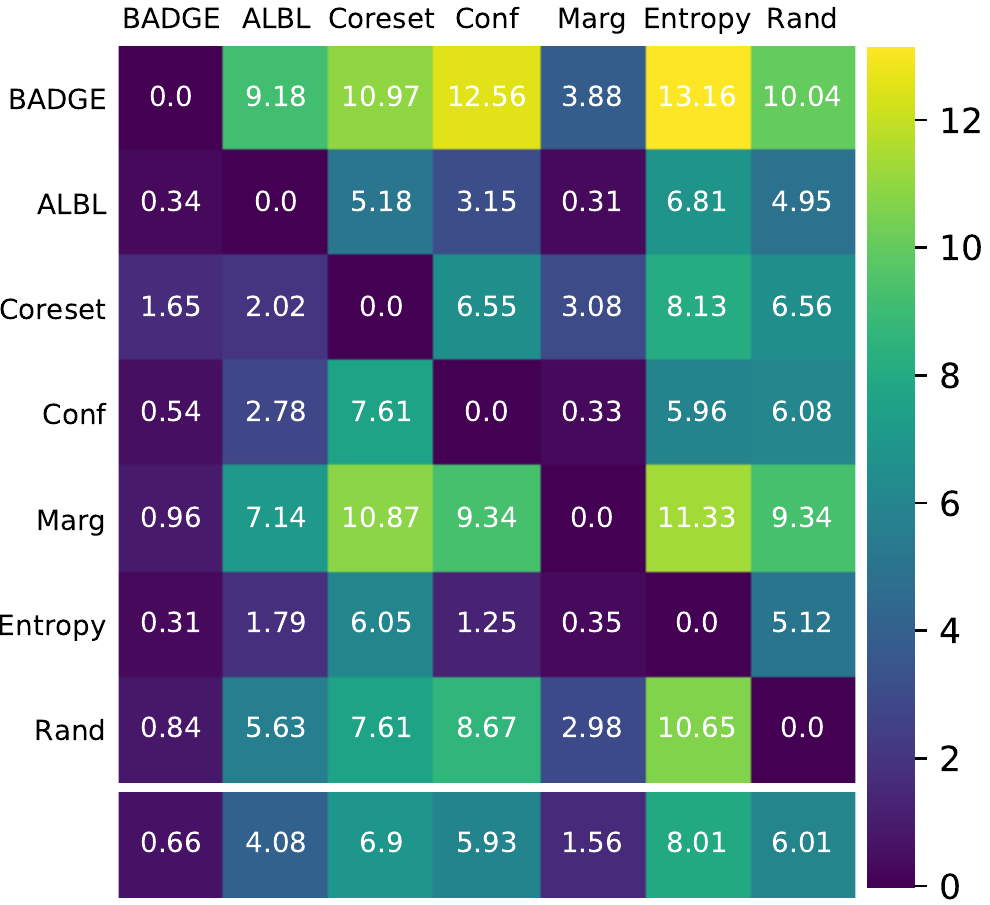}}
  \end{center}
  \vskip -0.2cm
  \caption{\small{A pairwise penalty matrix over all experiments. Element $P_{i,j}$ corresponds roughly to the number of times algorithm $i$ outperforms algorithm $j$. Column-wise averages at the bottom show overall performance (lower is better)}.\looseness=-1
  \label{fig:penalty}
 }
\vskip -0.7cm
\end{wrapfigure}
In this figure, \ouralg performs as well as diversity sampling when that strategy does best, and as well as uncertainty sampling once those methods start outpacing \coreset. This suggests that \ouralg is a good choice regardless of labeling budget.\looseness=-1

Separately, we notice that diversity sampling only seems to work well when either the model has good architectural priors (inductive biases) built in, or when the data are easy to learn. Otherwise, penultimate layer representations are not meaningful, and diverse sampling can be deleterious. For this reason, \coreset often performs worse than random on sufficiently complex data when not using a convolutional network (Figure~\ref{sp}). That is, the diversity induced by unconditional random sampling can often yield a batch that better represents the data.
Even when batch size is large and the model has helpful inductive biases, the  uncertainty information in \ouralg can give it an advantage over pure diversity approaches (Figure~\ref{lp}).
Comprehensive plots of this kind, spanning architecture, dataset, and batch size are in Appendix~\ref{sec:lcs}.\looseness=-1

\paragraph{Pairwise comparisons.}
%\vskip -0.5cm
We next show a comprehensive pairwise comparison of algorithms over all datasets ($D$), batch sizes ($B$),
model architectures ($A$), and label budgets ($L$).
From the learning curves, it can be observed that when label budgets are large enough, all algorithms eventually reach similar
performance, making the comparison between them uninteresting in the large sample limit.
For this reason, for each combination of $(D, B, A)$, we select a set of labeling budgets $L$ where learning is still progressing.
We experimented with three different batch sizes and eleven dataset-architecture pairs, making the total number of $(D, B, A)$ combinations $3 \times 11 = 33$.
Specifically, we compute $n_0$, the smallest number of labels where \rand's accuracy reaches 99\% of its final accuracy, and choose label budget $L$
from $\cbr{M+2^{m-1} B: m \in [\lfloor\log((n_0 - M)/B)\rfloor]}$.
%or $\cbr{M+i B: i \in \cbr{1,2,\ldots,\lfloor(n_0 - M)/B\rfloor}}$.
The calculation of scores in the penalty matrix $P$
follows the following protocol: For each $(D, B, A, L)$ combination and each pair of algorithms $(i, j)$, we have $5$ test errors (one for each repeated run), 
$\cbr{e_i^1, \ldots, e_i^5}$ and $\cbr{e_j^1, \ldots, e_j^5}$ respectively. We compute the $t$-score as $t = \nicefrac{\sqrt{5} \hat{\mu}}{\hat{\sigma}}$, where
\vspace{-0.2cm}
%where \akshay{In $z$-score and $\hat{\sigma}$ formula, it should be $\hat{\mu}$?}
\begin{wrapfigure}{r}{0.45\textwidth}
\vspace{-0.0cm}
\vskip -2.5cm
  \begin{center}
    \includegraphics[trim={0cm 0cm 2cm 2cm}, clip, width=\linewidth]{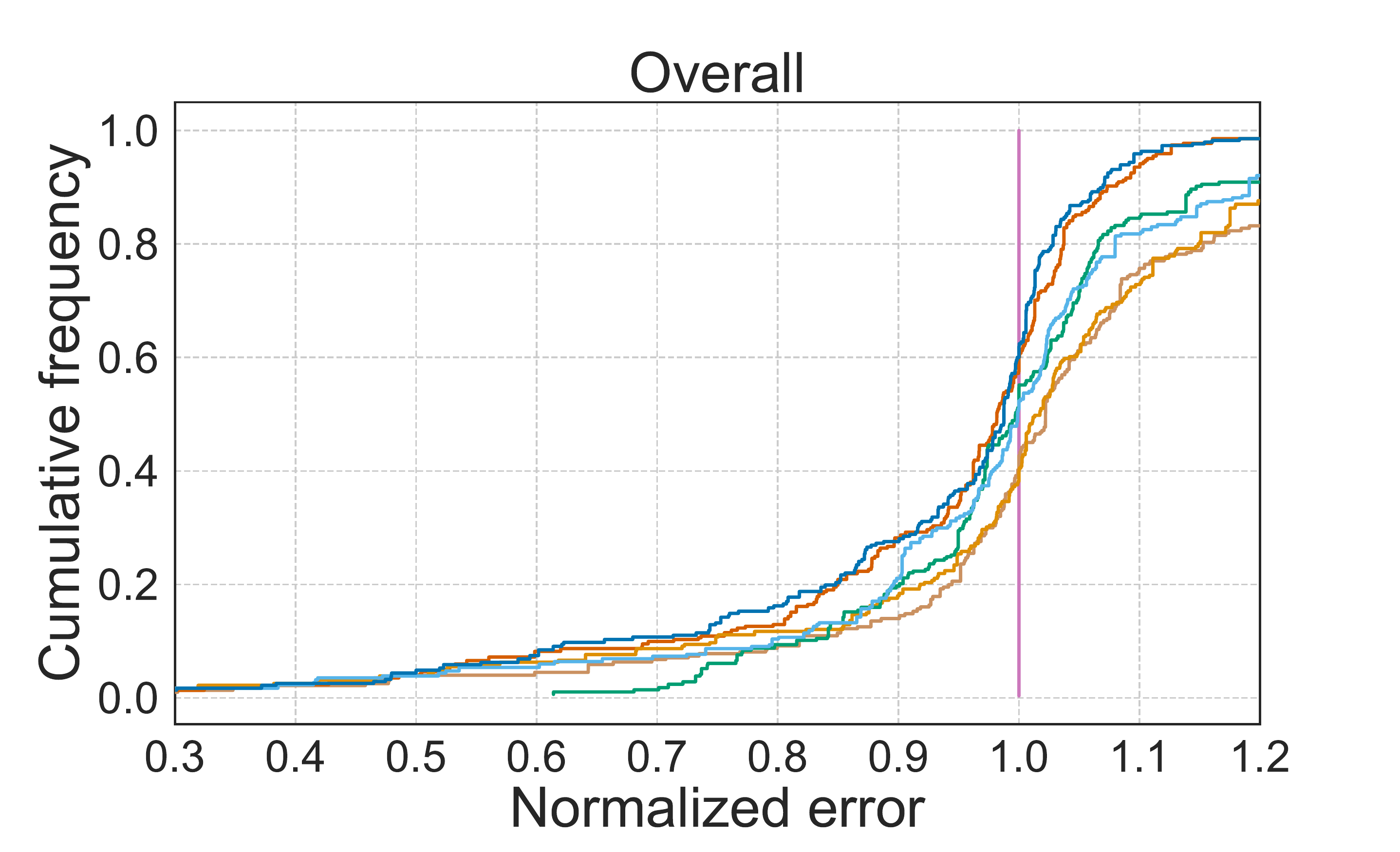}
  \end{center}
  \begin{subfigure}[b]{\linewidth}
    \vspace{-0.3cm}
    \includegraphics[trim={0cm 0cm 0cm 0cm}, clip, width=\textwidth]{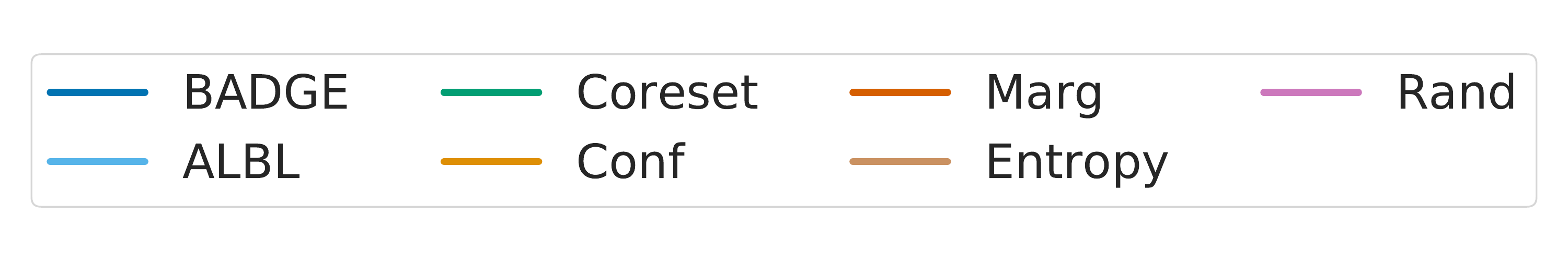}
  \end{subfigure}
  \vskip -0.2cm
  \caption{The cumulative distribution function of normalized errors for all acquisition functions.}
  \vskip -0.6cm
  \label{fig:cdf-all}
\end{wrapfigure}
{
\small
\[  \hat{\mu} = \frac{1}{5} \sum_{l=1}^5 (e_i^l - e_j^l), \;\;\;\;\;\;
  \hat{\sigma} = \sqrt{\frac{1}{4} \sum_{l=1}^5 (e_i^l - e_j^l - \hat{\mu})^2 }. \]
}
%\vspace{-0.1cm}
We use the two-sided $t$-test to compare pairs of algorithms: algorithm $i$ is said to {\em beat} algorithm $j$ in this setting if $t > 2.776$ (the critical point of $p$-value being $0.05$), and similarly algorithm $j$ beats
algorithm $i$ if $t < -2.776$.
For each $(D, B, A)$ combination, suppose there are $n_{D,B,A}$ different values of $L$. Then, for each $L$, if algorithm $i$ beats algorithm $j$, we accumulate a penalty of $1/n_{D,B,A}$ to $P_{i,j}$; otherwise, if
algorithm $j$ beats algorithm $i$, we accumulate a penalty of $1/n_{D,B,A}$ to $P_{j,i}$.
The choice of the penalty value $1/n_{D,B,A}$ is to ensure that every $(D, B, A)$ combination is assigned equal influence in the aggregated matrix. Therefore, the largest entry of $P$ is at most 33, the total number of $(D, B, A)$ combinations. %\chicheng{Add discussions on the max possible entries of the matrix}
Intuitively, each row $i$ indicates the number of settings in which
algorithm $i$ beats other algorithms and each column $j$ indicates the number of settings in which algorithm $j$
is beaten by another algorithm.

%\vskip -0.7cm

The penalty matrix in Figure~\ref{fig:penalty} summarizes all experiments, showing that \ouralg generally outperforms baselines. Matrices grouped by batch size and architecture in Appendix~\ref{sec:pairwise} show a similar trend.\looseness=-1

\vspace{-0.2cm}
\paragraph{Cumulative distribution functions of normalized errors.} For each $(D, B, A, L)$
combination, we compute the average error for each algorithm $i$ as $\bar{e}_{i} = \frac{1}{5} \sum_{l=1}^5 e_i^l$. To
ensure that the errors of these algorithms are on the same scale in all settings, we compute the normalized error
of every algorithm $i$, defined as $\ner_i = \bar{e}_i / \bar{e}_r$, where $r$ is the index of the \rand algorithm.
By definition, the normalized errors of the \rand algorithm are identically 1 in all settings.
Like with penalty matrices, for each $(D, B, A)$ combination, we only consider a subset of $L$ values from the set $\cbr{M+2^{m-1} B: m \in [\lfloor\log((n_0 - M)/B)\rfloor]}$. We assign a weight proportional to $1/n_{D,B,A}$ to each $(D,B,A,L)$ combination, where there are $n_{D,B,A}$ different $L$ values for this combination
of $(D, B, A)$.
We then plot the cumulative distribution
functions (CDFs) of the normalized errors of all algorithms: for a value of $x$, the $y$ value is the total weight of settings where
the algorithm has normalized error at most $x$; in general, an algorithm that has a higher CDF value has better performance.\looseness=-1
%For each batch size, we generate a CDF that highlights the effect of batch size to all the algorithms.

We plot the generated CDFs in Figures~\ref{fig:cdf-all},~\ref{figs:cdfs-batch-sizes} and~\ref{figs:cdfs-models}. We can see from Figure~\ref{fig:cdf-all} that \ouralg has the best overall performance.
In addition, from Figures~\ref{figs:cdfs-batch-sizes} and~\ref{figs:cdfs-models} in Appendix~\ref{sec:cdfs}, we can conclude that when batch size is small (100 or 1000) or when an MLP is used, both \ouralg and \marg perform best. However, in the regime when the batch size
is large (10000), \marg's performance degrades, while \ouralg, \albl and \coreset are the best performing approaches.
\vskip -1cm

%\vskip -0.001cm
\section{Related work}
\vskip -0.3cm
\label{sec:related}
Active learning is a been well-studied problem~\citep{S10, D11, H14}.
There are two major strategies for active learning---representative sampling and uncertainty sampling.

Representative sampling algorithms select batches of unlabeled examples that are representative
of the unlabeled set to ask for labels. It is based on the intuition that the sets of
representative examples chosen, once labeled, can act as a surrogate for the full dataset. Consequently, performing loss minimization on the surrogate suffices to ensure
a low error with respect to the full dataset. In the context of deep learning,
~\citet{sener2018active, geifman2017deep} select representative
examples based on core-set construction, a fundamental problem in computational geometry.
Inspired by generative adversarial learning, ~\citet{gissin2019discriminative} select
 samples that are maximally indistinguishable from the pool of unlabeled examples.\looseness=-1
%These works provide empirical evidence
%that representative sampling has better performance compared to other active
%learning methods when the number of examples chosen by the algorithm, \textit{the batch size}, is large.

On the other hand, uncertainty sampling is based on a different principle---to select new samples that maximally reduce the uncertainty the algorithm has
on the target classifier. In the context of linear classification,~\citet{TK01, schohn2000less, tur2005combining}
propose uncertainty sampling methods that query
 examples that lie closest to the current decision boundary. Some uncertainty sampling approaches have theoretical guarantees on statistical consistency~\citep{H14, A^2}. Such methods have also been
 recently generalized to deep learning. For instance, ~\citet{gal2017deep} use
 Dropout as an approximation of the posterior of the model parameters, and develop
 information-based uncertainty reduction criteria; inspired
 by recent advances on adversarial examples generation, ~\citet{ducoffe2018adversarial}
 use the distance between an example and one of its adversarial examples as an approximation
 of its distance to the current decision boundary, and uses it as the criterion of label queries. An ensemble of classifiers could also be used to effectively estimate uncertainty~\citep{ensembles}.

 % As we will see in the experiments, uncertainty sampling methods are not suitable
 % for active learning with large batch sizes

There are several existing approaches that support a hybrid of representative sampling and
uncertainty sampling. For example, ~\citet{baram2004online,hsu2015active} present meta-active learning
algorithms that can combine the advantages of different active learning algorithms.
Inspired by expected loss minimization, ~\citet{huang2010active} develop label query
criteria that balances between the representativeness and informativeness of examples.
Another method for this is Active Learning by Learning~\citep{hsu2015active}, which can select whether to exercise a diversity based algorithm or an uncertainty based algorithm at each round of training as a sequential decision process.

There is also a large body of literature on batch mode
active learning, where the learner
is asked to select a batch of samples within each round~\citep{guo2008discriminative, wang2015querying, pmlr-v28-chen13b, wei2015submodularity,batchbald}. In these works, batch selection is often formulated
as an optimization problem with objectives based on (upper bounds of)
 average log-likelihood, average squared loss, etc.

%~\citep{guo2008discriminative} proposes a batch-based querying criterion based on expected
%log-likelihood maximization.
%\citep{wang2015querying} formulate batch mode active learning inspired by square loss
%minimization.
%~\citep{pmlr-v28-chen13b, wei2015submodularity} view the batch selection problem
%as submodular optimization on objectives such as log-likelihood and
%expected number of label queries.

%Our work can also be seen as a new way of providing such a tradeoff.

A different query criterion based on expected gradient length (EGL) has
been proposed in the as well~\citep{settles2008multiple}.
In recent work,~\citet{huang2016active} show that the EGL criterion is related to
the $T$-optimality criterion in experimental design. They further demonstrate that the
samples selected by EGL are very different from those by entropy-based uncertainty criterion.
~\citet{zhang2017active} use the EGL criterion in active sentence and document
classification with CNNs. These approaches differ most substantially from \ouralg in that they do
not take into account the diversity of the examples queried within each batch.

%they showed that a query criteria based on the gradient length
%with respect to the representation layers has a better empirical performance compared
%to a query criteria based on the gradient length with respect the discriminative layer.

There is a wide array of theoretical articles that focus on the related problem of adaptive subsampling for fully-labeled datasets in regression
settings~\citep{han2016local, wang2018optimal, ting2018optimal}.
Empirical studies of batch stochastic gradient descent also employ adaptive sampling to ``emphasize''
hard or representative examples~\citep{zhang2017determinantal, chang2017active}.
These works aim at reducing computation costs or finding a better local optimal solution, as opposed to
reducing label costs.
Nevertheless, our work is inspired by their sampling criteria, which also emphasize samples that induce
large updates to the model.

%can also be used to develop active learning
%algorithms, by assigning a hypothesized label for each example and applying their sampling algorithms.

%In fact, our work can be seen as one that uses a
%To see this, one can assign a ``guessed'' label for each example based on the current
%classifier, and use their framework to derive sampling probability for each example~\citep{wei2015submodularity}.

As mentioned earlier, our sampling criterion has resemblance to sampling from $k$-determinantal
point processes~\citep{kulesza2011k}. Note that in multiclass classification settings,
our gradient-based embedding
of an example can be viewed as the outer product of the original embedding in the penultimate layer
and a probability score vector that encodes the uncertainty information on this example
(see Section~\ref{sec:alg}).
In this view, the penultimate layer embedding characterizes the diversity of each example,
whereas the probability score vector characterizes the quality of each example.
The $k$-DPP is also a natural probabilistic tool for sampling that trades off between
quality and diversity~\citep[See][Section 3.1]{kulesza2012determinantal}. We remark that concurrent to our work, ~\citet{biyik2019batch} develops $k$-DPP based active learning algorithms based on this principle by explicitly designing diversity and uncertainty measures. \looseness=-1
\vskip -0.5cm

%\vskip -1cm
\section{Discussion}
%\vskip 0.4cm
We have established that \ouralg{} is empirically an effective deep active learning algorithm across different architectures and batch sizes, performing similar to or better than other active learning algorithms.  A fundamental remaining question is: "Why?" While deep learning is notoriously difficult to analyze theoretically, there are several intuitively appealing properties of \ouralg{}:
\begin{enumerate}
    \item The definition of uncertainty (a lower bound on the gradient magnitude of the last layer) guarantees some update of parameters.
    \item It optimizes for diversity as well as uncertainty, eliminating a failure mode of choosing many identical uncertain examples in a batch, and does so without requiring any hyperparameters.
    \item The randomization associated with the \kmeansp initialization sampler implies that, even for adversarially constructed datasets, it eventually converges to a good solution.
\end{enumerate}
\vskip -0.2cm
The combination of these properties appears to generate the robustness that we observe empirically.

\bibliographystyle{unsrtnat}
\bibliography{deepactive}

\begin{thebibliography}{3}
\providecommand{\natexlab}[1]{#1}
\providecommand{\url}[1]{\texttt{#1}}
\expandafter\ifx\csname urlstyle\endcsname\relax
  \providecommand{\doi}[1]{doi: #1}\else
  \providecommand{\doi}{doi: \begingroup \urlstyle{rm}\Url}\fi

\bibitem[Bengio \& LeCun(2007)Bengio and LeCun]{Bengio+chapter2007}
Yoshua Bengio and Yann LeCun.
\newblock Scaling learning algorithms towards {AI}.
\newblock In \emph{Large Scale Kernel Machines}. MIT Press, 2007.

\bibitem[Goodfellow et~al.(2016)Goodfellow, Bengio, Courville, and
  Bengio]{goodfellow2016deep}
Ian Goodfellow, Yoshua Bengio, Aaron Courville, and Yoshua Bengio.
\newblock \emph{Deep learning}, volume~1.
\newblock MIT Press, 2016.

\bibitem[Hinton et~al.(2006)Hinton, Osindero, and Teh]{Hinton06}
Geoffrey~E. Hinton, Simon Osindero, and Yee~Whye Teh.
\newblock A fast learning algorithm for deep belief nets.
\newblock \emph{Neural Computation}, 18:\penalty0 1527--1554, 2006.

\end{thebibliography}


\begin{thebibliography}{51}
\providecommand{\natexlab}[1]{#1}
\providecommand{\url}[1]{\texttt{#1}}
\expandafter\ifx\csname urlstyle\endcsname\relax
  \providecommand{\doi}[1]{doi: #1}\else
  \providecommand{\doi}{doi: \begingroup \urlstyle{rm}\Url}\fi

\bibitem[Cohn et~al.(1994)Cohn, Atlas, and Ladner]{cohn1994improving}
David Cohn, Les Atlas, and Richard Ladner.
\newblock Improving generalization with active learning.
\newblock \emph{Machine learning}, 1994.

\bibitem[Balcan et~al.(2006)Balcan, Beygelzimer, and Langford]{A^2}
Maria{-}Florina Balcan, Alina Beygelzimer, and John Langford.
\newblock Agnostic active learning.
\newblock In \emph{International Conference on Machine Learning}, 2006.

\bibitem[Beygelzimer et~al.(2010)Beygelzimer, Hsu, Langford, and
  Zhang]{beygelzimer2010agnostic}
Alina Beygelzimer, Daniel~J Hsu, John Langford, and Tong Zhang.
\newblock Agnostic active learning without constraints.
\newblock In \emph{Neural Information Processing Systems}, 2010.

\bibitem[Cesa-Bianchi et~al.(2009)Cesa-Bianchi, Gentile, and
  Orabona]{cesa2009robust}
Nicolo Cesa-Bianchi, Claudio Gentile, and Francesco Orabona.
\newblock Robust bounds for classification via selective sampling.
\newblock In \emph{International Conference on Machine Learning}, 2009.

\bibitem[Arthur and Vassilvitskii(2007)]{arthur2007k}
David Arthur and Sergei Vassilvitskii.
\newblock k-means++: The advantages of careful seeding.
\newblock In \emph{ACM-SIAM symposium on Discrete algorithms}, 2007.

\bibitem[Kulesza and Taskar(2011)]{kulesza2011k}
Alex Kulesza and Ben Taskar.
\newblock k-dpps: Fixed-size determinantal point processes.
\newblock In \emph{International Conference on Machine Learning}, 2011.

\bibitem[Derezi{\'n}ski and Warmuth(2018)]{derezinski2018reverse}
Micha{\l} Derezi{\'n}ski and Manfred~K Warmuth.
\newblock Reverse iterative volume sampling for linear regression.
\newblock \emph{The Journal of Machine Learning Research}, 19\penalty0 (1),
  2018.

\bibitem[Kang(2013)]{kang2013fast}
Byungkon Kang.
\newblock Fast determinantal point process sampling with application to
  clustering.
\newblock In \emph{Neural Information Processing Systems}, 2013.

\bibitem[Anari et~al.(2016)Anari, Gharan, and Rezaei]{anari2016monte}
Nima Anari, Shayan~Oveis Gharan, and Alireza Rezaei.
\newblock Monte carlo markov chain algorithms for sampling strongly rayleigh
  distributions and determinantal point processes.
\newblock In \emph{Conference on Learning Theory}, 2016.

\bibitem[Derezi{\'n}ski(2018)]{derezinski2018fast}
Micha{\l} Derezi{\'n}ski.
\newblock Fast determinantal point processes via distortion-free intermediate
  sampling.
\newblock \emph{arXiv preprint}, 2018.

\bibitem[Sener and Savarese(2018)]{sener2018active}
Ozan Sener and Silvio Savarese.
\newblock Active learning for convolutional neural networks: A core-set
  approach.
\newblock In \emph{International Conference on Learning Representations}, 2018.

\bibitem[Wang and Shang(2014)]{wang2014new}
Dan Wang and Yi~Shang.
\newblock A new active labeling method for deep learning.
\newblock In \emph{2014 International joint conference on neural networks},
  2014.

\bibitem[Roth and Small(2006)]{roth2006margin}
Dan Roth and Kevin Small.
\newblock Margin-based active learning for structured output spaces.
\newblock In \emph{European Conference on Machine Learning}, 2006.

\bibitem[Hsu and Lin(2015)]{hsu2015active}
Wei-Ning Hsu and Hsuan-Tien Lin.
\newblock Active learning by learning.
\newblock In \emph{Association for the advancement of artificial intelligence},
  2015.

\bibitem[He et~al.(2016)He, Zhang, Ren, and Sun]{resnet}
Kaiming He, Xiangyu Zhang, Shaoqing Ren, and Jian Sun.
\newblock Deep residual learning for image recognition.
\newblock In \emph{Proceedings of the IEEE conference on computer vision and
  pattern recognition}, pages 770--778, 2016.

\bibitem[Simonyan and Zisserman(2014)]{vgg}
Karen Simonyan and Andrew Zisserman.
\newblock Very deep convolutional networks for large-scale image recognition.
\newblock \emph{arXiv preprint}, 2014.

\bibitem[Netzer et~al.(2011)Netzer, Wang, Coates, Bissacco, Wu, and Ng]{svhn}
Yuval Netzer, Tao Wang, Adam Coates, Alessandro Bissacco, Bo~Wu, and Andrew~Y
  Ng.
\newblock Reading digits in natural images with unsupervised feature learning.
\newblock 2011.

\bibitem[Krizhevsky(2009)]{cifar}
Alex Krizhevsky.
\newblock Learning multiple layers of features from tiny images.
\newblock Technical report, Citeseer, 2009.

\bibitem[LeCun et~al.(1998)LeCun, Bottou, Bengio, Haffner, et~al.]{mnist}
Yann LeCun, L{\'e}on Bottou, Yoshua Bengio, Patrick Haffner, et~al.
\newblock Gradient-based learning applied to document recognition.
\newblock \emph{IEEE}, 1998.

\bibitem[Ash and Adams(2019)]{ash2019difficulty}
Jordan~T Ash and Ryan~P Adams.
\newblock On the difficulty of warm-starting neural network training.
\newblock \emph{arXiv preprint}, 2019.

\bibitem[Yang et~al.(2017)Yang, Lee, Chung, Wu, Chen, and Lin]{yang2017libact}
Yao-Yuan Yang, Shao-Chuan Lee, Yu-An Chung, Tung-En Wu, Si-An Chen, and
  Hsuan-Tien Lin.
\newblock libact: Pool-based active learning in python.
\newblock \emph{arXiv preprint}, 2017.

\bibitem[Paszke et~al.(2017)Paszke, Gross, Chintala, Chanan, Yang, DeVito, Lin,
  Desmaison, Antiga, and Lerer]{pytorch}
Adam Paszke, Sam Gross, Soumith Chintala, Gregory Chanan, Edward Yang, Zachary
  DeVito, Zeming Lin, Alban Desmaison, Luca Antiga, and Adam Lerer.
\newblock Automatic differentiation in pytorch.
\newblock 2017.

\bibitem[Settles(2010)]{S10}
Burr Settles.
\newblock Active learning literature survey.
\newblock \emph{University of Wisconsin, Madison}, 2010.

\bibitem[Dasgupta(2011)]{D11}
Sanjoy Dasgupta.
\newblock Two faces of active learning.
\newblock \emph{Theoretical computer science}, 2011.

\bibitem[Hanneke(2014)]{H14}
Steve Hanneke.
\newblock Theory of disagreement-based active learning.
\newblock \emph{Foundations and Trends in Machine Learning}, 2014.

\bibitem[Geifman and El-Yaniv(2017)]{geifman2017deep}
Yonatan Geifman and Ran El-Yaniv.
\newblock Deep active learning over the long tail.
\newblock \emph{arXiv preprint}, 2017.

\bibitem[Gissin and Shalev-Shwartz(2019)]{gissin2019discriminative}
Daniel Gissin and Shai Shalev-Shwartz.
\newblock Discriminative active learning.
\newblock \emph{arXiv preprint}, 2019.

\bibitem[Tong and Koller(2001)]{TK01}
Simon Tong and Daphne Koller.
\newblock Support vector machine active learning with applications to text
  classification.
\newblock \emph{Journal of machine learning research}, 2001.

\bibitem[Schohn and Cohn(2000)]{schohn2000less}
Greg Schohn and David Cohn.
\newblock Less is more: Active learning with support vector machines.
\newblock In \emph{International Conference on Machine Learning}, 2000.

\bibitem[Tur et~al.(2005)Tur, Hakkani-T{\"u}r, and Schapire]{tur2005combining}
Gokhan Tur, Dilek Hakkani-T{\"u}r, and Robert~E Schapire.
\newblock Combining active and semi-supervised learning for spoken language
  understanding.
\newblock \emph{Speech Communication}, 2005.

\bibitem[Gal et~al.(2017)Gal, Islam, and Ghahramani]{gal2017deep}
Yarin Gal, Riashat Islam, and Zoubin Ghahramani.
\newblock Deep bayesian active learning with image data.
\newblock In \emph{International Conference on Machine Learning}, 2017.

\bibitem[Ducoffe and Precioso(2018)]{ducoffe2018adversarial}
Melanie Ducoffe and Frederic Precioso.
\newblock Adversarial active learning for deep networks: a margin based
  approach.
\newblock \emph{arXiv preprint}, 2018.

\bibitem[Beluch et~al.(2018)Beluch, Genewein, N{\"u}rnberger, and
  K{\"o}hler]{ensembles}
William~H Beluch, Tim Genewein, Andreas N{\"u}rnberger, and Jan~M K{\"o}hler.
\newblock The power of ensembles for active learning in image classification.
\newblock In \emph{IEEE Conference on Computer Vision and Pattern Recognition},
  2018.

\bibitem[Baram et~al.(2004)Baram, Yaniv, and Luz]{baram2004online}
Yoram Baram, Ran~El Yaniv, and Kobi Luz.
\newblock Online choice of active learning algorithms.
\newblock \emph{Journal of Machine Learning Research}, 2004.

\bibitem[Huang et~al.(2010)Huang, Jin, and Zhou]{huang2010active}
Sheng-Jun Huang, Rong Jin, and Zhi-Hua Zhou.
\newblock Active learning by querying informative and representative examples.
\newblock In \emph{Neural Information Processing Systems}, 2010.

\bibitem[Guo and Schuurmans(2008)]{guo2008discriminative}
Yuhong Guo and Dale Schuurmans.
\newblock Discriminative batch mode active learning.
\newblock In \emph{Neural Information Processing Systems}, 2008.

\bibitem[Wang and Ye(2015)]{wang2015querying}
Zheng Wang and Jieping Ye.
\newblock Querying discriminative and representative samples for batch mode
  active learning.
\newblock \emph{Transactions on Knowledge Discovery from Data}, 2015.

\bibitem[Chen and Krause()]{pmlr-v28-chen13b}
Yuxin Chen and Andreas Krause.
\newblock Near-optimal batch mode active learning and adaptive submodular
  optimization.
\newblock In \emph{International Conference on Machine Learning}.

\bibitem[Wei et~al.(2015)Wei, Iyer, and Bilmes]{wei2015submodularity}
Kai Wei, Rishabh Iyer, and Jeff Bilmes.
\newblock Submodularity in data subset selection and active learning.
\newblock In \emph{International Conference on Machine Learning}, 2015.

\bibitem[Kirsch et~al.(2019)Kirsch, van Amersfoort, and Gal]{batchbald}
Andreas Kirsch, Joost van Amersfoort, and Yarin Gal.
\newblock Batchbald: Efficient and diverse batch acquisition for deep bayesian
  active learning.
\newblock In \emph{Neural Information Processing Systems 32}, 2019.

\bibitem[Settles et~al.(2008)Settles, Craven, and Ray]{settles2008multiple}
Burr Settles, Mark Craven, and Soumya Ray.
\newblock Multiple-instance active learning.
\newblock In \emph{Neural Information Processing Systems}, 2008.

\bibitem[Huang et~al.(2016)Huang, Child, and Rao]{huang2016active}
Jiaji Huang, Rewon Child, and Vinay Rao.
\newblock Active learning for speech recognition: the power of gradients.
\newblock \emph{arXiv preprint}, 2016.

\bibitem[Zhang et~al.(2017{\natexlab{a}})Zhang, Lease, and
  Wallace]{zhang2017active}
Ye~Zhang, Matthew Lease, and Byron~C Wallace.
\newblock Active discriminative text representation learning.
\newblock In \emph{AAAI Conference on Artificial Intelligence},
  2017{\natexlab{a}}.

\bibitem[Han et~al.(2016)Han, Tan, Yang, and Zhang]{han2016local}
Lei Han, Kean~Ming Tan, Ting Yang, and Tong Zhang.
\newblock Local uncertainty sampling for large-scale multi-class logistic
  regression.
\newblock \emph{arXiv preprint}, 2016.

\bibitem[Wang et~al.(2018)Wang, Zhu, and Ma]{wang2018optimal}
HaiYing Wang, Rong Zhu, and Ping Ma.
\newblock Optimal subsampling for large sample logistic regression.
\newblock \emph{Journal of the American Statistical Association}, 2018.

\bibitem[Ting and Brochu(2018)]{ting2018optimal}
Daniel Ting and Eric Brochu.
\newblock Optimal subsampling with influence functions.
\newblock In \emph{Neural Information Processing Systems}, 2018.

\bibitem[Zhang et~al.(2017{\natexlab{b}})Zhang, Kjellstrom, and
  Mandt]{zhang2017determinantal}
Cheng Zhang, Hedvig Kjellstrom, and Stephan Mandt.
\newblock Determinantal point processes for mini-batch diversification.
\newblock \emph{Uncertainty in Artificial Intelligence}, 2017{\natexlab{b}}.

\bibitem[Chang et~al.(2017)Chang, Learned-Miller, and
  McCallum]{chang2017active}
Haw-Shiuan Chang, Erik Learned-Miller, and Andrew McCallum.
\newblock Active bias: Training more accurate neural networks by emphasizing
  high variance samples.
\newblock In \emph{Neural Information Processing Systems}, 2017.

\bibitem[Kulesza et~al.(2012)Kulesza, Taskar, et~al.]{kulesza2012determinantal}
Alex Kulesza, Ben Taskar, et~al.
\newblock Determinantal point processes for machine learning.
\newblock \emph{Foundations and Trends in Machine Learning}, 2012.

\bibitem[B{\i}y{\i}k et~al.(2019)B{\i}y{\i}k, Wang, Anari, and
  Sadigh]{biyik2019batch}
Erdem B{\i}y{\i}k, Kenneth Wang, Nima Anari, and Dorsa Sadigh.
\newblock Batch active learning using determinantal point processes.
\newblock \emph{arXiv preprint}, 2019.

\bibitem[Mussmann and Liang(2018)]{mussmann2018uncertainty}
Stephen Mussmann and Percy~S Liang.
\newblock Uncertainty sampling is preconditioned stochastic gradient descent on
  zero-one loss.
\newblock In \emph{Neural Information Processing Systems}, 2018.

\end{thebibliography}

\appendix
%\newgeometry{left=1.5cm, right=1.5cm}
%\nolinenumbers

\section{The \kmeansp seeding algorithm}
\label{sec:kmeansp}
Here we briefly review the \kmeansp seeding algorithm by~\citep{arthur2007k}. Its basic idea is
to perform sequential sampling of $k$ centers, where each new center is sampled from the
ground set with probability proportional to the squared distance to its nearest center.
It is shown in~\citep{arthur2007k} that the set of centers returned is guaranteed to
approximate the $k$-means objective function in expectation, thus ensuring diversity.

\begin{algorithm}
\begin{algorithmic}
\REQUIRE Ground set $G \subset \RR^d$, target size $k$.

\ENSURE Center set $C$ of size $k$.

\STATE $C_1 \gets \cbr{c_1}$, where $c_1$ is sampled uniformly at random from $G$.

\FOR{$t = 2,\ldots,k$:}

\STATE Define $D_t(x) := \min_{c \in C_{t-1}} \| x - c \|_2$.

\STATE $c_t \gets$ Sample $x$ from $G$ with probability $\frac{D_t(x)^2}{\sum_{x \in G} D_t(x)^2}$.

\STATE $C_t \gets C_{t-1} \cup \cbr{c_t}$.

\ENDFOR

\RETURN $C_k$.

\end{algorithmic}
  \caption{The \kmeansp seeding algorithm~\citep{arthur2007k}}
  \label{alg:kmeansp}
\end{algorithm}

\section{\ouralg for binary logistic regression}
\label{sec:bin-lr}

%\paragraph{Example 1: binary classification with softmax activation.}
We consider instantiating \ouralg for binary logistic regression, where $\Ycal = \cbr{-1,+1}$. Given a linear classifier $w$, we define the predictive probability of $w$ on $x$ as $p_w(y|x,\theta) = \sigma(y w \cdot x)$, where $\sigma(z) = \frac{1}{1+e^{-z}}$ is the sigmoid funciton.

Recall that $\hat{y} = \hat{y}(x)$ is the hallucinated label:
\[ \hat{y}(x) = \begin{cases} +1, & p_w(+1|x,\theta) > 1/2, \\ -1, & p_w(+1|x,\theta) \leq 1/2. \end{cases} \]

The binary logistic loss of classifier $w$ on example $(x,y)$ is defined as:
\[ \ell(w, (x, y)) = \ln(1 + \exp(-y w \cdot x)). \]

Now, given model $w$ and example $x$, we define $\hat{g}_x = \frac{\partial}{\partial w} \ell(w, (x, \hat{y})) = (1 - p_w(\hat{y}|x,\theta)) \cdot (-\hat{y} \cdot x)$ as the loss gradient induced by the example with hallucinated label, and $\tilde{g}_x = \frac{\partial}{\partial w} \ell(w, (x, y)) = (1 - p_w(y|x,\theta)) \cdot (-y \cdot x)$ as the loss gradient induced by the example with true label.

Suppose that \ouralg only selects examples from region $S_w = \cbr{x: w \cdot x = 0}$, then as $p_w(+1|x,\theta) = p_w(-1|x,\theta) = \frac12$, we have that
for all $x$ in $S_w$, $\hat{g}_x = s_x \cdot g_x$ for some $s_x \in \cbr{\pm 1}$.
This implies that,
sampling from a DPP induced by ${\hat{g}_x}$'s is equivalent to sampling from a DPP induced by $g_x$'s.
It is noted in~\citet{mussmann2018uncertainty} that uncertainty sampling (i.e. sampling from $D_{|S_w}$) implicitly performs preconditioned stochastic gradient descent on the expected 0-1 loss. In addition, it has been shown that DPP sampling over gradients may reduce the variance of the mini-batch stochastic gradient updates~\citep{zhang2017determinantal}; this suggests that \ouralg, when restricted its sampling over low-margin regions ($S_w$), improves over uncertainty sampling by collecting examples that together induce lower-variance updates on the gradient direction of expected 0-1 loss.

%with %respect to the last layer,
% %last layer, as
%\[
%   \frac{\partial}{\partial w} \ell(f(x;\theta), y) = (1 - p_f(\hat{y}|x,\theta)) \cdot (-y \cdot g(x;V)),
%\]

%neural network $f$ where its output layer has a sigmoid activation $\sigma(z) = \frac{1}{1+e^{-z}}$. Specifically, $f$ is
%parametrized by $\theta = (w, V)$, where $w$ is the weight of the last layer, and $V$ consists of weights of all previous layers. In addition,
%$f(x;\theta) = \sigma(w \cdot g(x;V))$, where $g$ is the nonlinear function that maps input $x$ to the network's penultimate layer.
%We also define the predictive probability of $f(\cdot, \theta)$ given $x$ as $p_f(y|x,\theta) = \sigma(y w \cdot g(x;V))$.
%

%is the hypothesized label of $x$.
%This implies that,
%\[ \ell(f(x;\theta), y) = \ln(1 + \exp(-\hat{y}(x) w \cdot g(x;V))). \]
%Define $\hat{g}_x = \frac{\partial}{\partial w} \ell(f(x;\theta),\hat{y})$ as the loss gradient induced by the hypothetical labeled examples with %respect to the last layer,
%and $\tilde{g}_x = \frac{\partial}{\partial w} \ell(f(x;\theta), y)$ as the loss gradient induced by the actual labeled examples with respect to the %last layer, as
%\[
%   \frac{\partial}{\partial w} \ell(f(x;\theta), y) = (1 - p_f(\hat{y}|x,\theta)) \cdot (-y \cdot g(x;V)),
%\]

\vskip 0.5cm
\section{All learning curves}
\label{sec:lcs}
\vskip -0.3cm
We plot all learning curves (test accuracy as a function of the number of labeled example queried) in
Figures~\ref{fig:6-lc-full} to~\ref{fig:cifar10-lc-full}.
In addition, we zoom into regions of the learning curves that discriminates the performance of all algorithms in Figures~\ref{fig:6-lc} to~\ref{fig:cifar10-lc}.

%0.4cm 0cm 27.9cm 0cm
%1.5cm 0cm 1.6cm 0cm

\begin{figure}
  \centering
  \includegraphics[trim={0.4cm 0cm 27.9cm 0cm}, clip, width=0.012\textwidth]{figs/learning_curves/all_algs_Accuracy_Data=_SVHN__Model=_rn__nQuery=_100__TrainAug=_0___.pdf}
  \includegraphics[trim={1.5cm 0cm 1.6cm 0cm}, clip,  width=0.32\textwidth]{{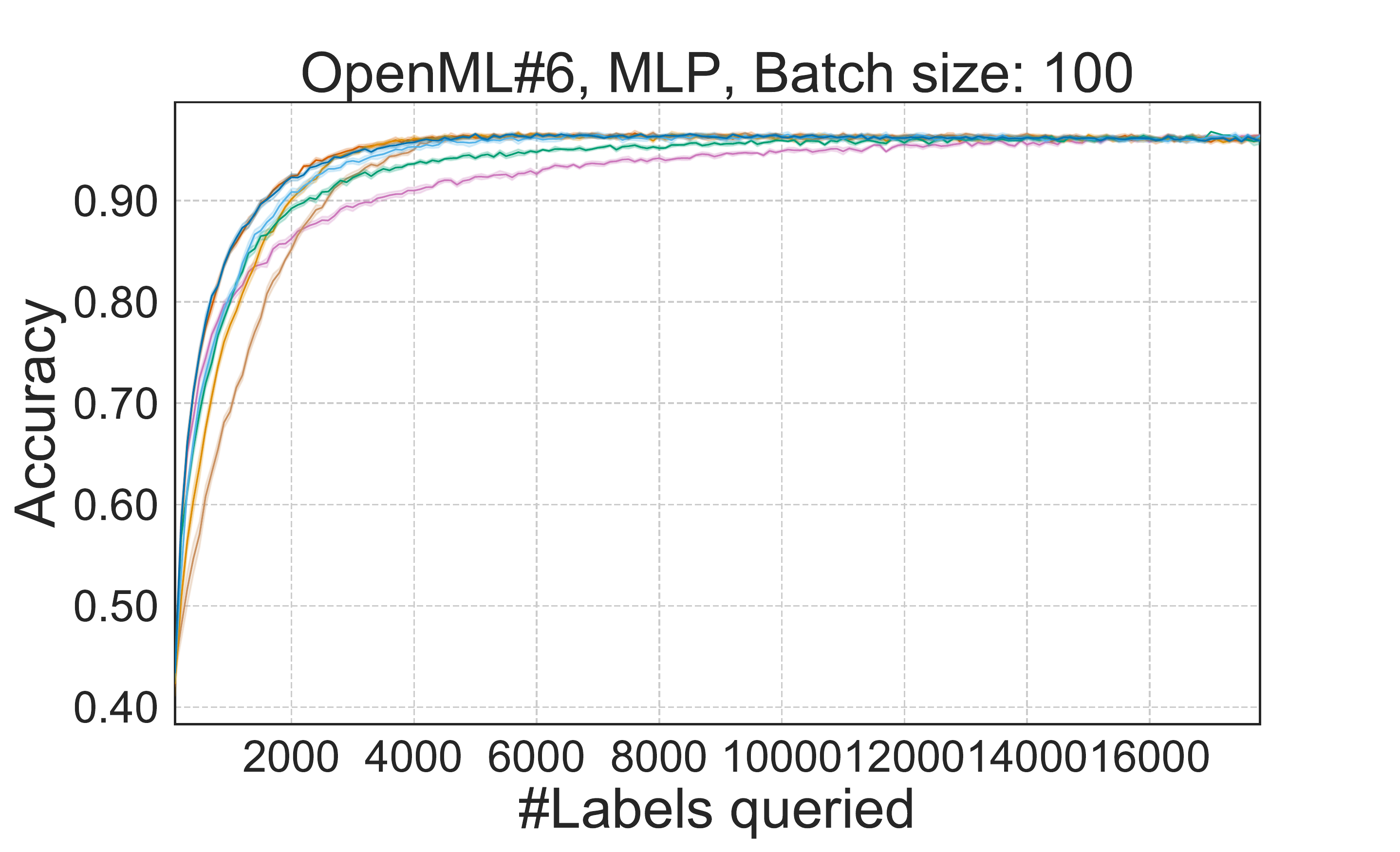}}
  \hfill
  \includegraphics[trim={1.5cm 0cm 1.6cm 0cm}, clip,  width=0.32\textwidth]{{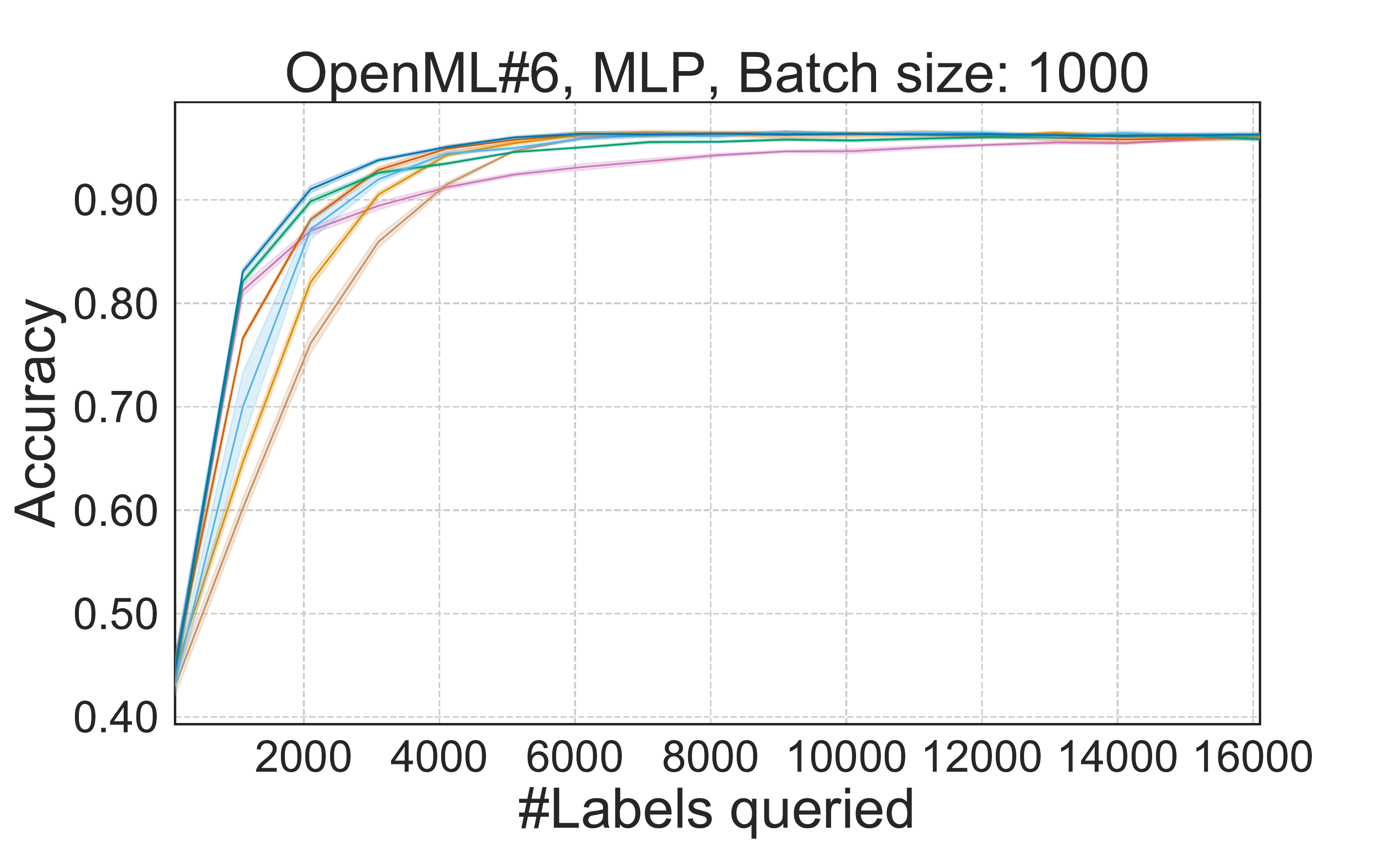}}
  \hfill
  \includegraphics[trim={1.5cm 0cm 1.6cm 0cm}, clip,  width=0.32\textwidth]{{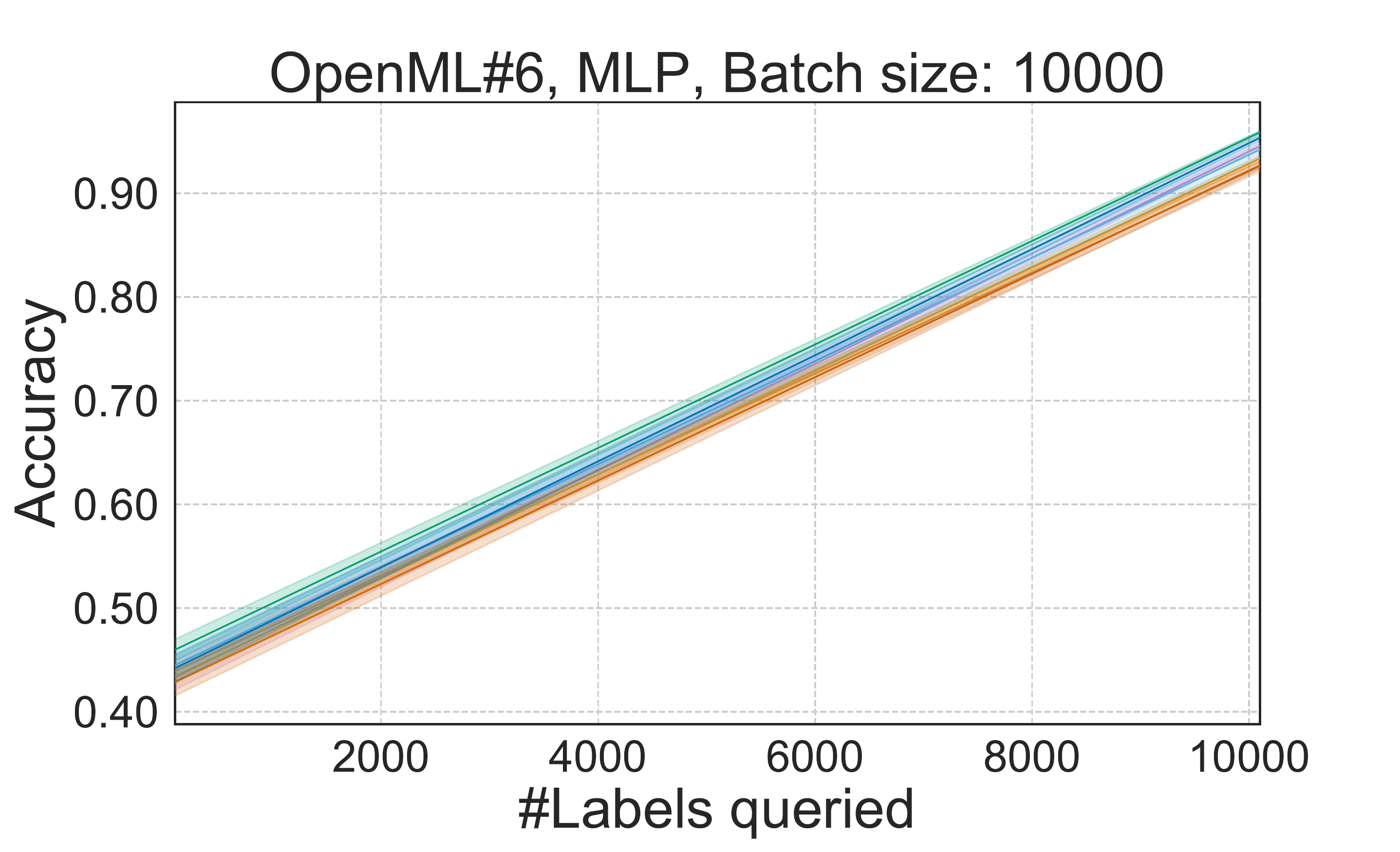}}
  \\
  \centering
  \begin{subfigure}[b]{\linewidth}
    \includegraphics[trim={0cm 0cm 0cm 0cm}, clip, width=\textwidth]{figs/legends/legend.pdf}
  \end{subfigure}
  \vskip -0.5cm
\caption{Full learning curves for OpenML \#6 with MLP.}
\label{fig:6-lc-full}
\end{figure}

\begin{figure}
  \centering
  \includegraphics[trim={0.4cm 0cm 27.9cm 0cm}, clip, width=0.012\textwidth]{figs/learning_curves/all_algs_Accuracy_Data=_SVHN__Model=_rn__nQuery=_100__TrainAug=_0___.pdf}
  \includegraphics[trim={1.5cm 0cm 1.6cm 0cm}, clip,  width=0.32\textwidth]{{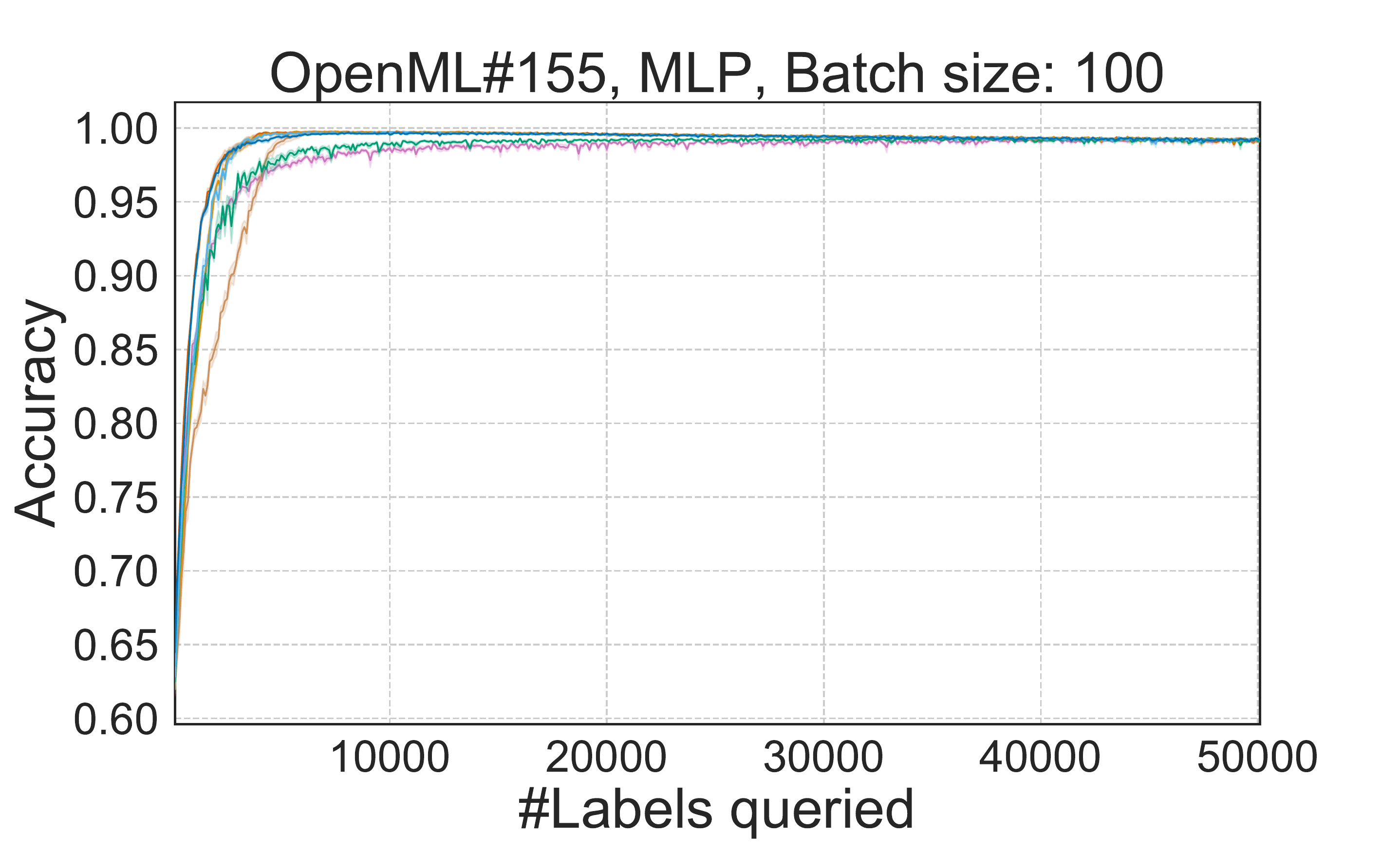}}
  \hfill
  \includegraphics[trim={1.5cm 0cm 1.6cm 0cm}, clip,  width=0.32\textwidth]{{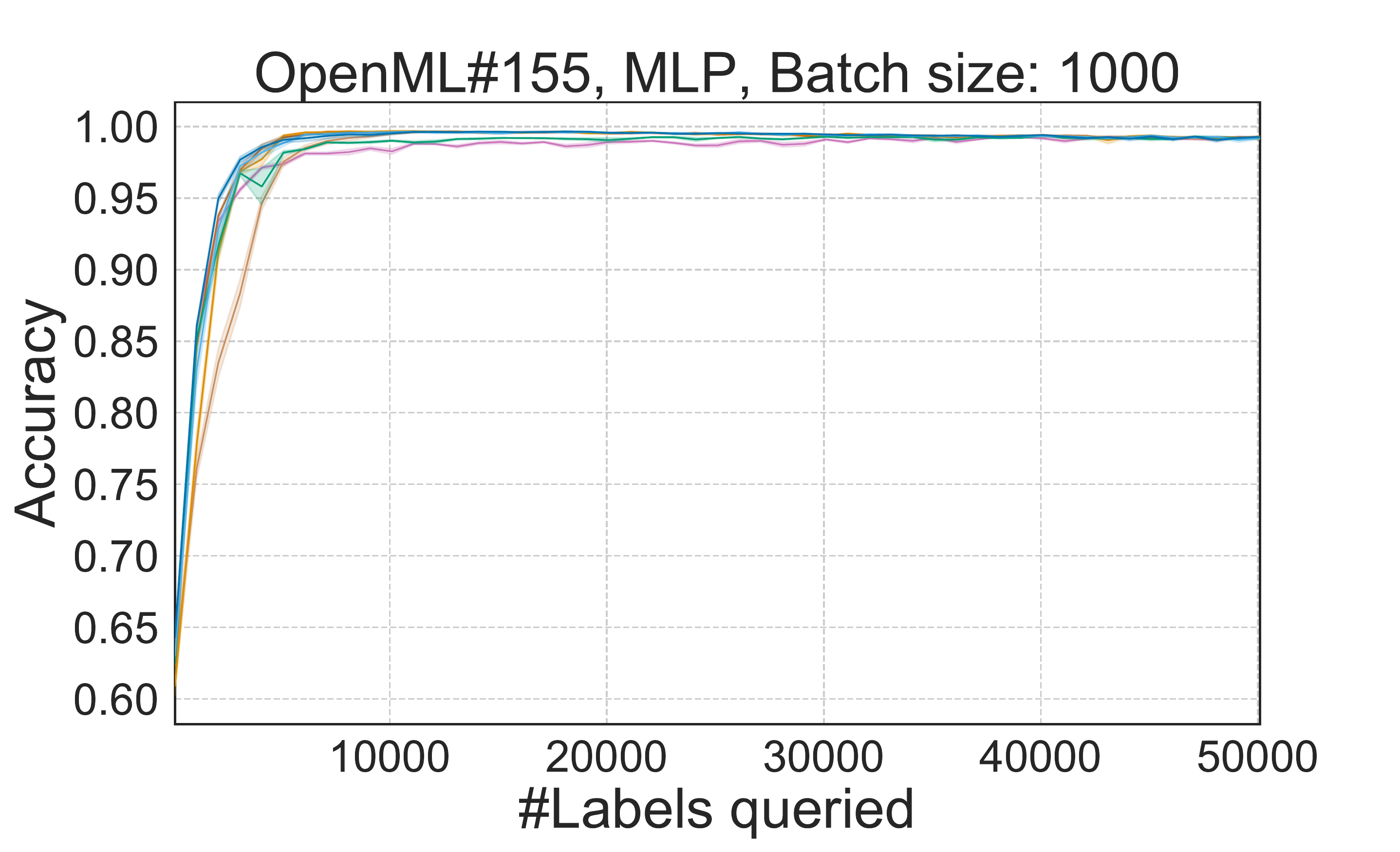}}
  \hfill
  \includegraphics[trim={1.5cm 0cm 1.6cm 0cm}, clip,  width=0.32\textwidth]{{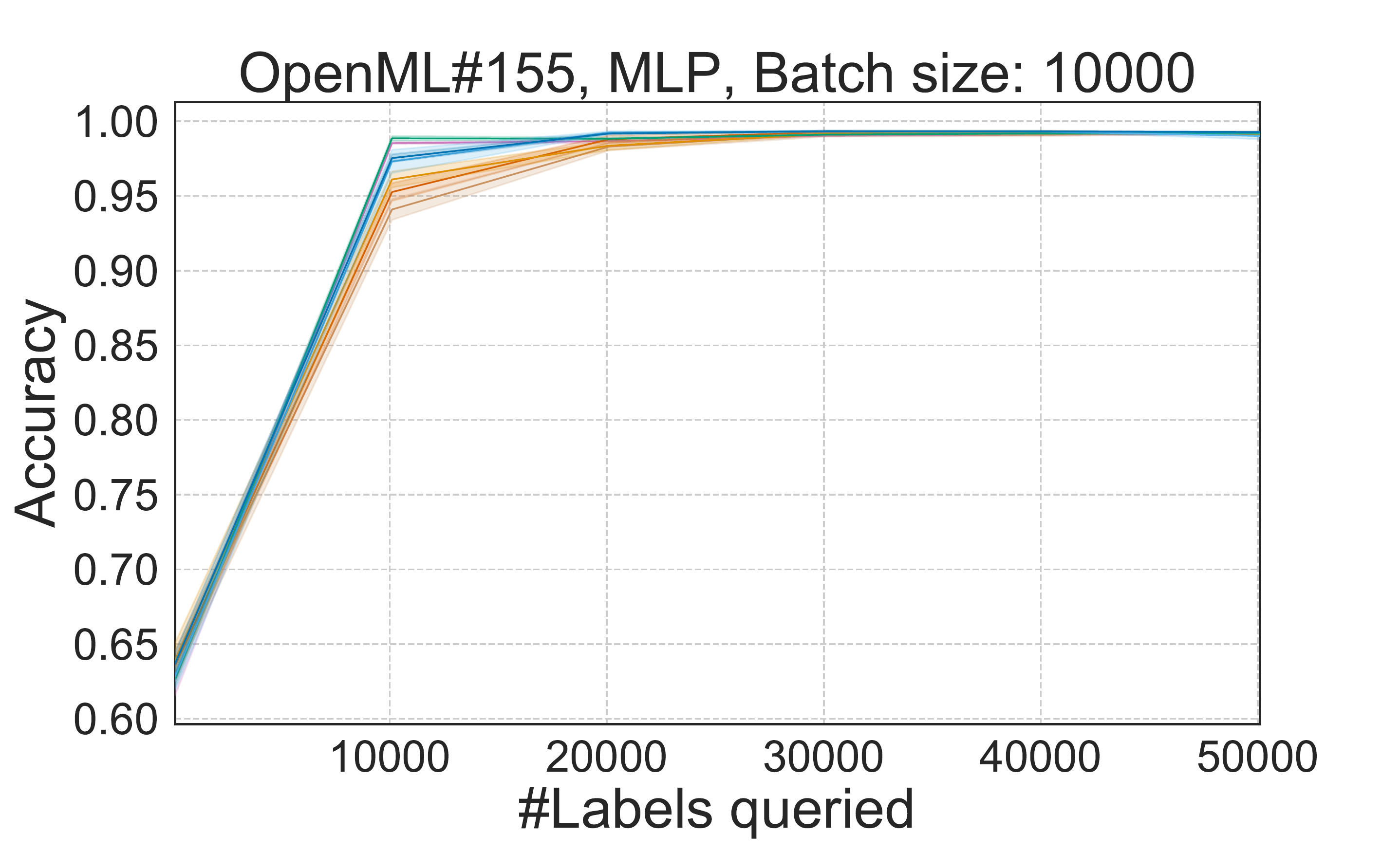}}
  \\
  \centering
  \begin{subfigure}[b]{\linewidth}
    \includegraphics[trim={0cm 0cm 0cm 0cm}, clip, width=\textwidth]{figs/legends/legend.pdf}
  \end{subfigure}
  \vskip -0.6cm
\caption{Full learning curves for OpenML \#155 with MLP.}
\vskip -0.5cm
\label{fig:155-lc-full}
\end{figure}

\begin{figure}
  \centering
    \includegraphics[trim={0.4cm 0cm 27.9cm 0cm}, clip, width=0.012\textwidth]{figs/learning_curves/all_algs_Accuracy_Data=_SVHN__Model=_rn__nQuery=_100__TrainAug=_0___.pdf}
  \includegraphics[trim={1.5cm 0cm 1.6cm 0cm}, clip, width=0.32\textwidth]{{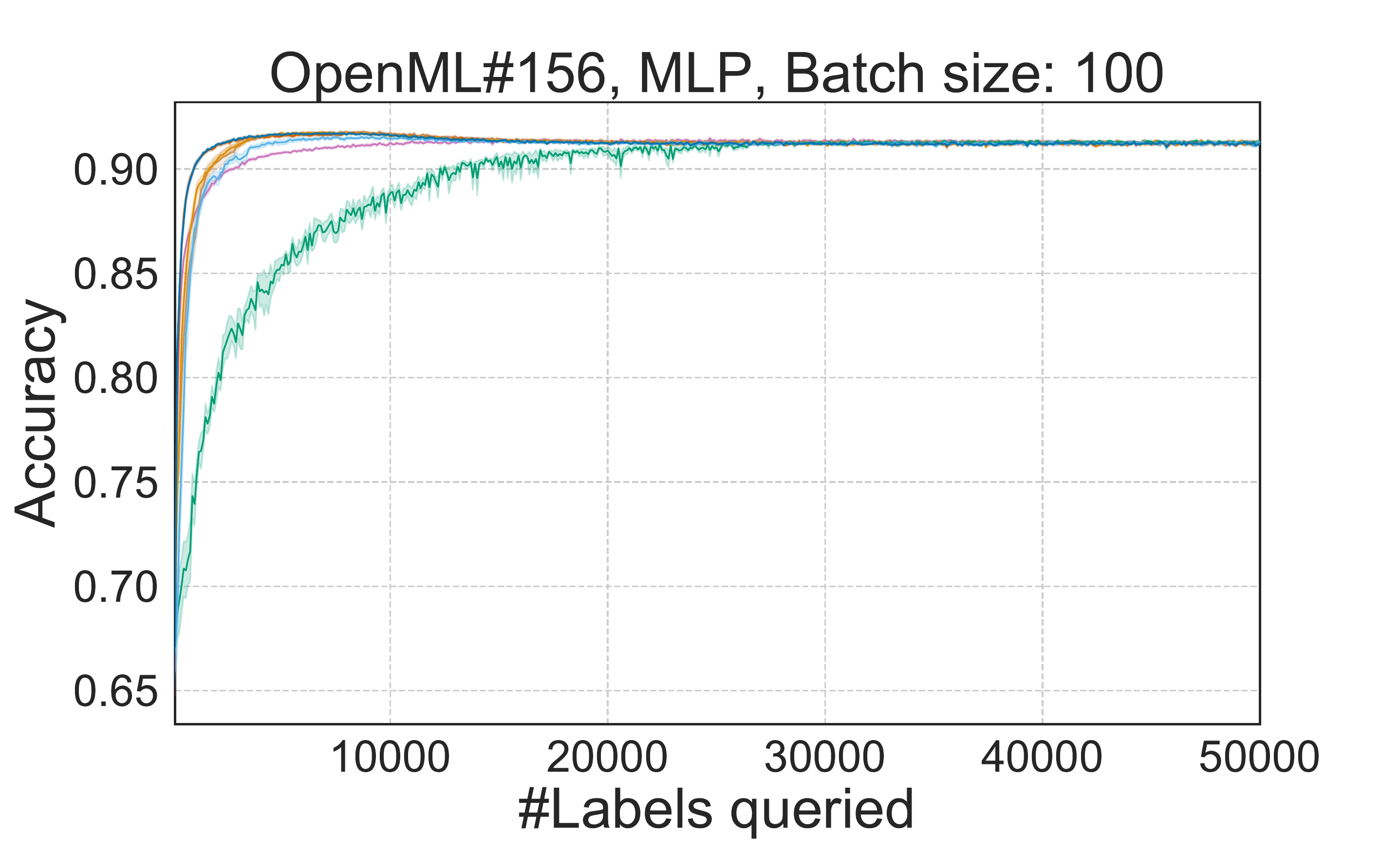}}
  \hfill
  \includegraphics[trim={1.5cm 0cm 1.6cm 0cm}, clip, width=0.32\textwidth]{{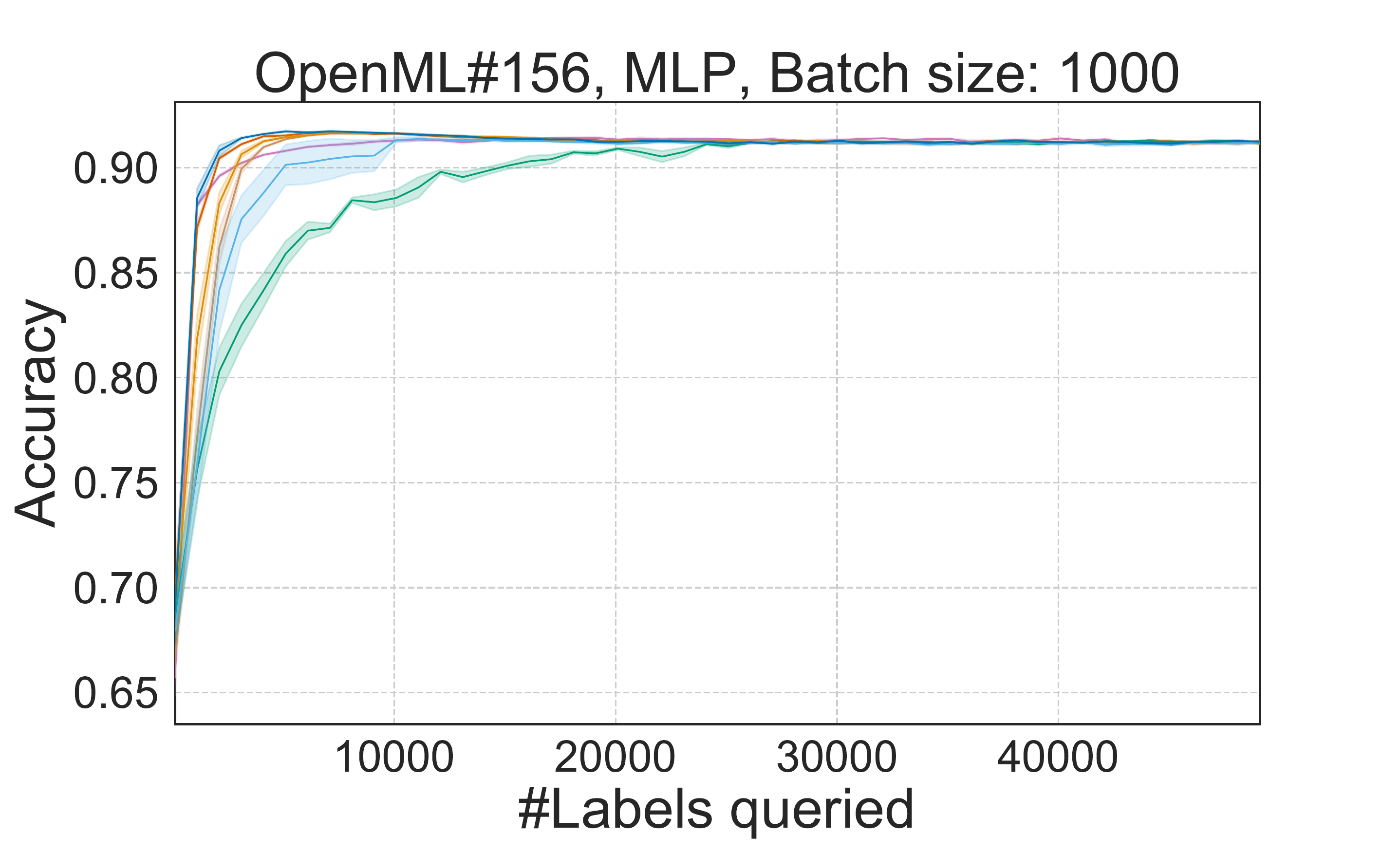}}
  \hfill
  \includegraphics[trim={1.5cm 0cm 1.6cm 0cm}, clip, width=0.32\textwidth]{{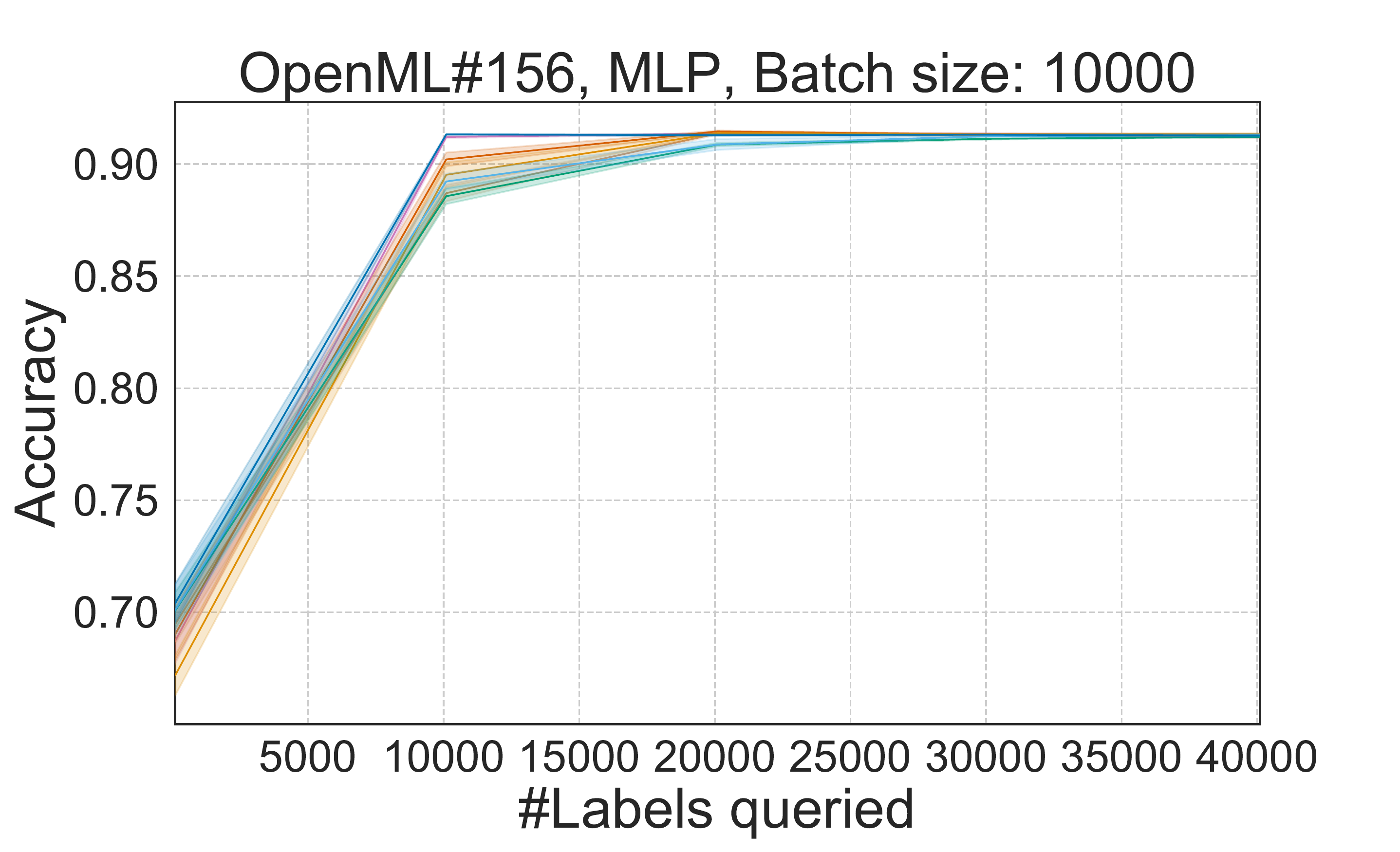}}
  \\
  \centering
  \begin{subfigure}[b]{\linewidth}
    \includegraphics[trim={0cm 0cm 0cm 0cm}, clip, width=\textwidth]{figs/legends/legend.pdf}
  \end{subfigure}
\caption{Full learning curves for OpenML \#156 with MLP.}
\label{fig:156-lc-full}
\end{figure}

\begin{figure}
  \centering
      \includegraphics[trim={0.4cm 0cm 27.9cm 0cm}, clip, width=0.012\textwidth]{figs/learning_curves/all_algs_Accuracy_Data=_SVHN__Model=_rn__nQuery=_100__TrainAug=_0___.pdf}
  \includegraphics[trim={1.5cm 0cm 1.6cm 0cm}, clip, width=0.32\textwidth]{{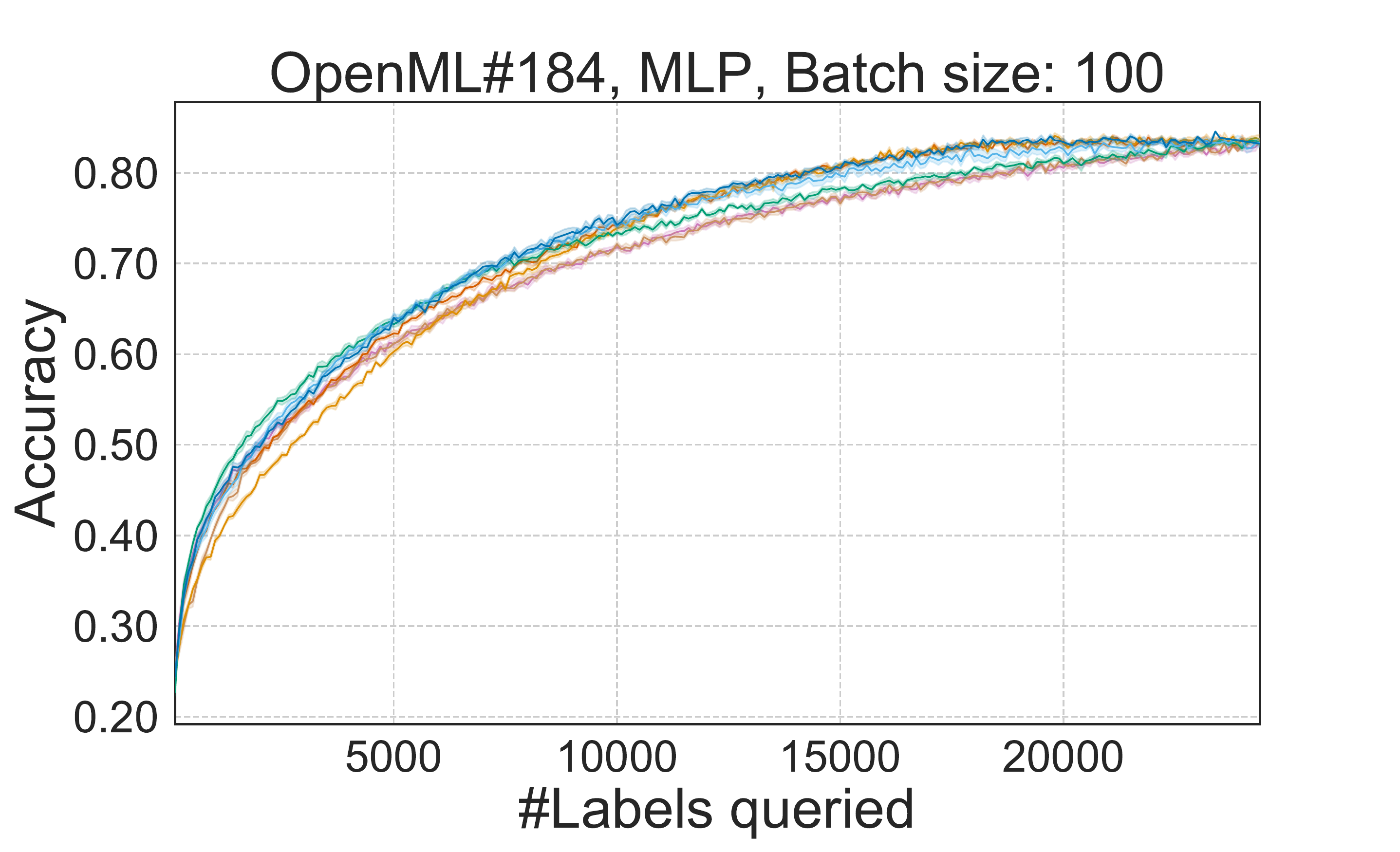}}
  \hfill
  \includegraphics[trim={1.5cm 0cm 1.6cm 0cm}, clip, width=0.32\textwidth]{{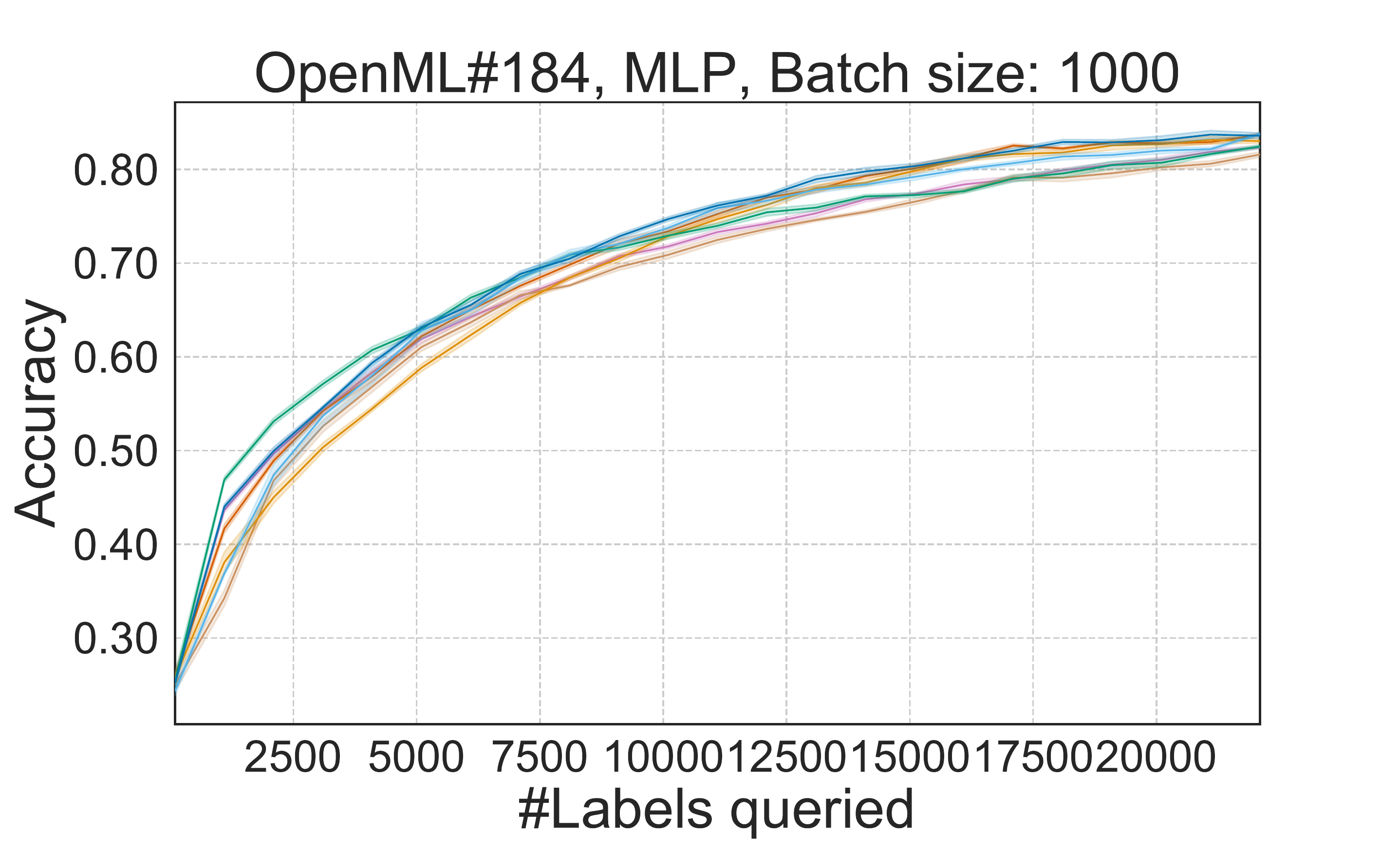}}
  \hfill
  \includegraphics[trim={1.5cm 0cm 1.6cm 0cm}, clip, width=0.32\textwidth]{{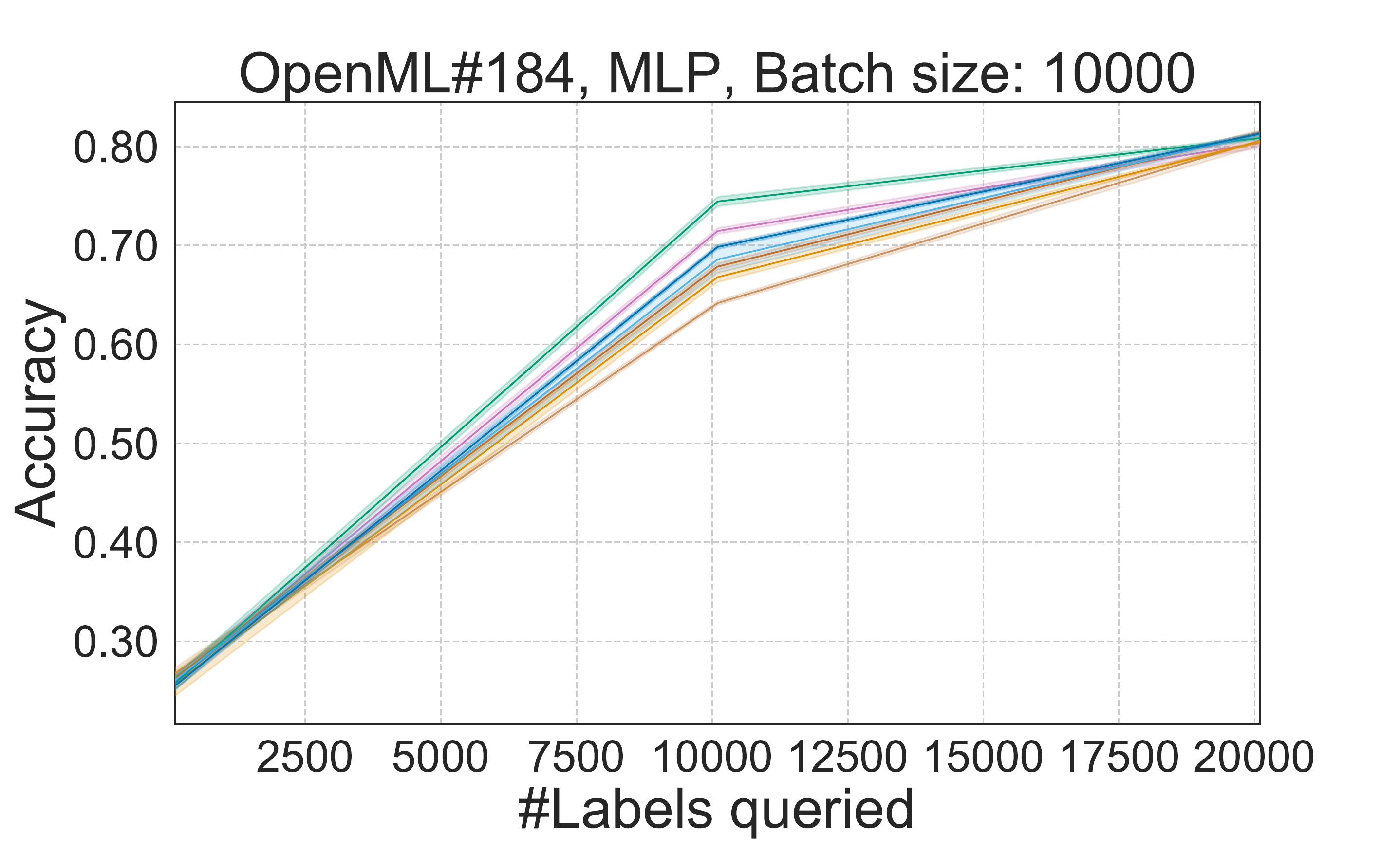}}
  \\
  \centering
  \begin{subfigure}[b]{\linewidth}
    \includegraphics[trim={0cm 0cm 0cm 0cm}, clip, width=\textwidth]{figs/legends/legend.pdf}
  \end{subfigure}
\caption{Full learning curves for OpenML \#184 with MLP.}
\label{fig:184-lc-full}
\end{figure}

\begin{figure}
  \centering
      \includegraphics[trim={0.4cm 0cm 27.9cm 0cm}, clip, width=0.012\textwidth]{figs/learning_curves/all_algs_Accuracy_Data=_SVHN__Model=_rn__nQuery=_100__TrainAug=_0___.pdf}
  \includegraphics[trim={1.5cm 0cm 1.6cm 0cm}, clip, width=0.32\textwidth]{{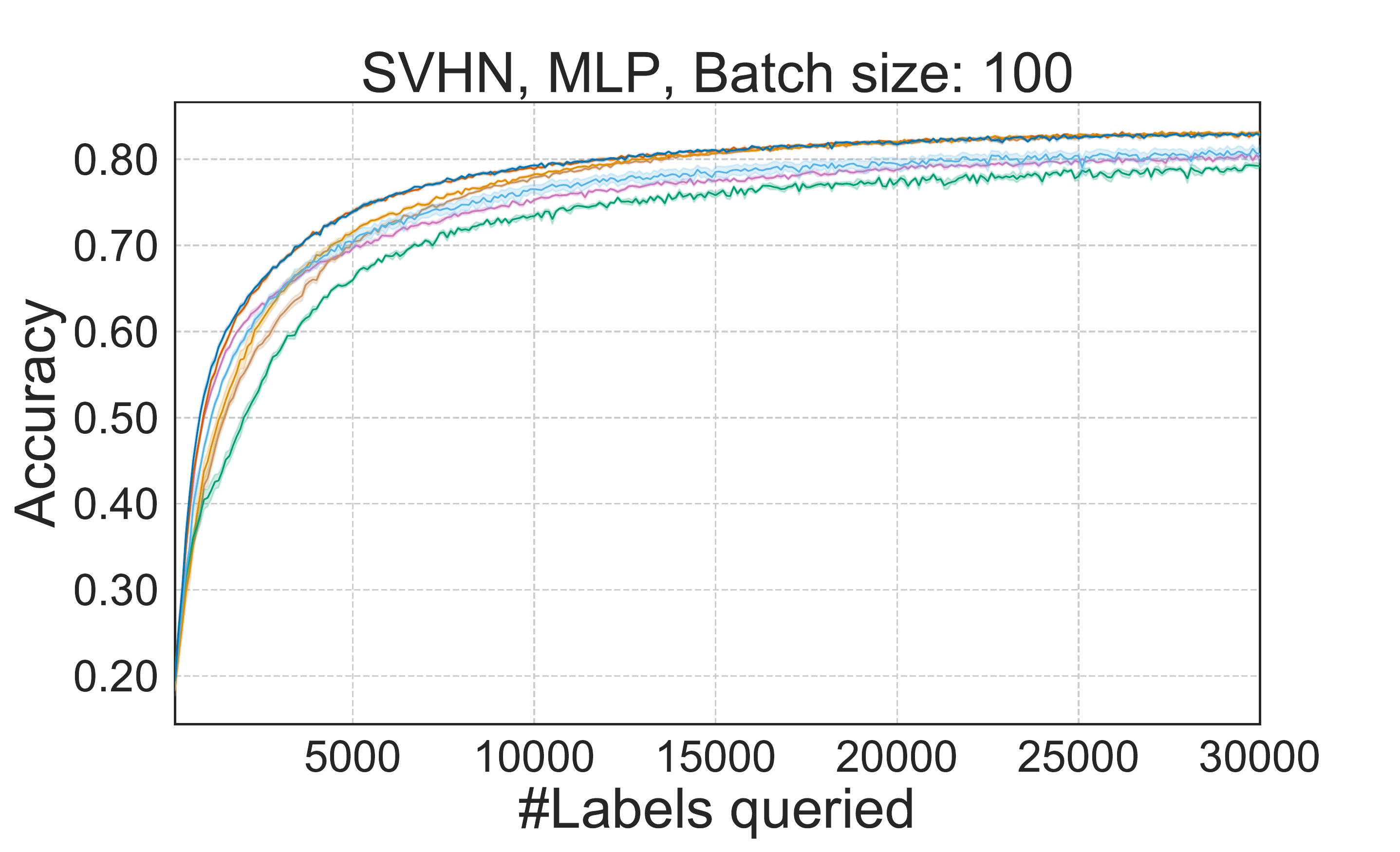}}
  \hfill
  \includegraphics[trim={1.5cm 0cm 1.6cm 0cm}, clip, width=0.32\textwidth]{{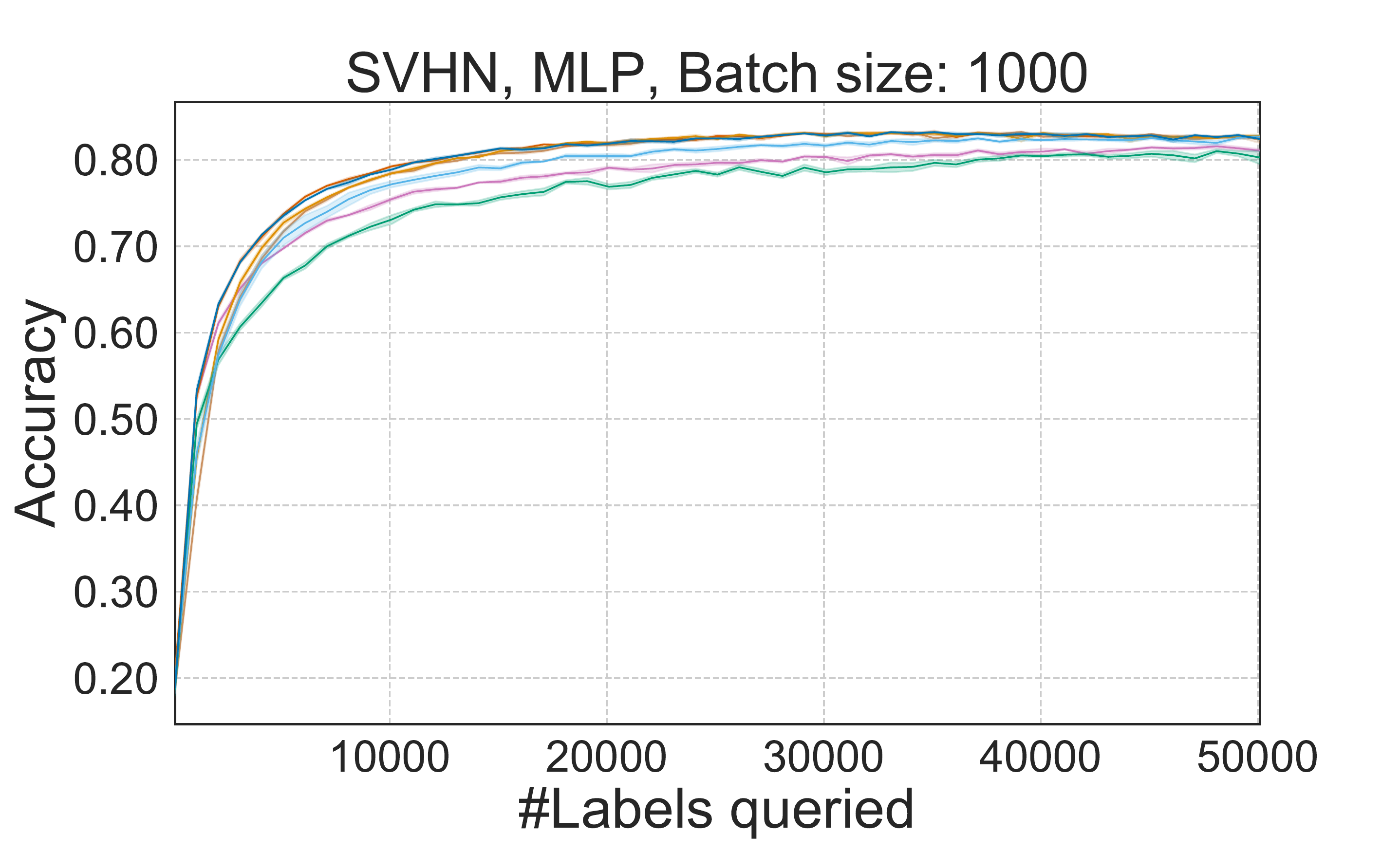}}
  \hfill
  \includegraphics[trim={1.5cm 0cm 1.6cm 0cm}, clip, trim={2.18cm 0cm 2cm 0cm}, clip,width=0.32\textwidth]{{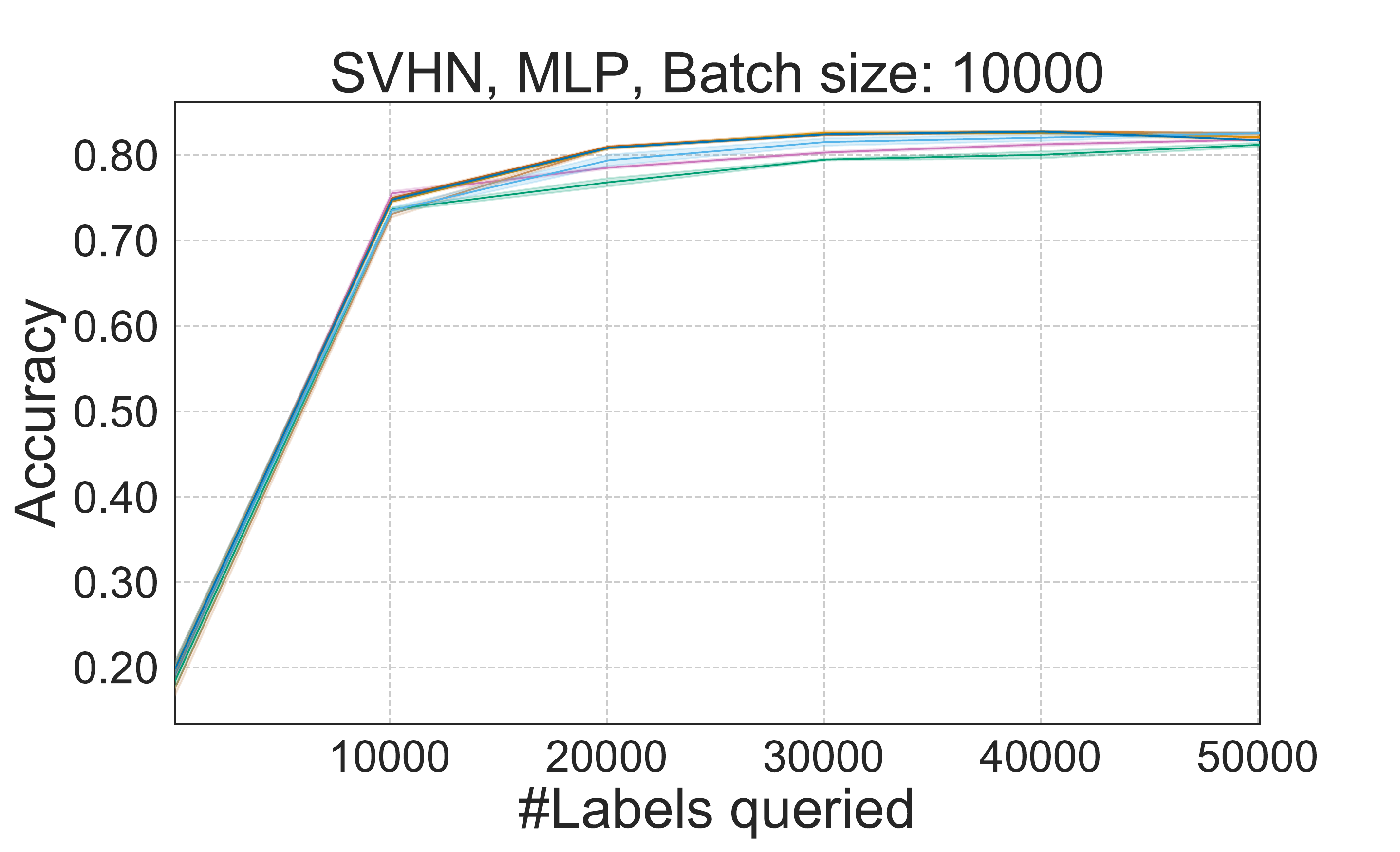}}

    \includegraphics[trim={0.4cm 0cm 27.9cm 0cm}, clip, width=0.012\textwidth]{figs/learning_curves/all_algs_Accuracy_Data=_SVHN__Model=_rn__nQuery=_100__TrainAug=_0___.pdf}
  \includegraphics[trim={1.5cm 0cm 1.6cm 0cm}, clip, width=0.32\textwidth]{{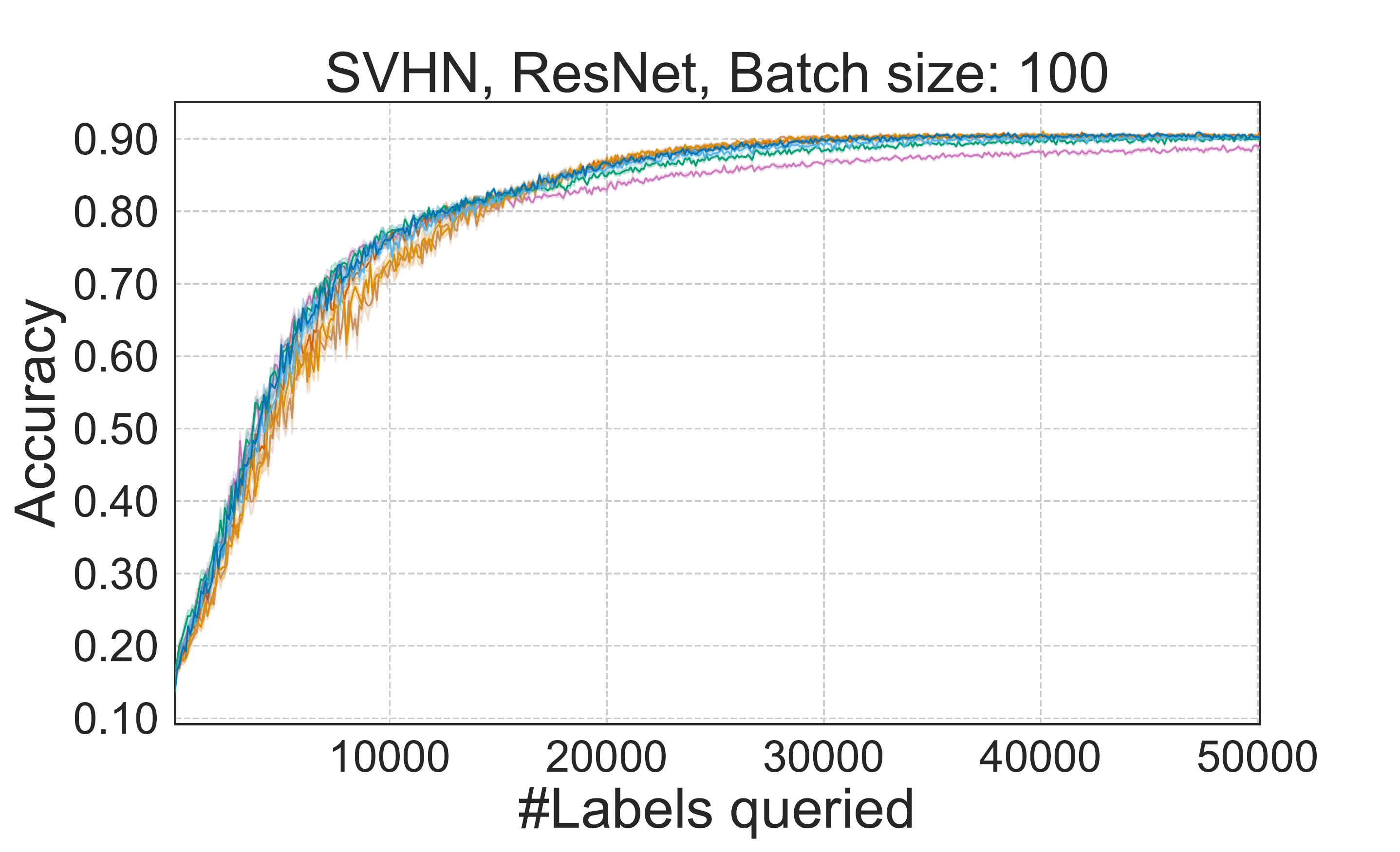}}
  \hfill
  \includegraphics[trim={1.5cm 0cm 1.6cm 0cm}, clip, width=0.32\textwidth]{{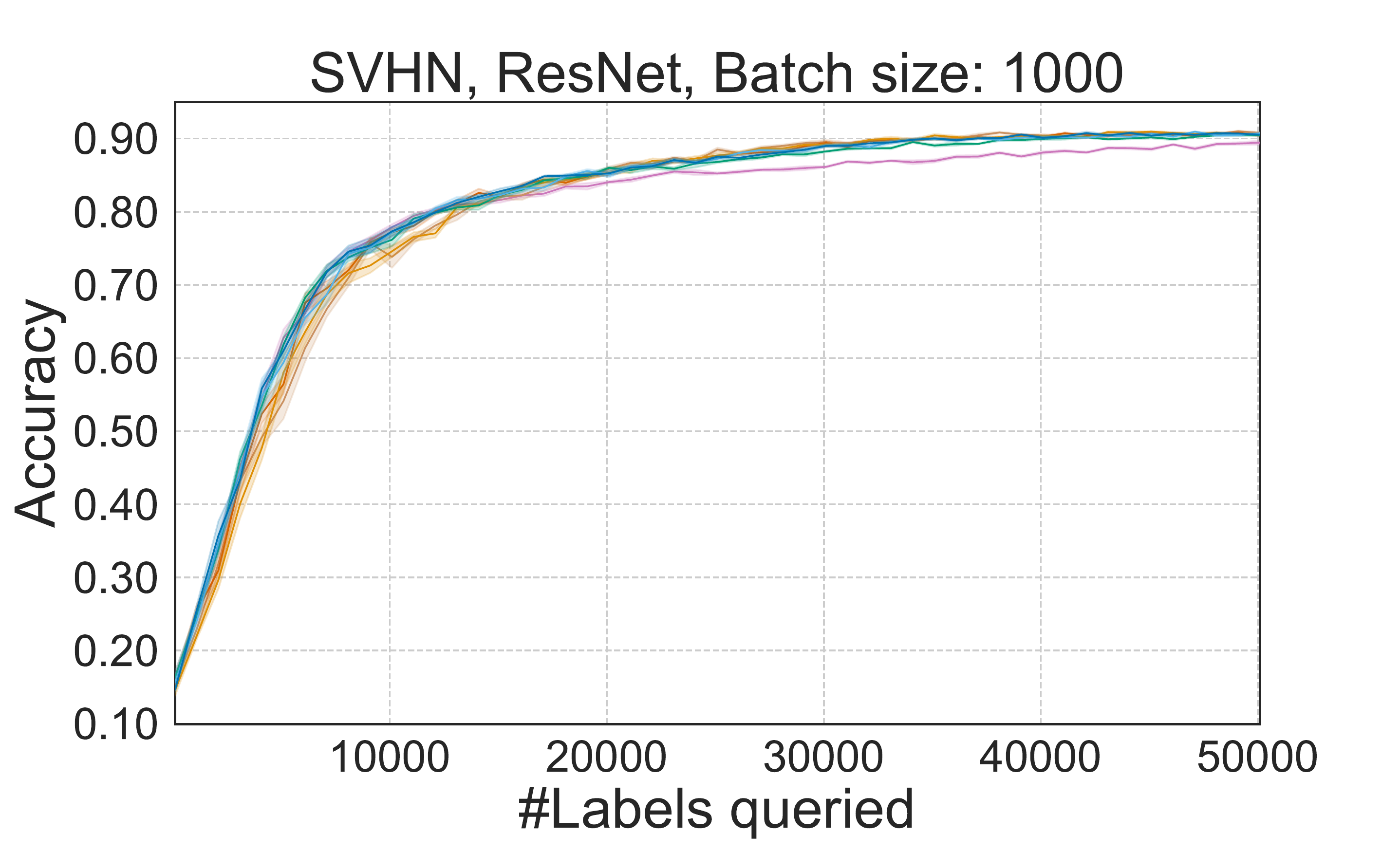}}
  \hfill
  \includegraphics[trim={1.5cm 0cm 1.6cm 0cm}, clip, width=0.32\textwidth]{{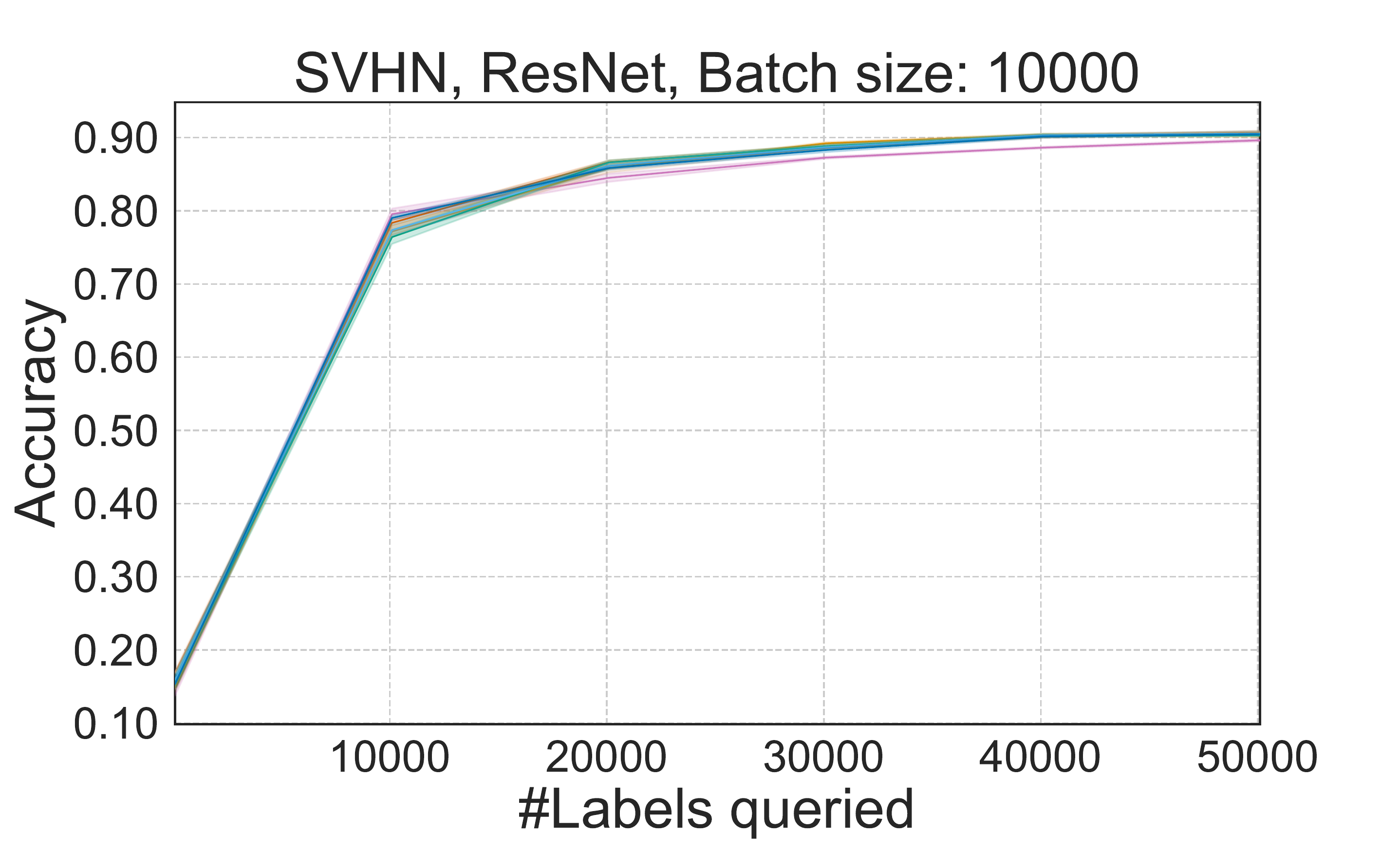}}

    \includegraphics[trim={0.4cm 0cm 27.9cm 0cm}, clip, width=0.012\textwidth]{figs/learning_curves/all_algs_Accuracy_Data=_SVHN__Model=_rn__nQuery=_100__TrainAug=_0___.pdf}
  \includegraphics[trim={1.5cm 0cm 1.6cm 0cm}, clip, width=0.32\textwidth]{{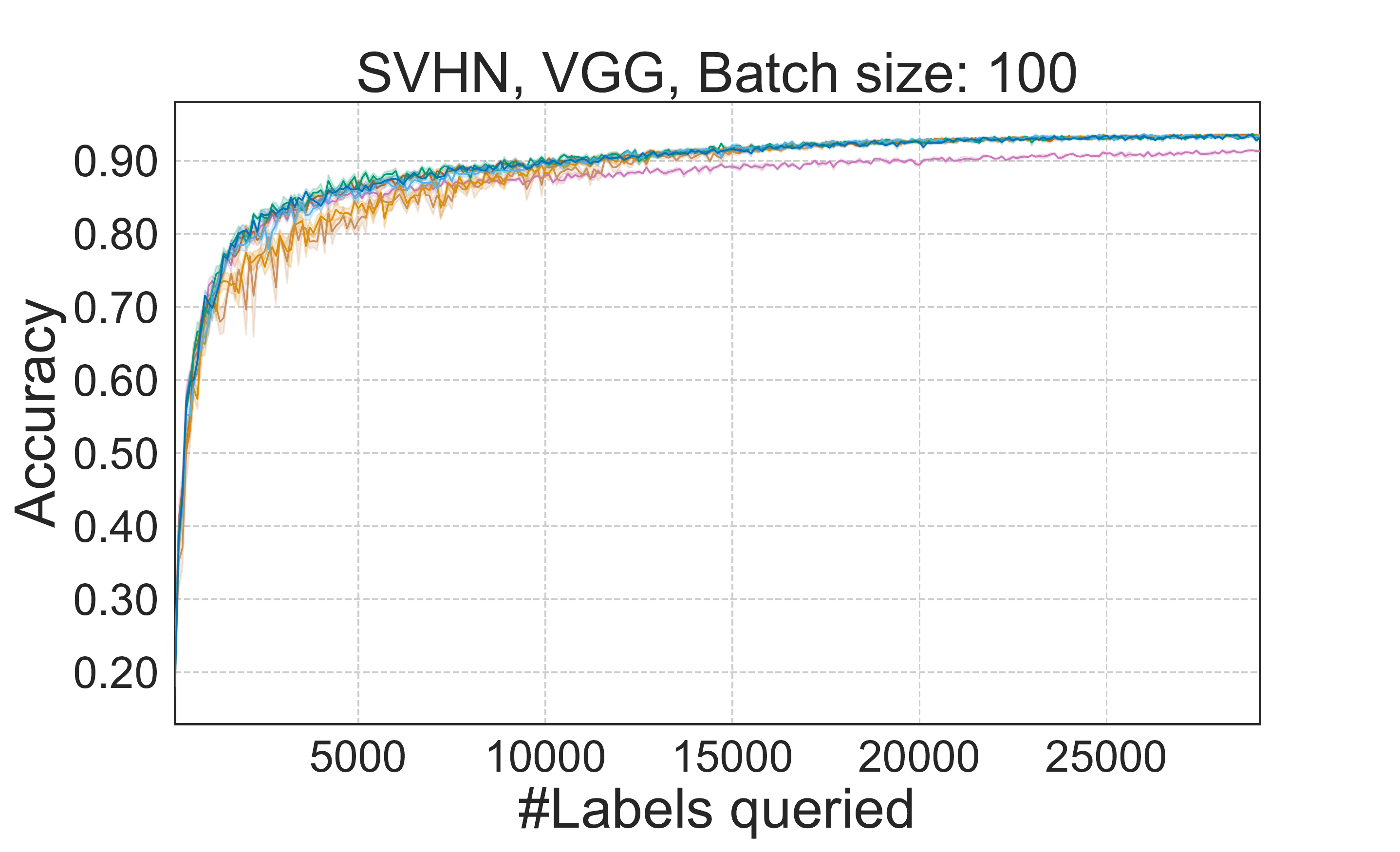}}
  \hfill
  \includegraphics[trim={1.5cm 0cm 1.6cm 0cm}, clip, width=0.32\textwidth]{{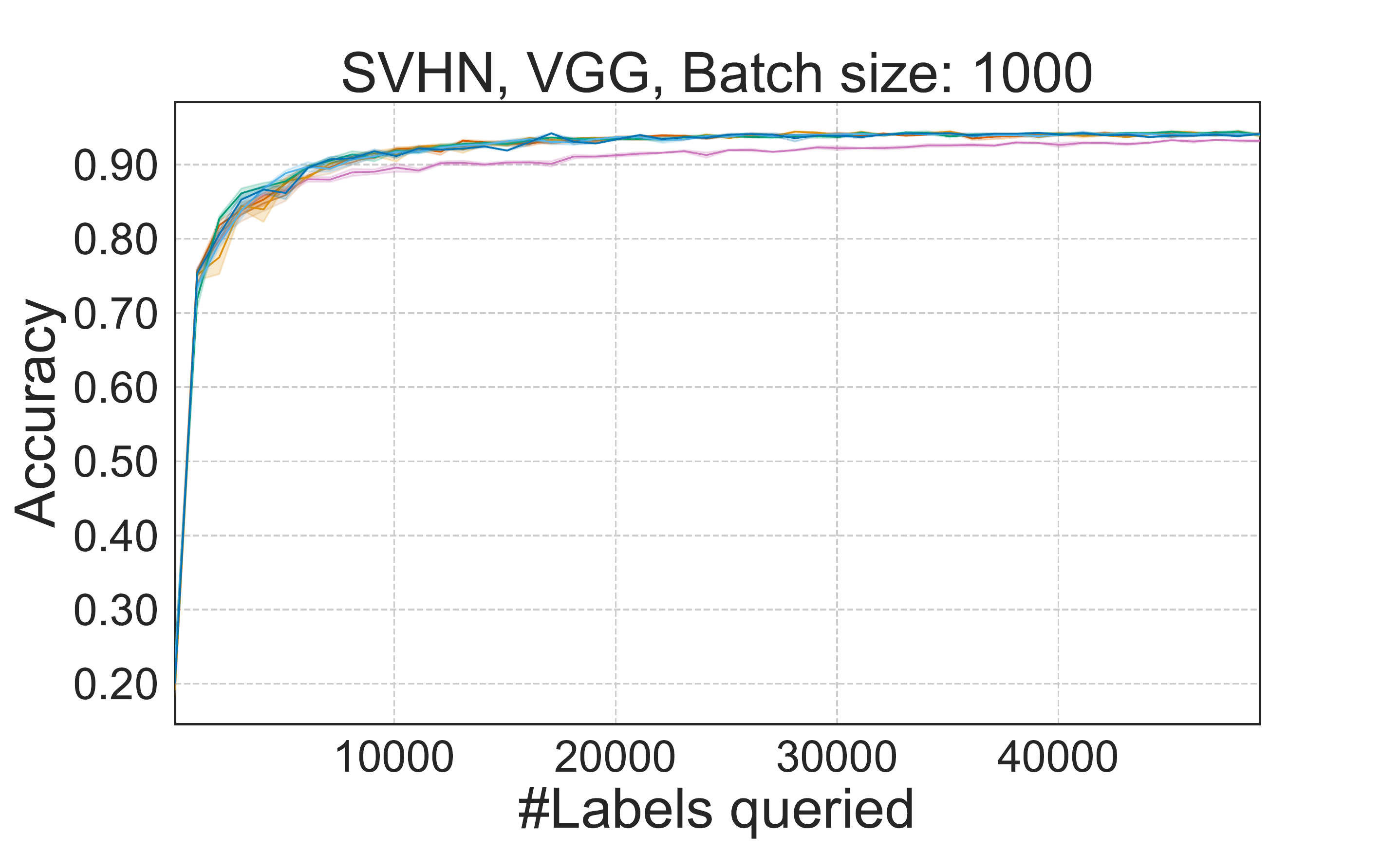}}
  \hfill
  \includegraphics[trim={1.5cm 0cm 1.6cm 0cm}, clip, width=0.32\textwidth]{{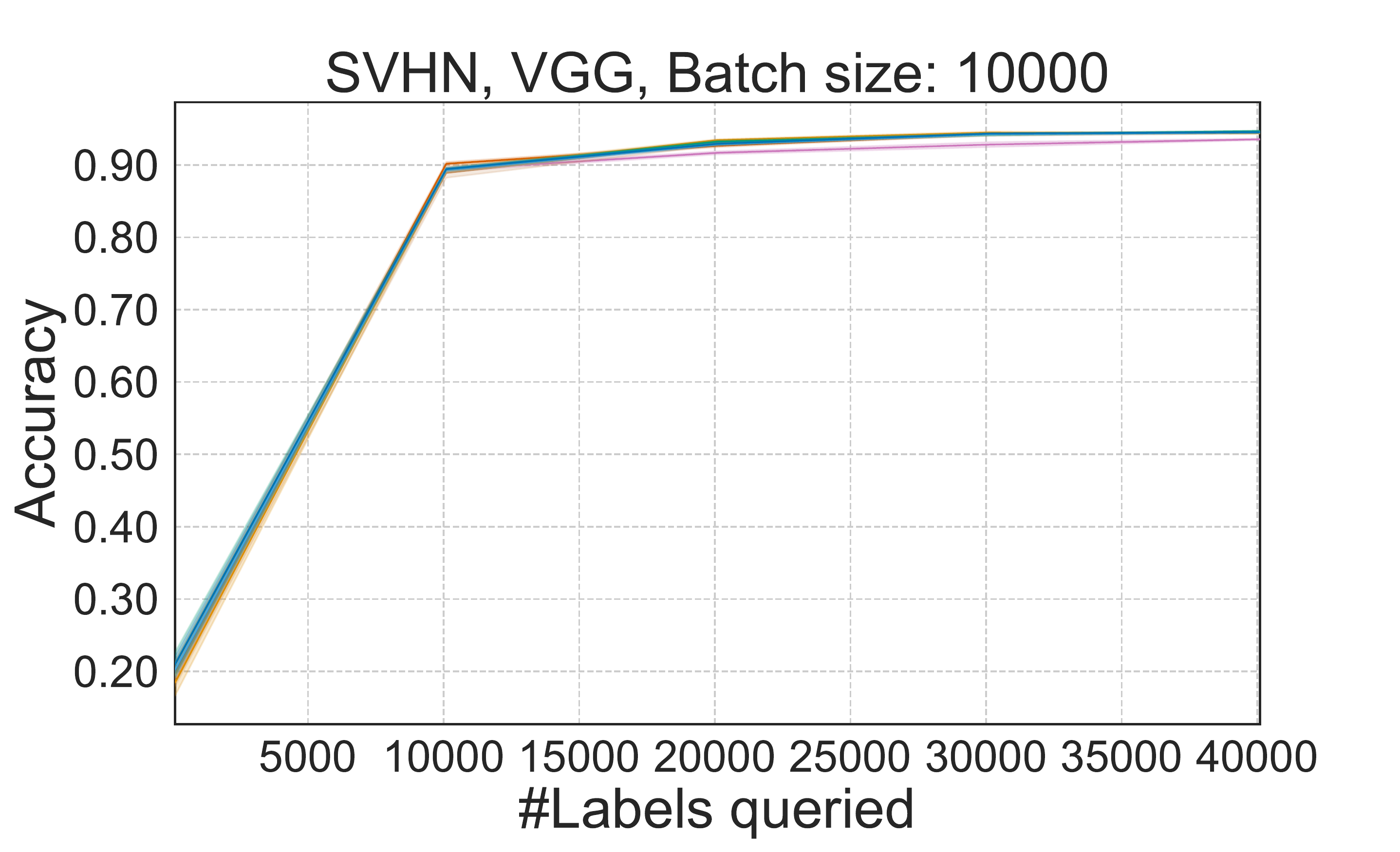}}
  \\
  \centering
  \begin{subfigure}[b]{\linewidth}
    \includegraphics[trim={0cm 0cm 0cm 0cm}, clip, width=\textwidth]{figs/legends/legend.pdf}
  \end{subfigure}

\caption{Full learning curves for SVHN with MLP, ResNet and VGG.}
\label{fig:svhn-lc-full}
\end{figure}

\begin{figure}
  \centering
      \includegraphics[trim={0.4cm 0cm 27.9cm 0cm}, clip, width=0.012\textwidth]{figs/learning_curves/all_algs_Accuracy_Data=_SVHN__Model=_rn__nQuery=_100__TrainAug=_0___.pdf}
  \includegraphics[trim={1.5cm 0cm 1.6cm 0cm}, clip, width=0.32\textwidth]{{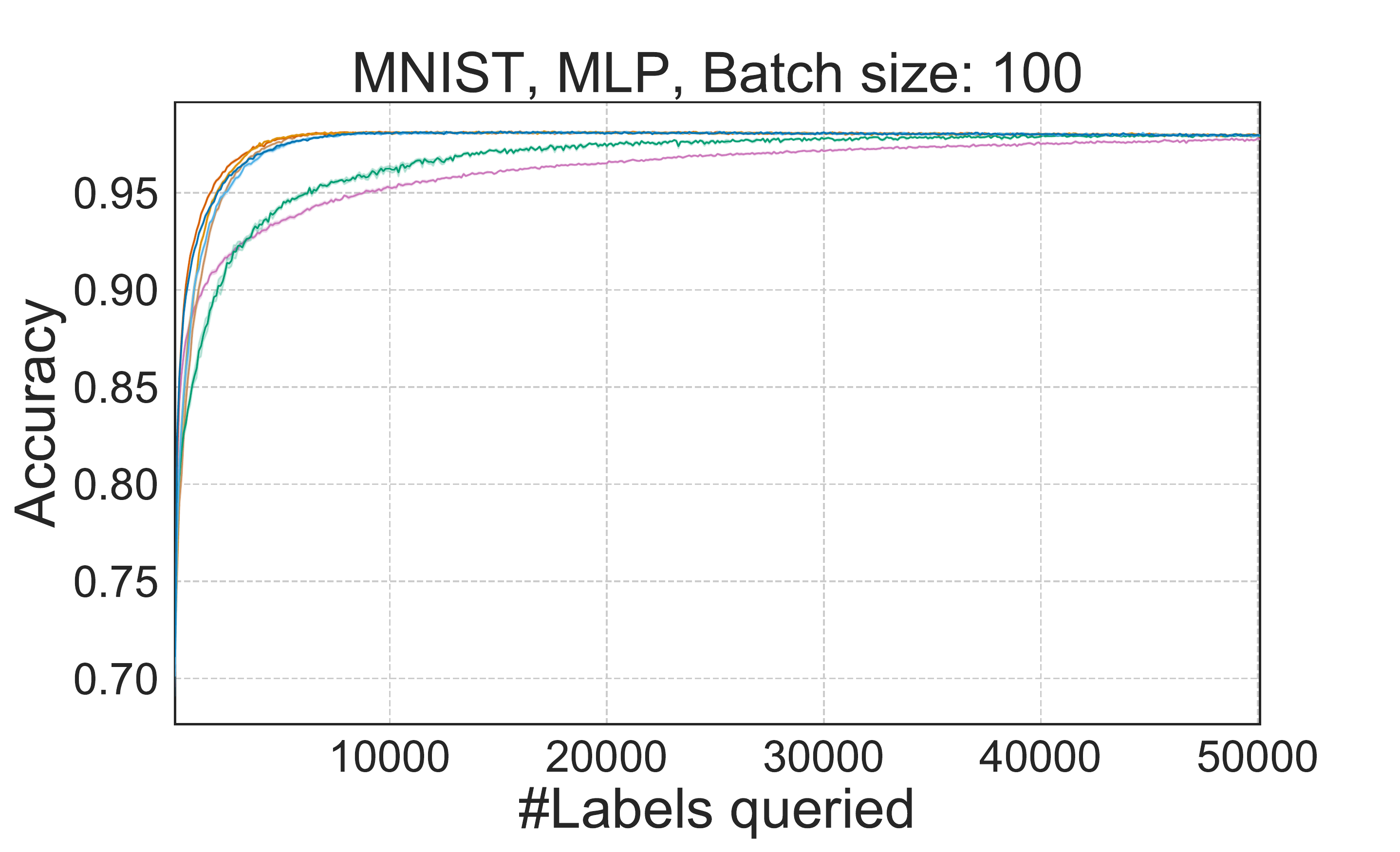}}
  \hfill
  \includegraphics[trim={1.5cm 0cm 1.6cm 0cm}, clip, width=0.32\textwidth]{{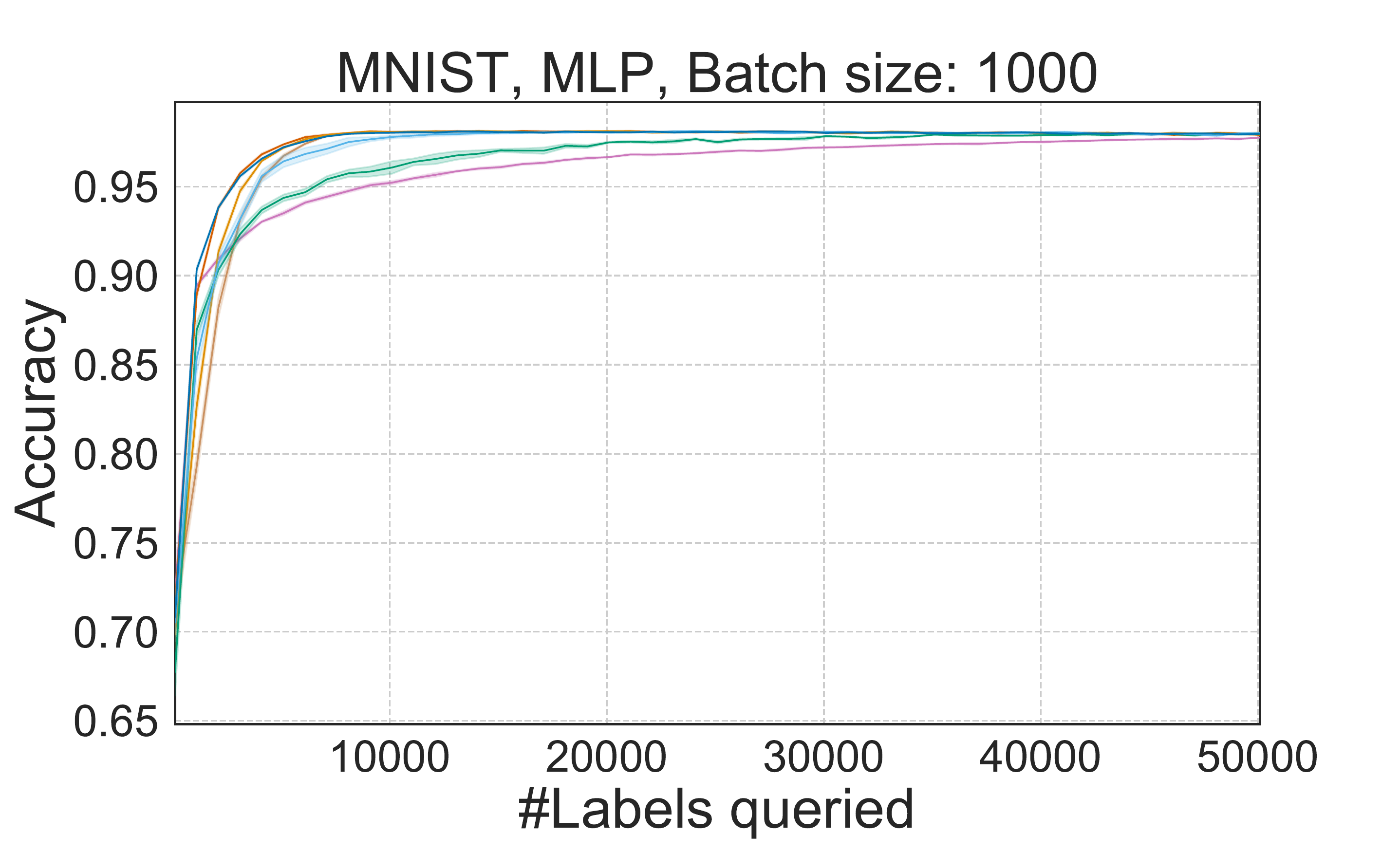}}
  \hfill
  \includegraphics[trim={1.5cm 0cm 1.6cm 0cm}, clip, width=0.32\textwidth]{{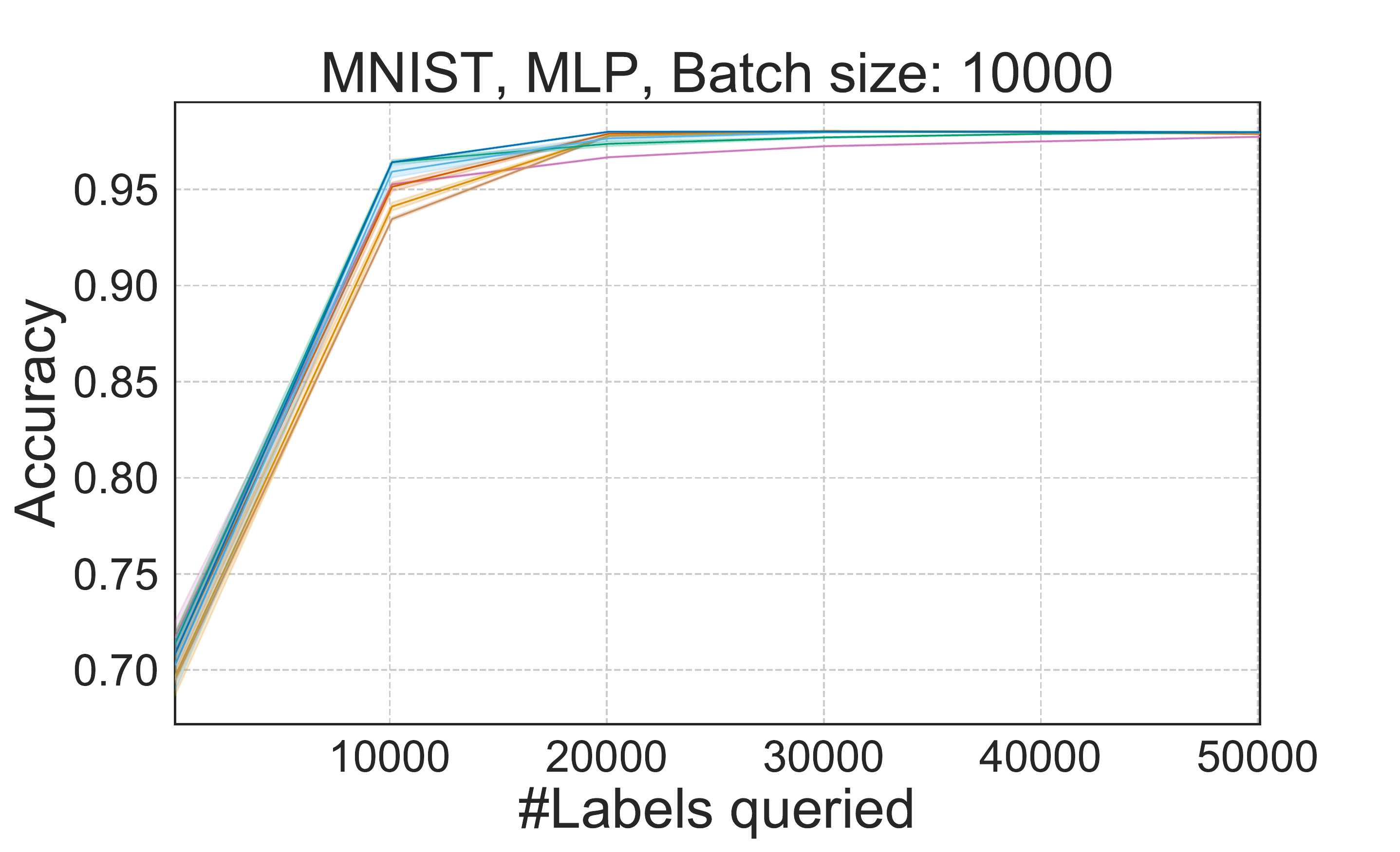}}
  \\
  \centering
  \begin{subfigure}[b]{\linewidth}
    \includegraphics[trim={0cm 0cm 0cm 0cm}, clip, width=\textwidth]{figs/legends/legend.pdf}
  \end{subfigure}

\caption{Full learning curves for MNIST with MLP.}
\label{fig:mnist-lc-full}
\end{figure}

\begin{figure}
  \centering
      \includegraphics[trim={0.4cm 0cm 27.9cm 0cm}, clip, width=0.012\textwidth]{figs/learning_curves/all_algs_Accuracy_Data=_SVHN__Model=_rn__nQuery=_100__TrainAug=_0___.pdf}
  \includegraphics[trim={1.5cm 0cm 1.6cm 0cm}, clip, width=0.32\textwidth]{{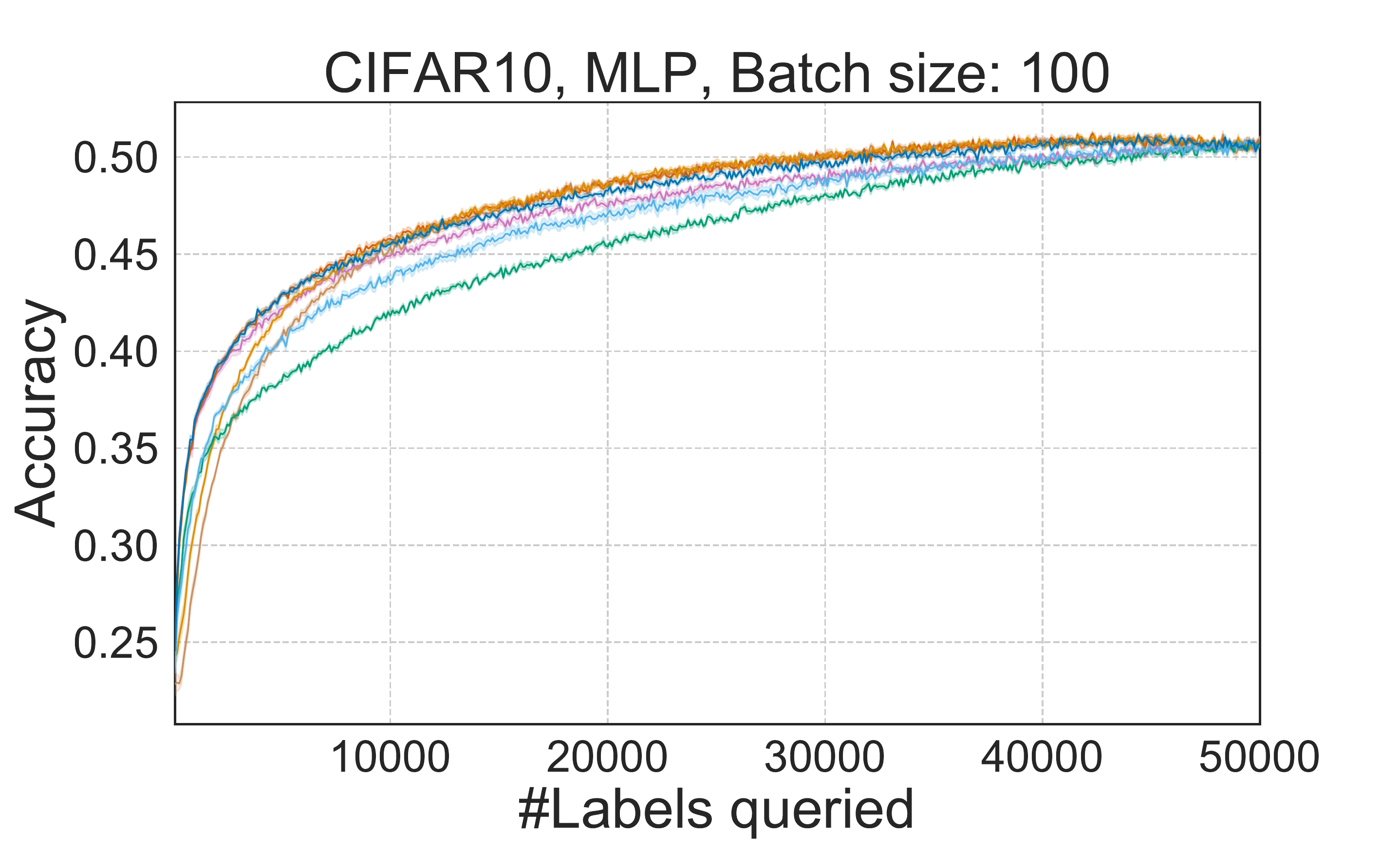}}
  \hfill
  \includegraphics[trim={1.5cm 0cm 1.6cm 0cm}, clip, width=0.32\textwidth]{{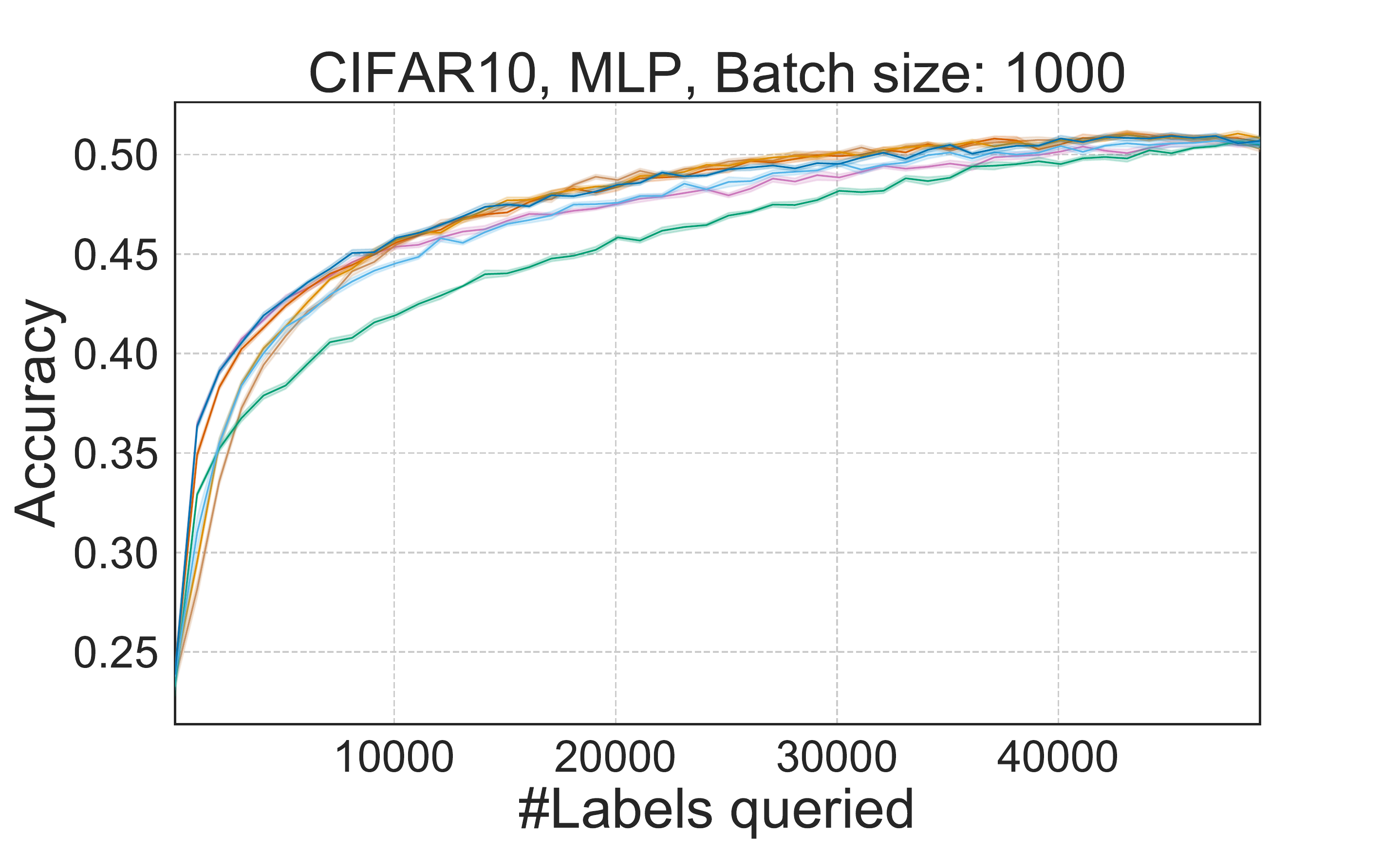}}
  \hfill
  \includegraphics[trim={1.5cm 0cm 1.6cm 0cm}, clip, width=0.32\textwidth]{{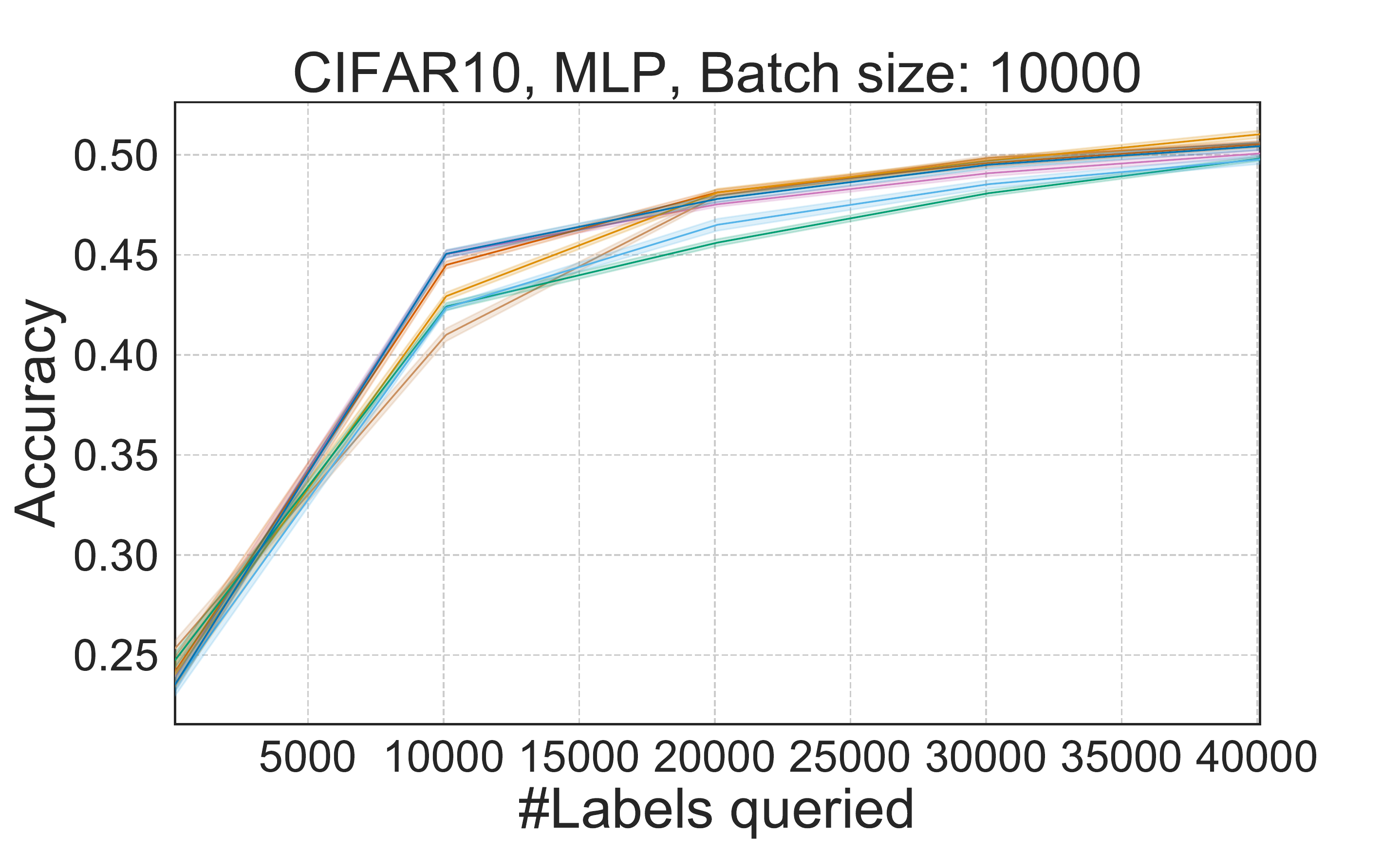}}

    \includegraphics[trim={0.4cm 0cm 27.9cm 0cm}, clip, width=0.012\textwidth]{figs/learning_curves/all_algs_Accuracy_Data=_SVHN__Model=_rn__nQuery=_100__TrainAug=_0___.pdf}
  \includegraphics[trim={1.5cm 0cm 1.6cm 0cm}, clip, width=0.32\textwidth]{{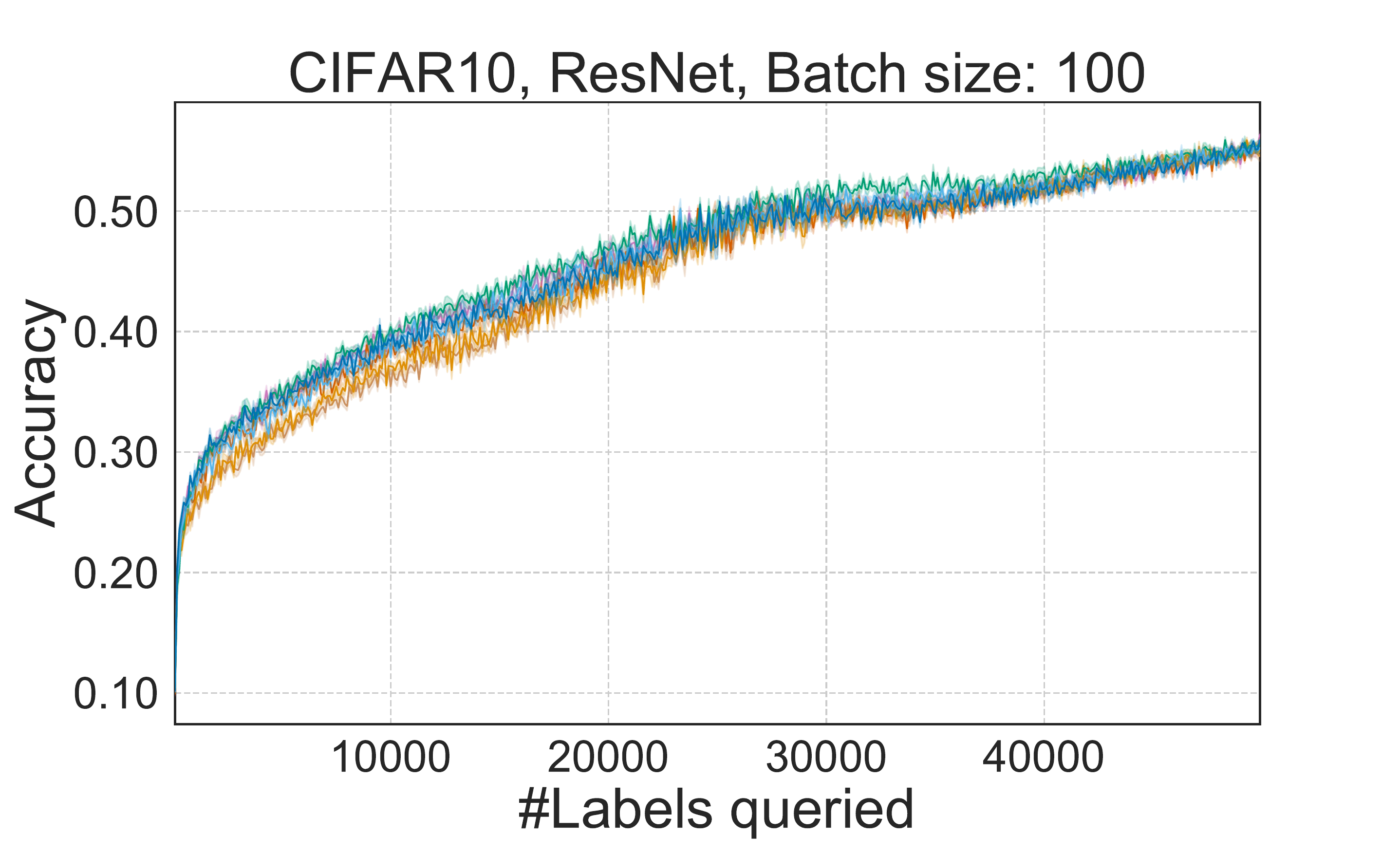}}
  \hfill
  \includegraphics[trim={1.5cm 0cm 1.6cm 0cm}, clip, width=0.32\textwidth]{{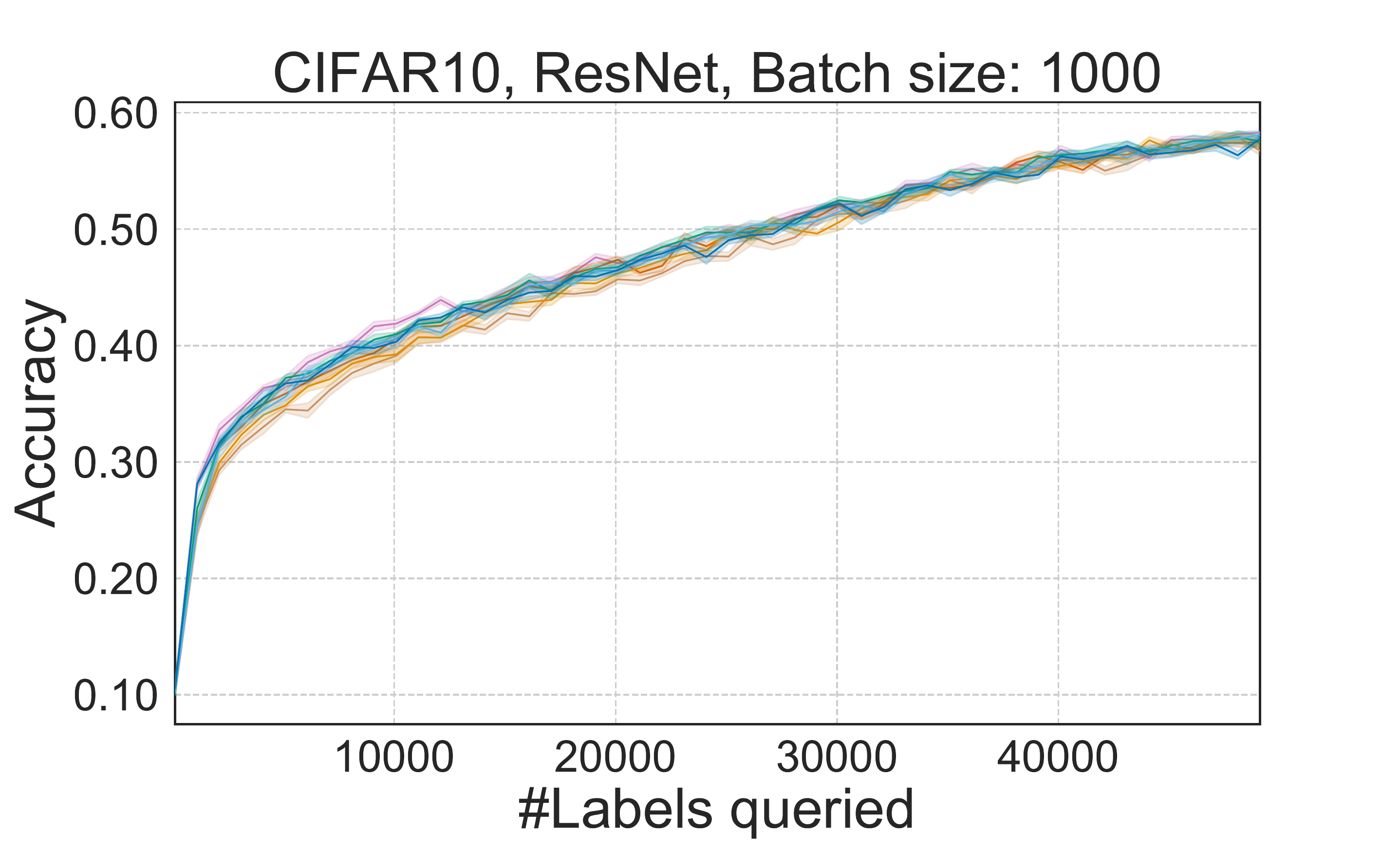}}
  \hfill
  \includegraphics[trim={1.5cm 0cm 1.6cm 0cm}, clip, width=0.32\textwidth]{{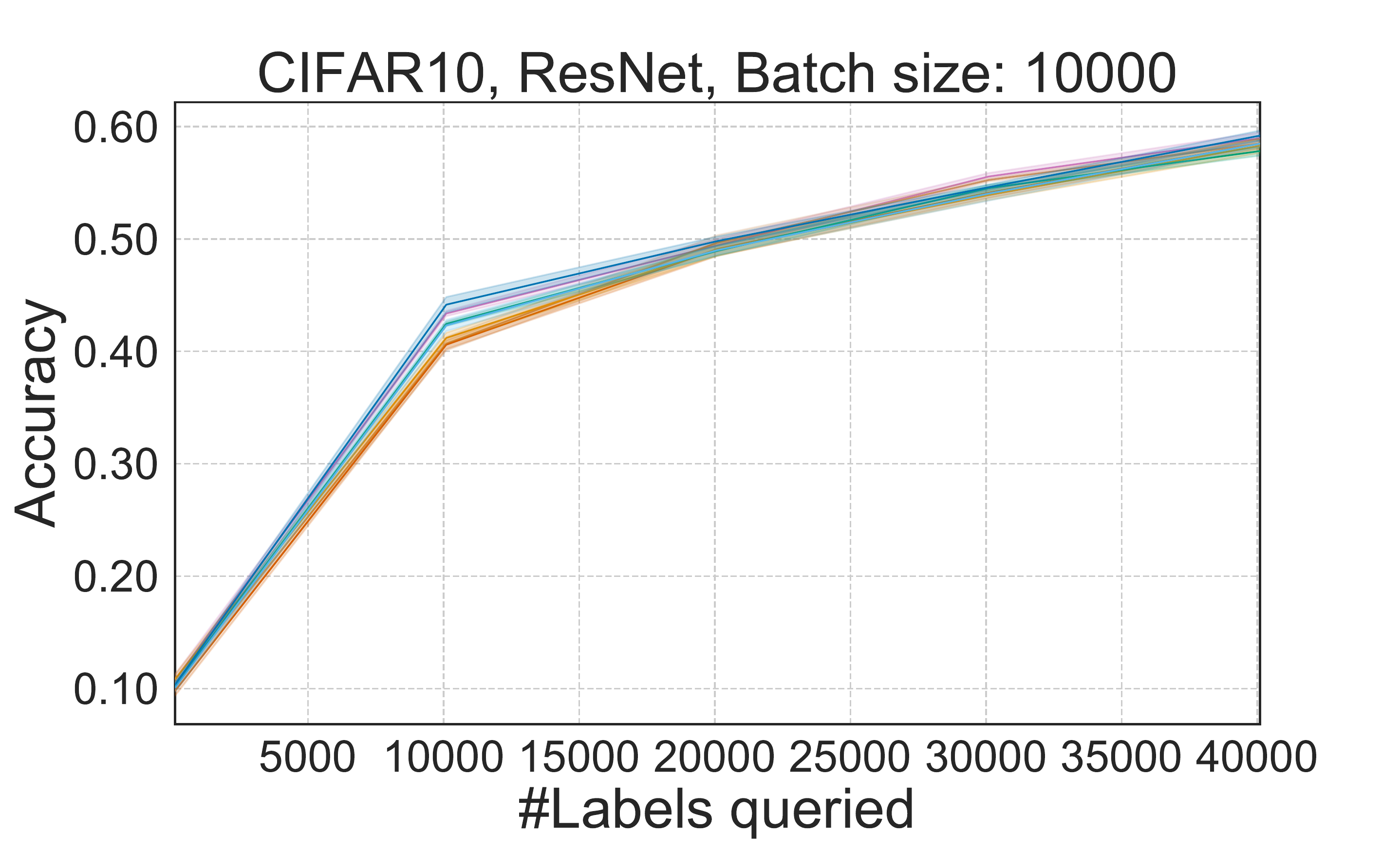}}

    \includegraphics[trim={0.4cm 0cm 27.9cm 0cm}, clip, width=0.012\textwidth]{figs/learning_curves/all_algs_Accuracy_Data=_SVHN__Model=_rn__nQuery=_100__TrainAug=_0___.pdf}
  \includegraphics[trim={1.5cm 0cm 1.6cm 0cm}, clip, width=0.32\textwidth]{{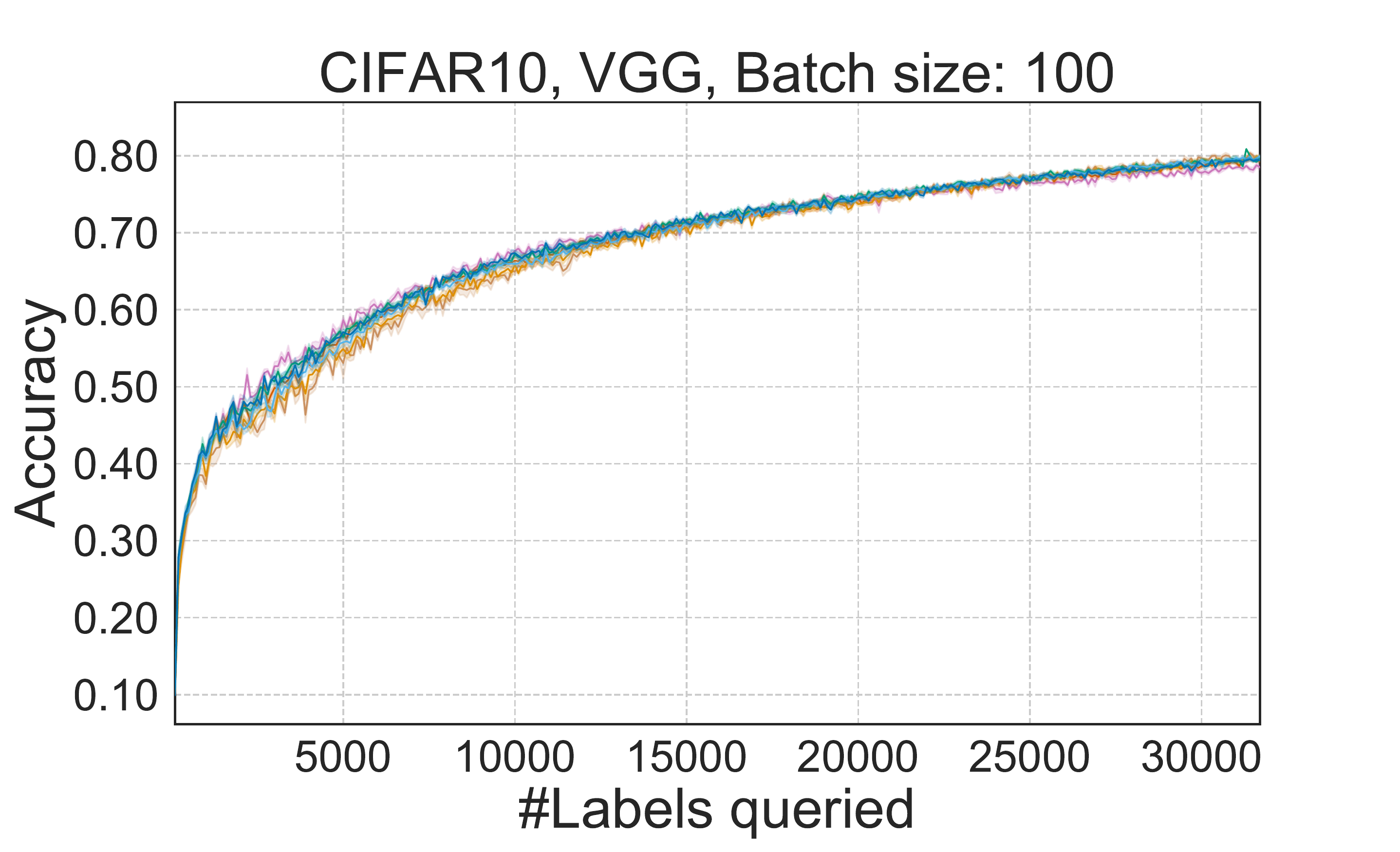}}
  \hfill
  \includegraphics[trim={1.5cm 0cm 1.6cm 0cm}, clip, width=0.32\textwidth]{{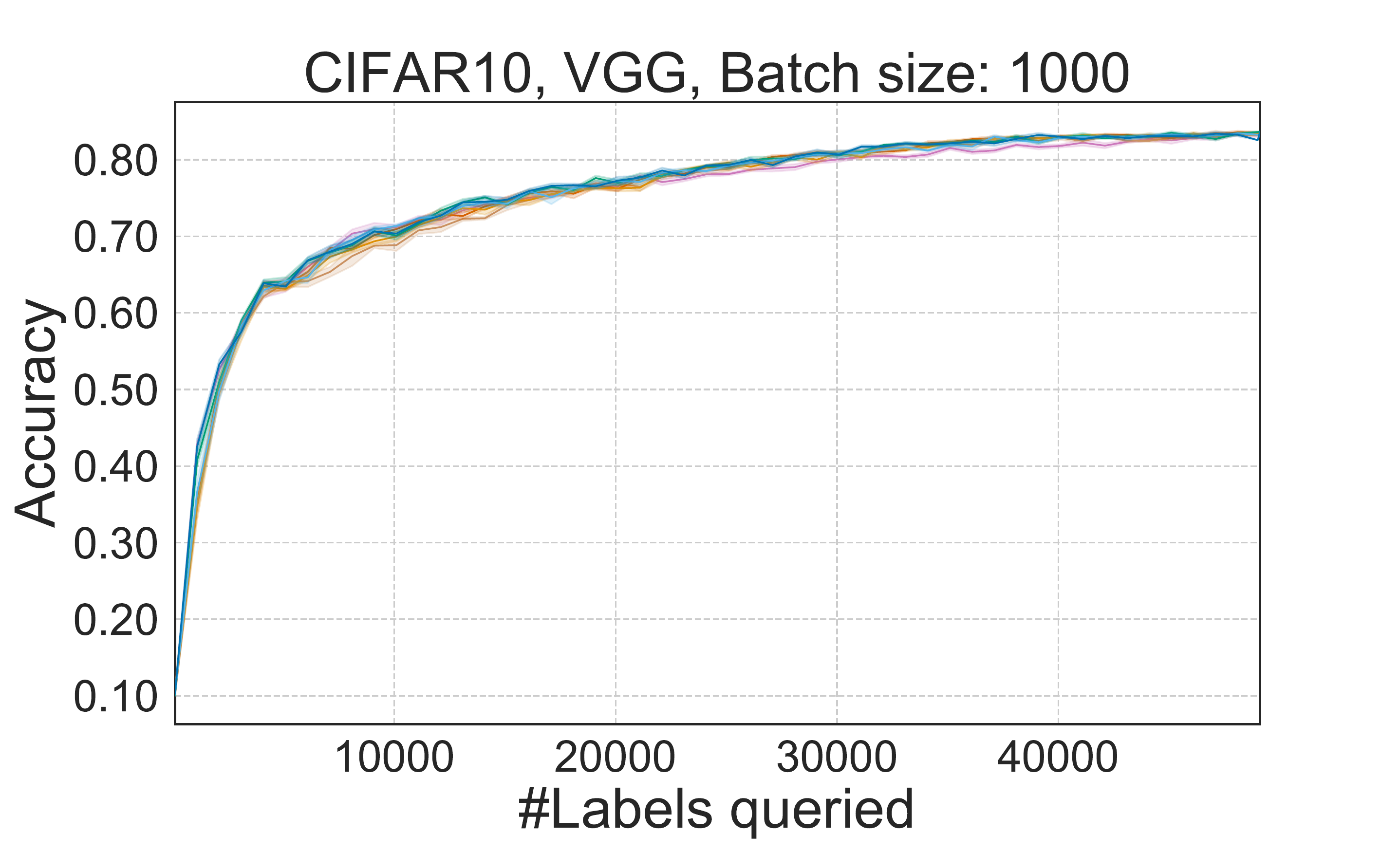}}
  \hfill
  \includegraphics[trim={1.5cm 0cm 1.6cm 0cm}, clip, width=0.32\textwidth]{{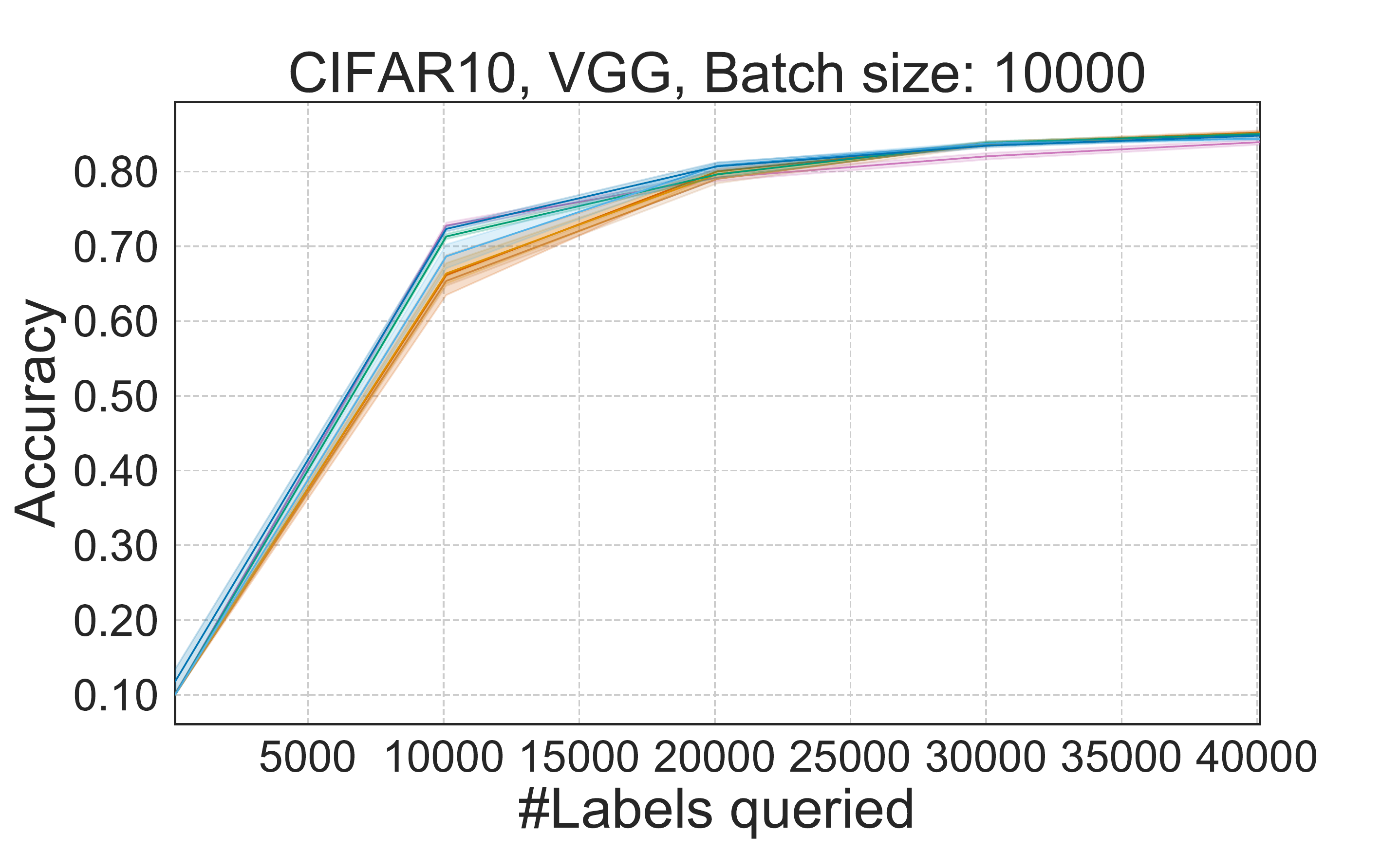}}
  \\
  \centering
  \begin{subfigure}[b]{\linewidth}
    \includegraphics[trim={0cm 0cm 0cm 0cm}, clip, width=\textwidth]{figs/legends/legend.pdf}
  \end{subfigure}

\caption{Full learning curves for CIFAR10 with MLP, ResNet and VGG.}
\label{fig:cifar10-lc-full}
\end{figure}

%, fig:155-lc, fig:156-lc, fig:184-lc, fig:svhn-lc, fig:mnist-lc,

\begin{figure}
  \centering
      \includegraphics[trim={0.4cm 0cm 27.9cm 0cm}, clip, width=0.012\textwidth]{figs/learning_curves/all_algs_Accuracy_Data=_SVHN__Model=_rn__nQuery=_100__TrainAug=_0___.pdf}
  \includegraphics[trim={1.5cm 0cm 1.6cm 0cm}, clip, width=0.32\textwidth]{{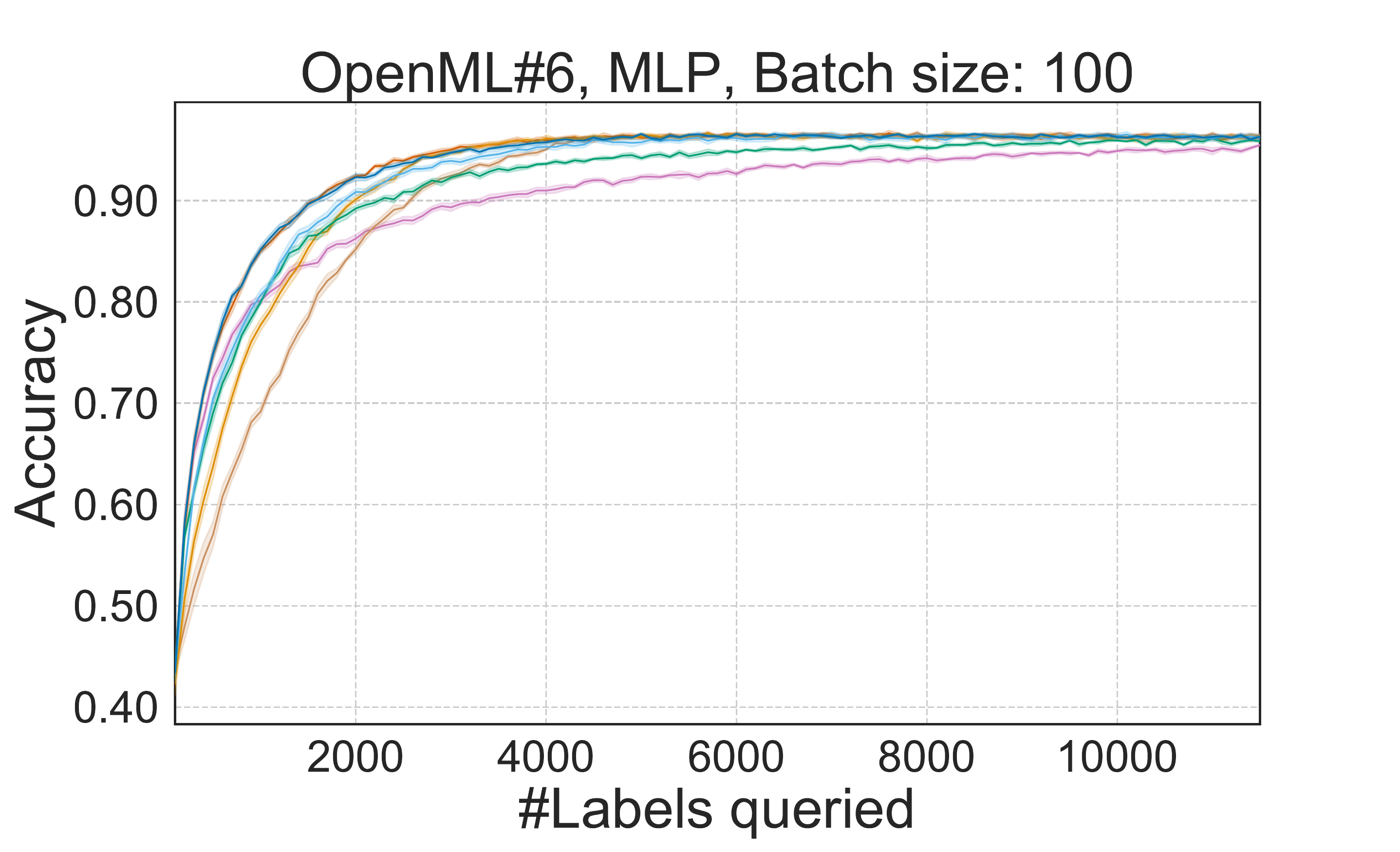}}
  \hfill
  \includegraphics[trim={1.5cm 0cm 1.6cm 0cm}, clip, width=0.32\textwidth]{{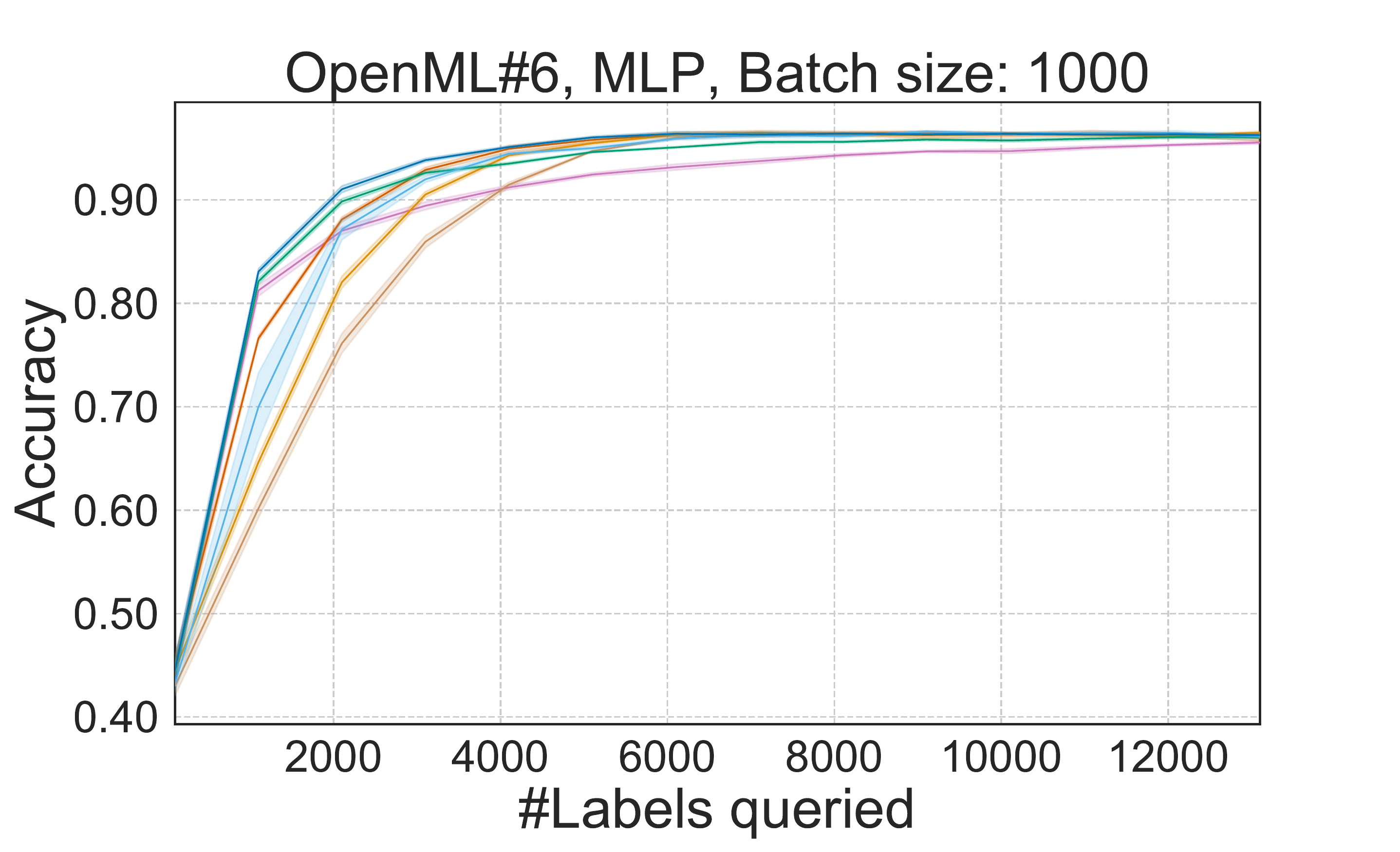}}
  \hfill
  \includegraphics[trim={1.5cm 0cm 1.6cm 0cm}, clip, width=0.32\textwidth]{{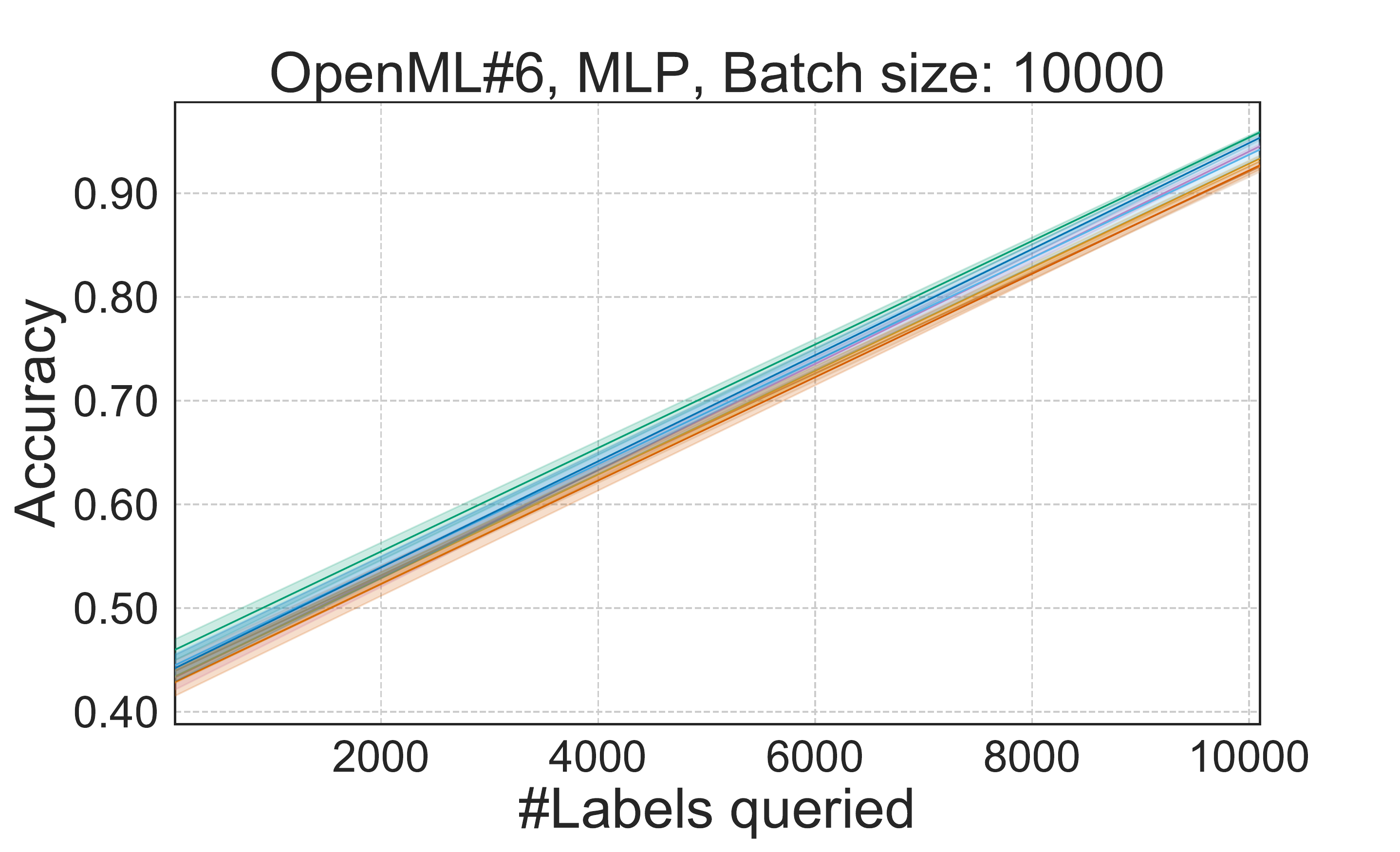}}
  \\
  \centering
  \begin{subfigure}[b]{\linewidth}
    \includegraphics[trim={0cm 0cm 0cm 0cm}, clip, width=\textwidth]{figs/legends/legend.pdf}
  \end{subfigure}
\caption{Zoomed-in learning curves for OpenML \#6 with MLP.}
\label{fig:6-lc}
\end{figure}

\begin{figure}
  \centering
      \includegraphics[trim={0.4cm 0cm 27.9cm 0cm}, clip, width=0.012\textwidth]{figs/learning_curves/all_algs_Accuracy_Data=_SVHN__Model=_rn__nQuery=_100__TrainAug=_0___.pdf}
  \includegraphics[trim={1.5cm 0cm 1.6cm 0cm}, clip, width=0.32\textwidth]{{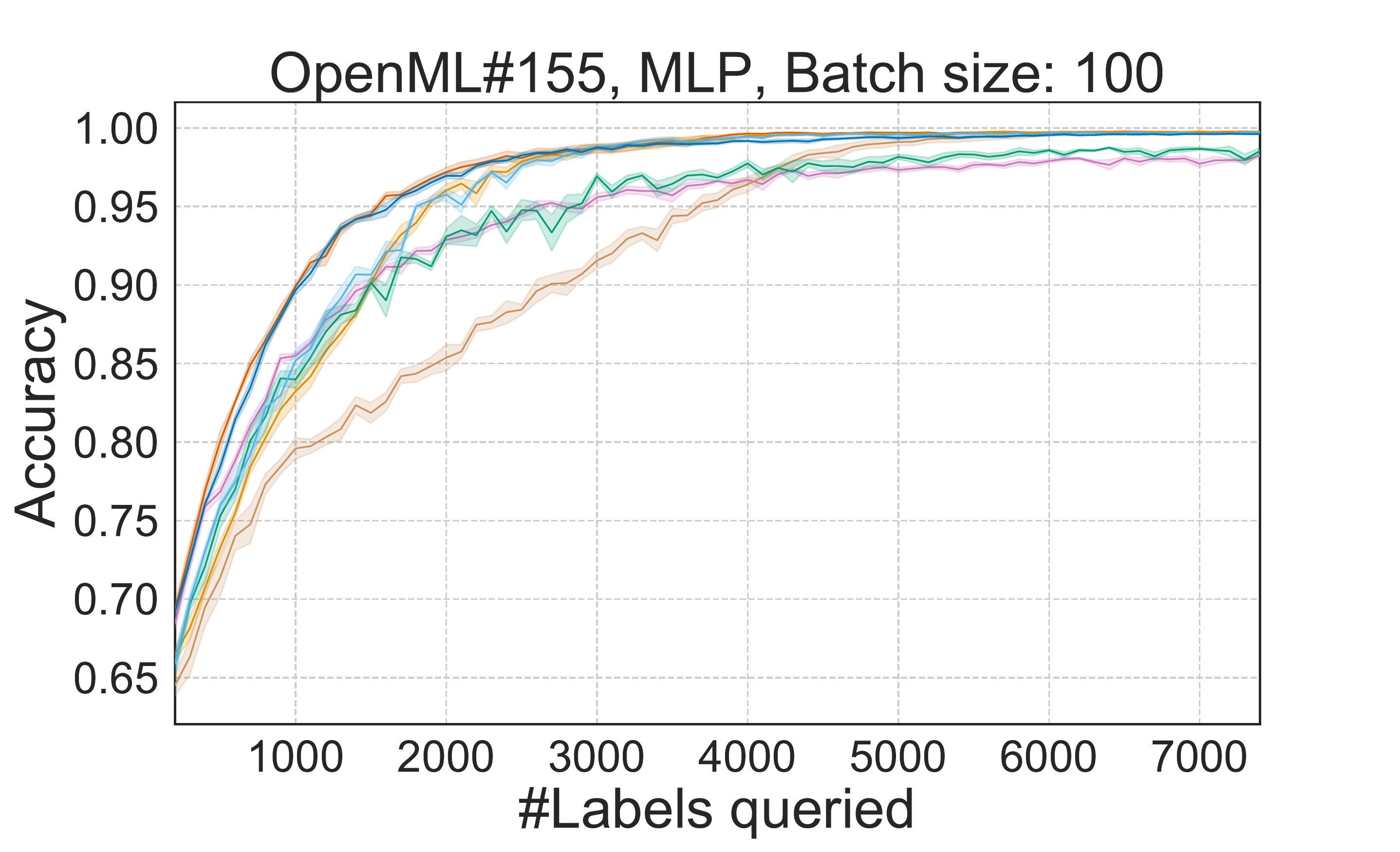}}
  \hfill
  \includegraphics[trim={1.5cm 0cm 1.6cm 0cm}, clip, width=0.32\textwidth]{{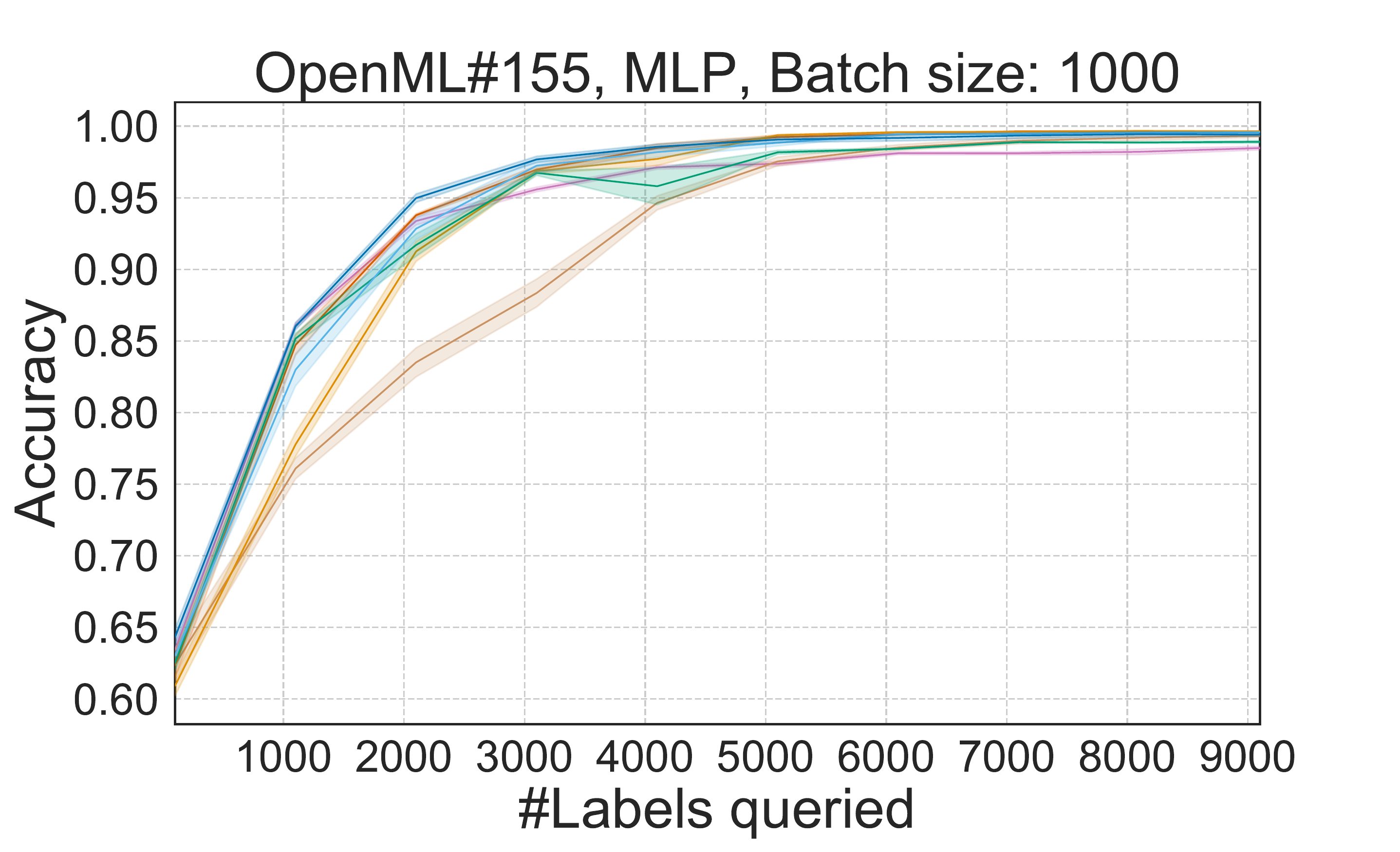}}
  \hfill
  \includegraphics[trim={1.5cm 0cm 1.6cm 0cm}, clip, width=0.32\textwidth]{{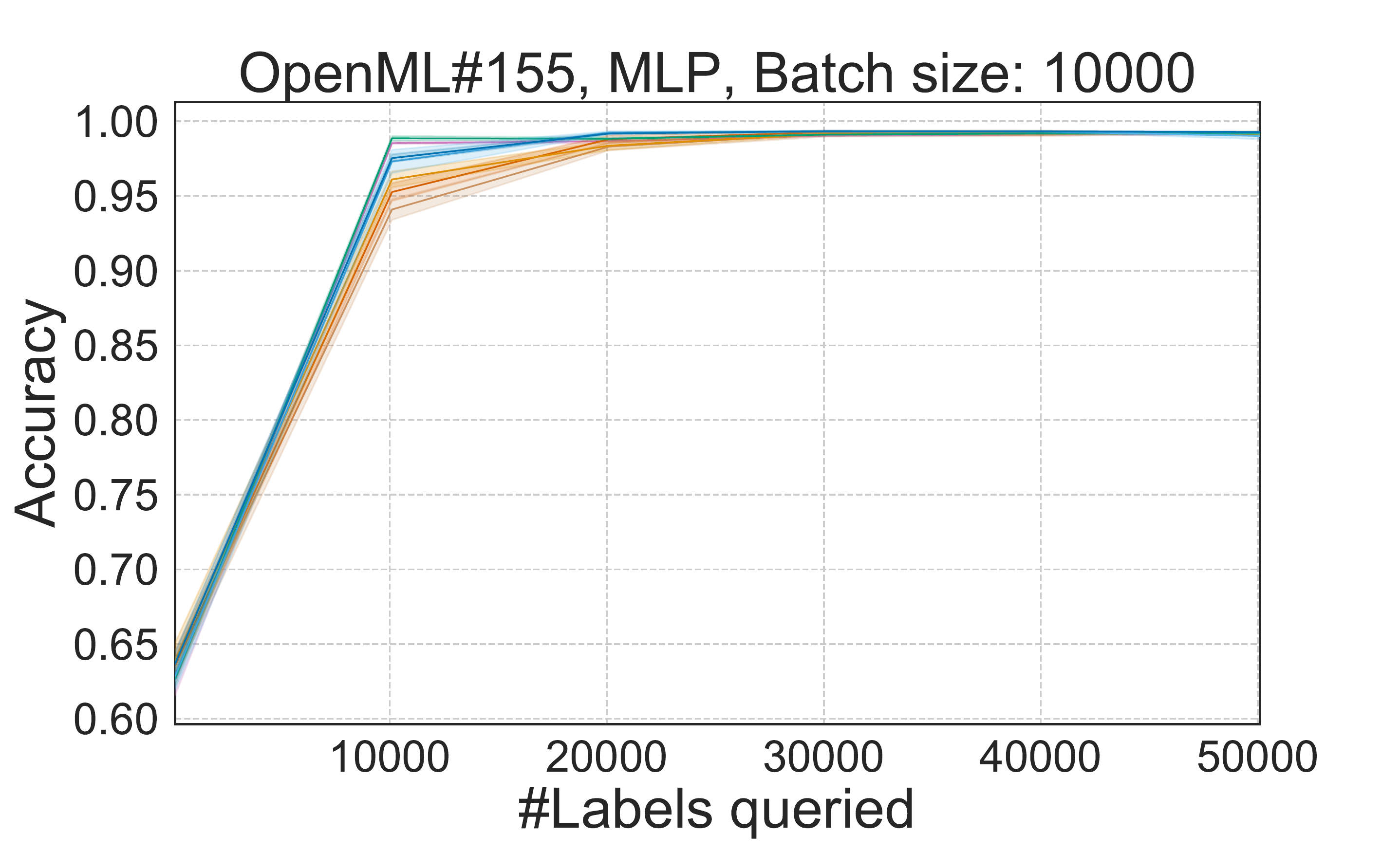}}
  \\
  \centering
  \begin{subfigure}[b]{\linewidth}
    \includegraphics[trim={0cm 0cm 0cm 0cm}, clip, width=\textwidth]{figs/legends/legend.pdf}
  \end{subfigure}
\caption{Zoomed-in learning curves for OpenML \#155 with MLP.}
\label{fig:155-lc}
\end{figure}

\begin{figure}
  \centering
      \includegraphics[trim={0.4cm 0cm 27.9cm 0cm}, clip, width=0.012\textwidth]{figs/learning_curves/all_algs_Accuracy_Data=_SVHN__Model=_rn__nQuery=_100__TrainAug=_0___.pdf}
  \includegraphics[trim={1.5cm 0cm 1.6cm 0cm}, clip, width=0.32\textwidth]{{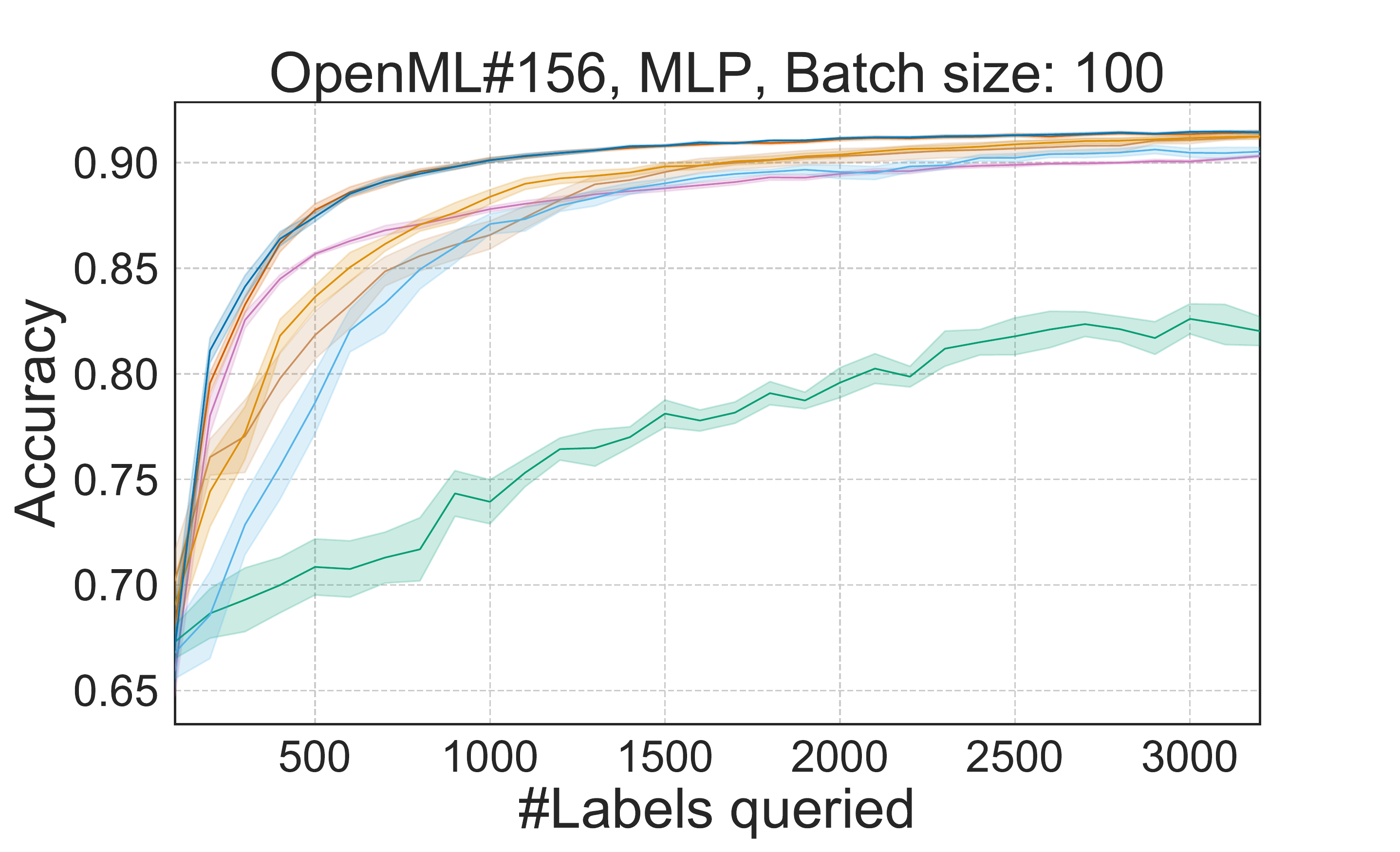}}
  \hfill
  \includegraphics[trim={1.5cm 0cm 1.6cm 0cm}, clip, width=0.32\textwidth]{{figs/learning_curves/all_algs_Accuracy_Data=_156__Model=_mlp__nQuery=_1000__TrainAug=_0___.pdf}}
  \hfill
  \includegraphics[trim={1.5cm 0cm 1.6cm 0cm}, clip, width=0.32\textwidth]{{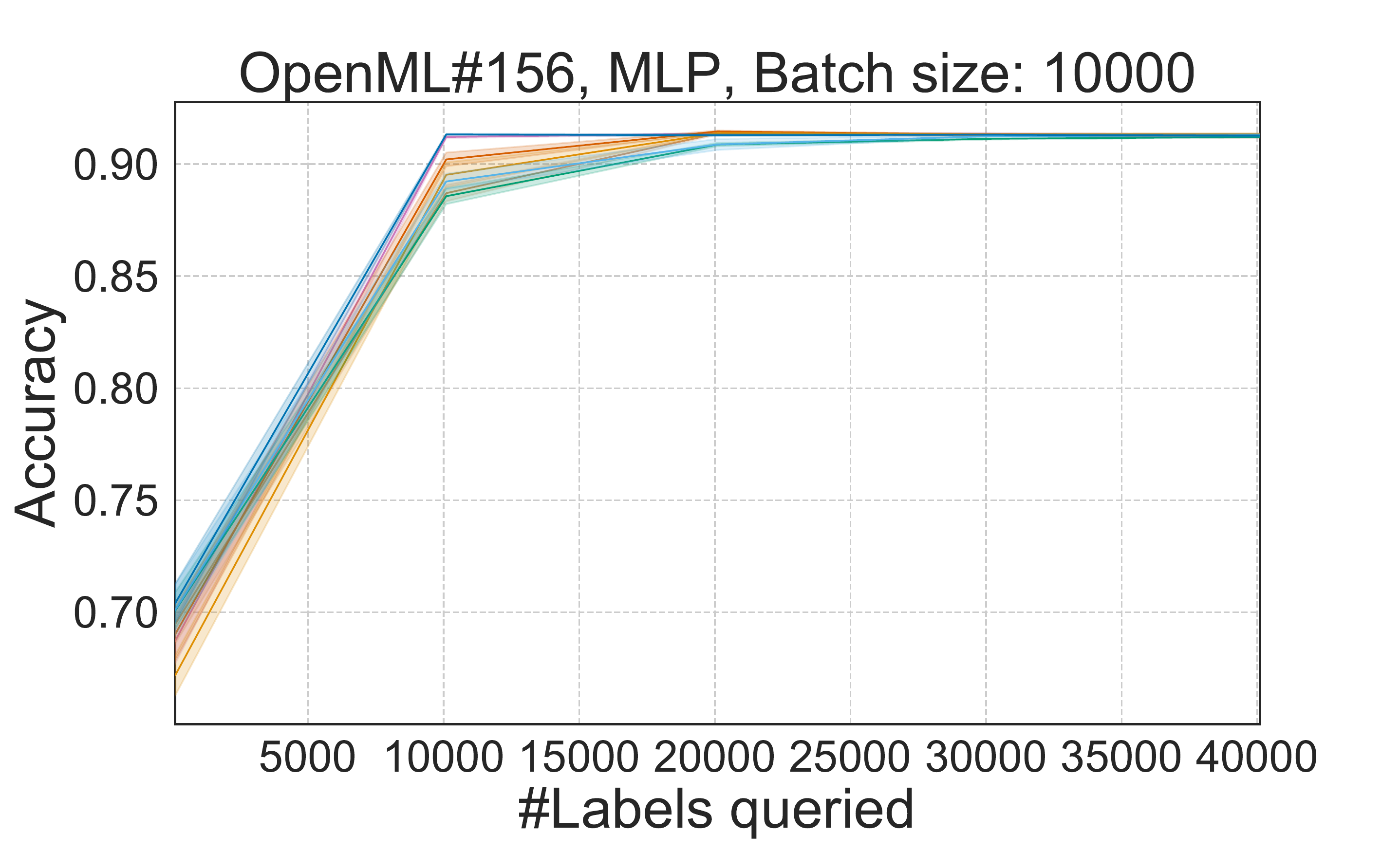}}
  \\
  \centering
  \begin{subfigure}[b]{\linewidth}
    \includegraphics[trim={0cm 0cm 0cm 0cm}, clip, width=\textwidth]{figs/legends/legend.pdf}
  \end{subfigure}
\caption{Zoomed-in learning curves for OpenML \#156 with MLP.}
\label{fig:156-lc}
\end{figure}

\begin{figure}
  \centering
      \includegraphics[trim={0.4cm 0cm 27.9cm 0cm}, clip, width=0.012\textwidth]{figs/learning_curves/all_algs_Accuracy_Data=_SVHN__Model=_rn__nQuery=_100__TrainAug=_0___.pdf}
  \includegraphics[trim={1.5cm 0cm 1.6cm 0cm}, clip, width=0.32\textwidth]{{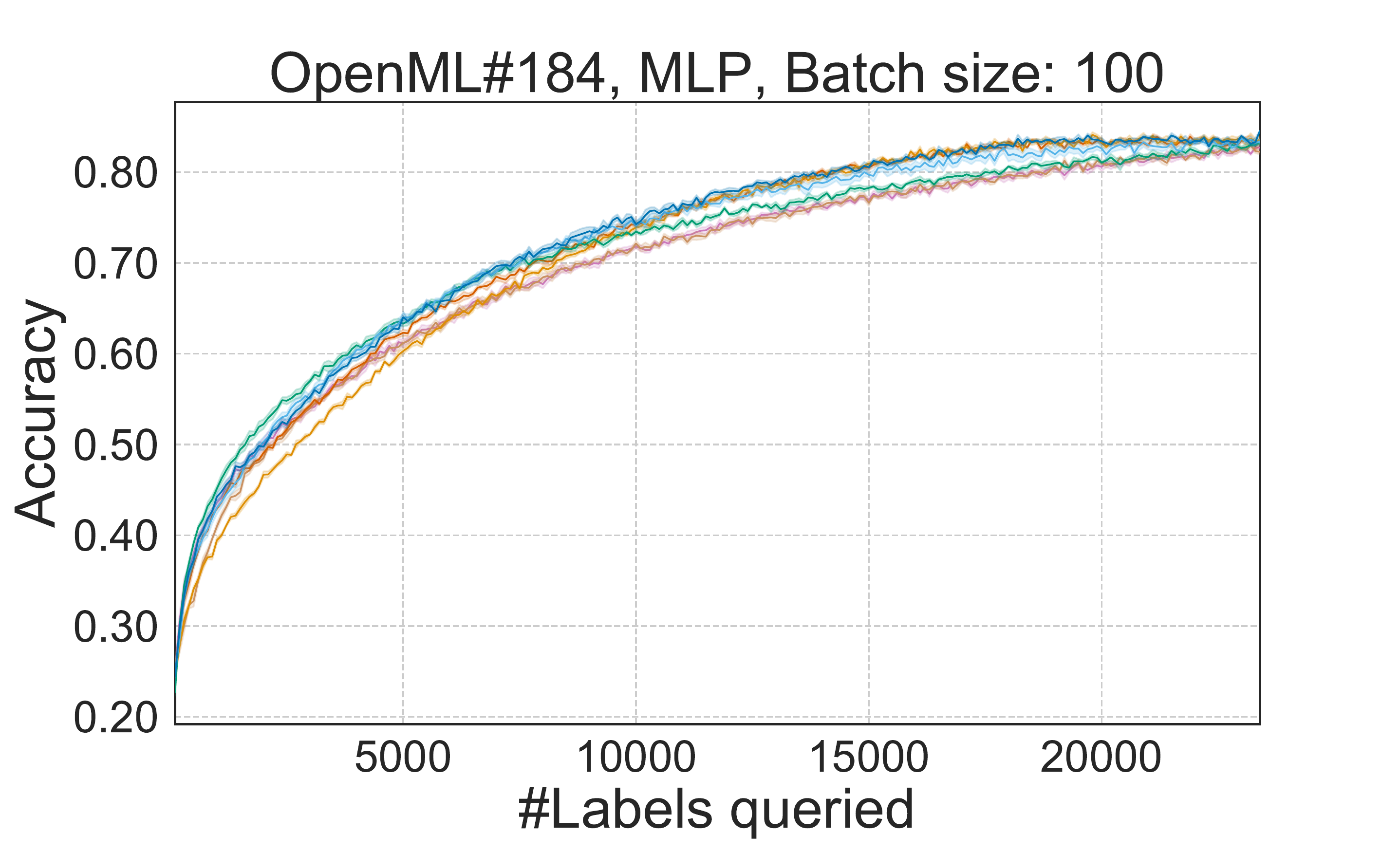}}
  \hfill
  \includegraphics[trim={1.5cm 0cm 1.6cm 0cm}, clip, width=0.32\textwidth]{{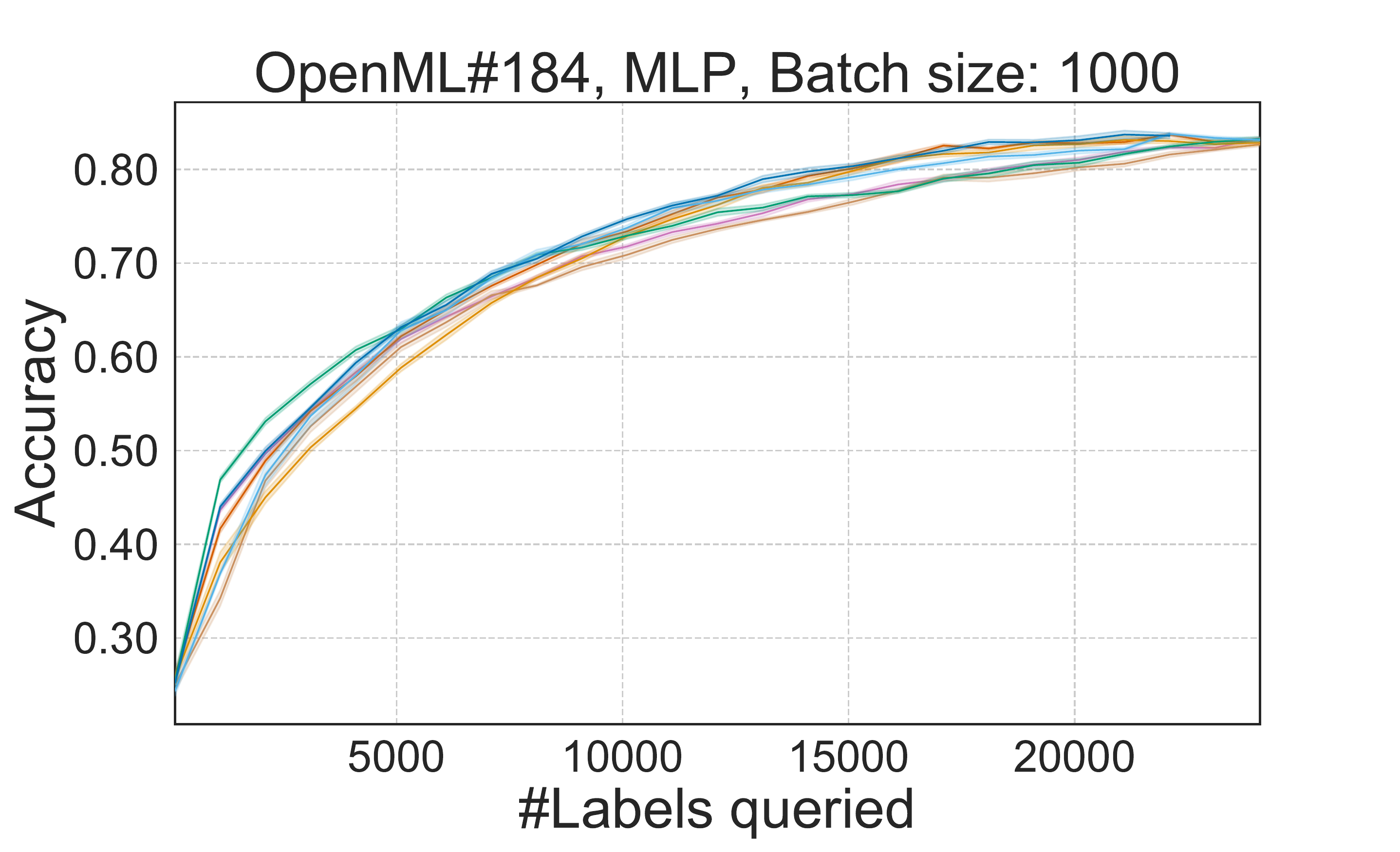}}
  \hfill
  \includegraphics[trim={1.5cm 0cm 1.6cm 0cm}, clip, width=0.32\textwidth]{{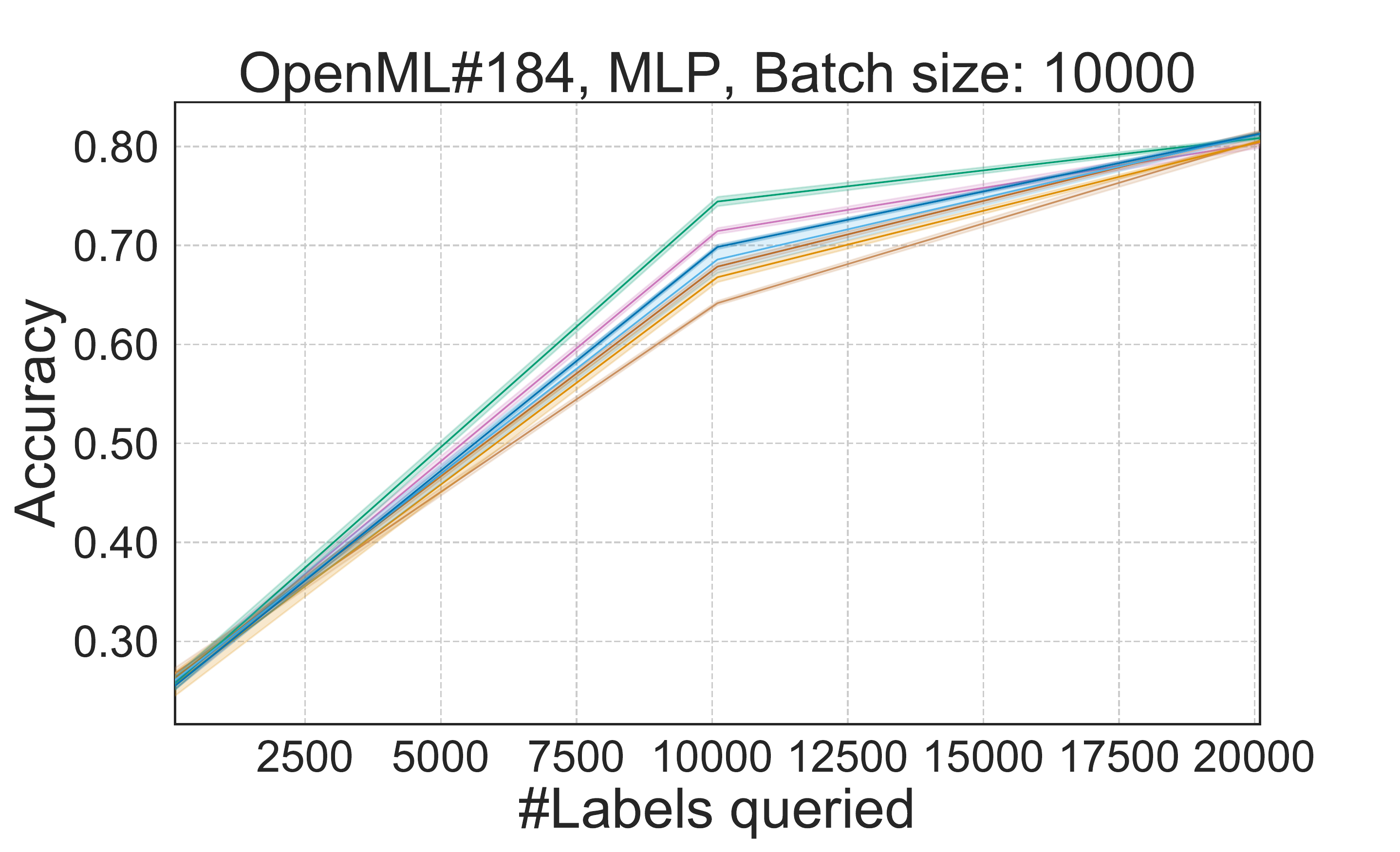}}
  \\
  \centering
  \begin{subfigure}[b]{\linewidth}
    \includegraphics[trim={0cm 0cm 0cm 0cm}, clip, width=\textwidth]{figs/legends/legend.pdf}
  \end{subfigure}
\caption{Zoomed-in learning curves for OpenML \#184 with MLP.}
\label{fig:184-lc}
\end{figure}

\begin{figure}
  \centering
      \includegraphics[trim={0.4cm 0cm 27.9cm 0cm}, clip, width=0.012\textwidth]{figs/learning_curves/all_algs_Accuracy_Data=_SVHN__Model=_rn__nQuery=_100__TrainAug=_0___.pdf}
  \includegraphics[trim={1.5cm 0cm 1.6cm 0cm}, clip, width=0.32\textwidth]{{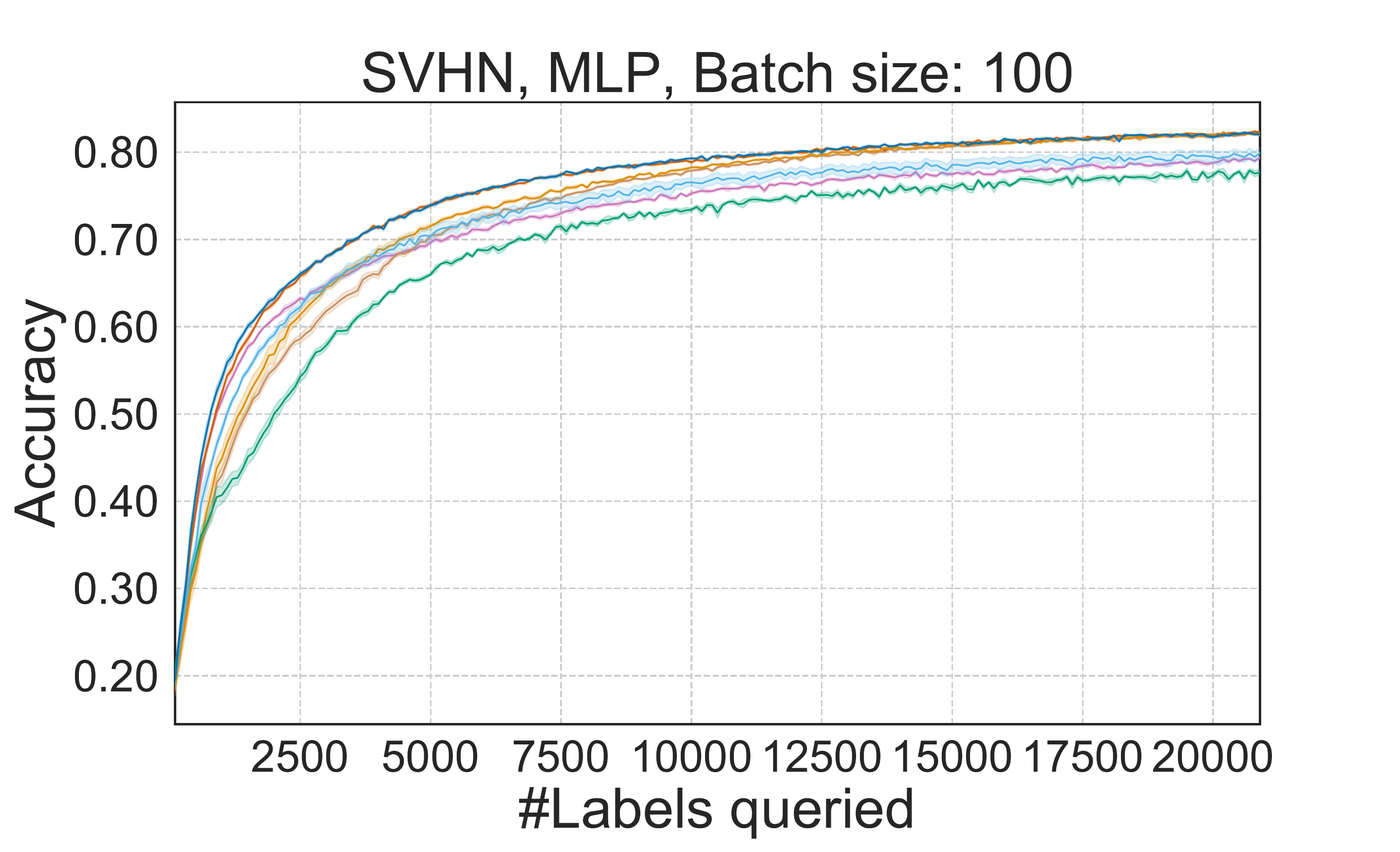}}
  \hfill
  \includegraphics[trim={1.5cm 0cm 1.6cm 0cm}, clip, width=0.32\textwidth]{{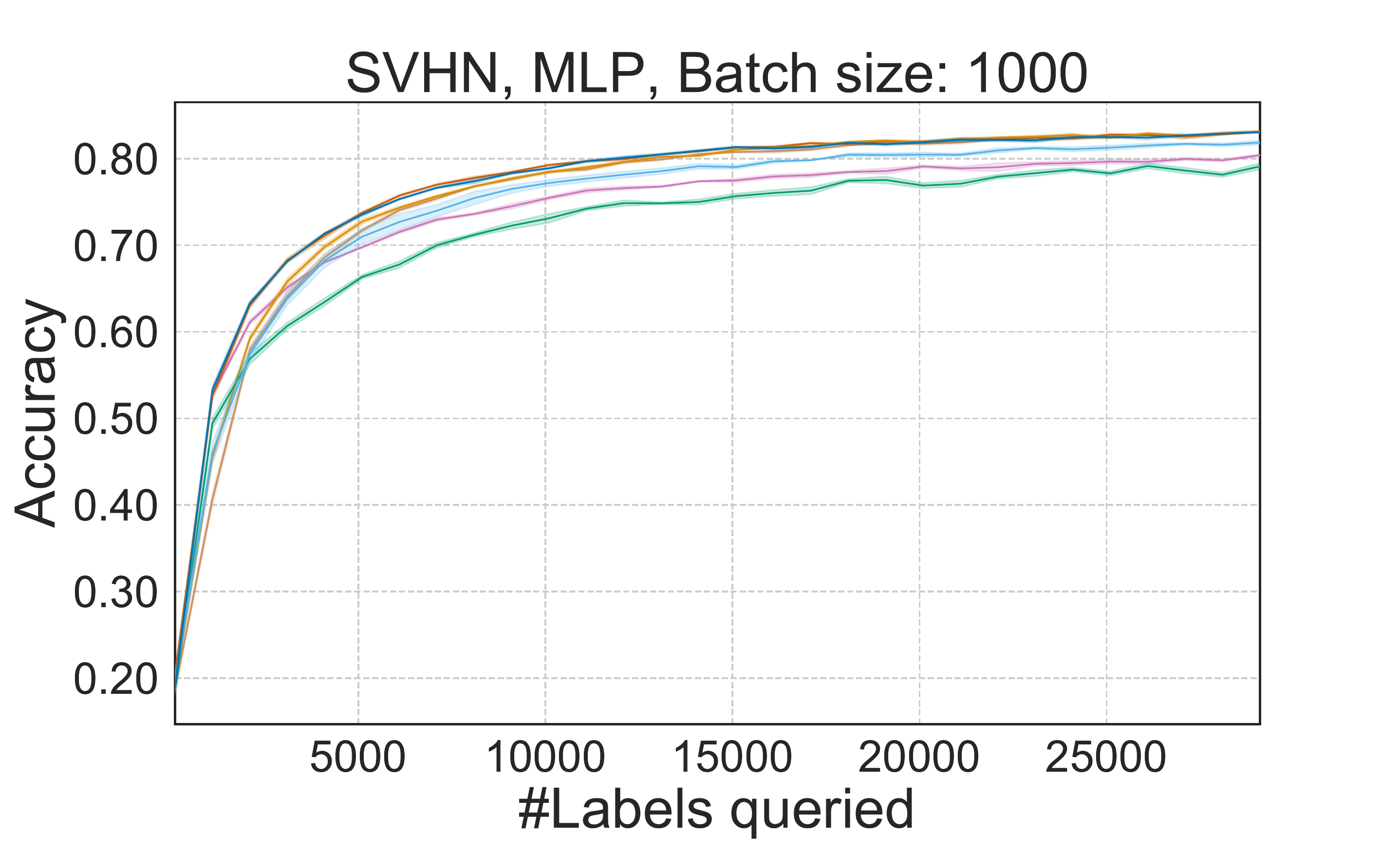}}
  \hfill
  \includegraphics[trim={1.5cm 0cm 1.6cm 0cm}, clip, width=0.32\textwidth]{{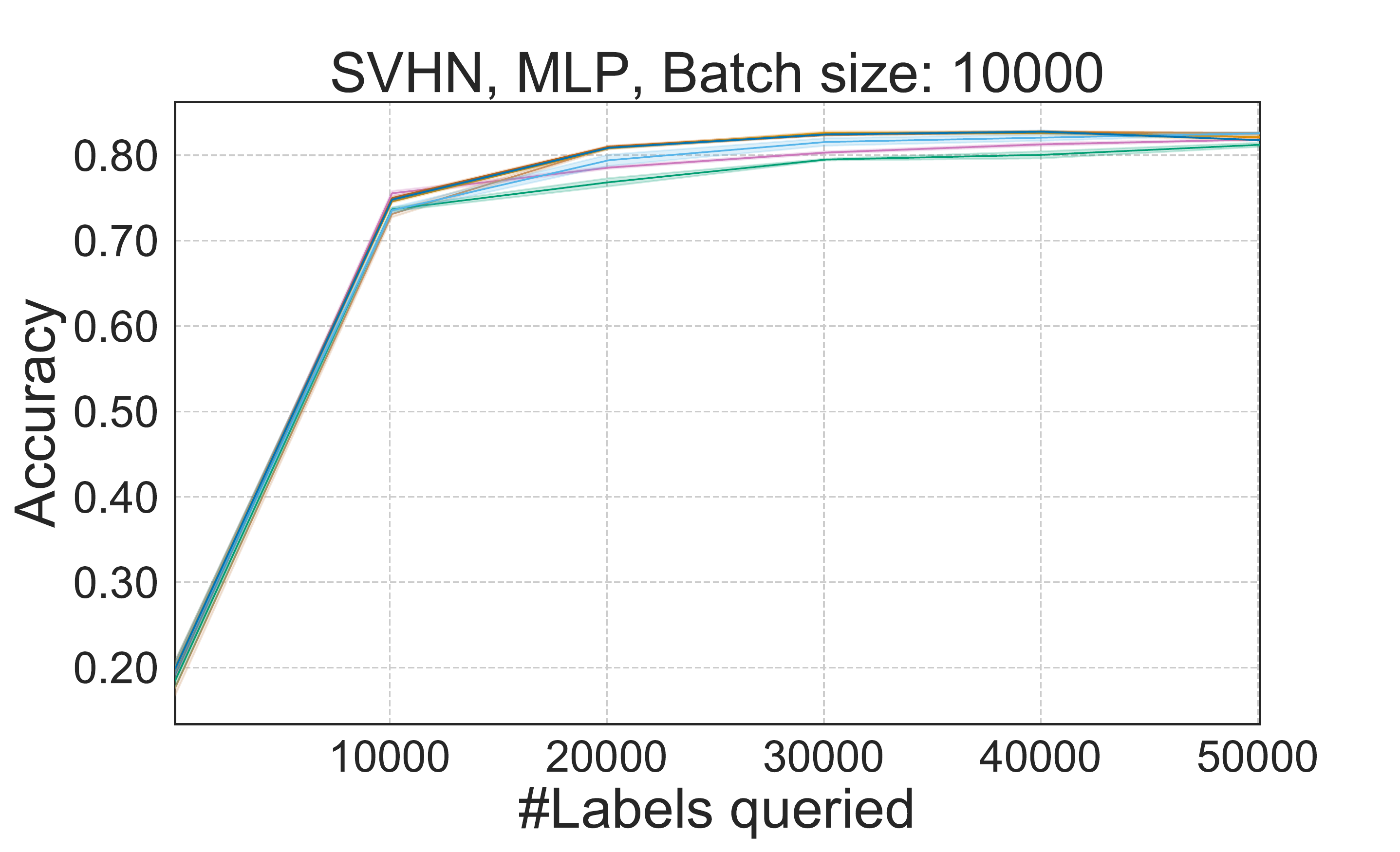}}

    \includegraphics[trim={0.4cm 0cm 27.9cm 0cm}, clip, width=0.012\textwidth]{figs/learning_curves/all_algs_Accuracy_Data=_SVHN__Model=_rn__nQuery=_100__TrainAug=_0___.pdf}
  \includegraphics[trim={1.5cm 0cm 1.6cm 0cm}, clip, width=0.32\textwidth]{{figs/learning_curves/all_algs_Accuracy_Data=_SVHN__Model=_rn__nQuery=_100__TrainAug=_0___.pdf}}
  \hfill
  \includegraphics[trim={1.5cm 0cm 1.6cm 0cm}, clip, width=0.32\textwidth]{{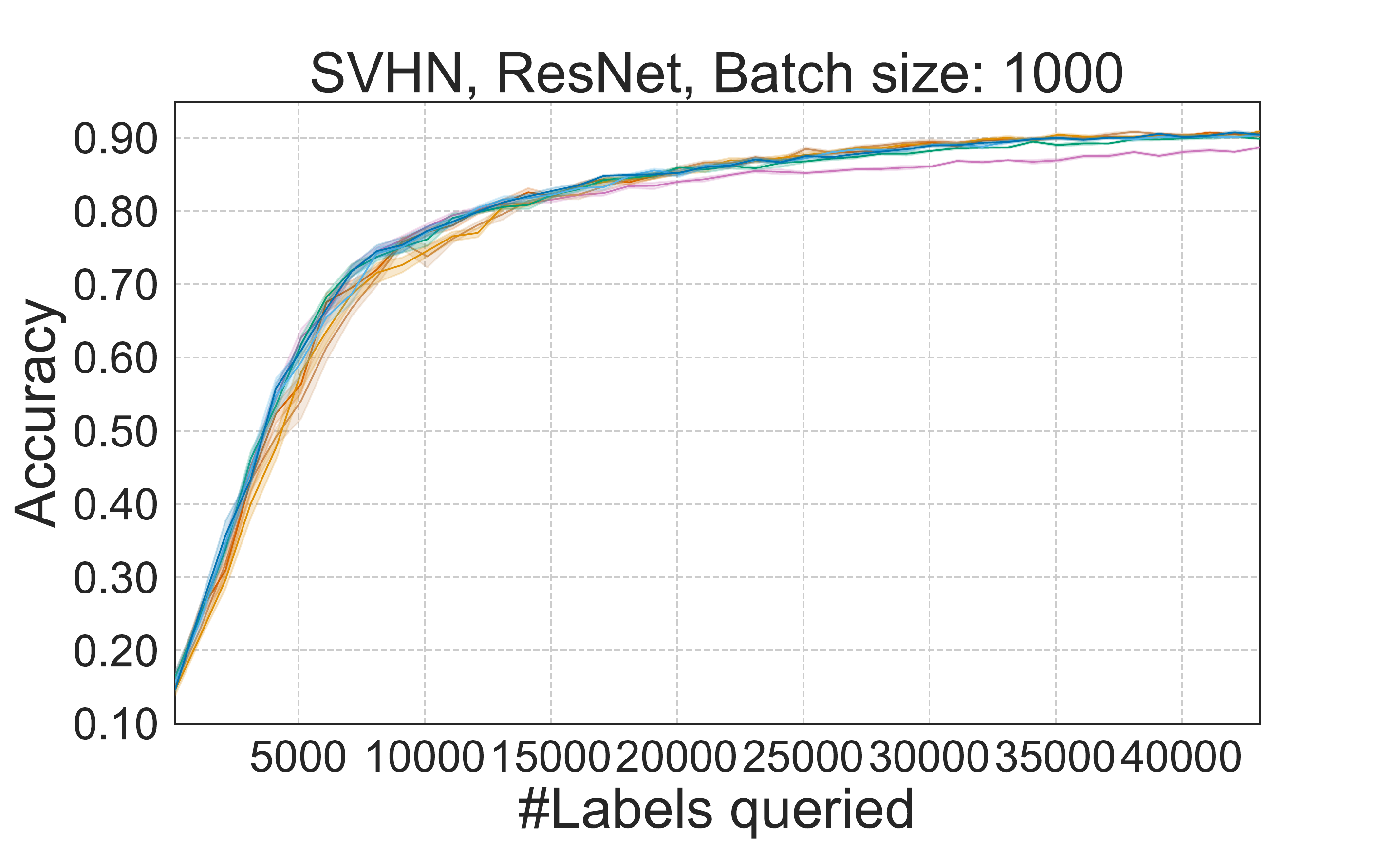}}
  \hfill
  \includegraphics[trim={1.5cm 0cm 1.6cm 0cm}, clip, width=0.32\textwidth]{{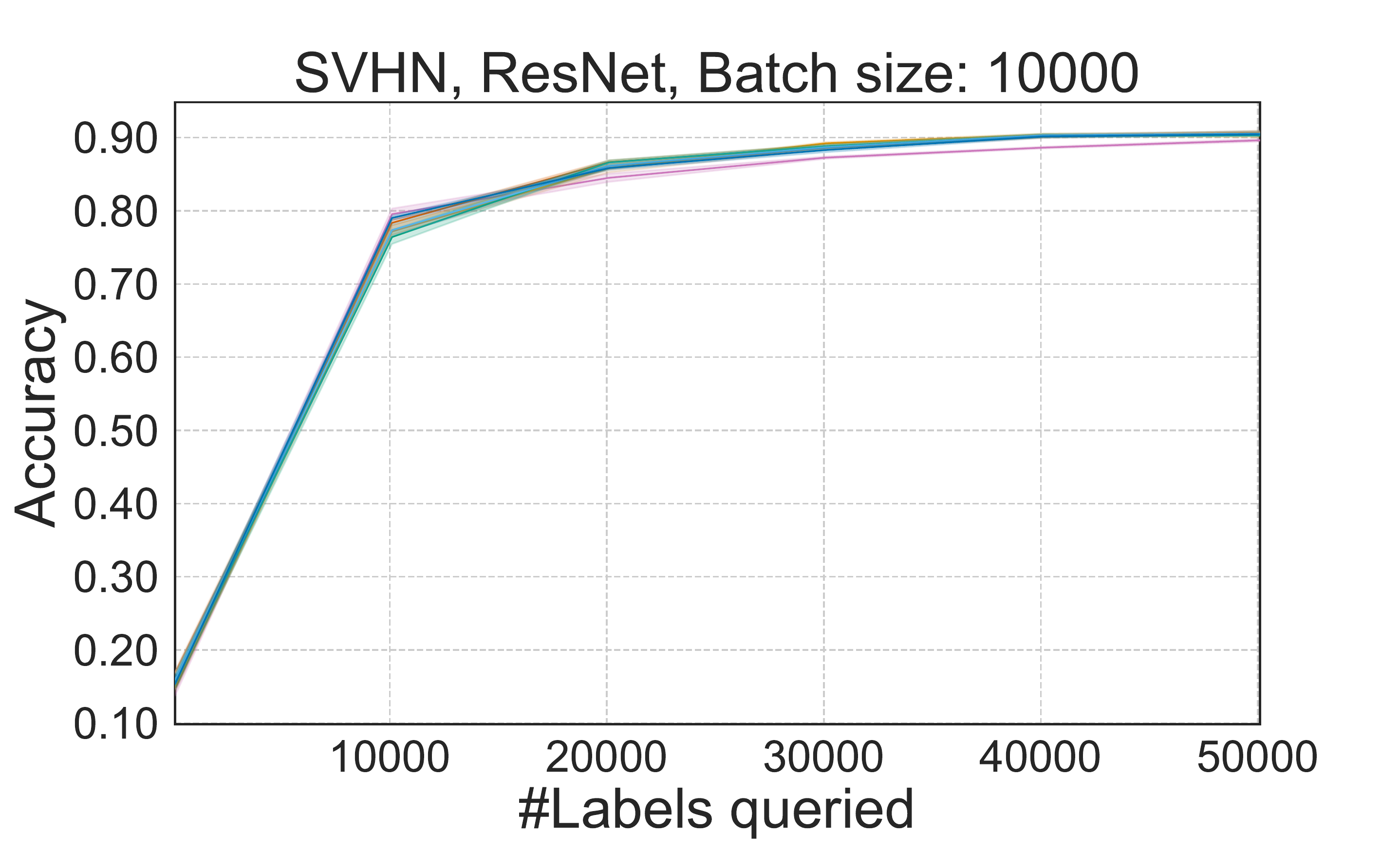}}

    \includegraphics[trim={0.4cm 0cm 27.9cm 0cm}, clip, width=0.012\textwidth]{figs/learning_curves/all_algs_Accuracy_Data=_SVHN__Model=_rn__nQuery=_100__TrainAug=_0___.pdf}
  \includegraphics[trim={1.5cm 0cm 1.6cm 0cm}, clip, width=0.32\textwidth]{{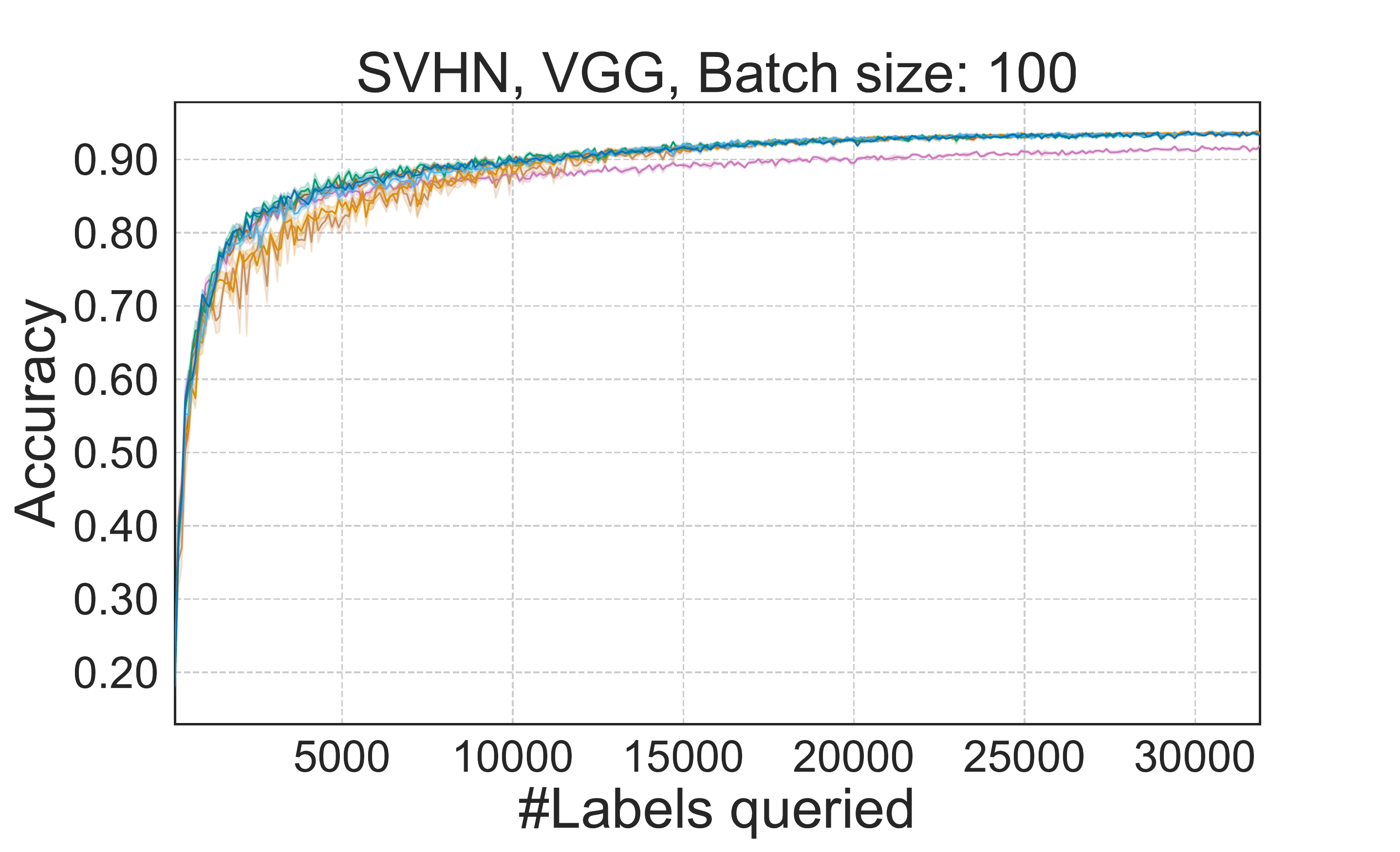}}
  \hfill
  \includegraphics[trim={1.5cm 0cm 1.6cm 0cm}, clip, width=0.32\textwidth]{{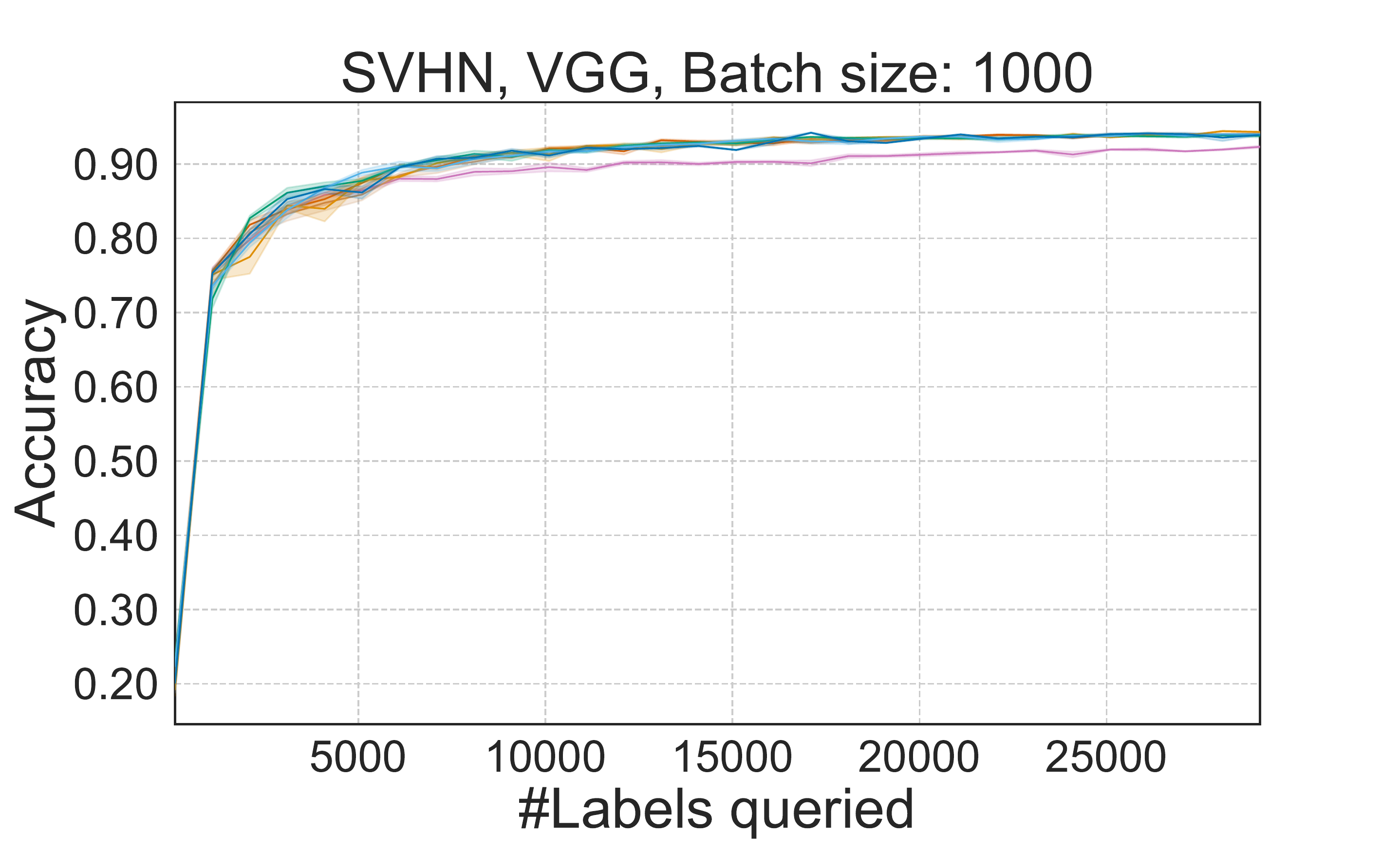}}
  \hfill
  \includegraphics[trim={1.5cm 0cm 1.6cm 0cm}, clip, width=0.32\textwidth]{{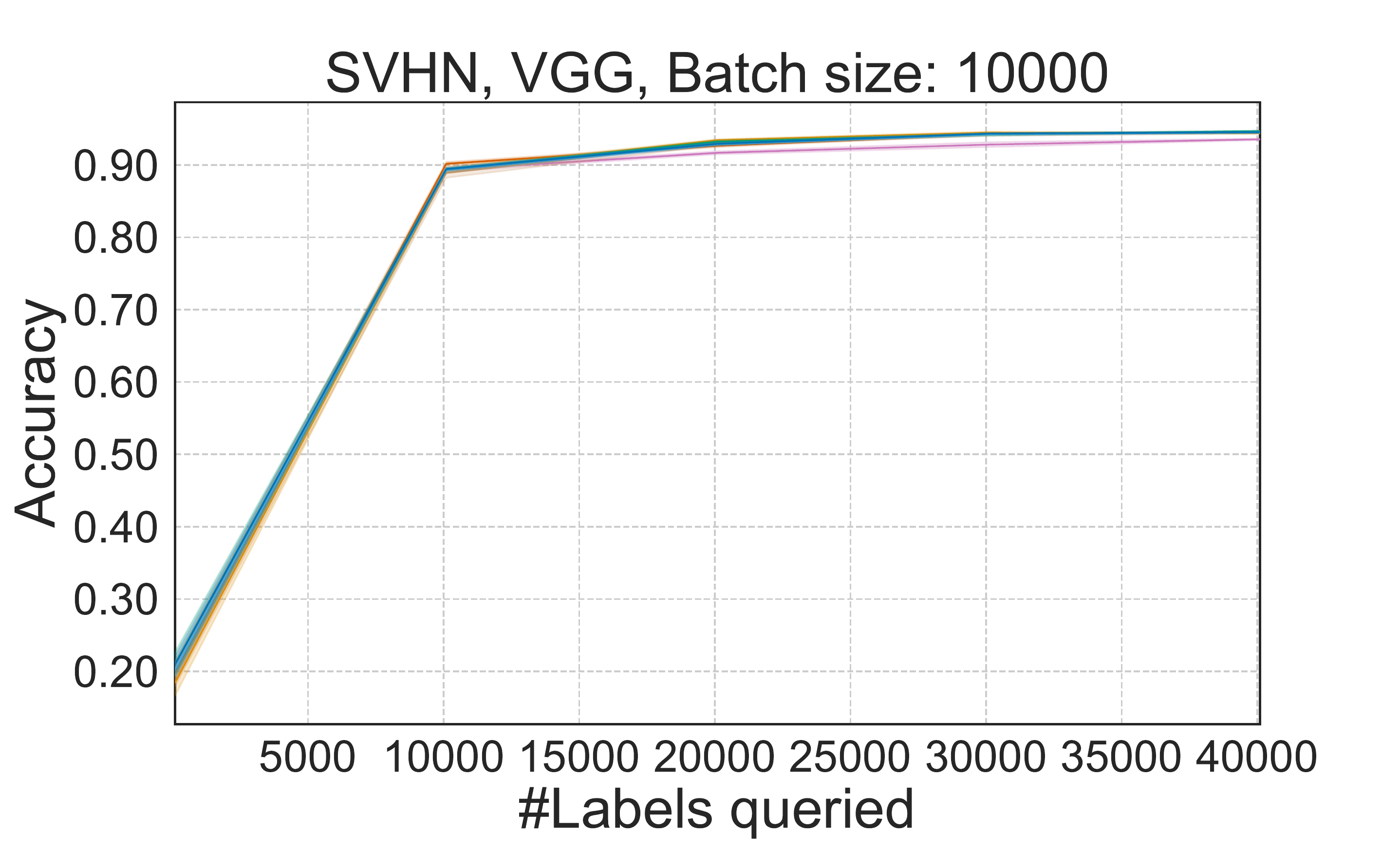}}
  \\
  \centering
  \begin{subfigure}[b]{\linewidth}
    \includegraphics[trim={0cm 0cm 0cm 0cm}, clip, width=\textwidth]{figs/legends/legend.pdf}
  \end{subfigure}

\caption{Zoomed-in learning curves for SVHN with MLP, ResNet and VGG.}
\label{fig:svhn-lc}
\end{figure}

\begin{figure}
  \centering
      \includegraphics[trim={0.4cm 0cm 27.9cm 0cm}, clip, width=0.012\textwidth]{figs/learning_curves/all_algs_Accuracy_Data=_SVHN__Model=_rn__nQuery=_100__TrainAug=_0___.pdf}
  \includegraphics[trim={1.5cm 0cm 1.6cm 0cm}, clip, width=0.32\textwidth]{{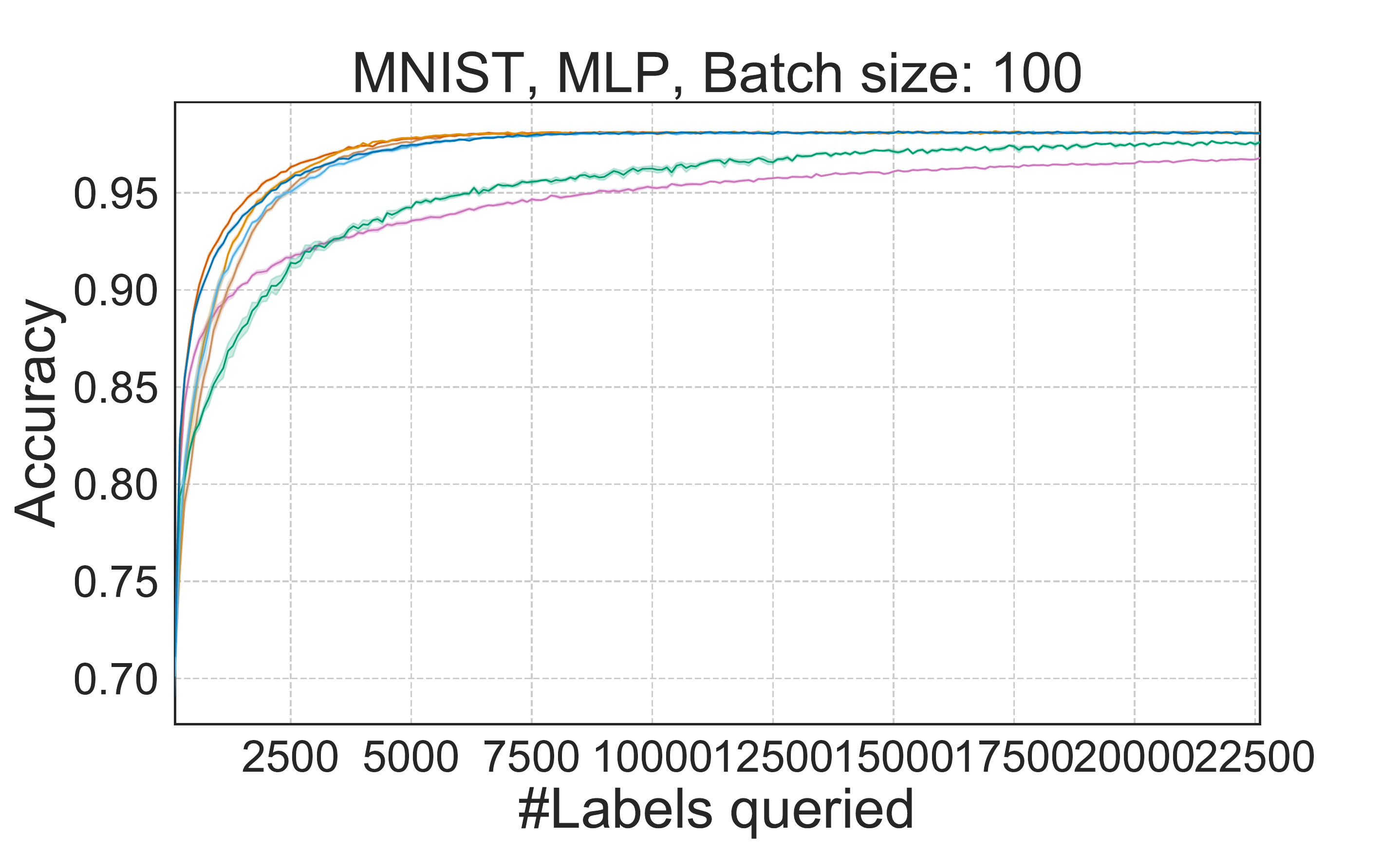}}
  \hfill
  \includegraphics[trim={1.5cm 0cm 1.6cm 0cm}, clip, width=0.32\textwidth]{{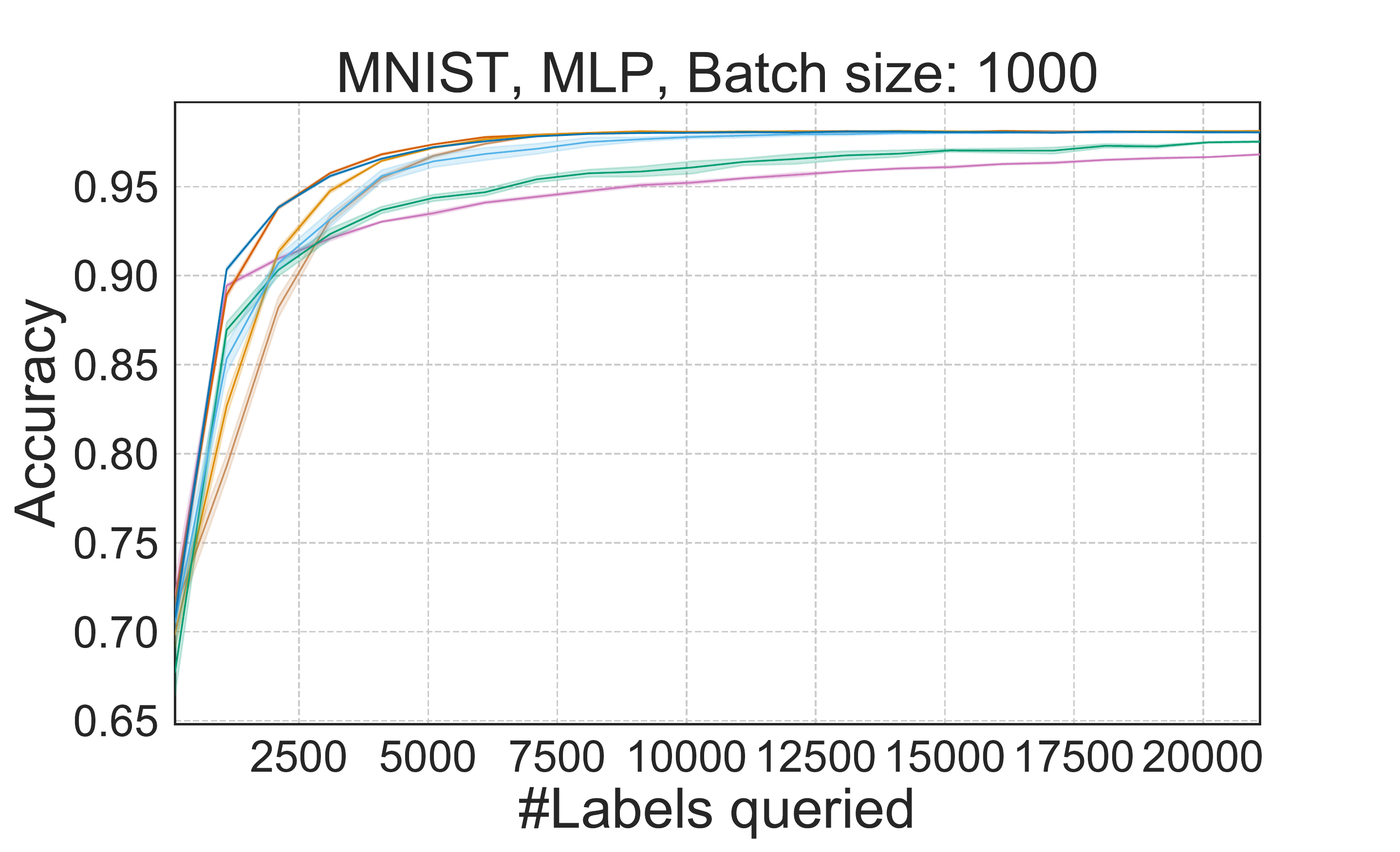}}
  \hfill
  \includegraphics[trim={1.5cm 0cm 1.6cm 0cm}, clip, width=0.32\textwidth]{{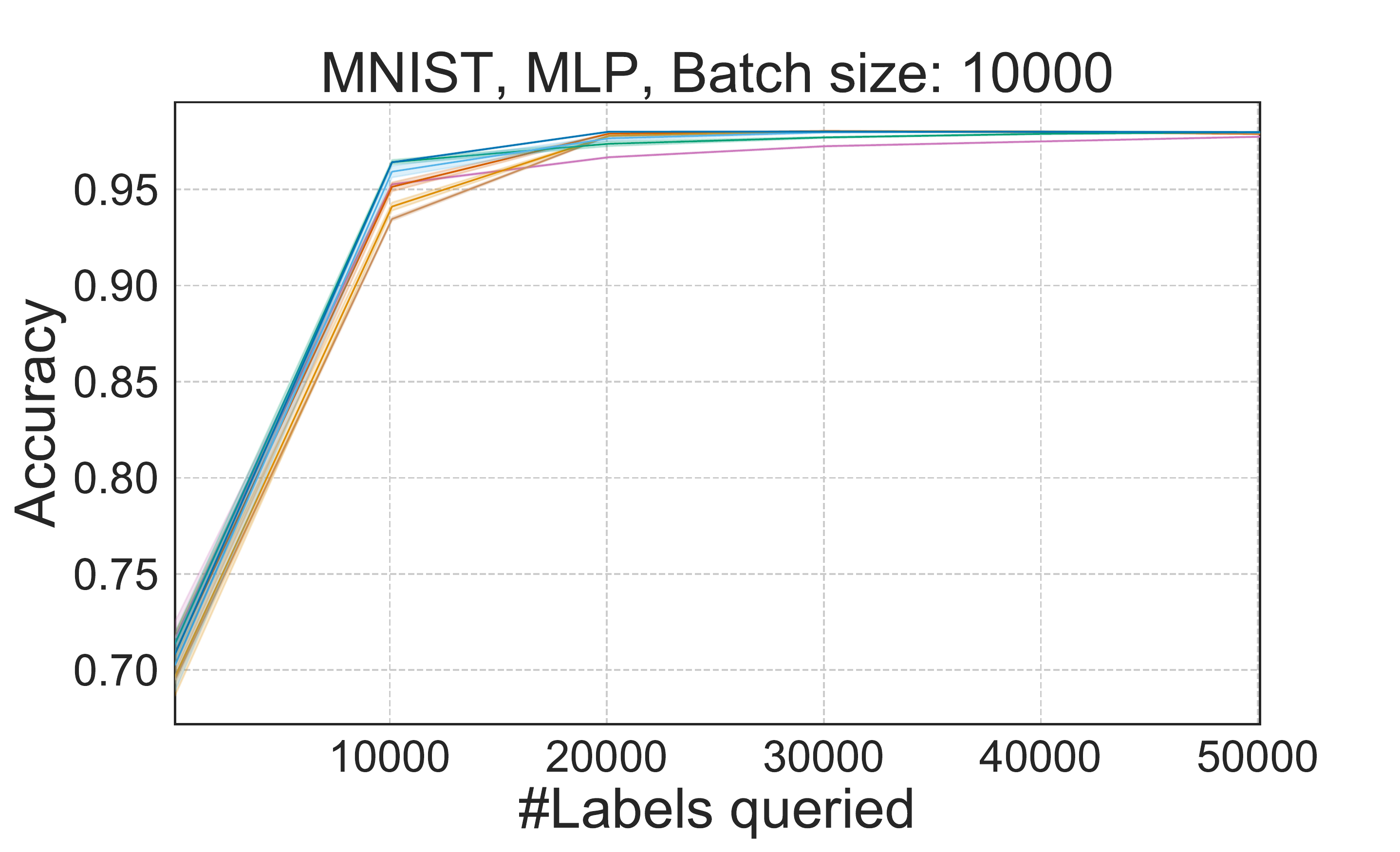}}
  \\
  \centering
  \begin{subfigure}[b]{\linewidth}
    \includegraphics[trim={0cm 0cm 0cm 0cm}, clip, width=\textwidth]{figs/legends/legend.pdf}
  \end{subfigure}

\caption{Zoomed-in learning curves for MNIST with MLP.}
\label{fig:mnist-lc}
\end{figure}

\begin{figure}
  \centering
      \includegraphics[trim={0.4cm 0cm 27.9cm 0cm}, clip, width=0.012\textwidth]{figs/learning_curves/all_algs_Accuracy_Data=_SVHN__Model=_rn__nQuery=_100__TrainAug=_0___.pdf}
  \includegraphics[trim={1.5cm 0cm 1.6cm 0cm}, clip, width=0.32\textwidth]{{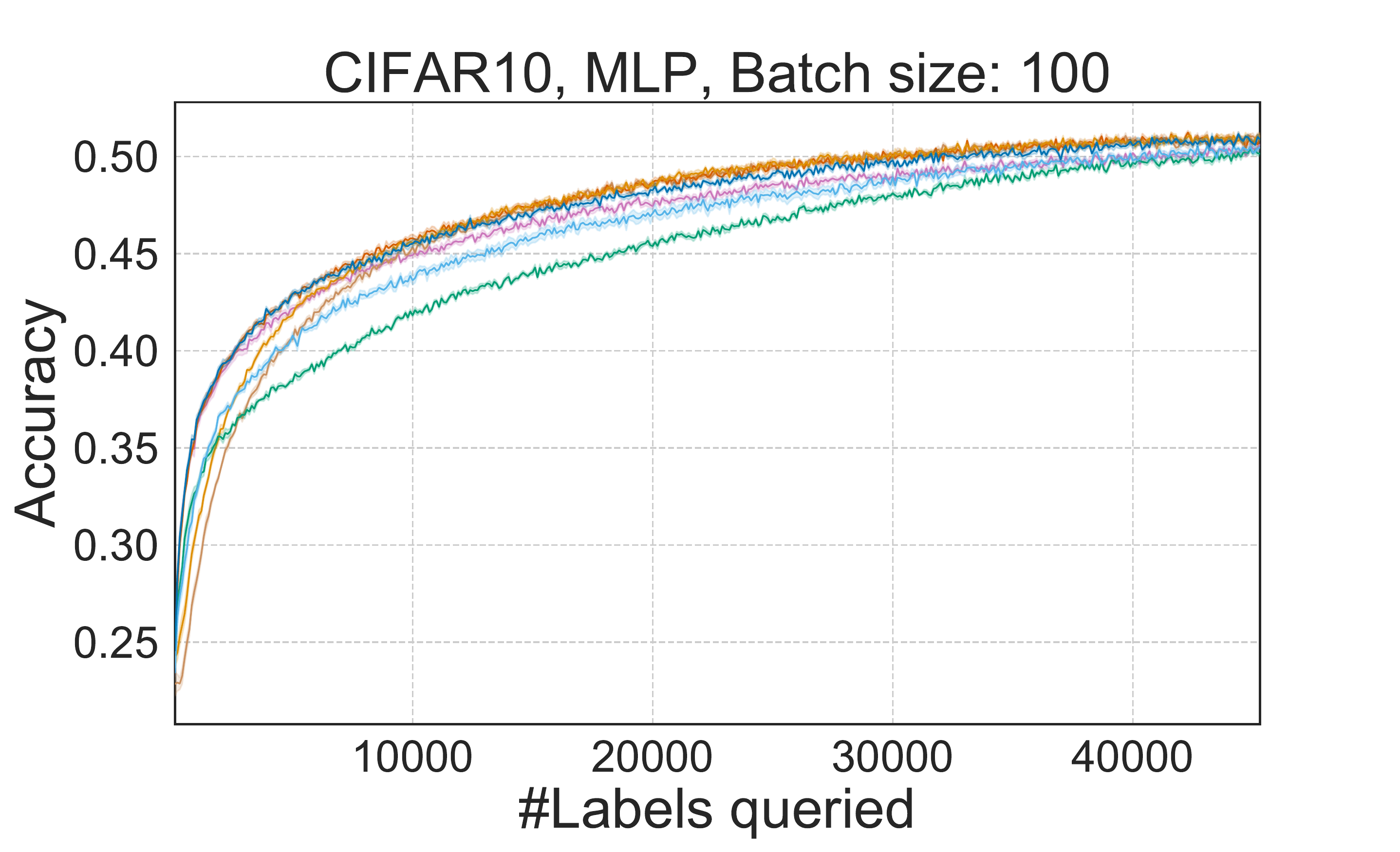}}
  \hfill
  \includegraphics[trim={1.5cm 0cm 1.6cm 0cm}, clip, width=0.32\textwidth]{{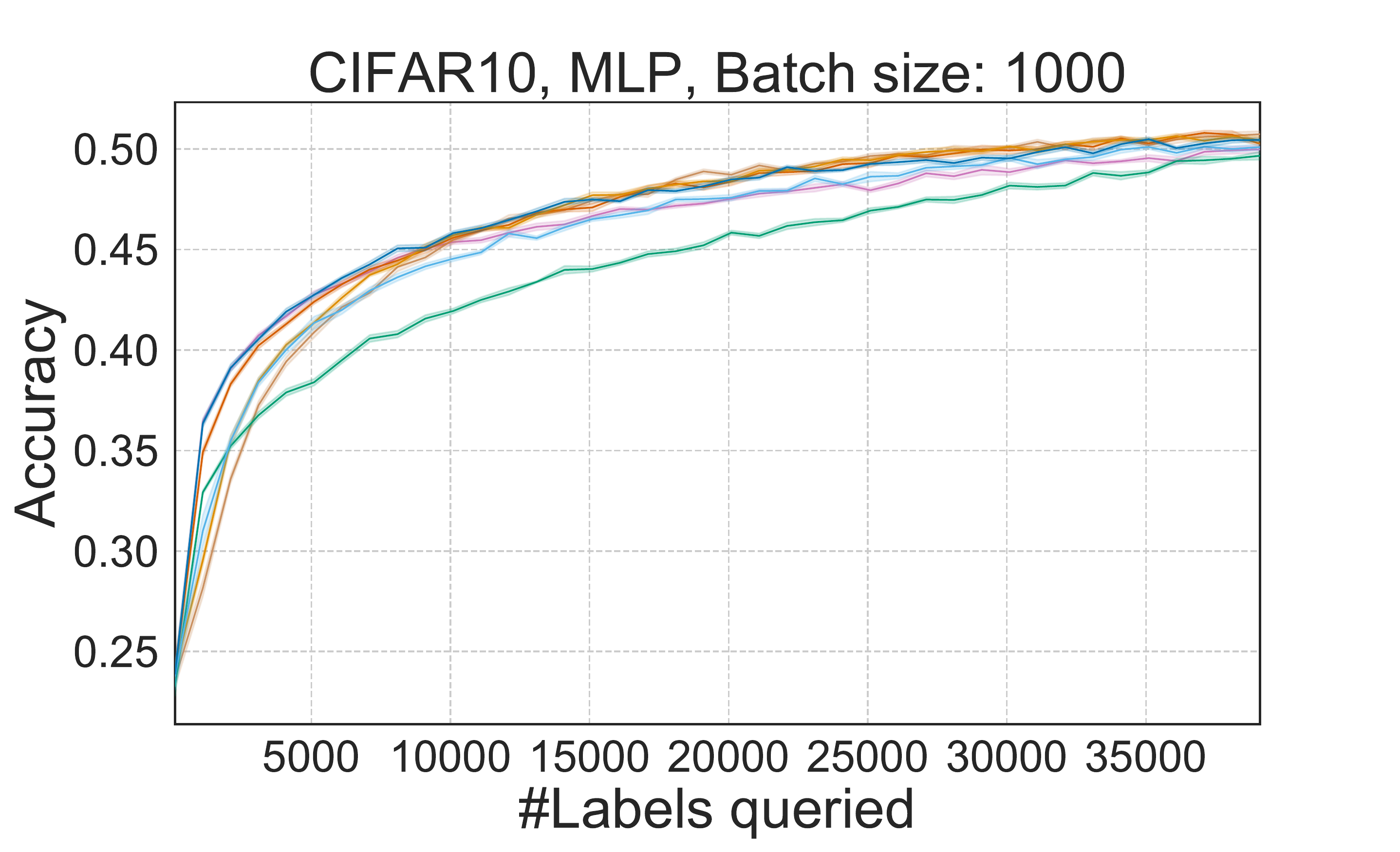}}
  \hfill
  \includegraphics[trim={1.5cm 0cm 1.6cm 0cm}, clip, width=0.32\textwidth]{{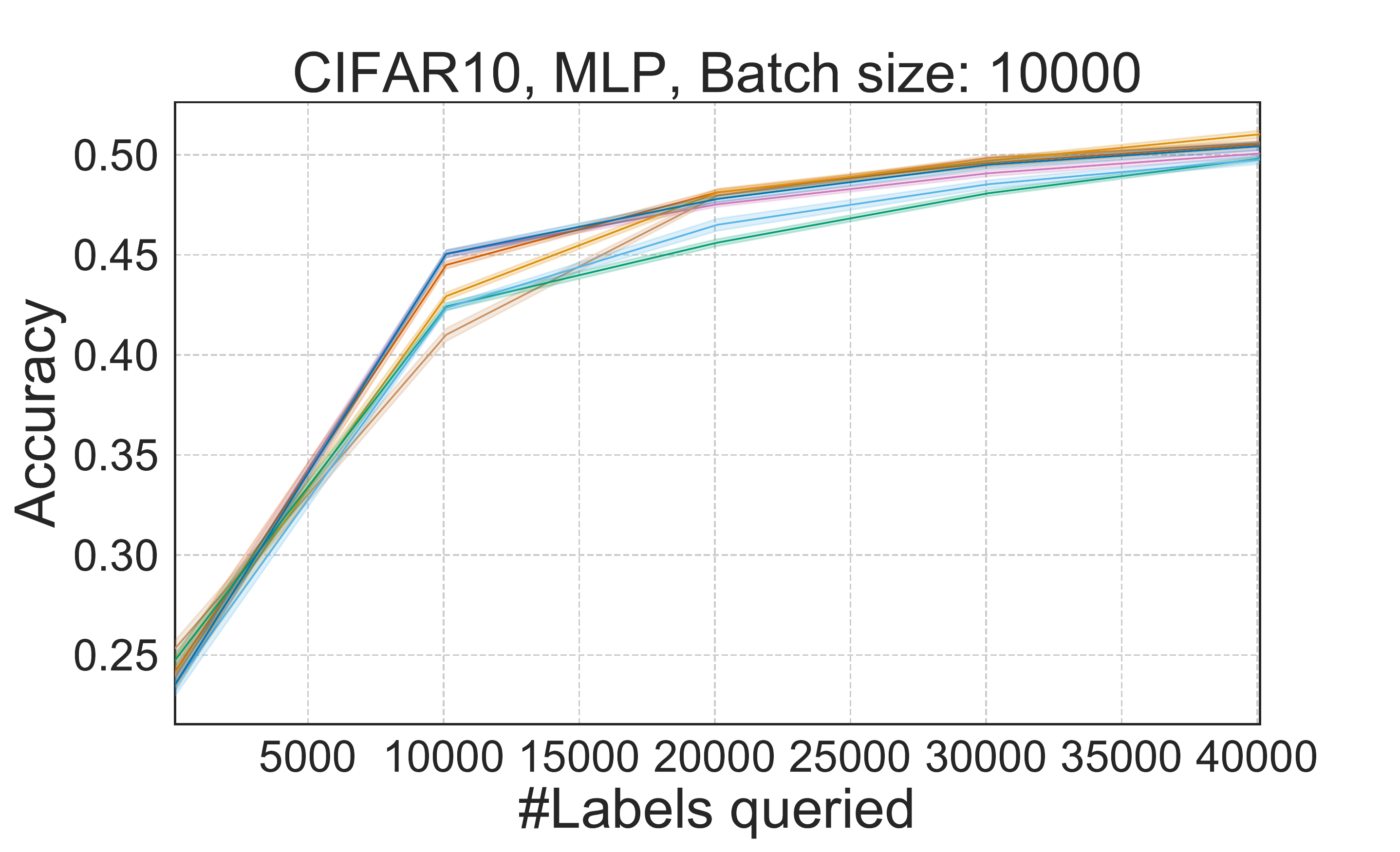}}

    \includegraphics[trim={0.4cm 0cm 27.9cm 0cm}, clip, width=0.012\textwidth]{figs/learning_curves/all_algs_Accuracy_Data=_SVHN__Model=_rn__nQuery=_100__TrainAug=_0___.pdf}
  \includegraphics[trim={1.5cm 0cm 1.6cm 0cm}, clip, width=0.32\textwidth]{{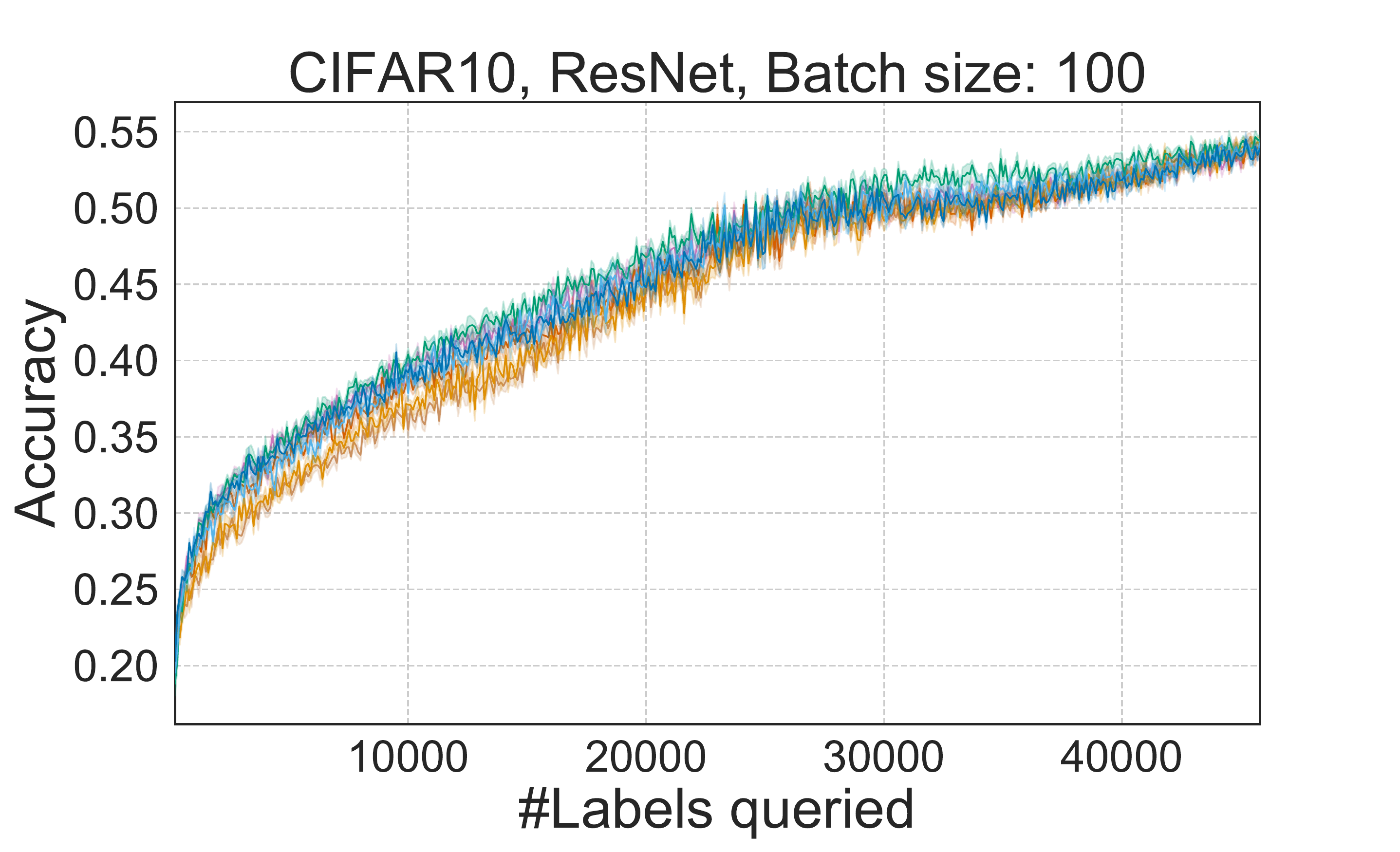}}
  \hfill
  \includegraphics[trim={1.5cm 0cm 1.6cm 0cm}, clip, width=0.32\textwidth]{{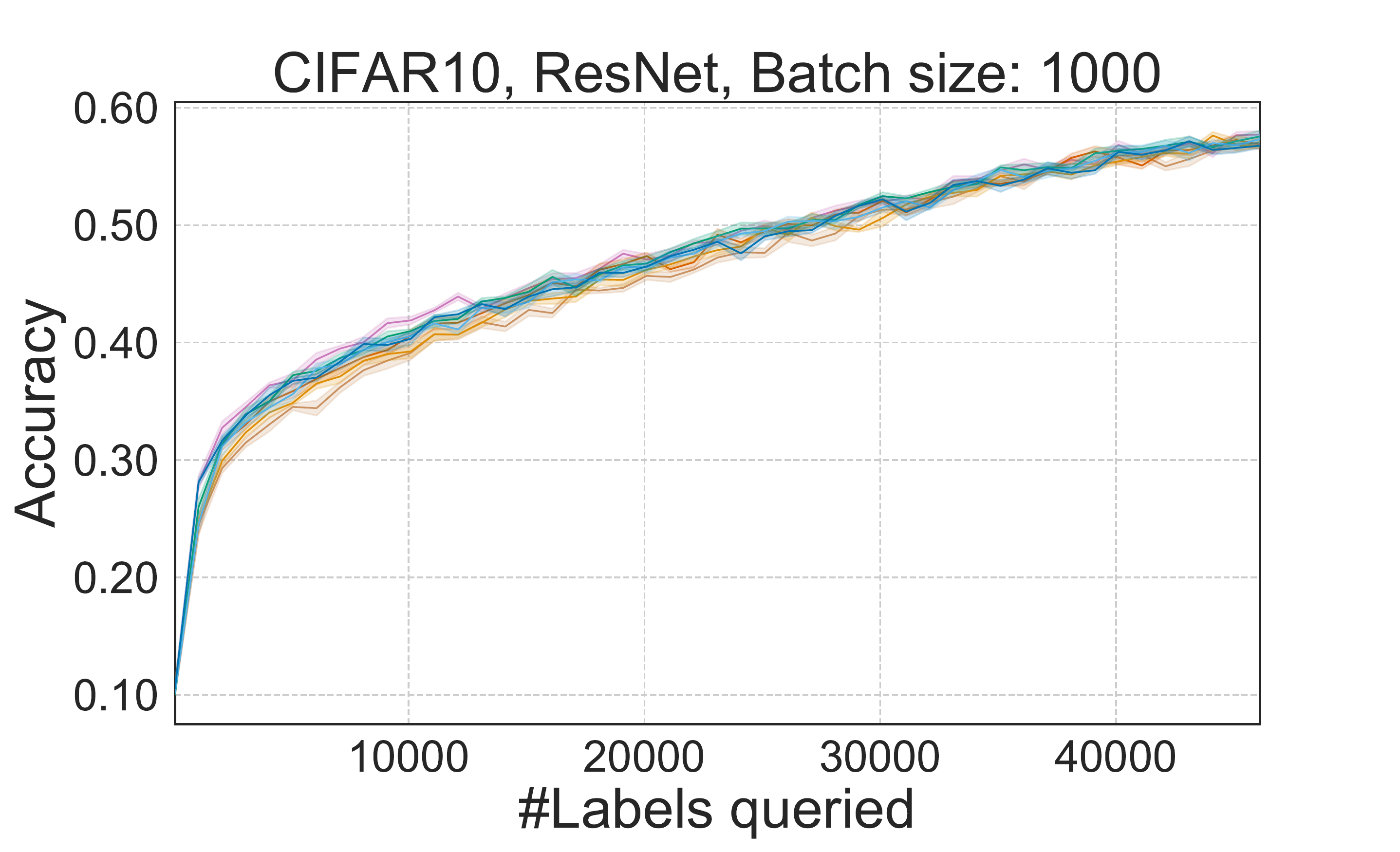}}
  \hfill
  \includegraphics[trim={1.5cm 0cm 1.6cm 0cm}, clip, width=0.32\textwidth]{{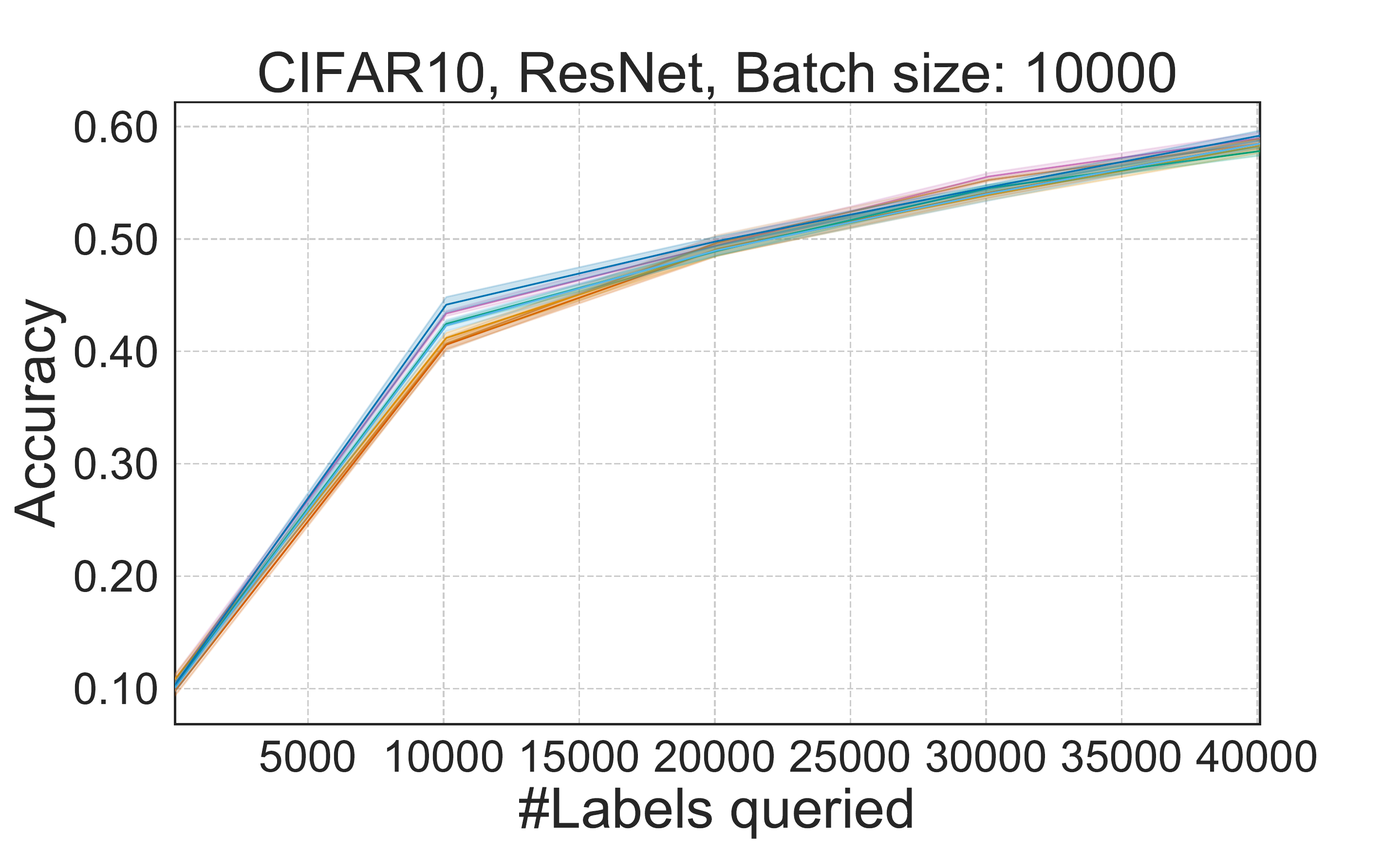}}

    \includegraphics[trim={0.4cm 0cm 27.9cm 0cm}, clip, width=0.012\textwidth]{figs/learning_curves/all_algs_Accuracy_Data=_SVHN__Model=_rn__nQuery=_100__TrainAug=_0___.pdf}
  \includegraphics[trim={1.5cm 0cm 1.6cm 0cm}, clip, width=0.32\textwidth]{{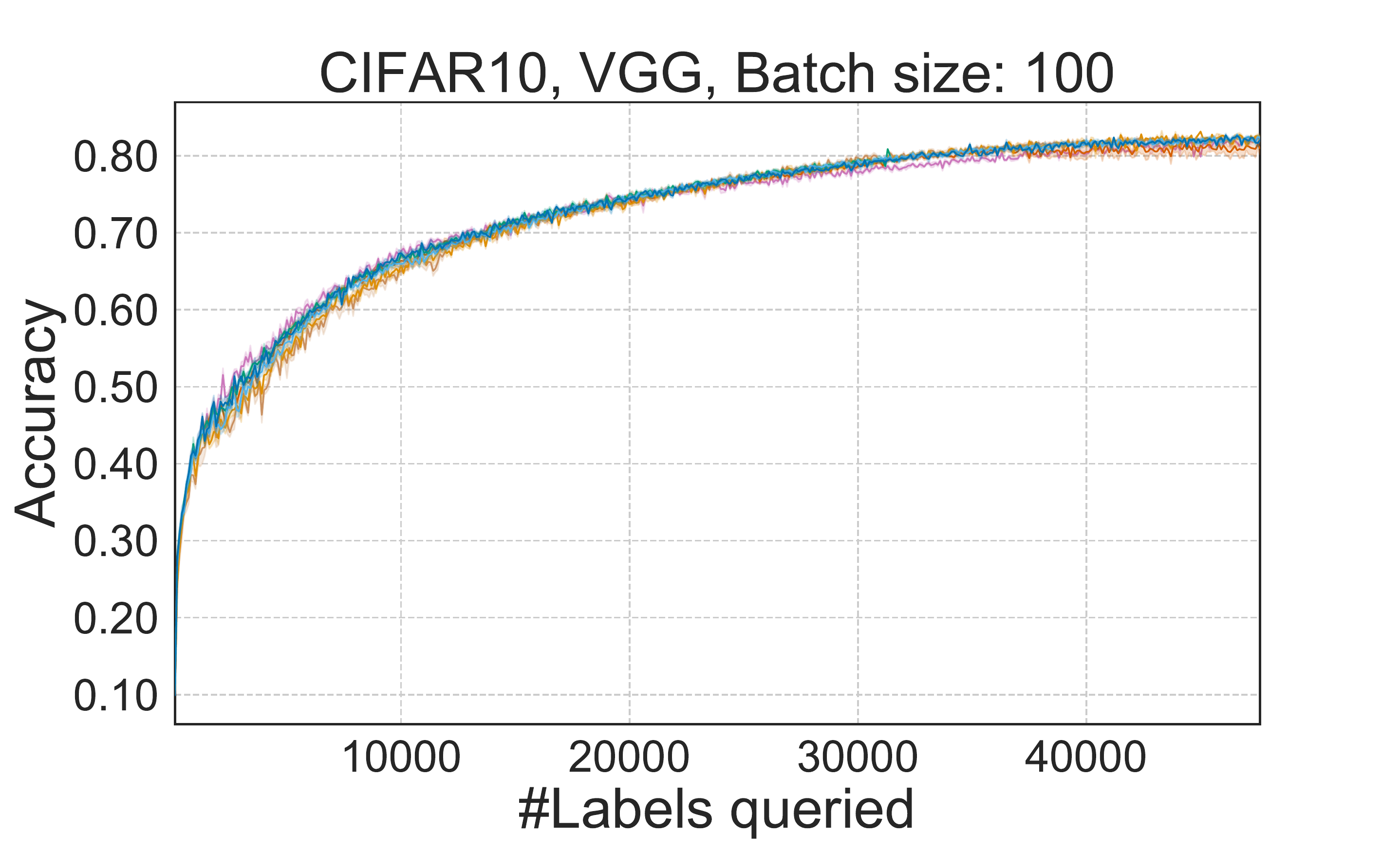}}
  \hfill
  \includegraphics[trim={1.5cm 0cm 1.6cm 0cm}, clip, width=0.32\textwidth]{{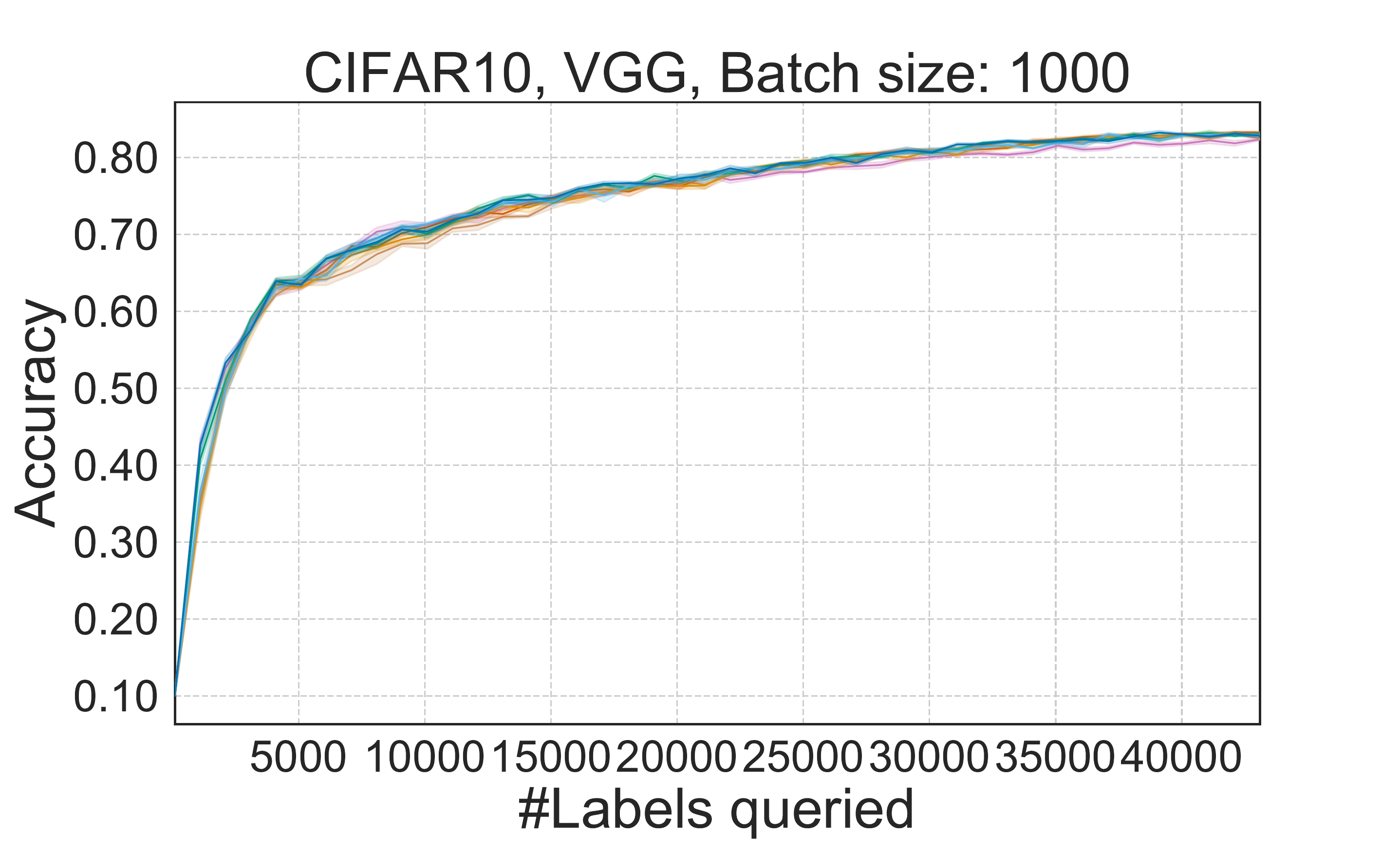}}
  \hfill
  \includegraphics[trim={1.5cm 0cm 1.6cm 0cm}, clip, width=0.32\textwidth]{{figs/learning_curves/all_algs_Accuracy_Data=_CIFAR10__Model=_vgg__nQuery=_10000__TrainAug=_0___.pdf}}
  \\
  \centering
  \begin{subfigure}[b]{\linewidth}
    \includegraphics[trim={0cm 0cm 0cm 0cm}, clip, width=\textwidth]{figs/legends/legend.pdf}
  \end{subfigure}

\caption{Zoomed-in learning curves for CIFAR10 with MLP, ResNet and VGG.}
\label{fig:cifar10-lc}
\end{figure}
\vskip -0.5cm
\section{Pairwise comparisons of algorithms}
\vskip -0.3cm
\label{sec:pairwise}

In addition to Figure~\ref{fig:penalty} in the main text, we also provide penalty
matrices (Figures~\ref{figs:pw-batch-sizes} and~\ref{figs:pw-models}), where the results are aggregated by conditioning on a fixed batch size (100, 1000 and 10000) or on a fixed
neural network model (MLP, ResNet and VGG). For each penalty matrix, the parenthesized number in its title is the total number of $(D, B, A)$ combinations aggregated; as discussed in Section~\ref{sec:experiments}, this is also an upper bound on all its entries.
It can be seen that uncertainty-based methods (e.g. \marg) perform well only in small batch size regimes (100) or when using MLP models; representative sampling based methods (e.g. \coreset) only perform well in large batch size regimes (10000) or when using ResNet or VGG models.
In contrast, \ouralg's performance is competitive across all batch sizes and neural network models.

\begin{figure}
  \centering
      \includegraphics[trim={0.38cm 0cm 27.7cm 0cm}, clip, width=0.022\textwidth]{figs/learning_curves/all_algs_Accuracy_Data=_SVHN__Model=_rn__nQuery=_100__TrainAug=_0___.pdf}
  \includegraphics[trim={4cm 0cm 1.8cm 0cm},clip,width=0.32\textwidth]{{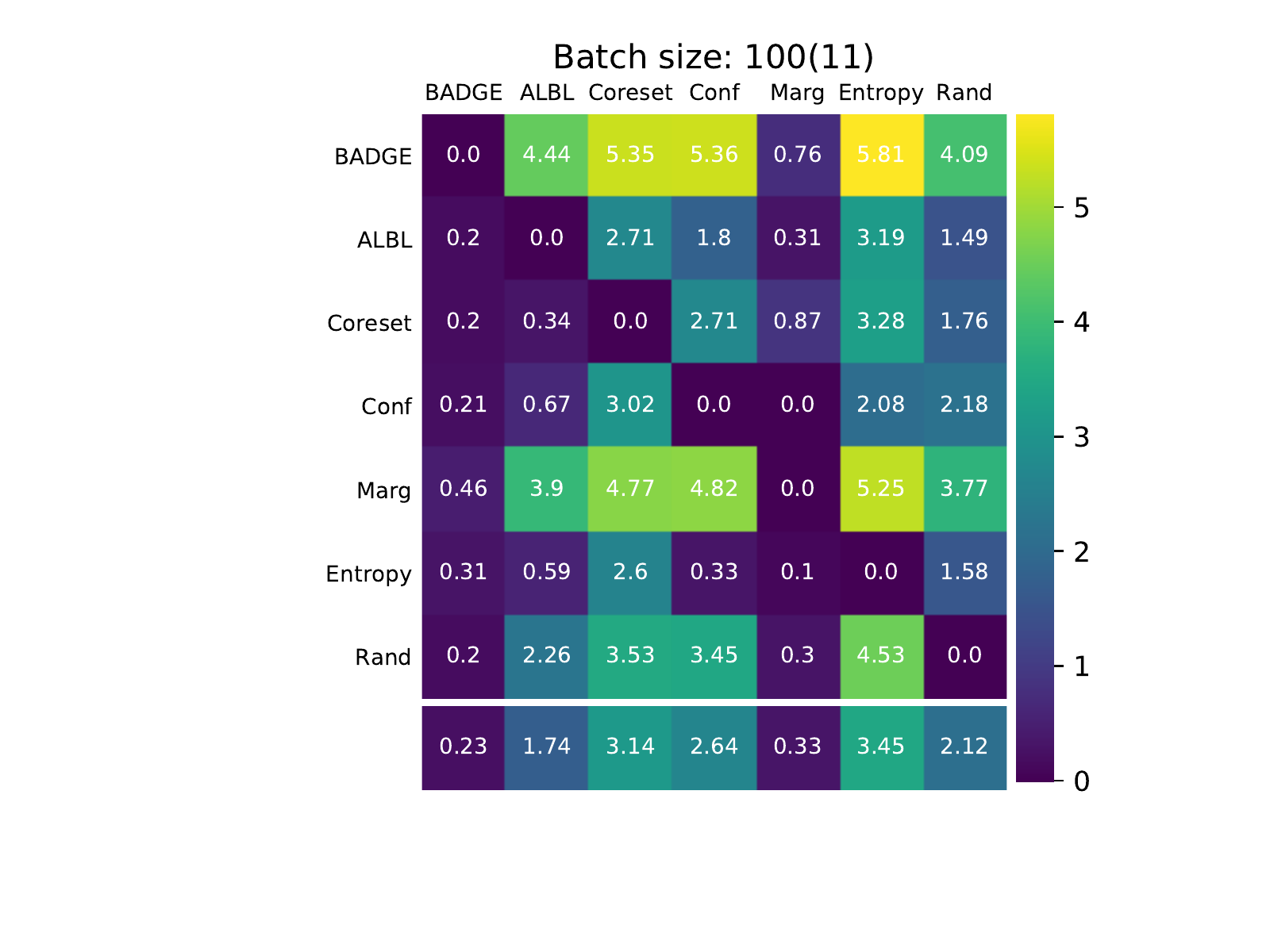}}
  \hfill
  \includegraphics[trim={4cm 0cm 1.8cm 0cm},clip,width=0.32\textwidth]{{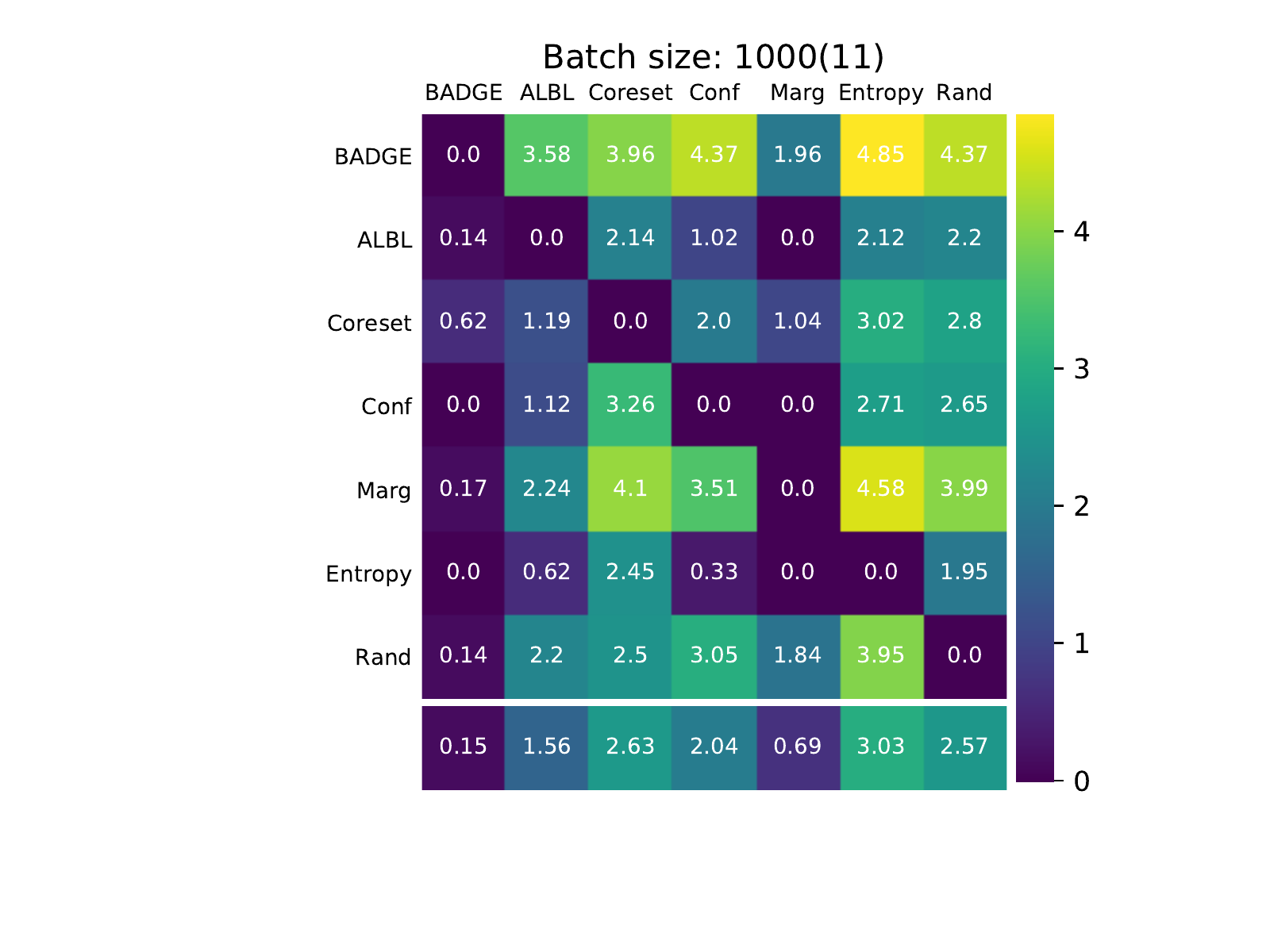}}
  \hfill
  \includegraphics[trim={4cm 0cm 1.8cm 0cm},clip,width=0.32\textwidth]{{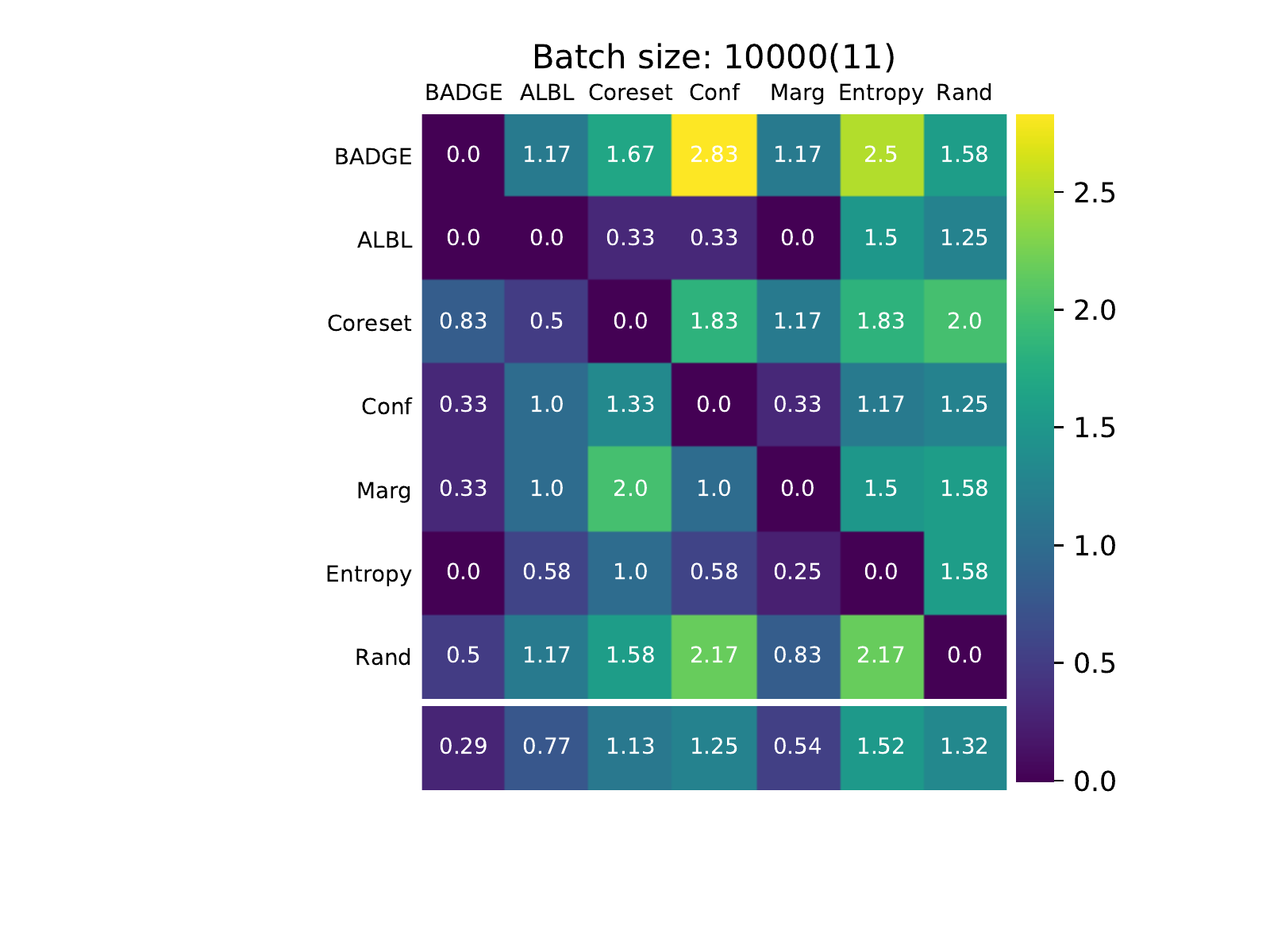}}
\caption{Pairwise penalty matrices of the algorithms, grouped by different batch sizes. The parenthesized number in the title is the total number of $(D, B, A)$ combinations aggregated, which is also an upper bound on all its entries. Element $(i, j)$ corresponds roughly to the number of times algorithm $i$ beats algorithm $j$. Column-wise averages at the bottom show aggregate performance (lower is better). From left to right: batch size = 100, 1000, 10000.}
\label{figs:pw-batch-sizes}
\end{figure}

\begin{figure}
  \centering
      \includegraphics[trim={0.38cm 0cm 27.7cm 0cm}, clip, width=0.022\textwidth]{figs/learning_curves/all_algs_Accuracy_Data=_SVHN__Model=_rn__nQuery=_100__TrainAug=_0___.pdf}
  \includegraphics[trim={4cm 0cm 1.8cm 0cm},clip,width=0.32\textwidth]{{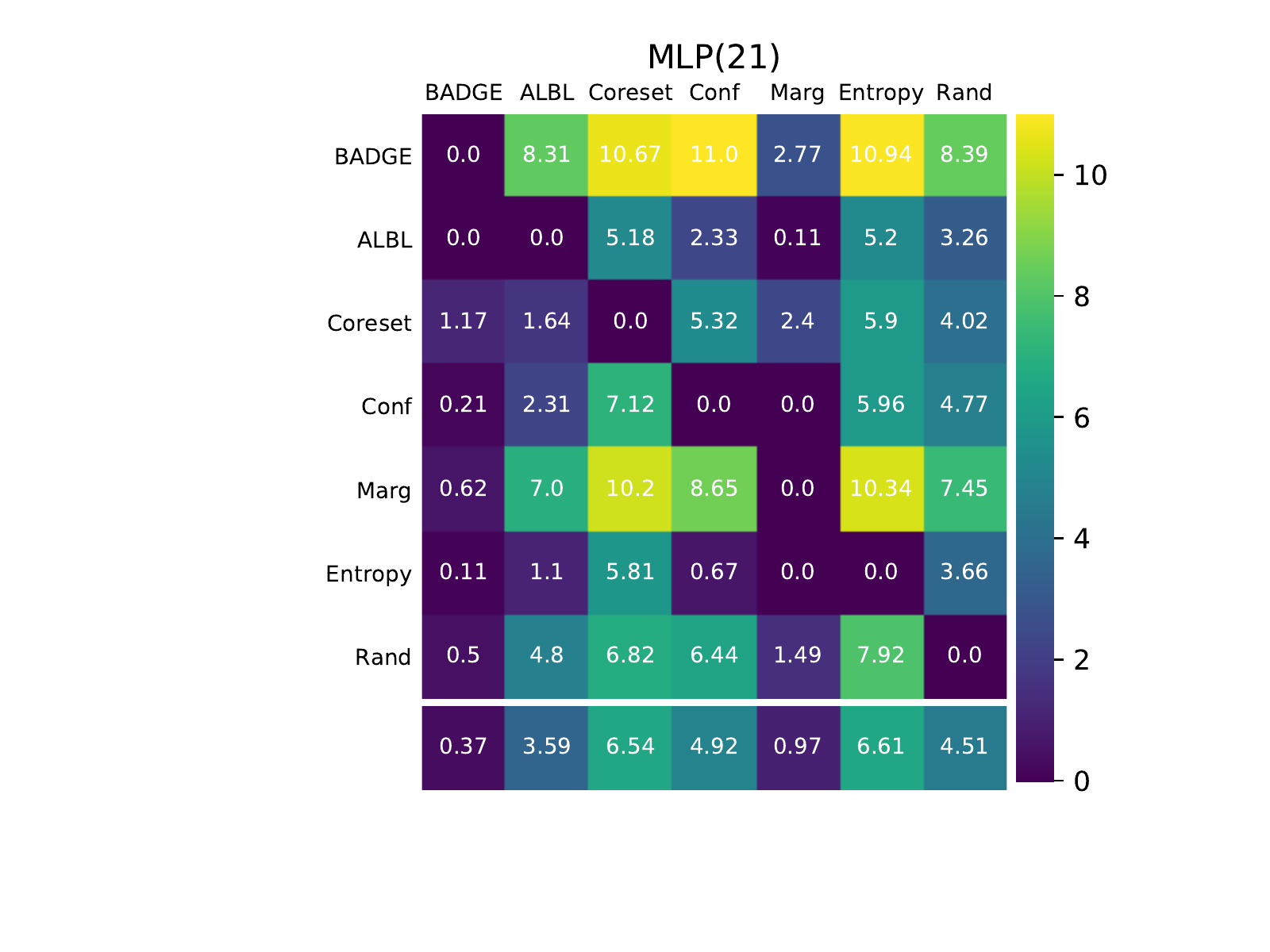}}
  \hfill
  \includegraphics[trim={4cm 0cm 1.8cm 0cm},clip,width=0.32\textwidth]{{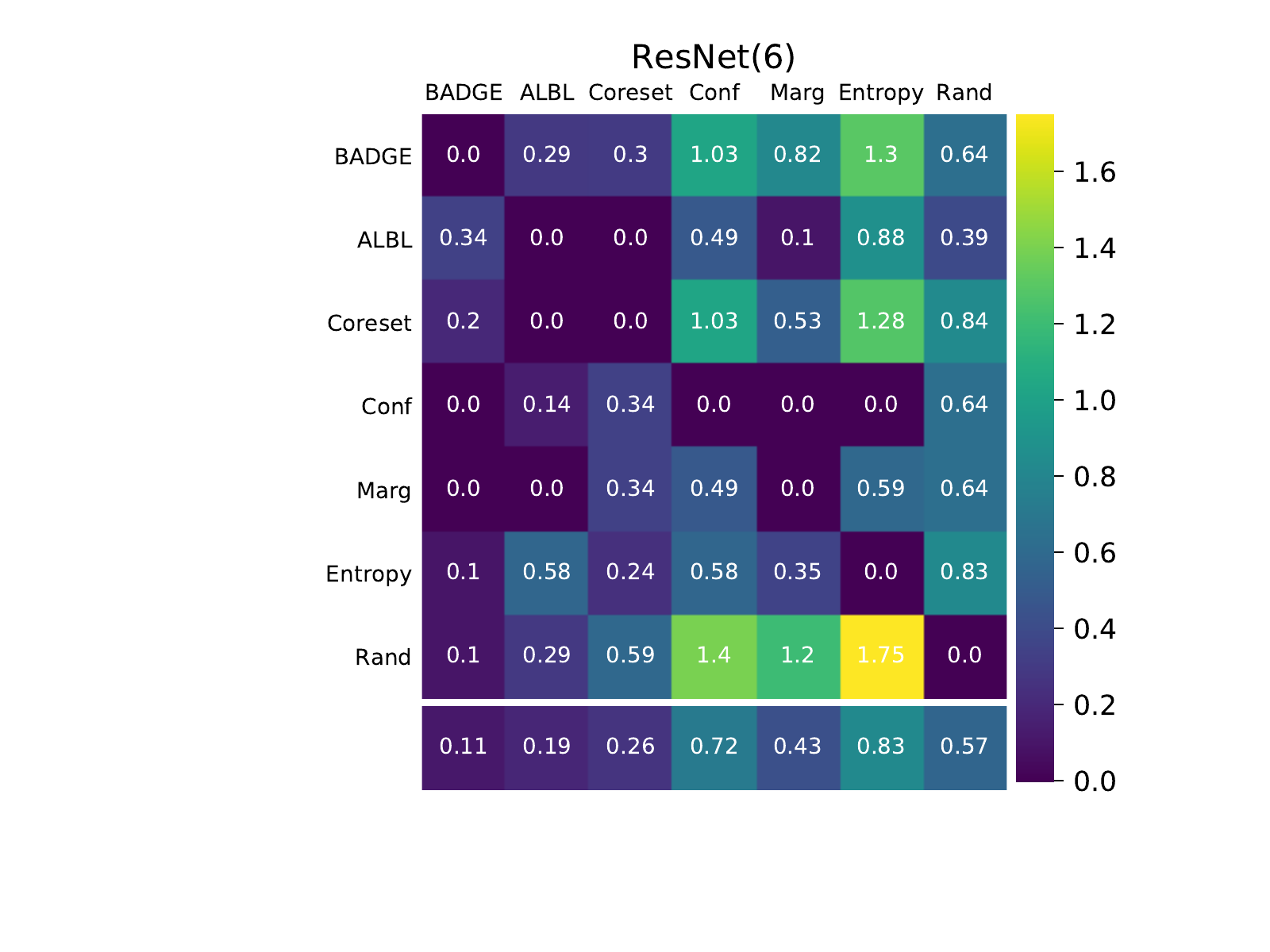}}
  \hfill
  \includegraphics[trim={4cm 0cm 1.8cm 0cm},clip,width=0.32\textwidth]{{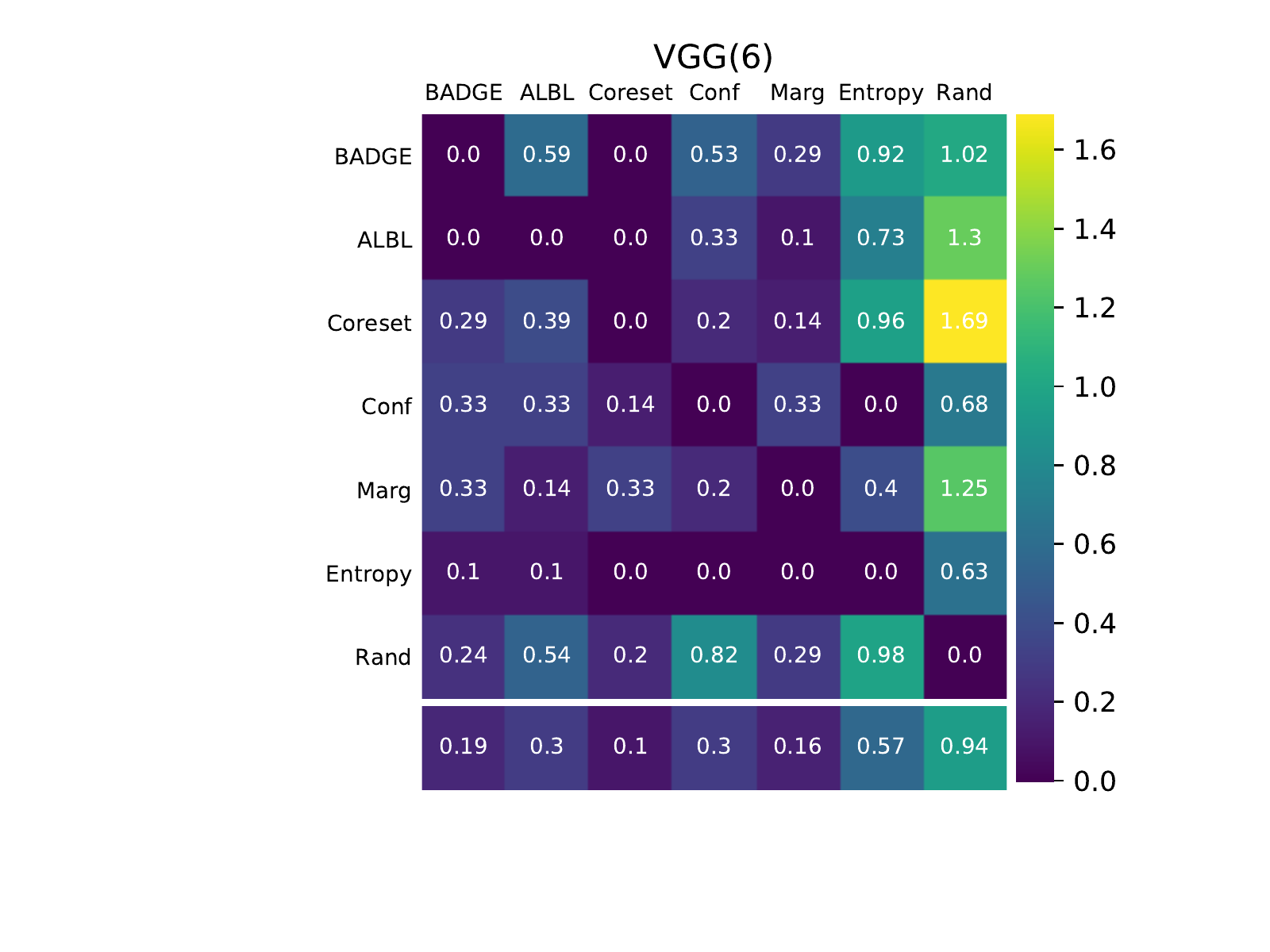}}
\caption{Pairwise penalty matrices of the algorithms, grouped by different neural network models. The parenthesized number in the title is the total number of $(D, B, A)$ combinations aggregated, which is also an upper bound on all its entries. Element $(i, j)$ corresponds roughly to the number of times algorithm $i$ beats algorithm $j$. Column-wise averages at the bottom show aggregate performance (lower is better). From left to right: MLP, ResNet and VGG.}
\label{figs:pw-models}
\end{figure}

%\begin{figure}
%  \centering
%  \includegraphics[width=0.5\textwidth]{{figs/comp_matrices/linear.pdf}}
%  \hfill
%  \includegraphics[width=0.5\textwidth]{{figs/comp_matrices/log.pdf}}
%  \caption{Pairwise comparison of normalized errors of the algorithms; the $(i,j)$-th entry in each matrix indicates the number of settings
%  algorithm $i$ significantly beats algorithm $j$. The left graph accumulates settings where the sample sizes is from the set $\cbr{100,200,300,50000}$; the right graph accumulates settings where the sample size is from the set $\cbr{100, 200, 400, 800, \ldots, 25600}$.}
%  \label{fig:pairwise}
%\end{figure}

\section{CDFs of normalized errors of different algorithms}
\label{sec:cdfs}

In addition to Figure~\ref{fig:cdf-all} that aggregates over all settings, we show here the CDFs of normalized errors by conditioning on fixed batch sizes (100, 1000 and 10000) in Figure~\ref{figs:cdfs-batch-sizes}, and show the CDFs of normalized errors by conditioning on fixed neural network models (MLP, ResNet and VGG) in Figure~\ref{figs:cdfs-models}.

%\includegraphics[trim=left bottom right top, clip]
%trim={0.3cm 0cm 2.5cm 0cm}, clip, width=0.347\textwidth
\begin{figure}
  \centering
      \includegraphics[trim={0.5cm 0cm 27.3cm 0cm}, clip, width=0.017\textwidth]{{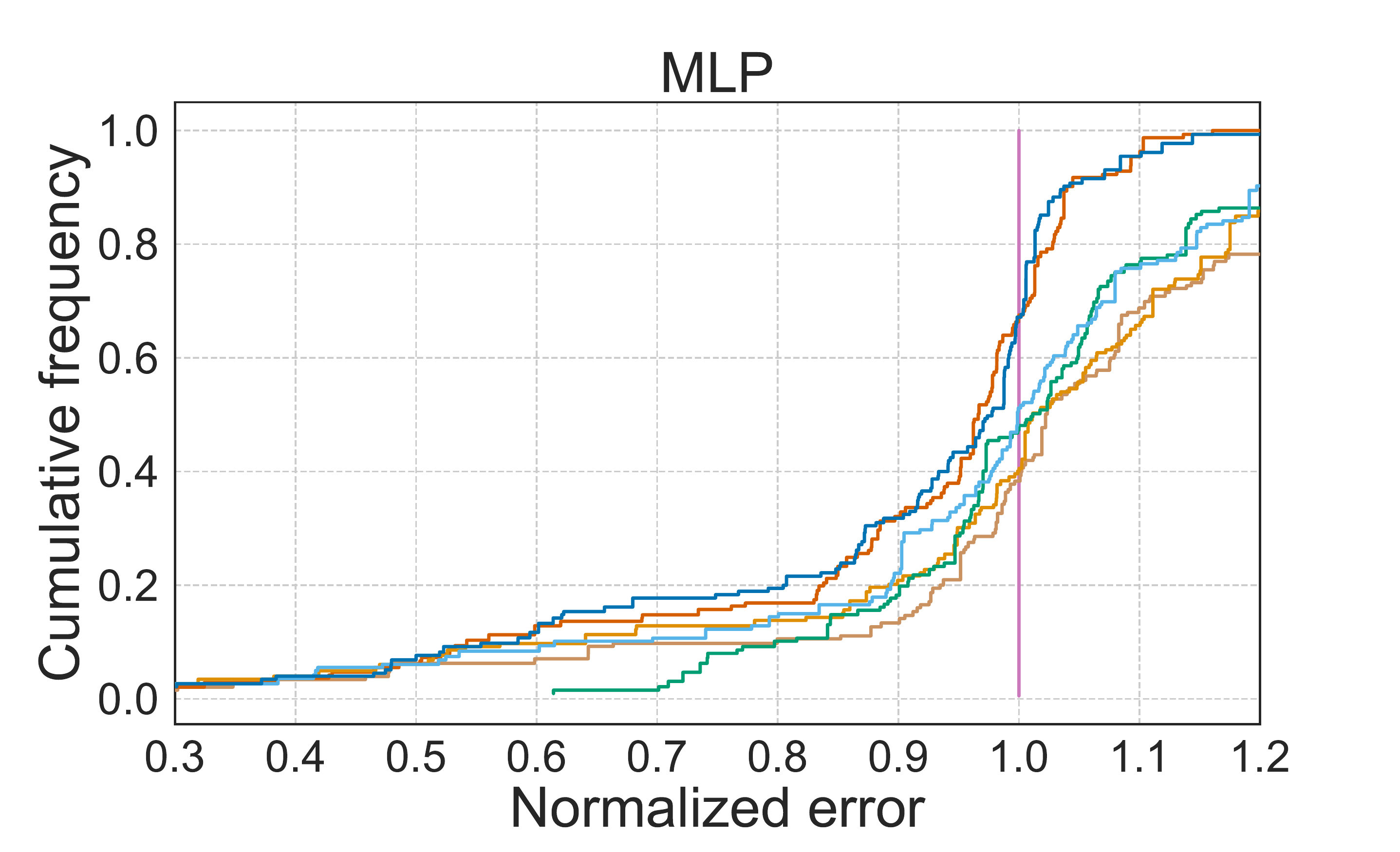}}
  \includegraphics[trim={2cm 0cm 2cm 0cm}, clip, width=0.32\textwidth]{{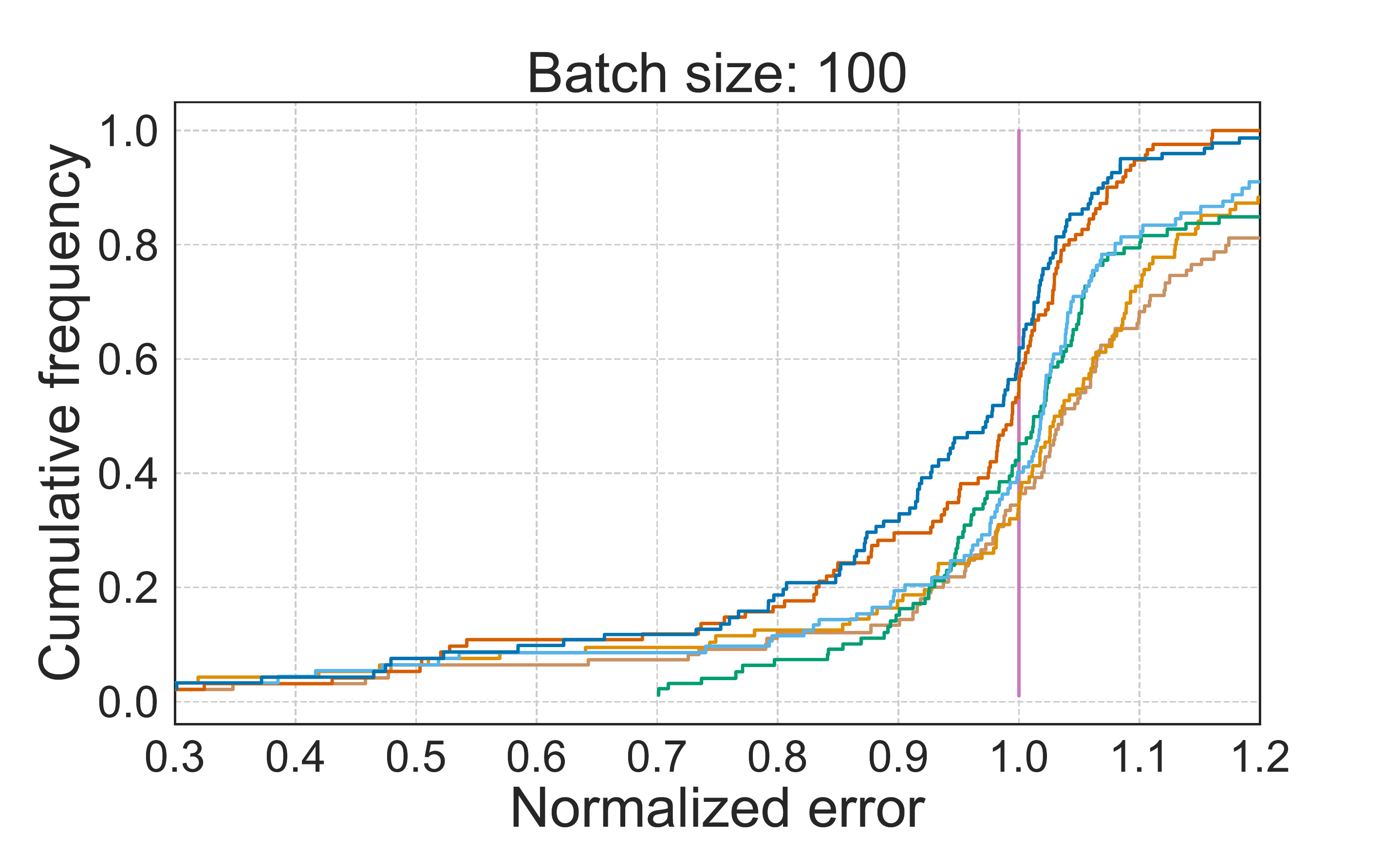}}
  \hfill
  \includegraphics[trim={2cm 0cm 2cm 0cm}, clip, width=0.32\textwidth]{{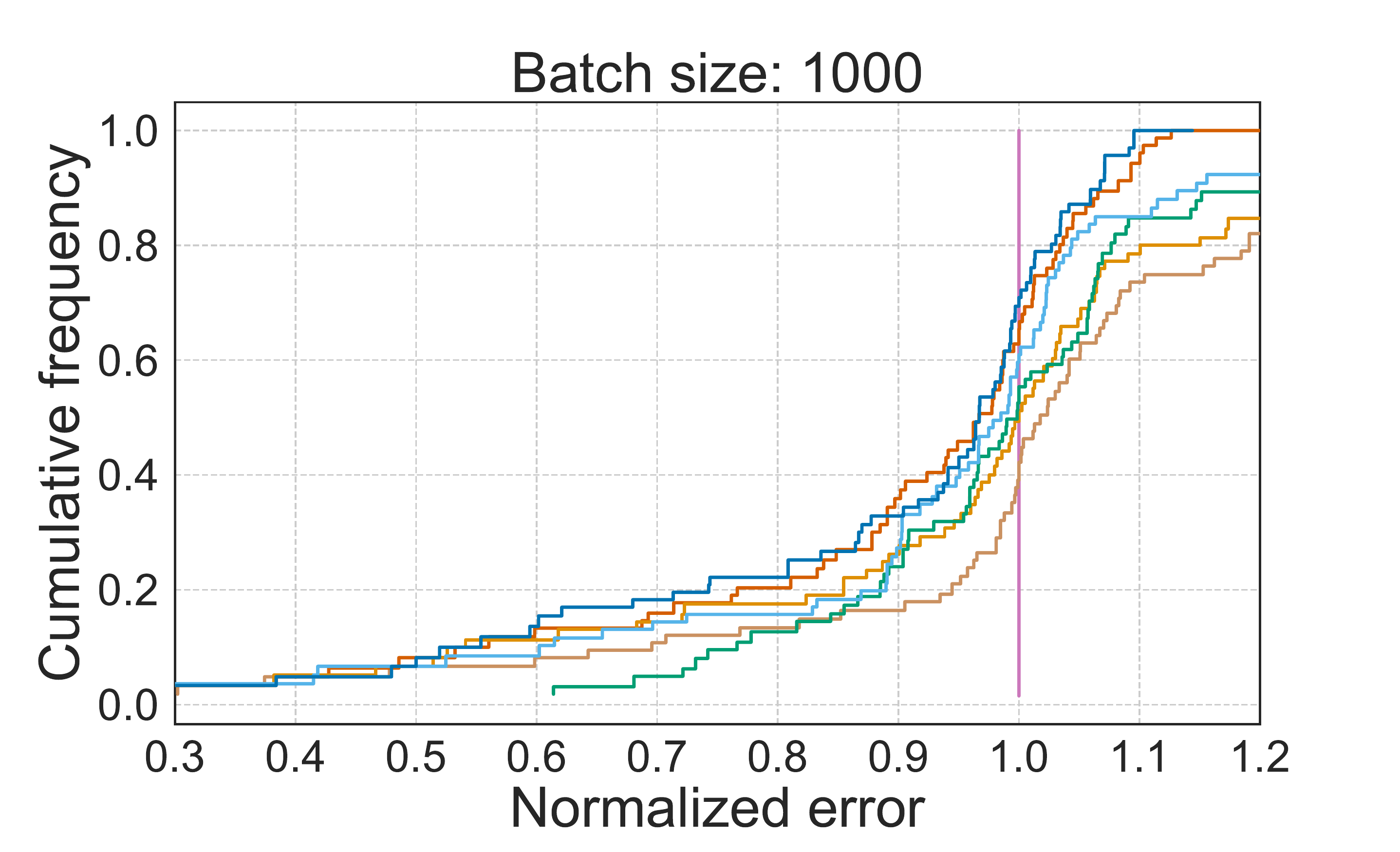}}
  \hfill
  \includegraphics[trim={2cm 0cm 2cm 0cm}, clip, width=0.32\textwidth]{{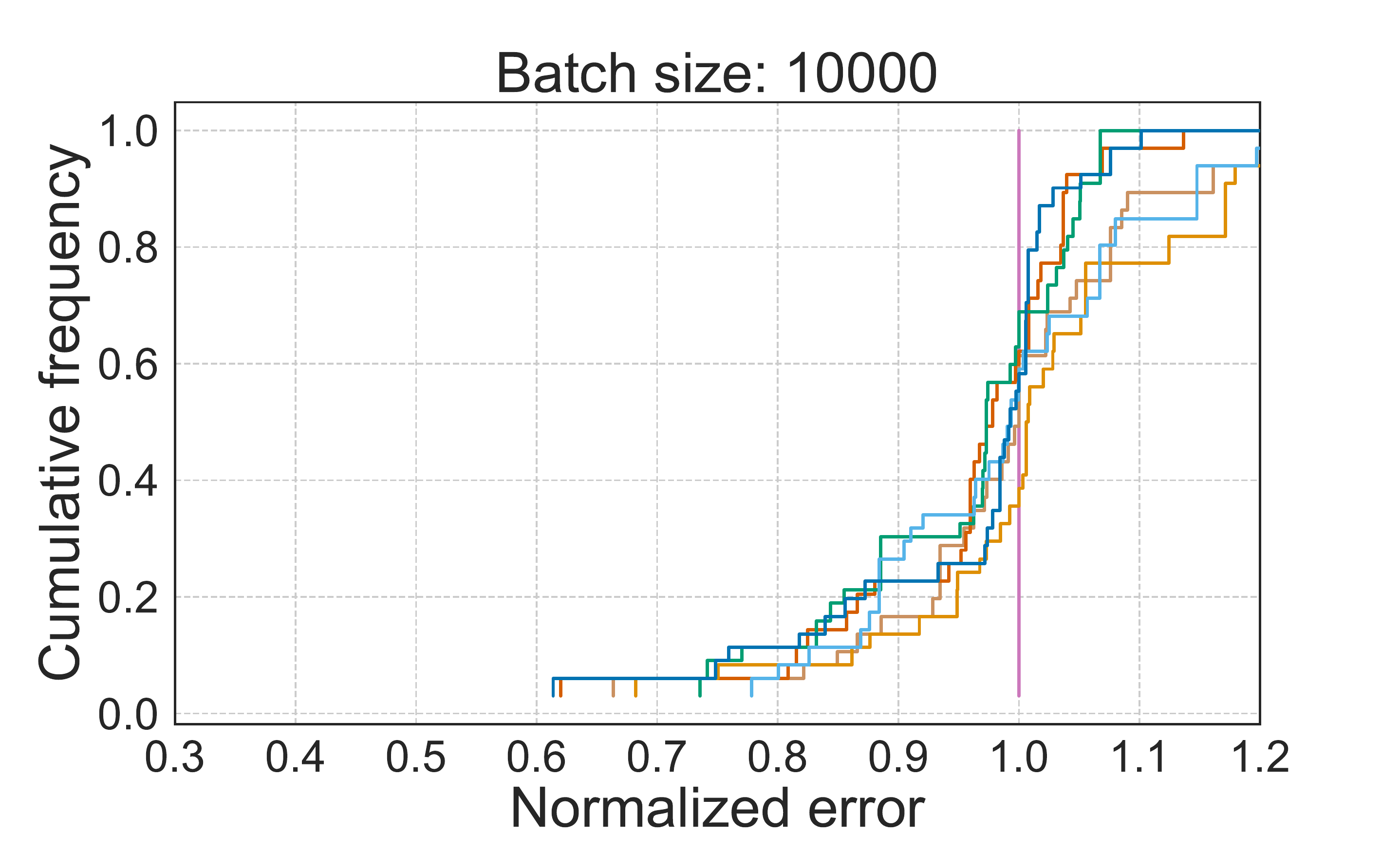}}
  \\
  \centering
  \begin{subfigure}[b]{\linewidth}
    \includegraphics[trim={0cm 0cm 0cm 0cm}, clip, width=\textwidth]{figs/legends/legend.pdf}
  \end{subfigure}
\caption{CDFs of normalized errors of the algorithms, group by different batch sizes. Higher CDF indicates better performance. From left to right: batch size = 100, 1000, 10000.}
\label{figs:cdfs-batch-sizes}
\end{figure}

\begin{figure}
  \centering
    \includegraphics[trim={0.5cm 0cm 27.3cm 0cm}, clip, width=0.017\textwidth]{{figs/cdfs/cdf_Model=_mlp__.pdf}}
  \includegraphics[trim={2cm 0cm 2cm 0cm}, clip, width=0.32\textwidth]{{figs/cdfs/cdf_Model=_mlp__.pdf}}
  \hfill
  \includegraphics[trim={2cm 0cm 2cm 0cm}, clip, width=0.32\textwidth]{{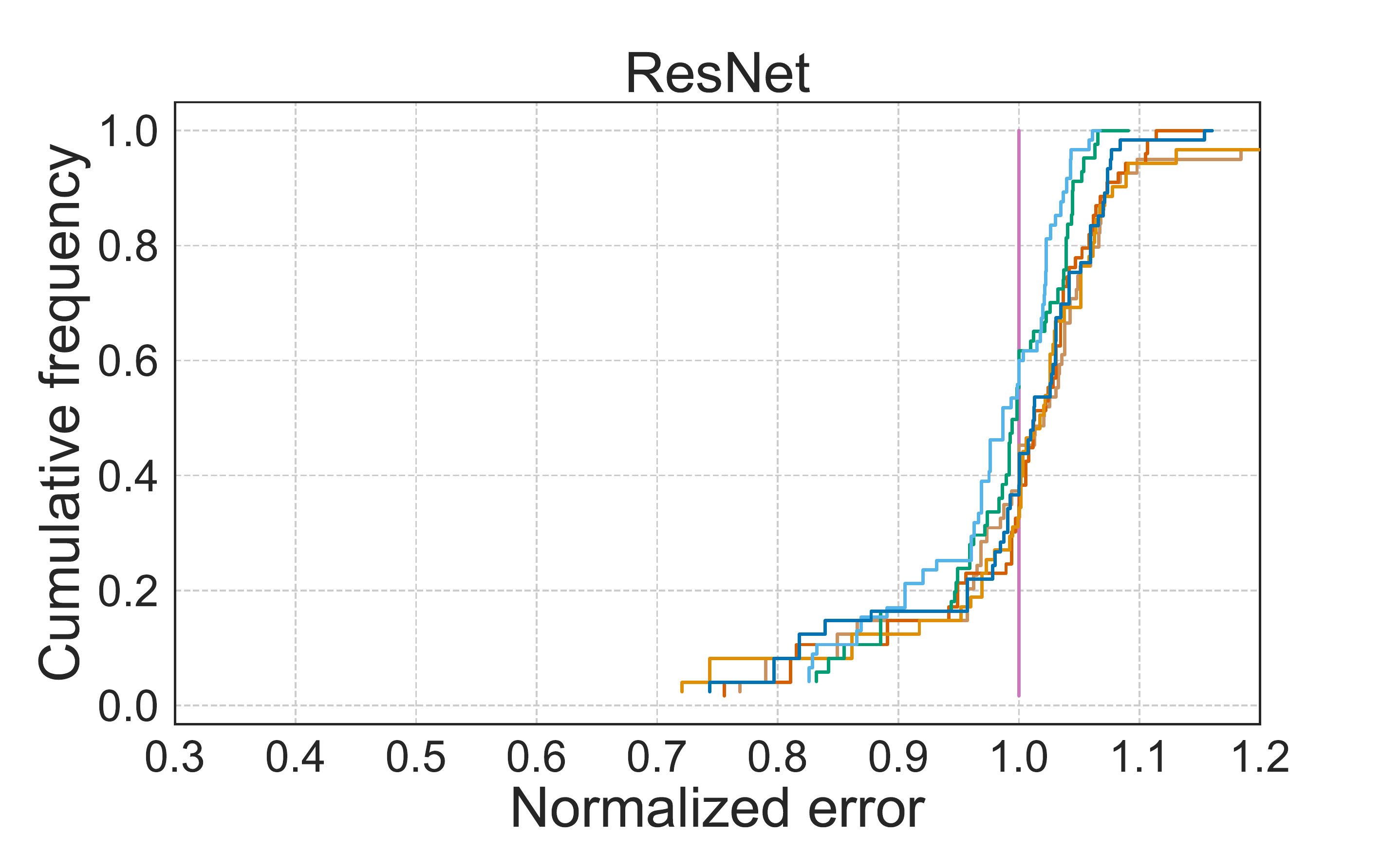}}
  \hfill
  \includegraphics[trim={2cm 0cm 2cm 0cm}, clip, width=0.32\textwidth]{{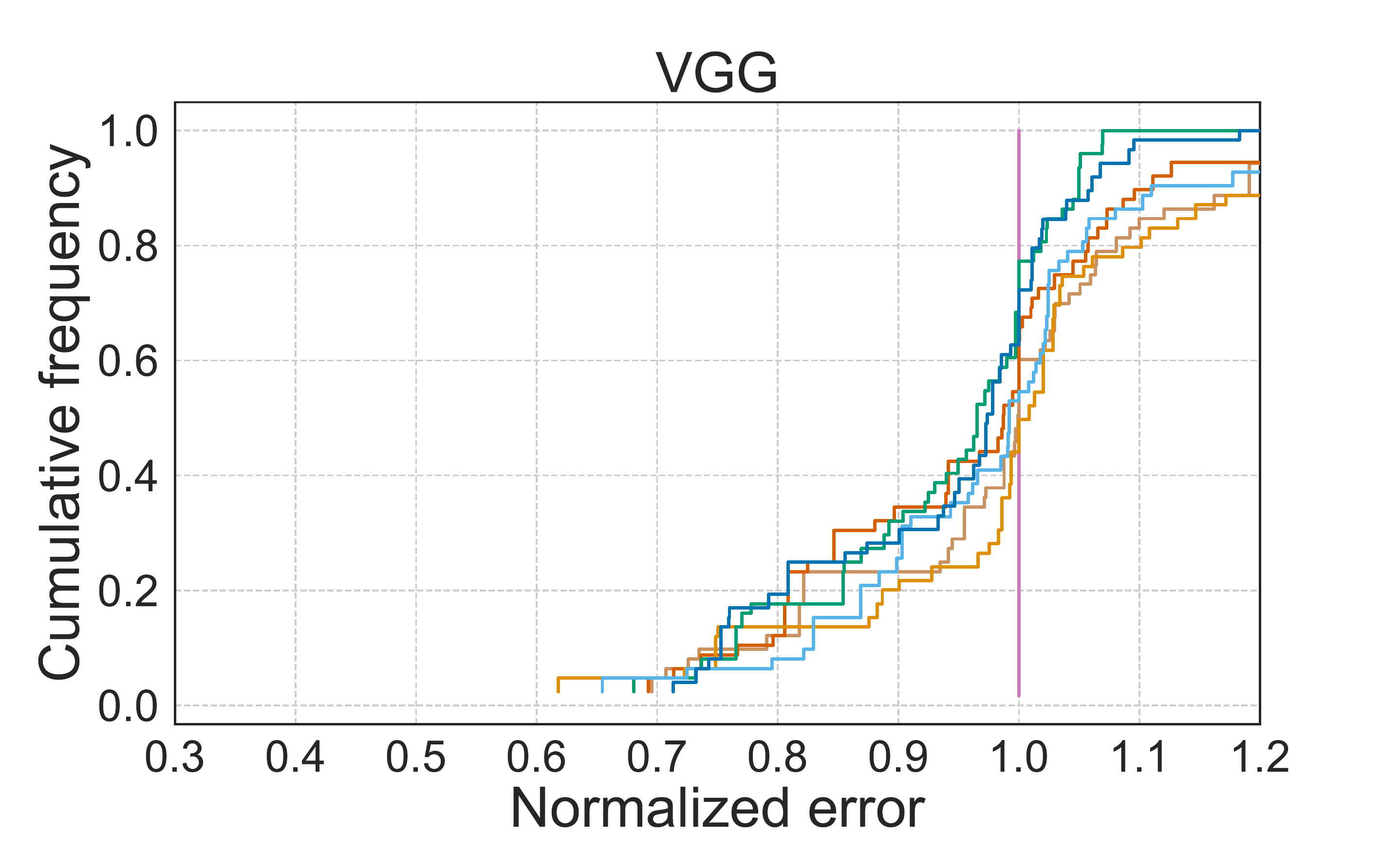}}
  \\
  \centering
  \begin{subfigure}[b]{\linewidth}
    \includegraphics[trim={0cm 0cm 0cm 0cm}, clip, width=\textwidth]{figs/legends/legend.pdf}
  \end{subfigure}
\caption{CDFs of normalized errors of the algorithms, group by different neural network models. Higher CDF indicates better performance. From left to right: MLP, ResNet and VGG.}
\label{figs:cdfs-models}
\end{figure}

\section{Batch uncertainty and diversity}
\label{sec:batchdiv}
Figure~\ref{fig:k-dpp-3} gives a comparison of sampling methods with gradient embedding in two settings (OpenML \# 6, MLP, batchsize 100 and SVHN, ResNet, batchsize 1000), in terms of uncertainty and diversity of examples selected within batches. These two properties are measured by average $\ell_2$ norm and determinant of the Gram matrix of gradient embedding, respectively.
It can be seen that, \kmeansp (\ouralg) induces good batch diversity in both settings.
\conf generally selects examples with high uncertainty, but in some iterations of OpenML \#6, the batch diversity is relatively low, as evidenced by the corresponding log Gram determinant being $-\infty$. These areas are indicated by gaps in the learning curve for \conf. Situations where there are many gaps in the \conf plot seem to correspond to situations in which \conf performs poorly in terms of accuracy (see Figure~\ref{fig:6-lc} for the corresponding learning curve). Both $k$-DPP and \ffkc (an algorithm that approximately minimizes $k$-center objective) select batches that have lower diversity than \kmeansp (\ouralg).
%\chicheng{Added some discussions here.}

\begin{figure}
\centering
\begin{subfigure}[b]{0.45\linewidth}
\includegraphics[trim={0cm 0cm 0cm 0cm}, clip, width=\textwidth]{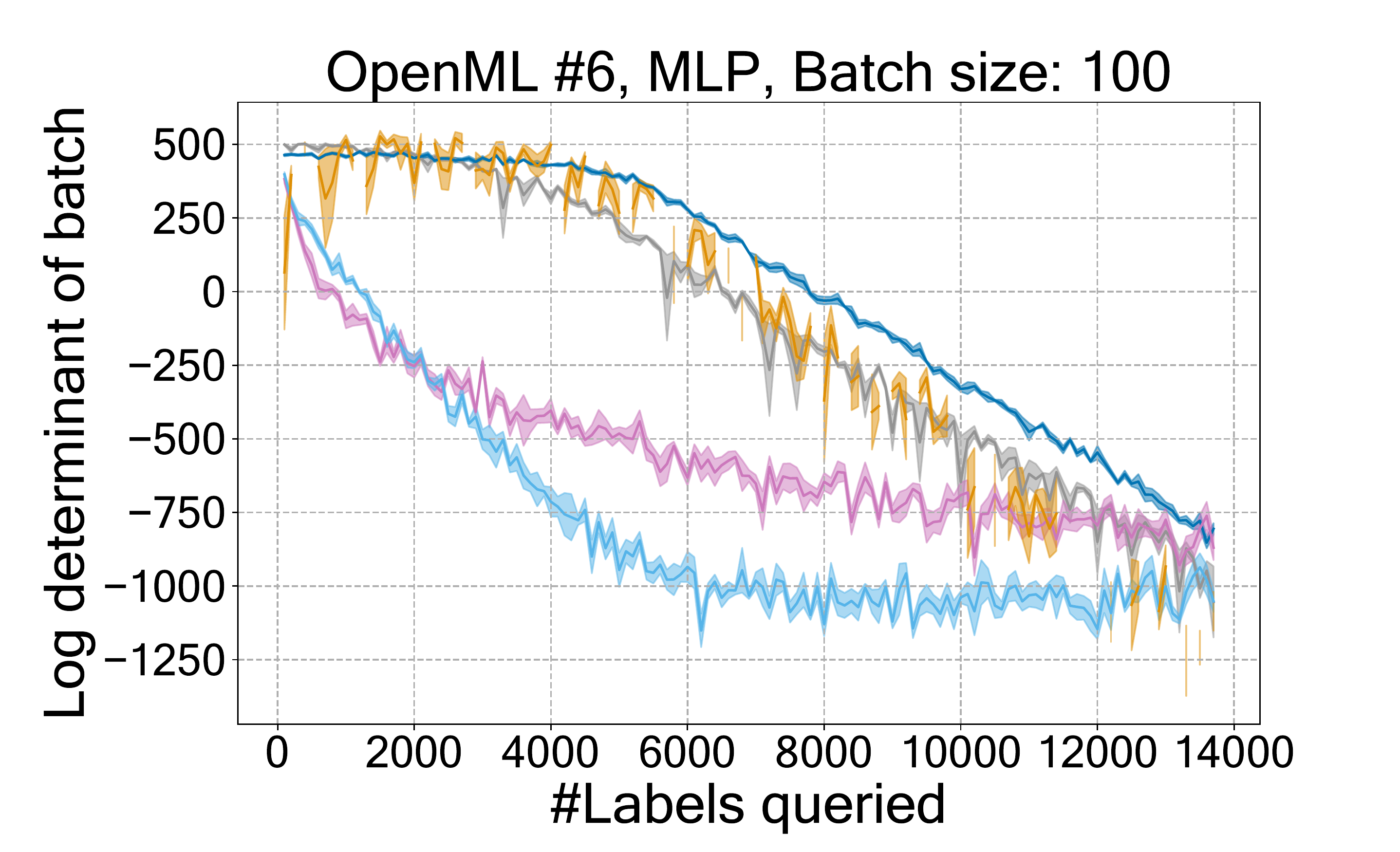}
\caption{}
\end{subfigure}
\hfill
\begin{subfigure}[b]{0.45\linewidth}
  \includegraphics[trim={0cm 0cm 0cm 0cm}, clip, width=\textwidth]{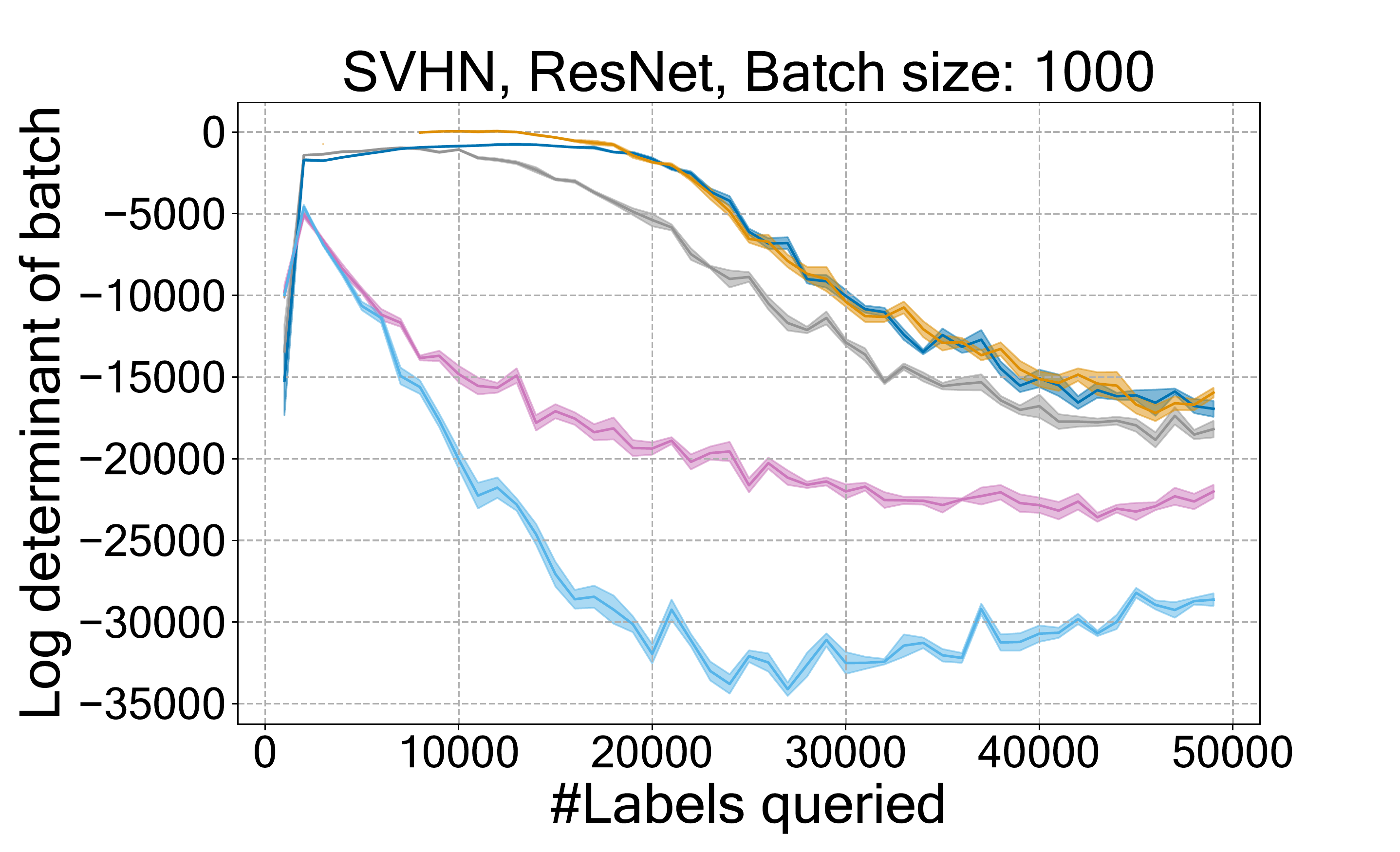}
\caption{}
\end{subfigure}
\hfill
\begin{subfigure}[b]{0.45\linewidth}
  \includegraphics[trim={0cm 0cm 0cm 0cm}, clip, width=\textwidth]{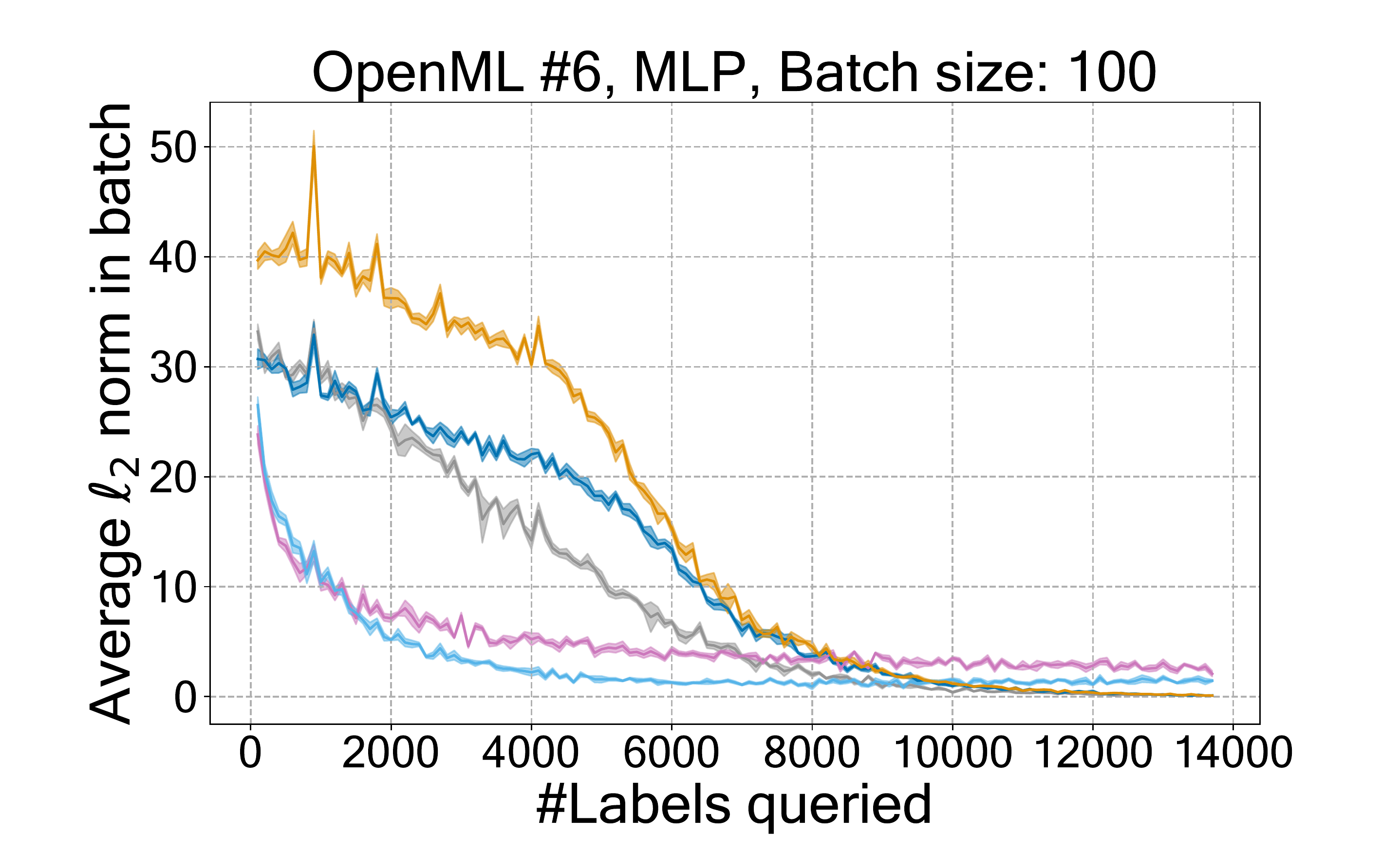}
  \caption{}
\end{subfigure}
\hfill
\begin{subfigure}[b]{0.45\linewidth}
  \includegraphics[trim={0cm 0cm 0cm 0cm}, clip, width=\textwidth]{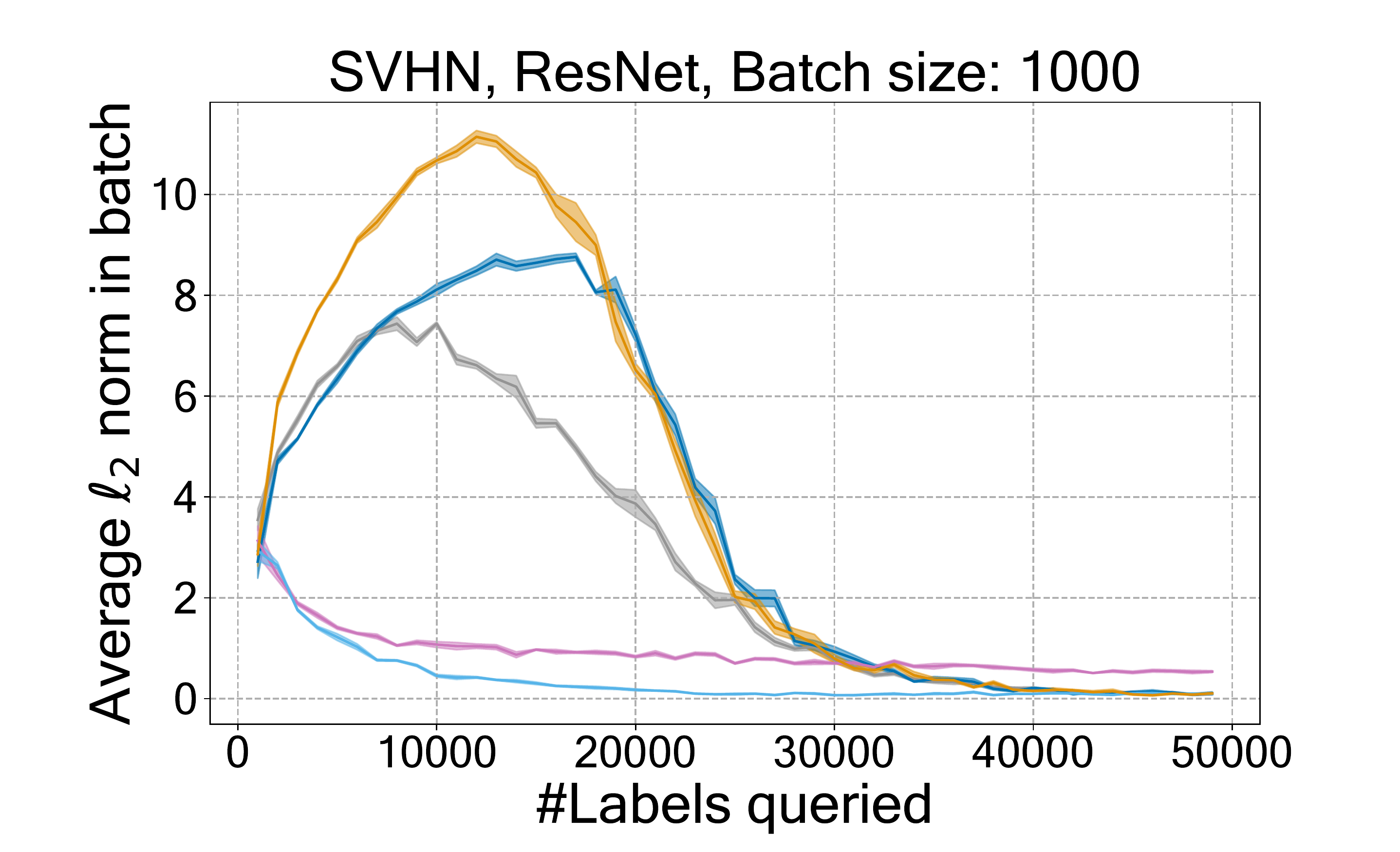}
  \caption{}
\end{subfigure}
\begin{subfigure}[b]{\linewidth}
  \includegraphics[trim={0cm 0cm 0cm 0cm}, clip, width=0.9\textwidth]{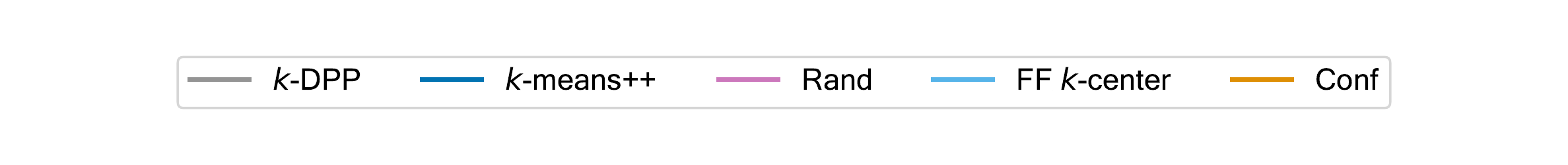}
\end{subfigure}
\caption{A comparison of batch selection algorithms in gradient space. Plots \textbf{a and b} show the log determinants of the Gram matrices of gradient embeddings within batches as learning progresses. Plots \textbf{c and d} show the average embedding magnitude (a measurement of predictive uncertainty) in the selected batch. The $k$-centers sampler finds points that are not as diverse or high-magnitude as other samplers. Notice also that \kmeansp tends to actually select samples that are both more diverse and higher-magnitude than a $k$-DPP, a potential pathology of the $k$-DPP's degree of stochastisity. Among all algorithms, \conf has the largest average norm of gradient embeddings within a batch; however, in OpenML \#6, and the first few interations of SVHN, some batches have a log Gram determinant of $-\infty$ (shown as gaps in the curve), which shows that \conf sometimes selects batches that are inferior in diversity.\looseness=-1}
% \vskip -0.25cm
\label{fig:k-dpp-3}
\end{figure}

\section{Comparison of \kmeansp and $k$-DPP in batch selection}
\label{sec:comp}

In Figures~\ref{fig:comp-6} to~\ref{fig:comp-cifar10}, we give running time and test accuracy comparisons between \kmeansp and $k$-DPP for selecting examples based on gradient embedding in batch mode active learning.
We implement the $k$-DPP sampling using the MCMC algorithm from~\citep{kang2013fast}, which has a time complexity of $O(\tau \cdot (k^2 + kd))$ and space complexity of $O(k^2 + kd)$, where $\tau$ is the number of sampling steps. We set $\tau$ as $\lfloor 5 k \ln k \rfloor$ in our experiment.
The comparisons for batch size 10000 are not shown here as the implementation of $k$-DPP sampling runs out of memory.

It can be seen from the figures that, although $k$-DPP and \kmeansp are based on different sampling criteria, the classification accuracies of their induced active learning algorithm are similar. In addition, when large batch sizes are required (e.g. $k = 1000$), the running times of $k$-DPP sampling are generally much higher than those of \kmeansp.

\begin{figure}
  \centering
  \includegraphics[trim={0.3cm 0cm 2.5cm 0cm}, clip, width=0.23\textwidth]{{figs/dpp_learning_curves/comp_Accuracy_Data=_6__Model=_mlp__nQuery=_100__TrainAug=_0___.pdf}}
  \includegraphics[trim={0.3cm 0cm 2.5cm 0cm}, clip, width=0.23\textwidth]{{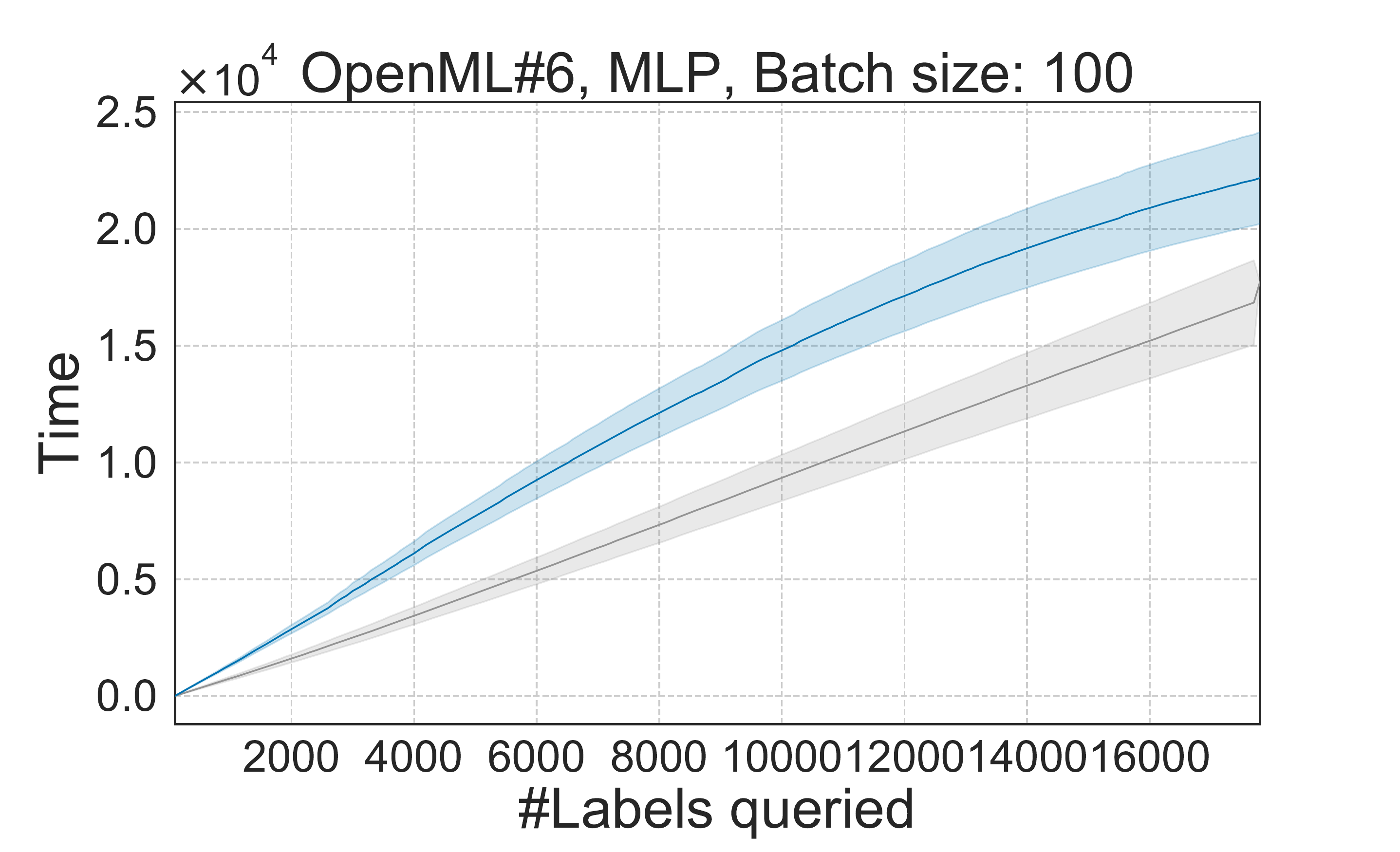}}
  \includegraphics[trim={0.3cm 0cm 2.5cm 0cm}, clip, width=0.23\textwidth]{{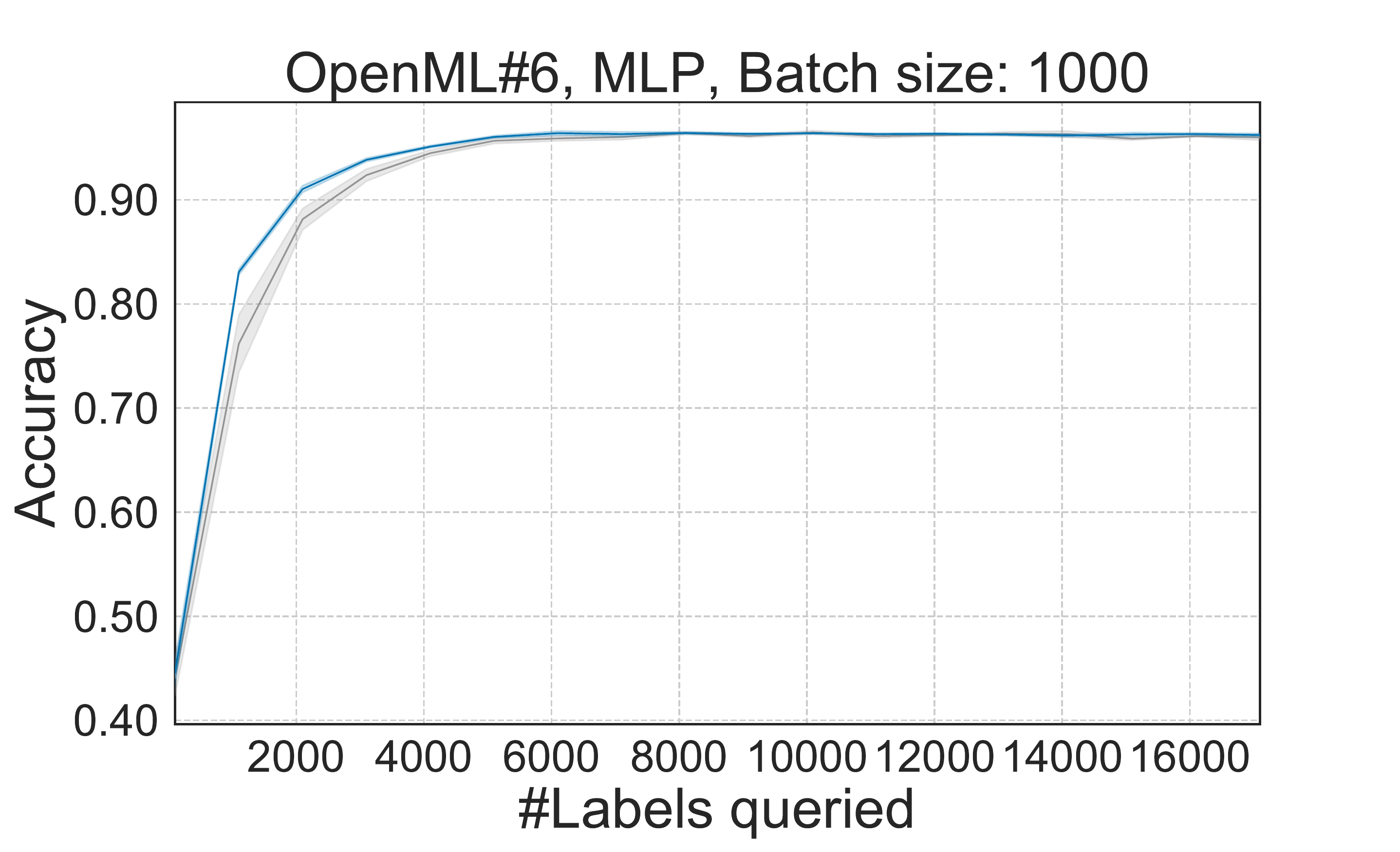}}
  \includegraphics[trim={0.3cm 0cm 2.5cm 0cm}, clip, width=0.23\textwidth]{{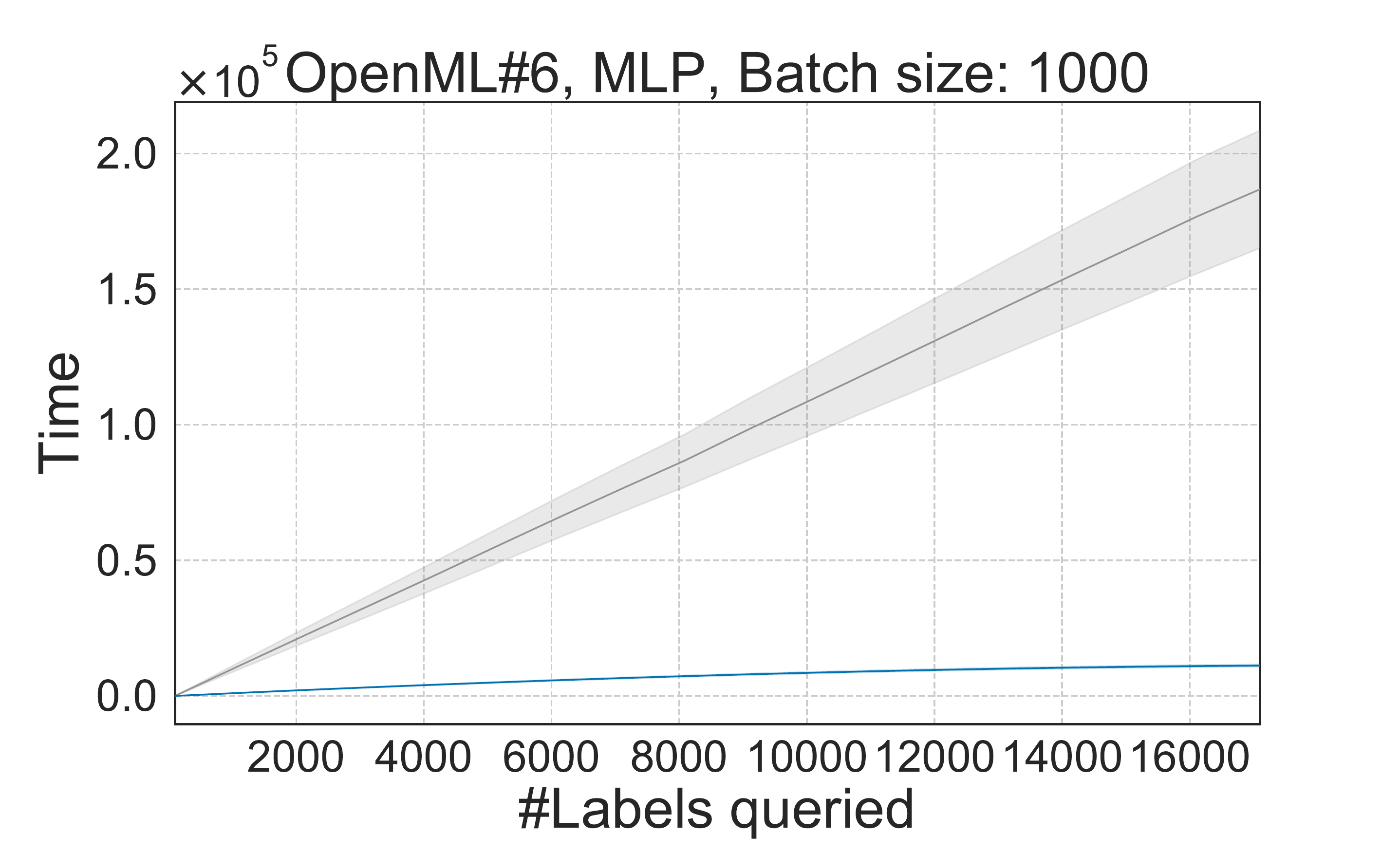}}
  \\
  \centering
  \begin{subfigure}[b]{0.32\linewidth}
    \includegraphics[trim={0cm 0cm 0cm 0cm}, clip, width=\textwidth]{figs/dpp_learning_curves/legend.pdf}
  \end{subfigure}
\caption{Learning curves and running times for OpenML \#6 with MLP.}
\label{fig:comp-6}
\end{figure}

\begin{figure}
  \centering
  \includegraphics[trim={0.3cm 0cm 2.5cm 0cm}, clip, width=0.24\textwidth]{{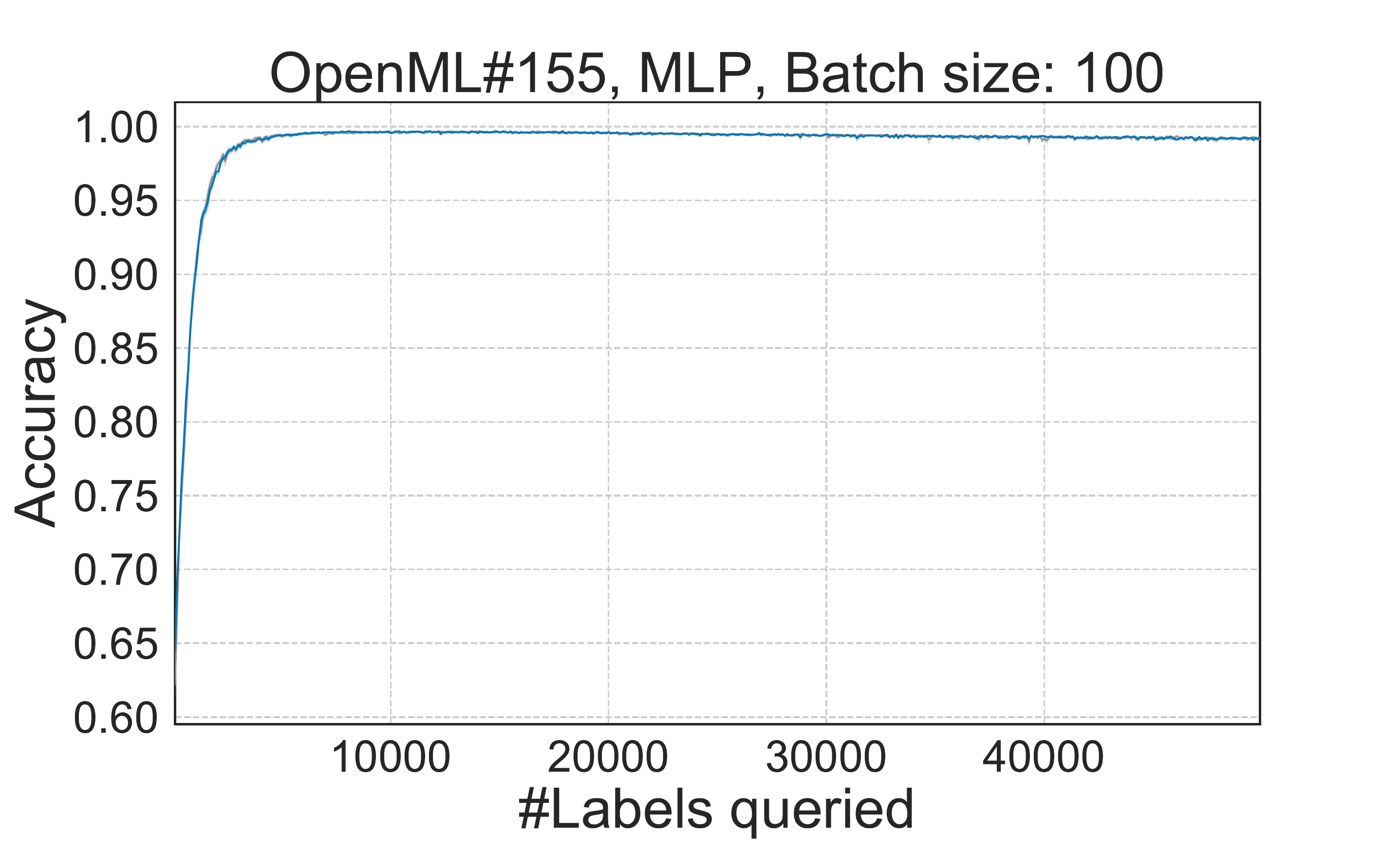}}
  \includegraphics[trim={0.3cm 0cm 2.5cm 0cm}, clip, width=0.24\textwidth]{{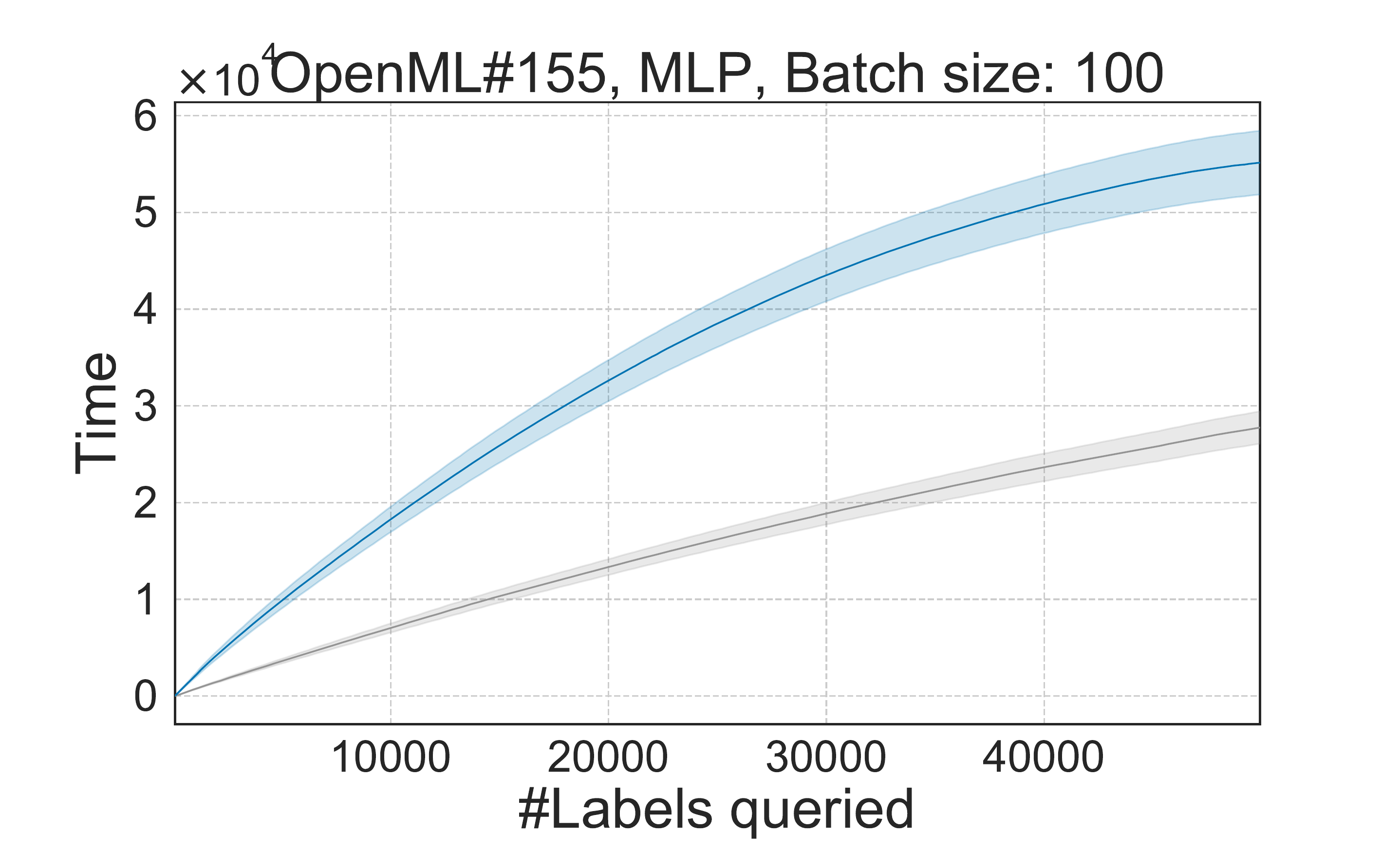}}
  \includegraphics[trim={0.3cm 0cm 2.5cm 0cm}, clip, width=0.24\textwidth]{{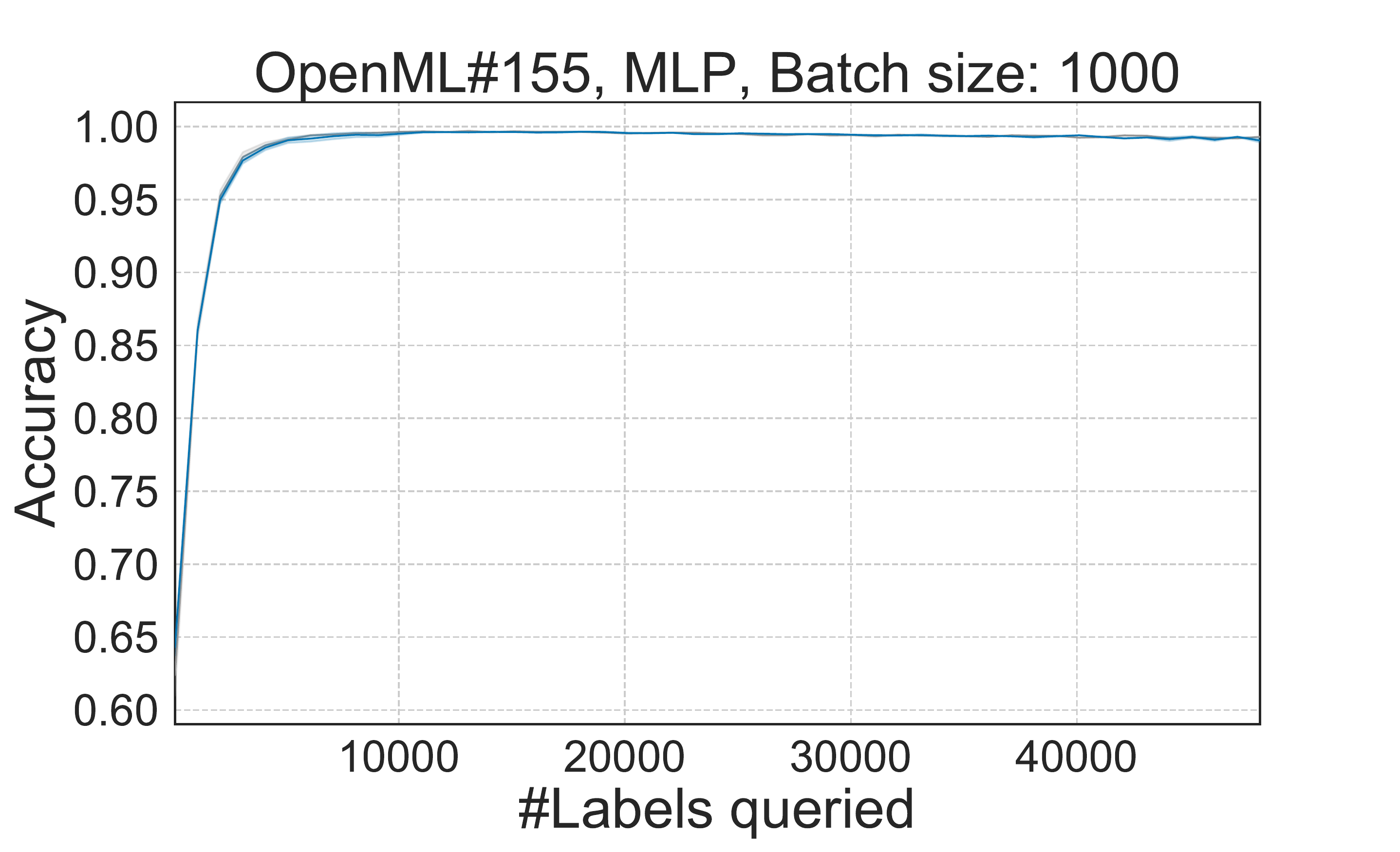}}
  \includegraphics[trim={0.3cm 0cm 2.5cm 0cm}, clip, width=0.24\textwidth]{{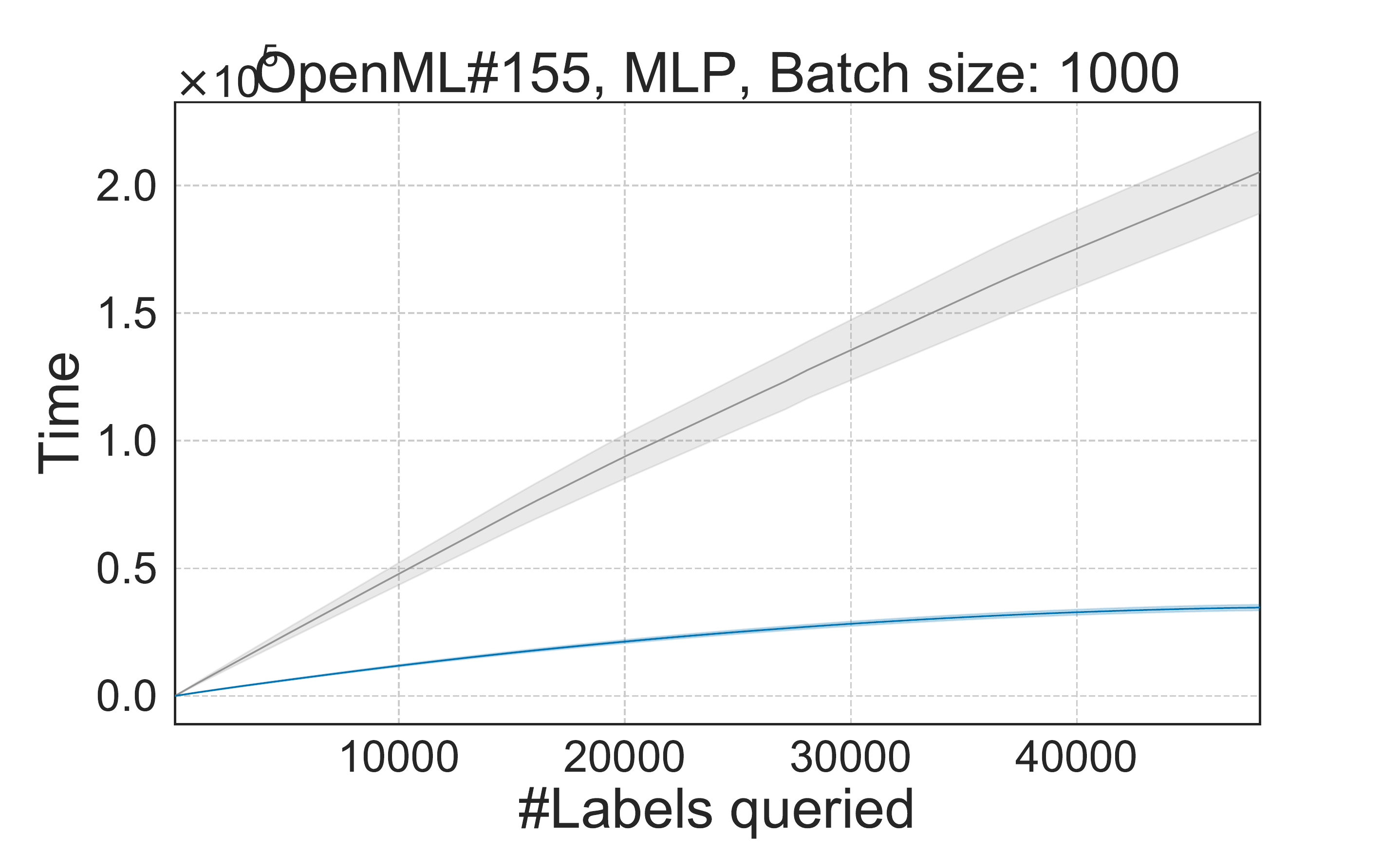}}
  \\
  \centering
  \begin{subfigure}[b]{0.32\linewidth}
    \includegraphics[trim={0cm 0cm 0cm 0cm}, clip, width=\textwidth]{figs/dpp_learning_curves/legend.pdf}
  \end{subfigure}
\caption{Learning curves and running times for OpenML \#155 with MLP.}
\end{figure}

\begin{figure}
  \centering
  \includegraphics[trim={0.3cm 0cm 2.5cm 0cm}, clip, width=0.24\textwidth]{{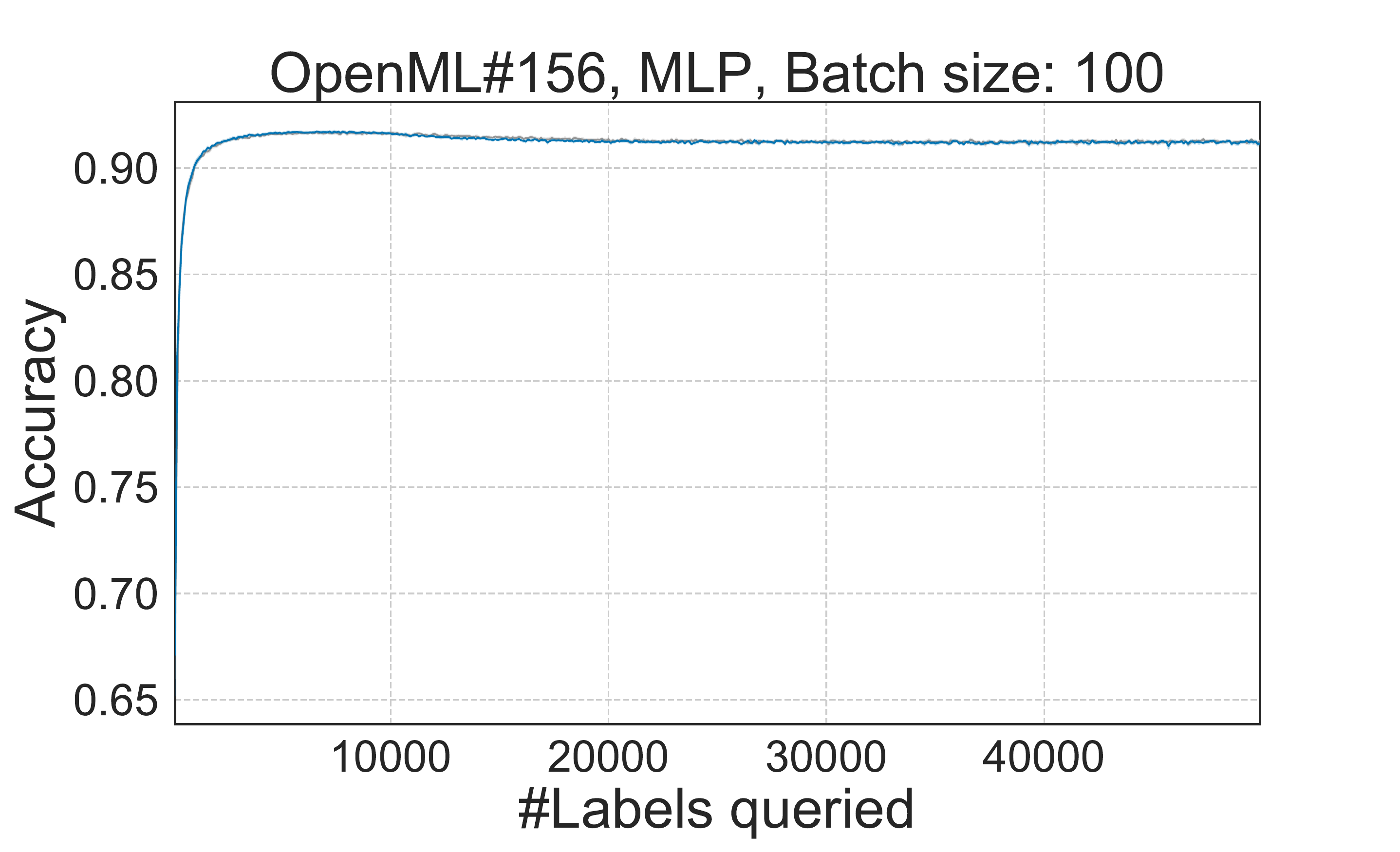}}
  \includegraphics[trim={0.3cm 0cm 2.5cm 0cm}, clip, width=0.24\textwidth]{{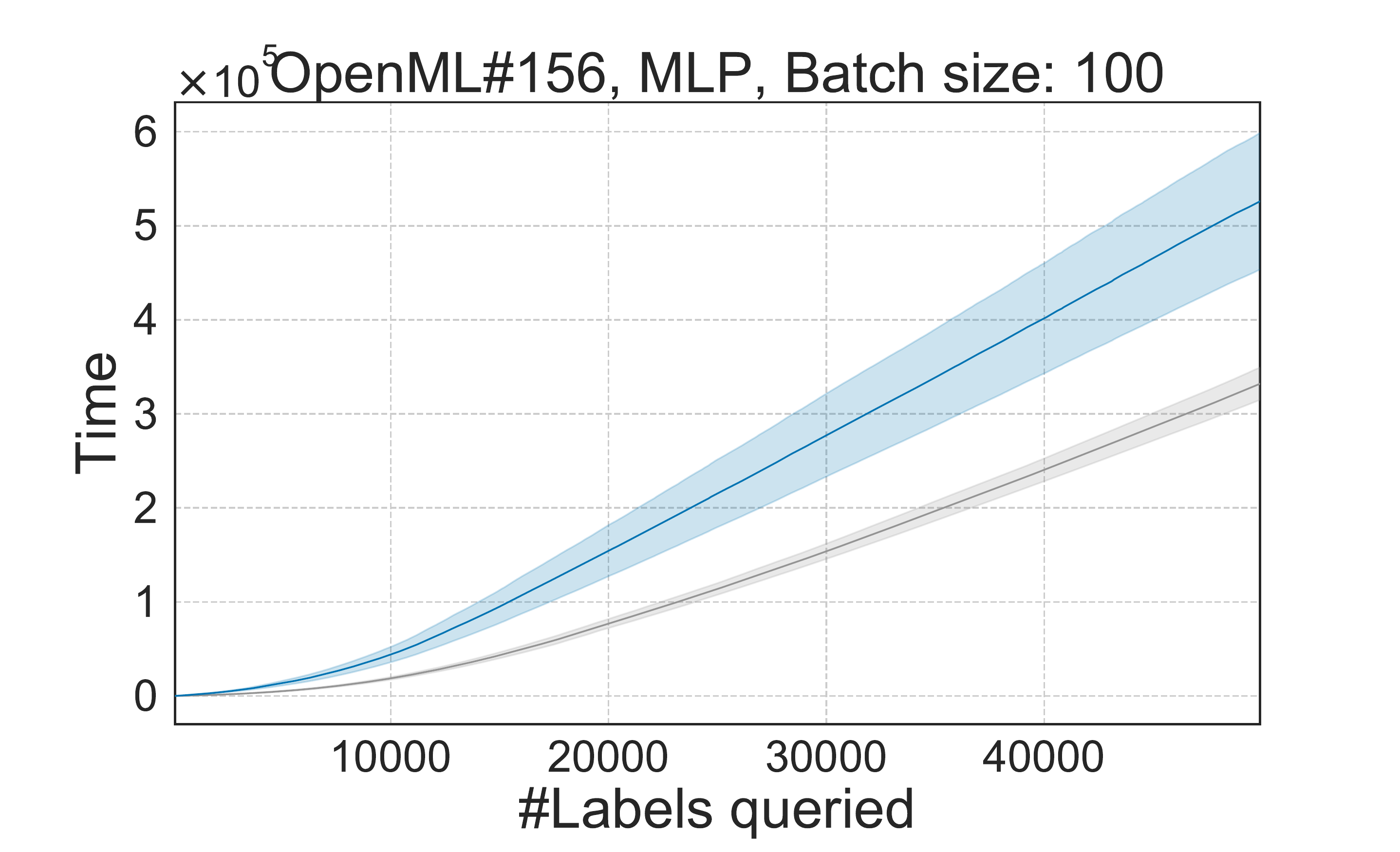}}
  \includegraphics[trim={0.3cm 0cm 2.5cm 0cm}, clip, width=0.24\textwidth]{{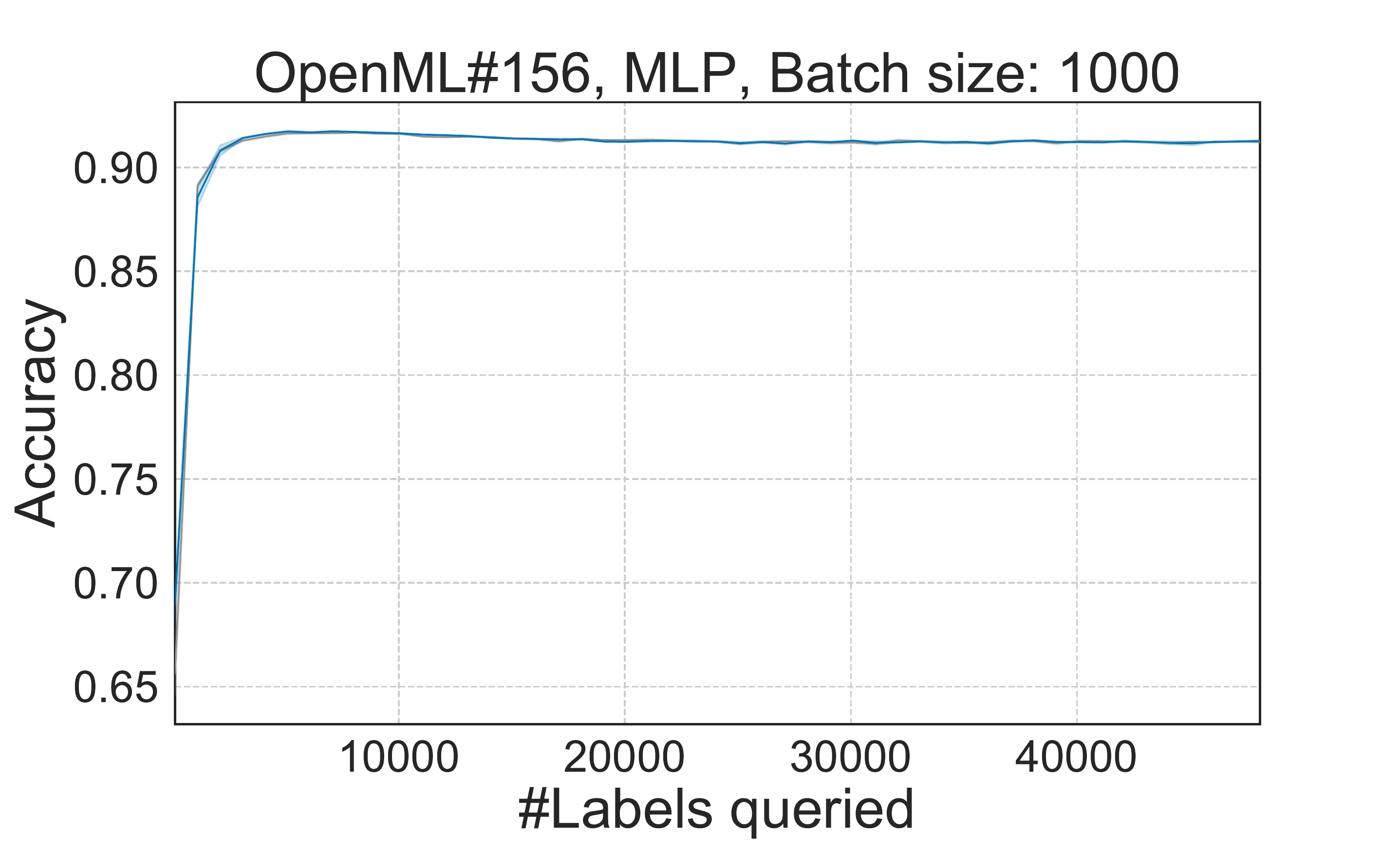}}
  \includegraphics[trim={0.3cm 0cm 2.5cm 0cm}, clip, width=0.24\textwidth]{{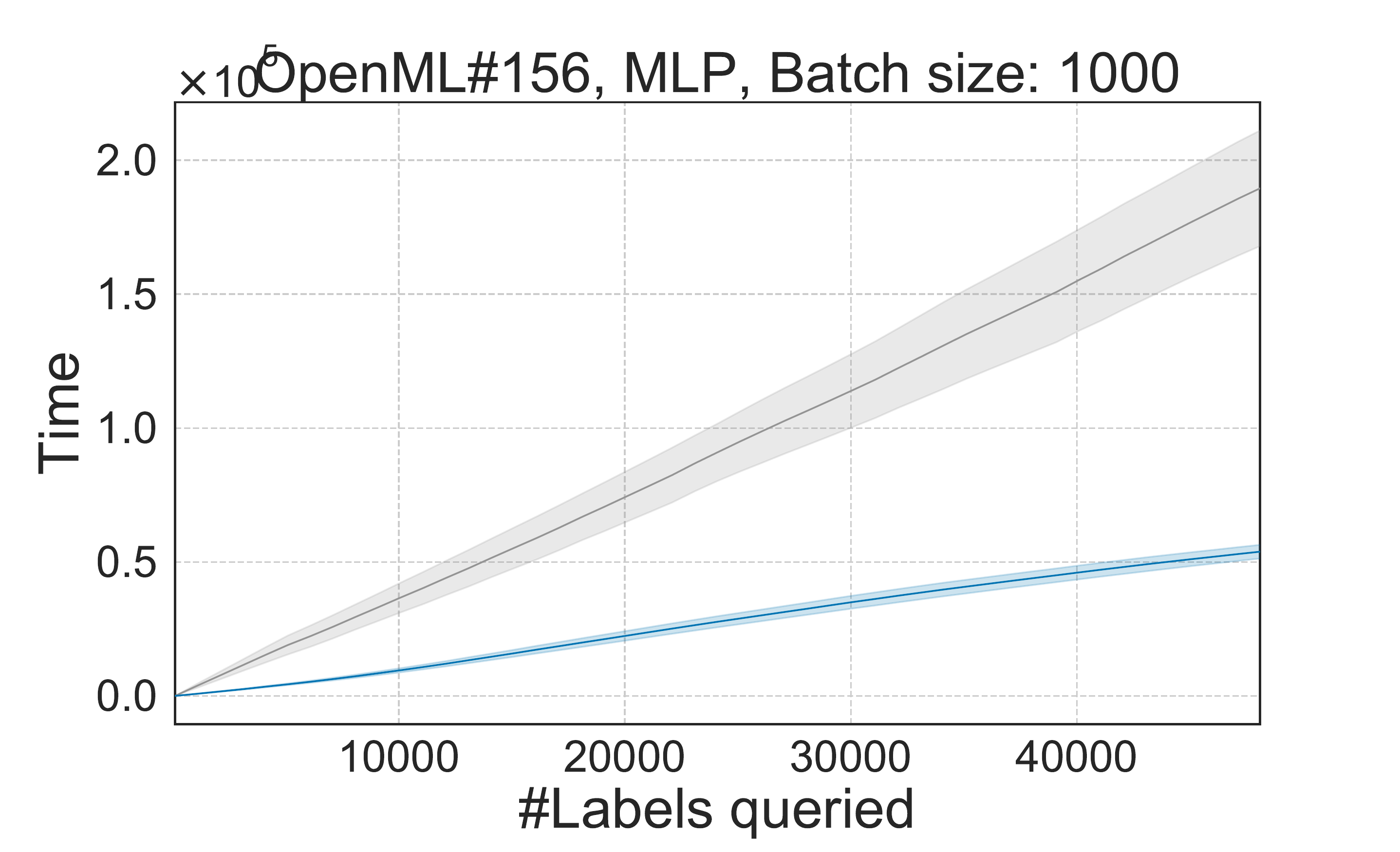}}
  \\
  \centering
  \begin{subfigure}[b]{0.32\linewidth}
    \includegraphics[trim={0cm 0cm 0cm 0cm}, clip, width=\textwidth]{figs/dpp_learning_curves/legend.pdf}
  \end{subfigure}
\caption{Learning curves and running times for OpenML \#156 with MLP.}
\end{figure}

\begin{figure}
  \centering
  \includegraphics[trim={0.3cm 0cm 2.5cm 0cm}, clip, width=0.24\textwidth]{{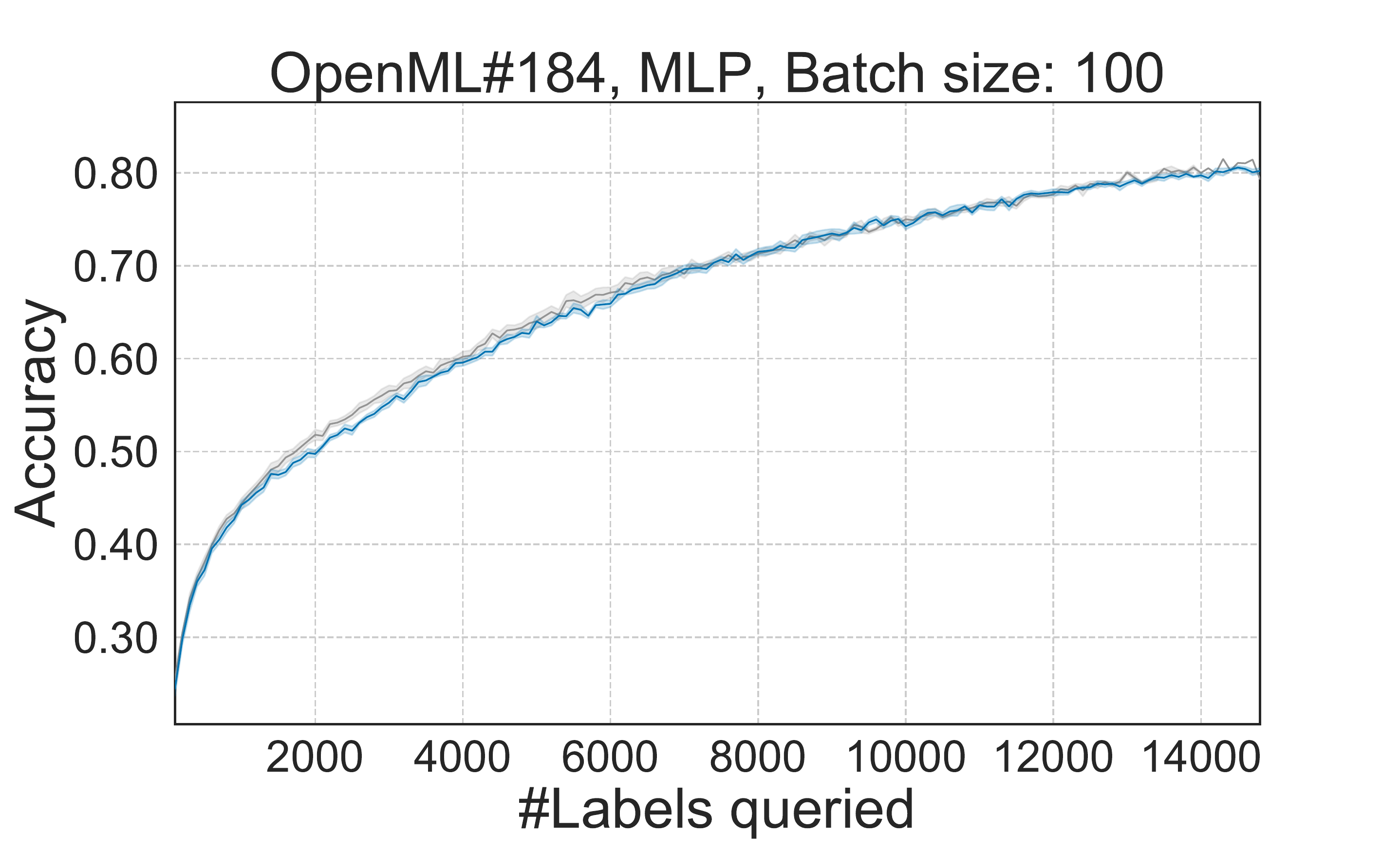}}
  \includegraphics[trim={0.3cm 0cm 2.5cm 0cm}, clip, width=0.24\textwidth]{{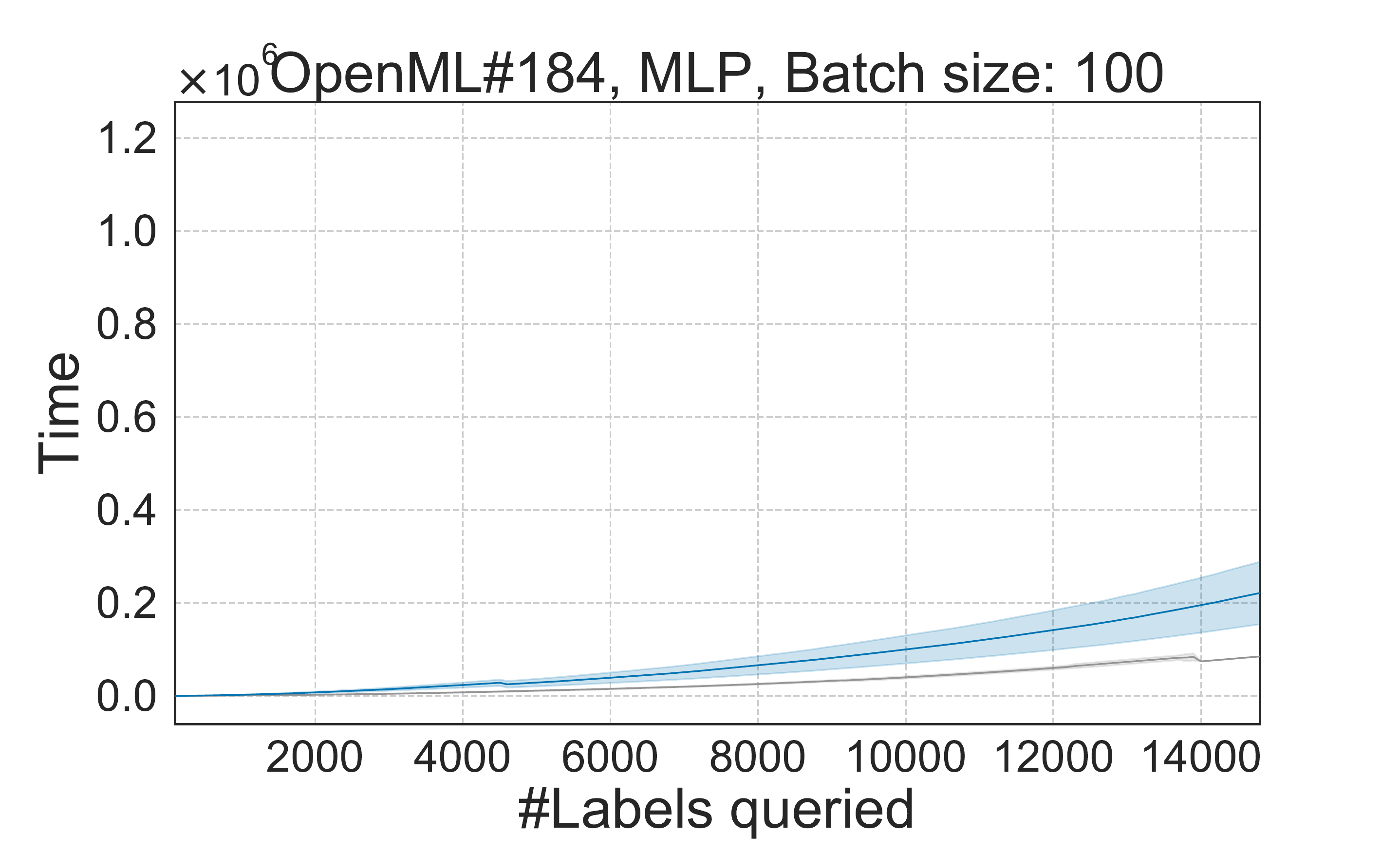}}
  \includegraphics[trim={0.3cm 0cm 2.5cm 0cm}, clip, width=0.24\textwidth]{{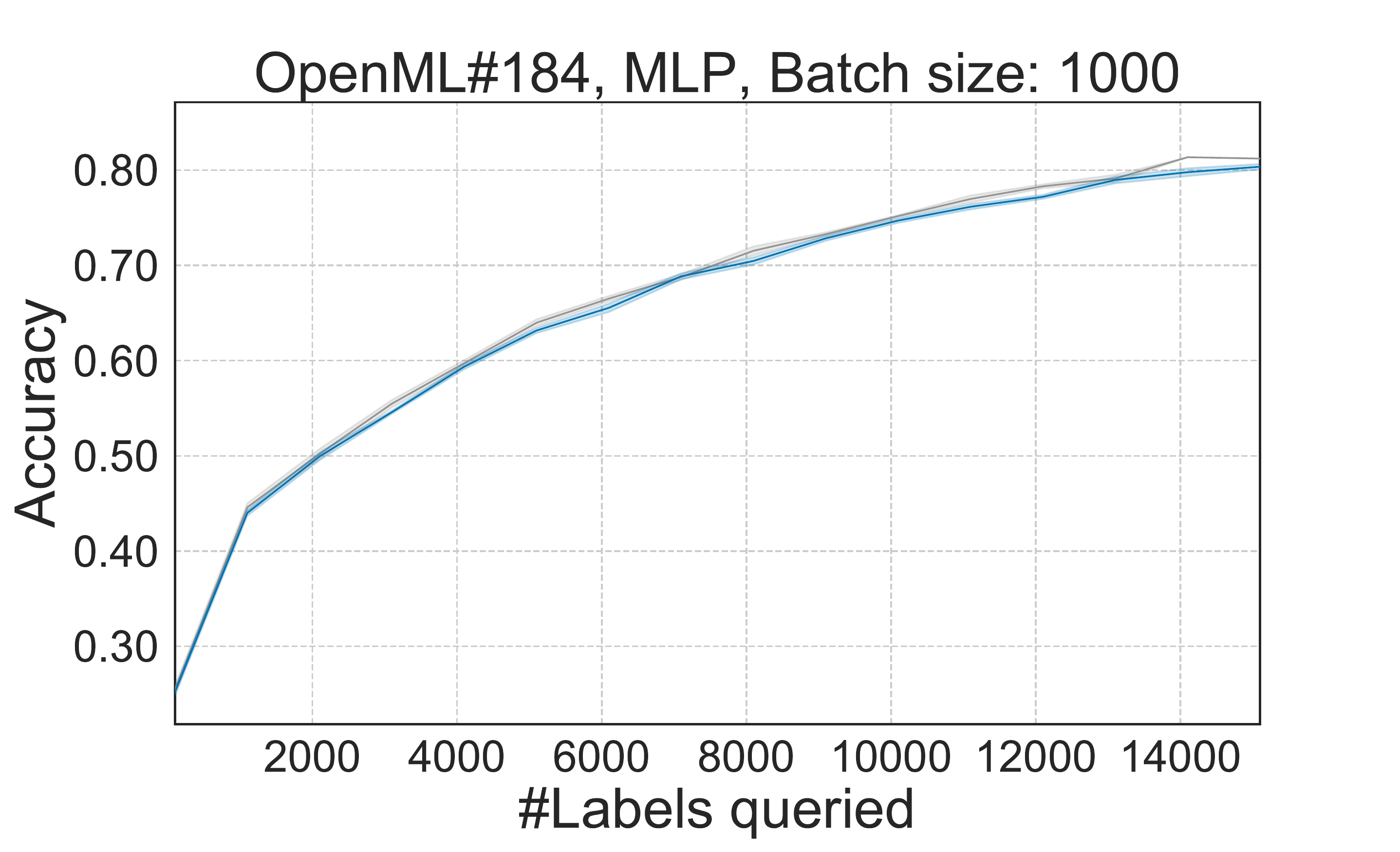}}
  \includegraphics[trim={0.3cm 0cm 2.5cm 0cm}, clip, width=0.24\textwidth]{{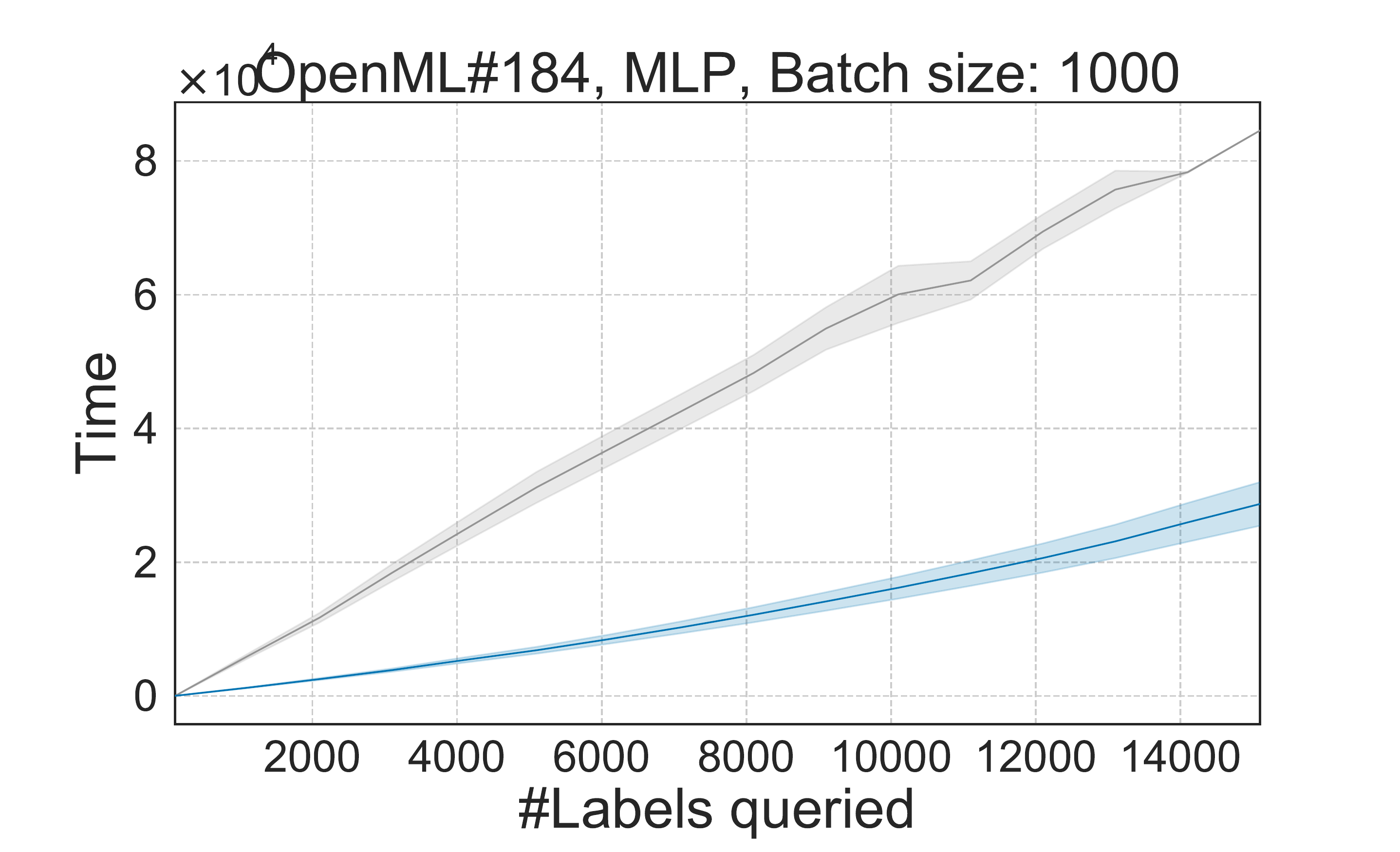}}
  \\
  \centering
  \begin{subfigure}[b]{0.32\linewidth}
    \includegraphics[trim={0cm 0cm 0cm 0cm}, clip, width=\textwidth]{figs/dpp_learning_curves/legend.pdf}
  \end{subfigure}
\caption{Learning curves and running times for OpenML \#184 with MLP.}
\end{figure}

\begin{figure}
  \centering
  \includegraphics[trim={0.3cm 0cm 2.5cm 0cm}, clip, width=0.24\textwidth]{{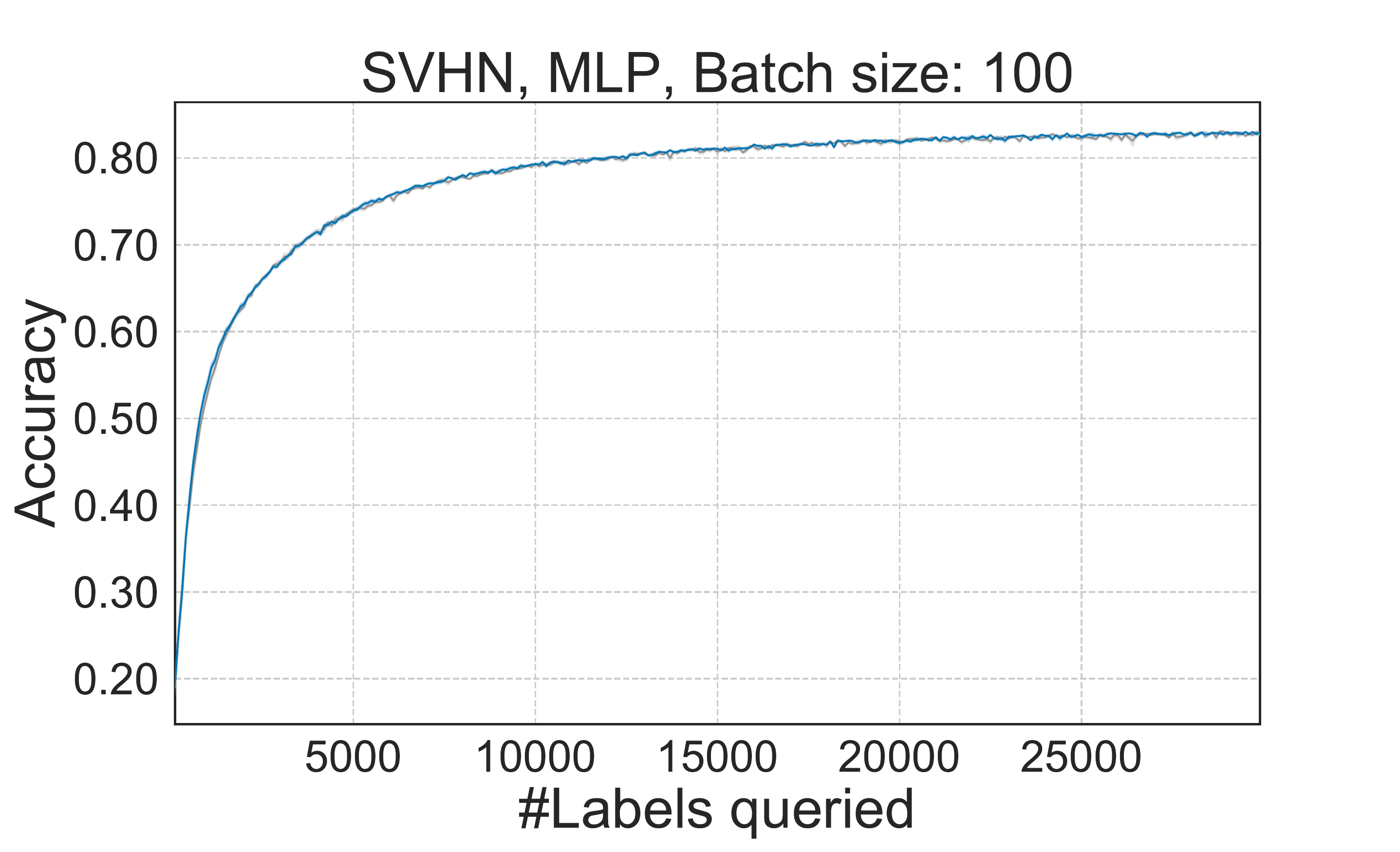}}
  \includegraphics[trim={0.3cm 0cm 2.5cm 0cm}, clip, width=0.24\textwidth]{{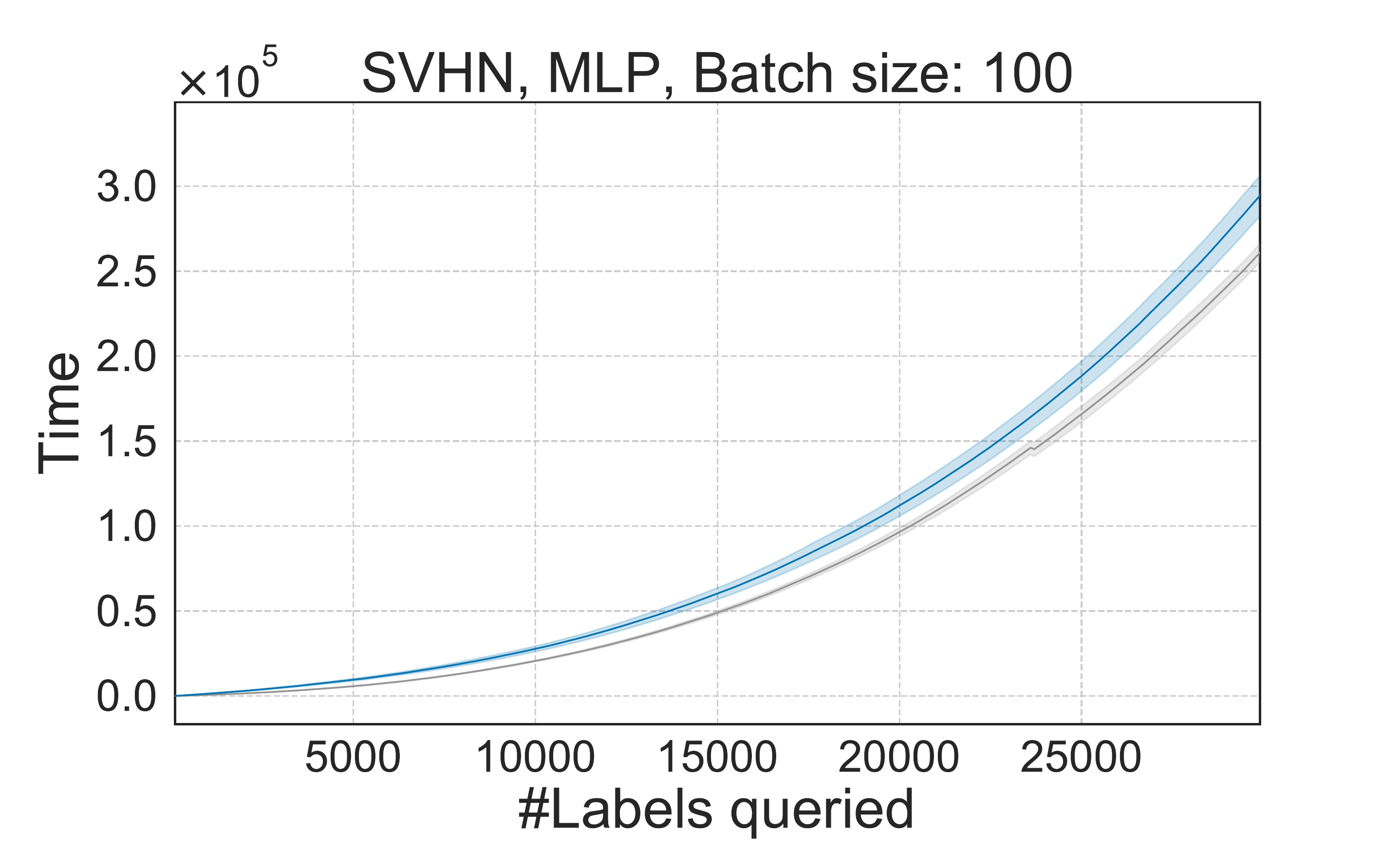}}
  \includegraphics[trim={0.3cm 0cm 2.5cm 0cm}, clip, width=0.24\textwidth]{{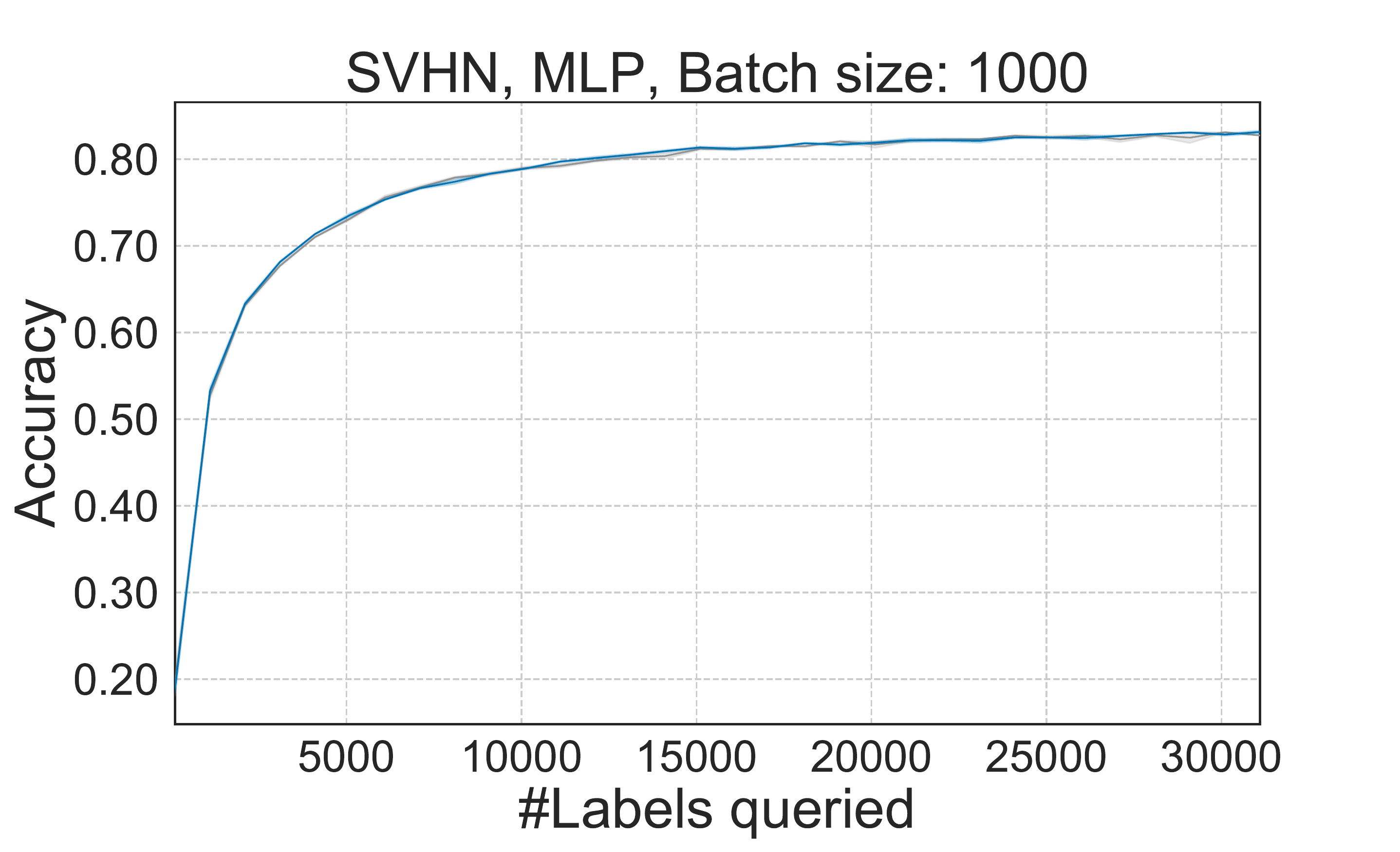}}
  \includegraphics[trim={0.3cm 0cm 2.5cm 0cm}, clip, width=0.24\textwidth]{{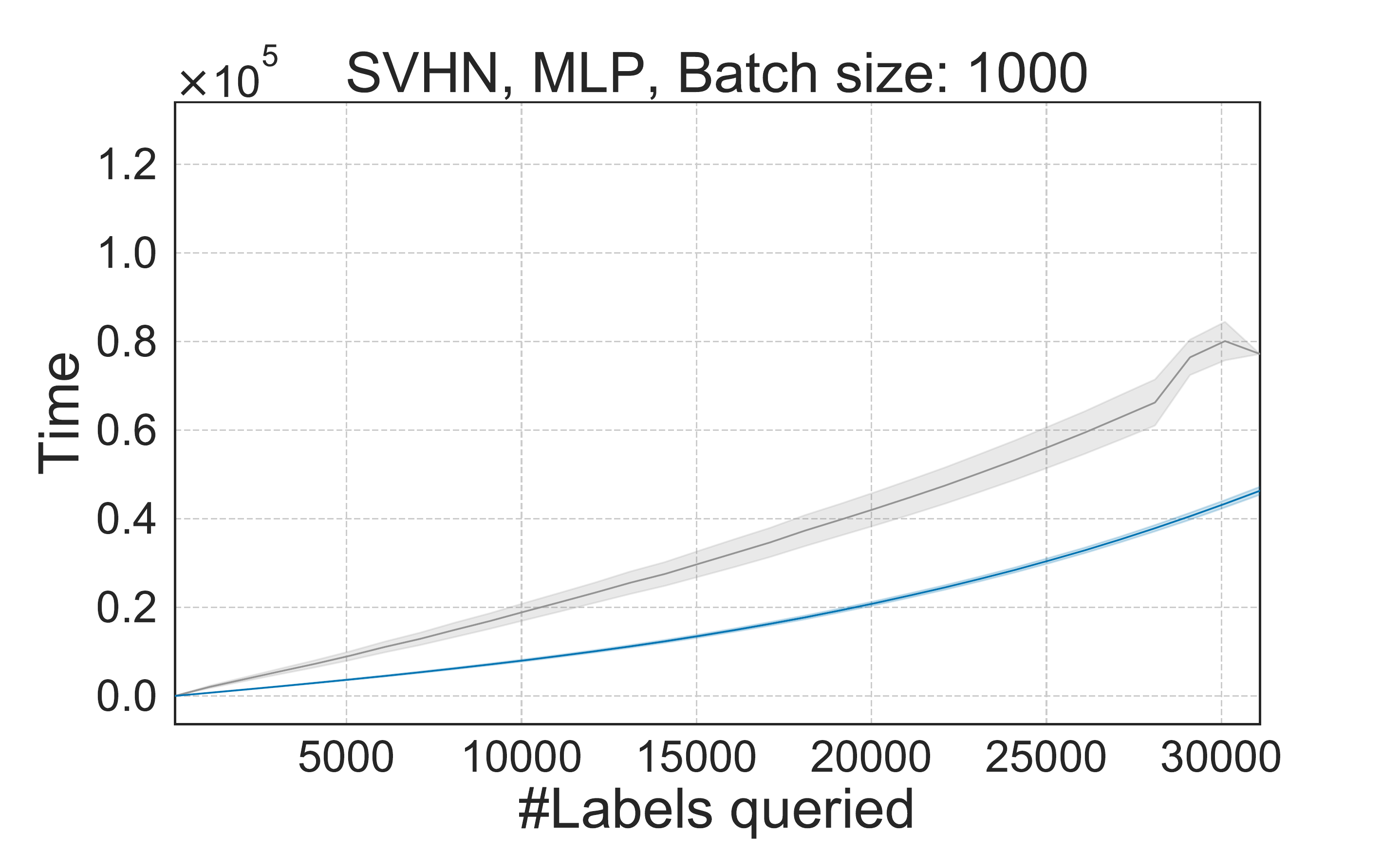}}
  \\
  \centering
  \includegraphics[trim={0.3cm 0cm 2.5cm 0cm}, clip, width=0.24\textwidth]{{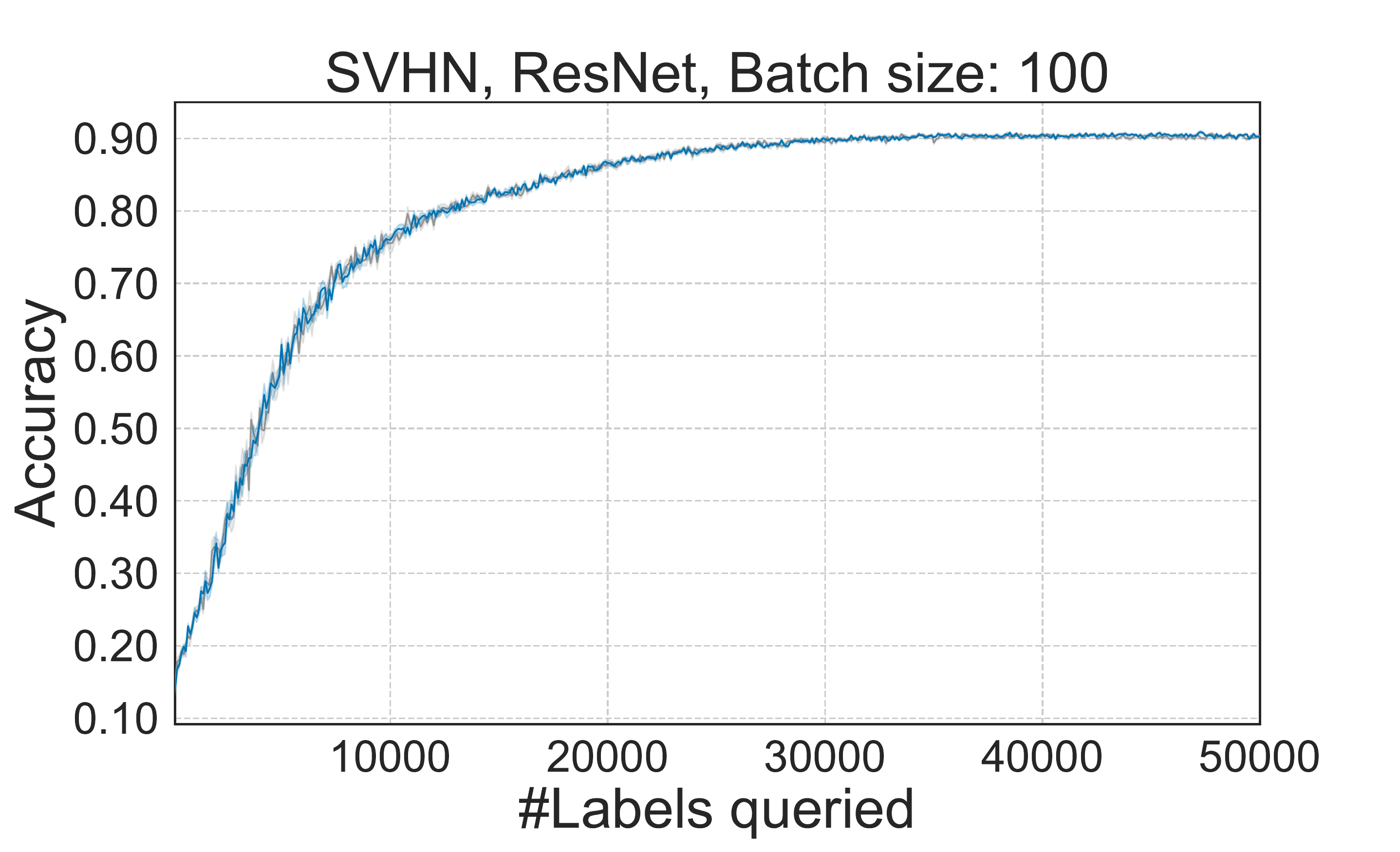}}
  \includegraphics[trim={0.3cm 0cm 2.5cm 0cm}, clip, width=0.24\textwidth]{{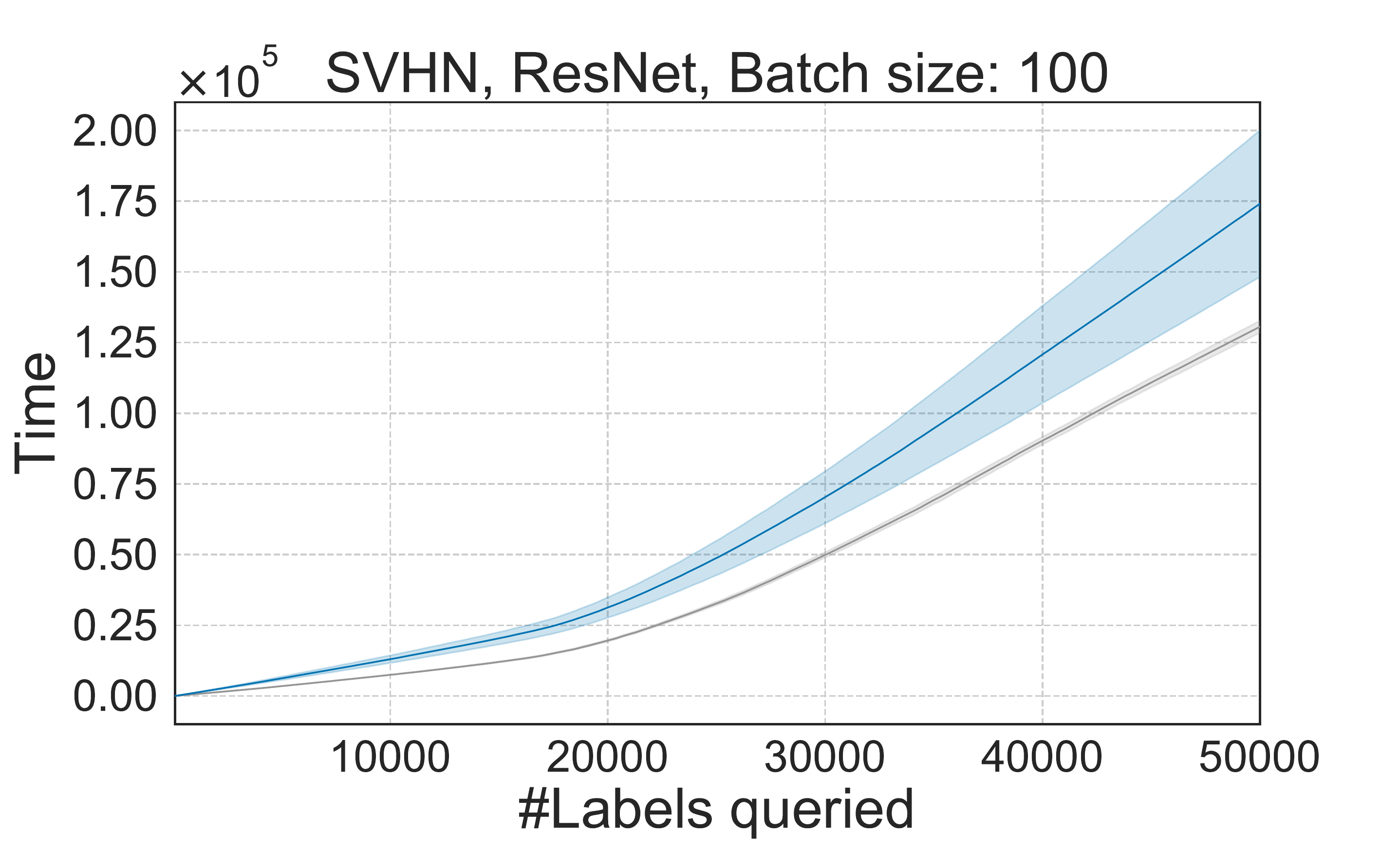}}
  \includegraphics[trim={0.3cm 0cm 2.5cm 0cm}, clip, width=0.24\textwidth]{{figs/dpp_learning_curves/comp_Accuracy_Data=_SVHN__Model=_rn__nQuery=_1000__TrainAug=_0___.pdf}}
  \includegraphics[trim={0.3cm 0cm 2.5cm 0cm}, clip, width=0.24\textwidth]{{figs/dpp_learning_curves/comp_Time_Data=_SVHN__Model=_rn__nQuery=_1000__TrainAug=_0___.pdf}}
  \\
  \centering
  \begin{subfigure}[b]{0.32\linewidth}
    \includegraphics[trim={0cm 0cm 0cm 0cm}, clip, width=\textwidth]{figs/dpp_learning_curves/legend.pdf}
  \end{subfigure}
\caption{Learning curves and running times for SVHN with MLP and ResNet.}
\end{figure}

\begin{figure}
  \centering
  \includegraphics[trim={0.3cm 0cm 2.5cm 0cm}, clip, width=0.24\textwidth]{{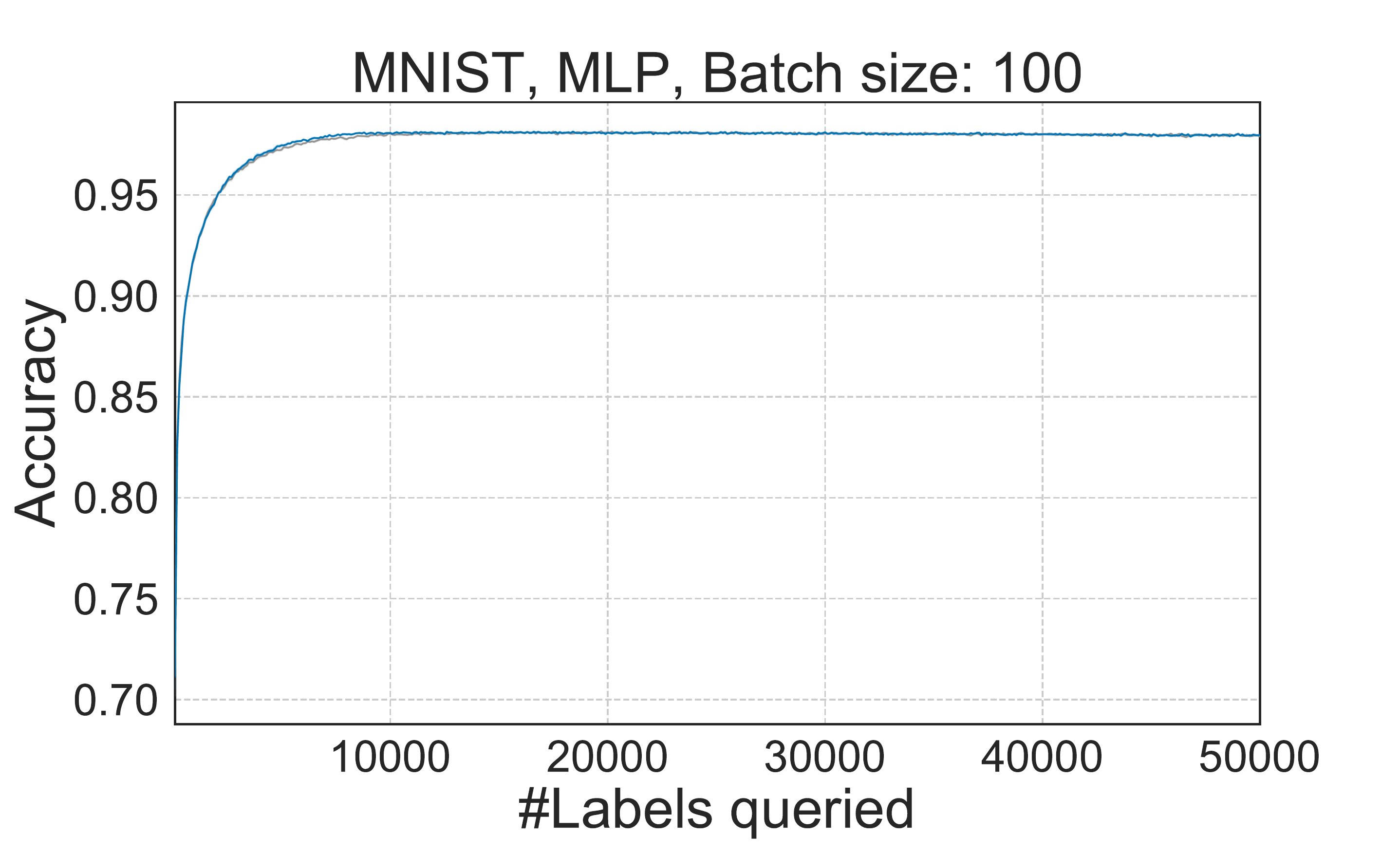}}
  \includegraphics[trim={0.3cm 0cm 2.5cm 0cm}, clip, width=0.24\textwidth]{{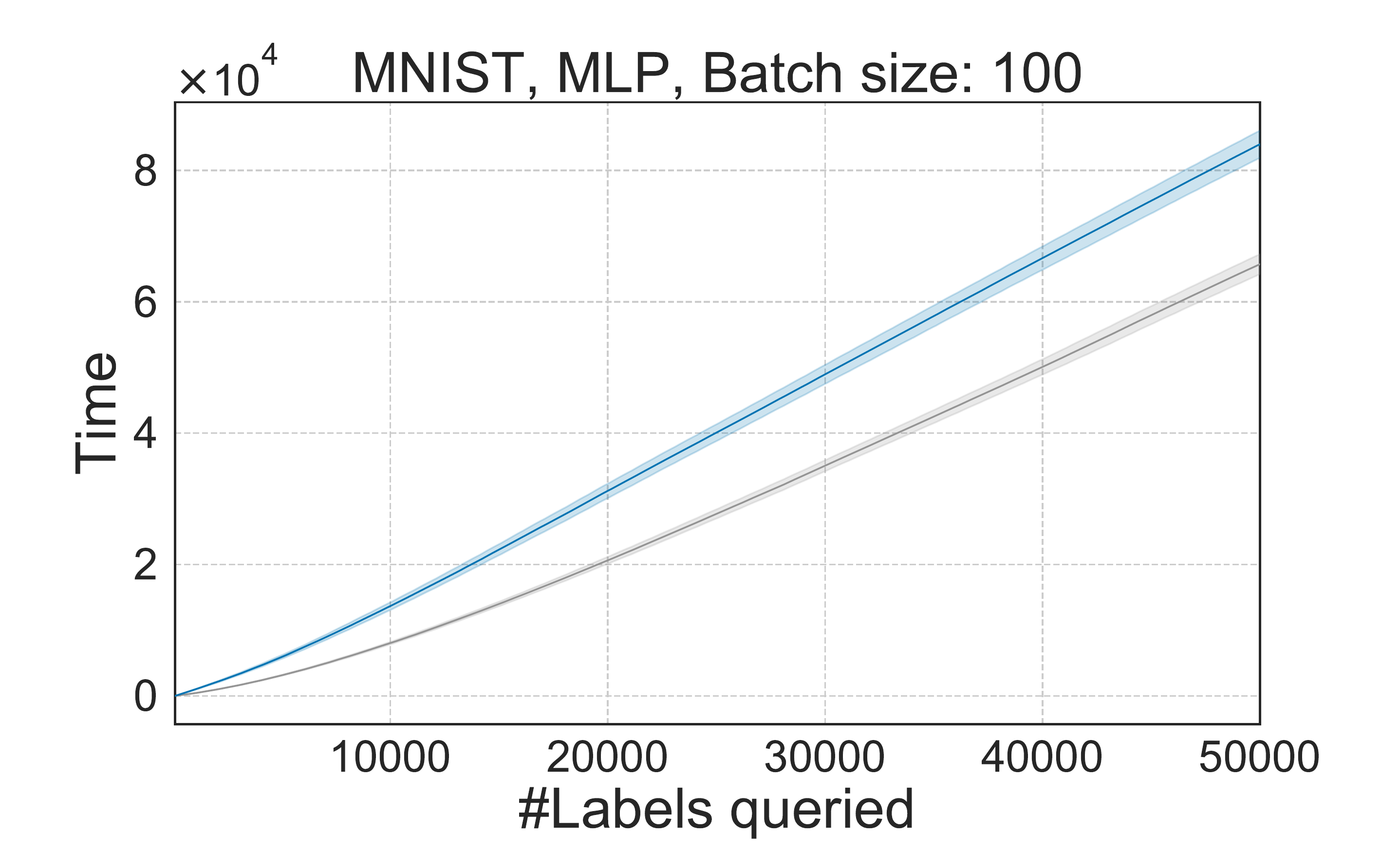}}
  \includegraphics[trim={0.3cm 0cm 2.5cm 0cm}, clip, width=0.24\textwidth]{{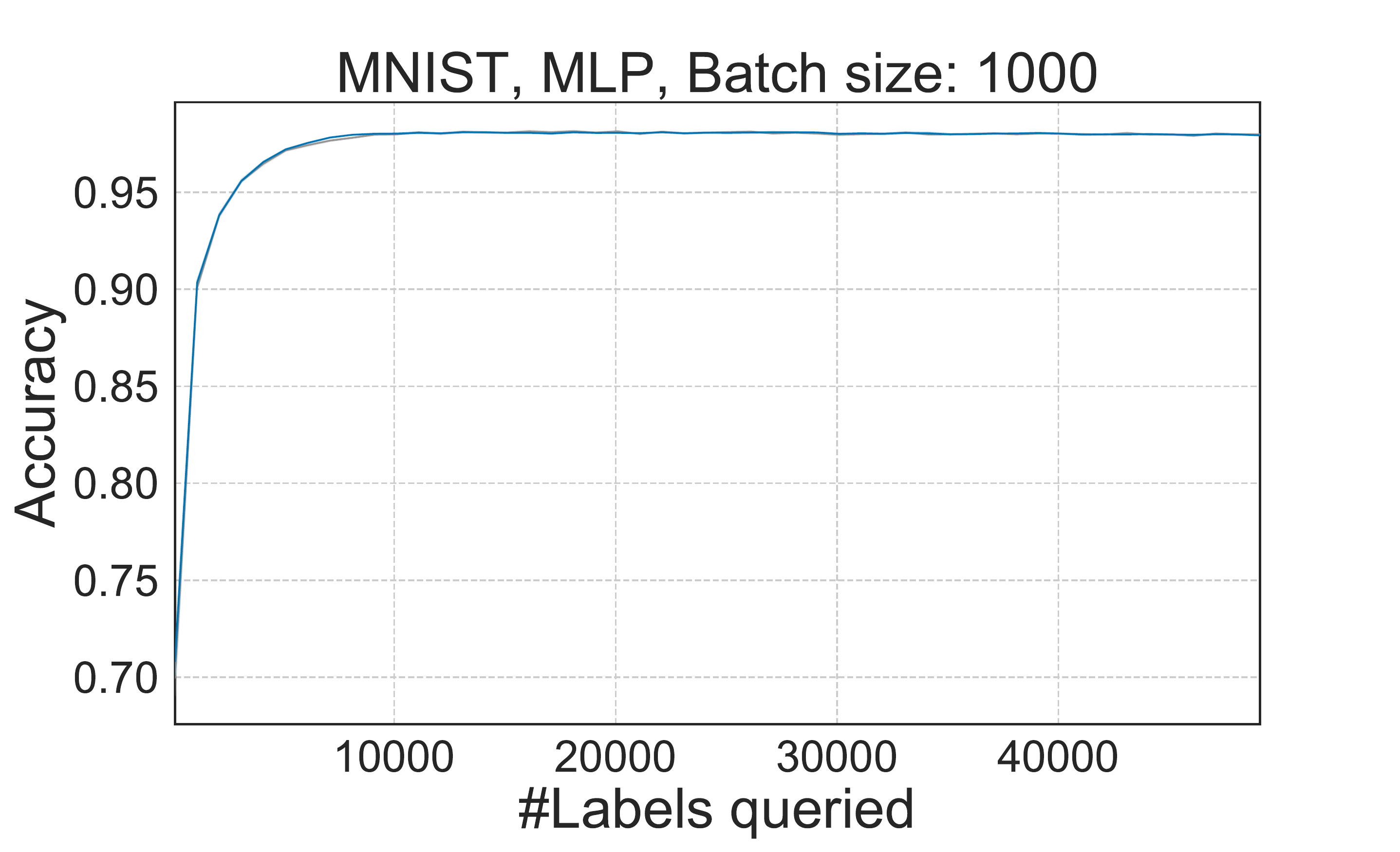}}
  \includegraphics[trim={0.3cm 0cm 2.5cm 0cm}, clip, width=0.24\textwidth]{{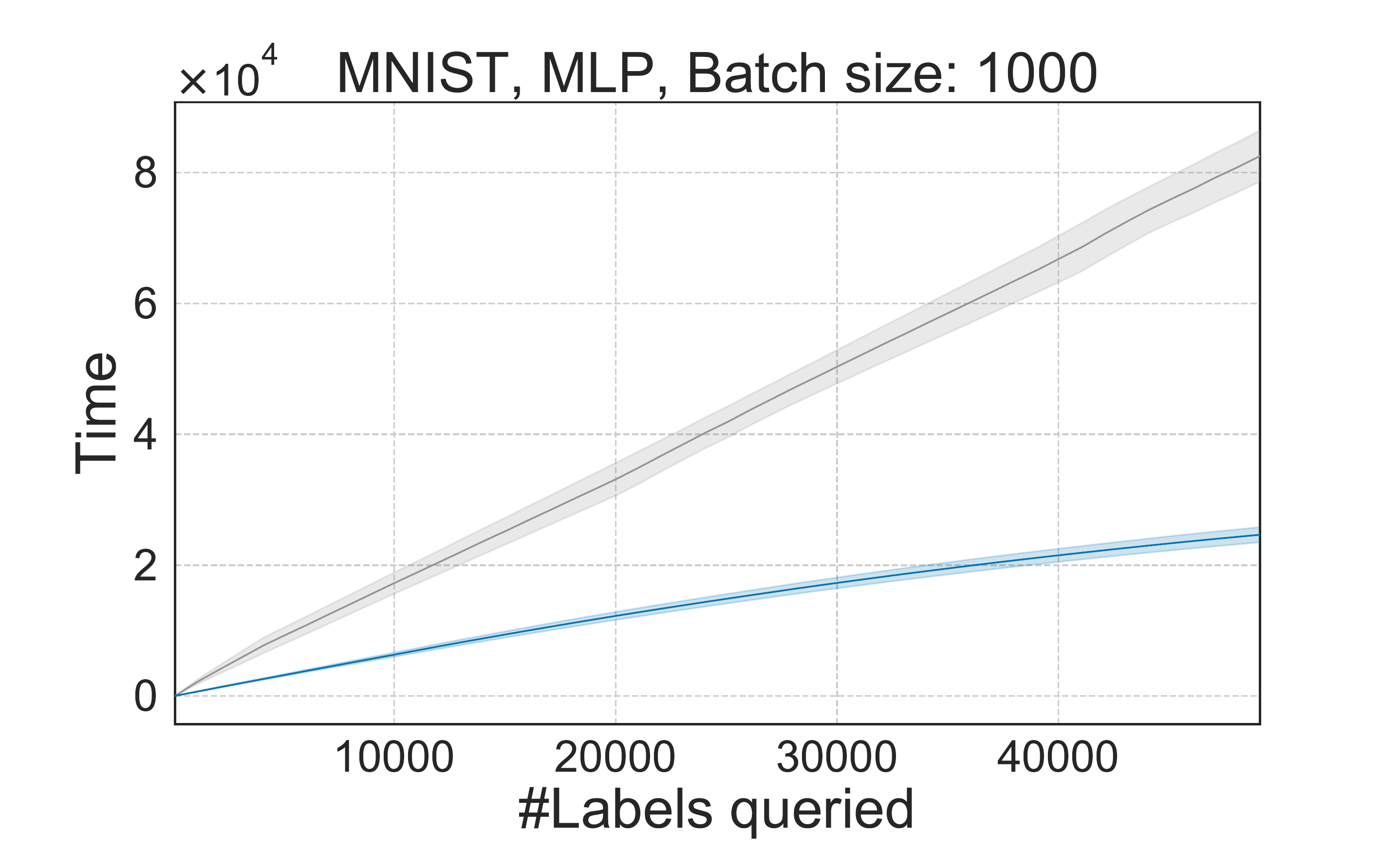}}
  \\
  \centering
  \begin{subfigure}[b]{0.32\linewidth}
    \includegraphics[trim={0cm 0cm 0cm 0cm}, clip, width=\textwidth]{figs/dpp_learning_curves/legend.pdf}
  \end{subfigure}
\caption{Learning curves and running times for MNIST with MLP.}
\end{figure}

\begin{figure}
  \centering
  \includegraphics[trim={0.3cm 0cm 2.5cm 0cm}, clip, width=0.24\textwidth]{{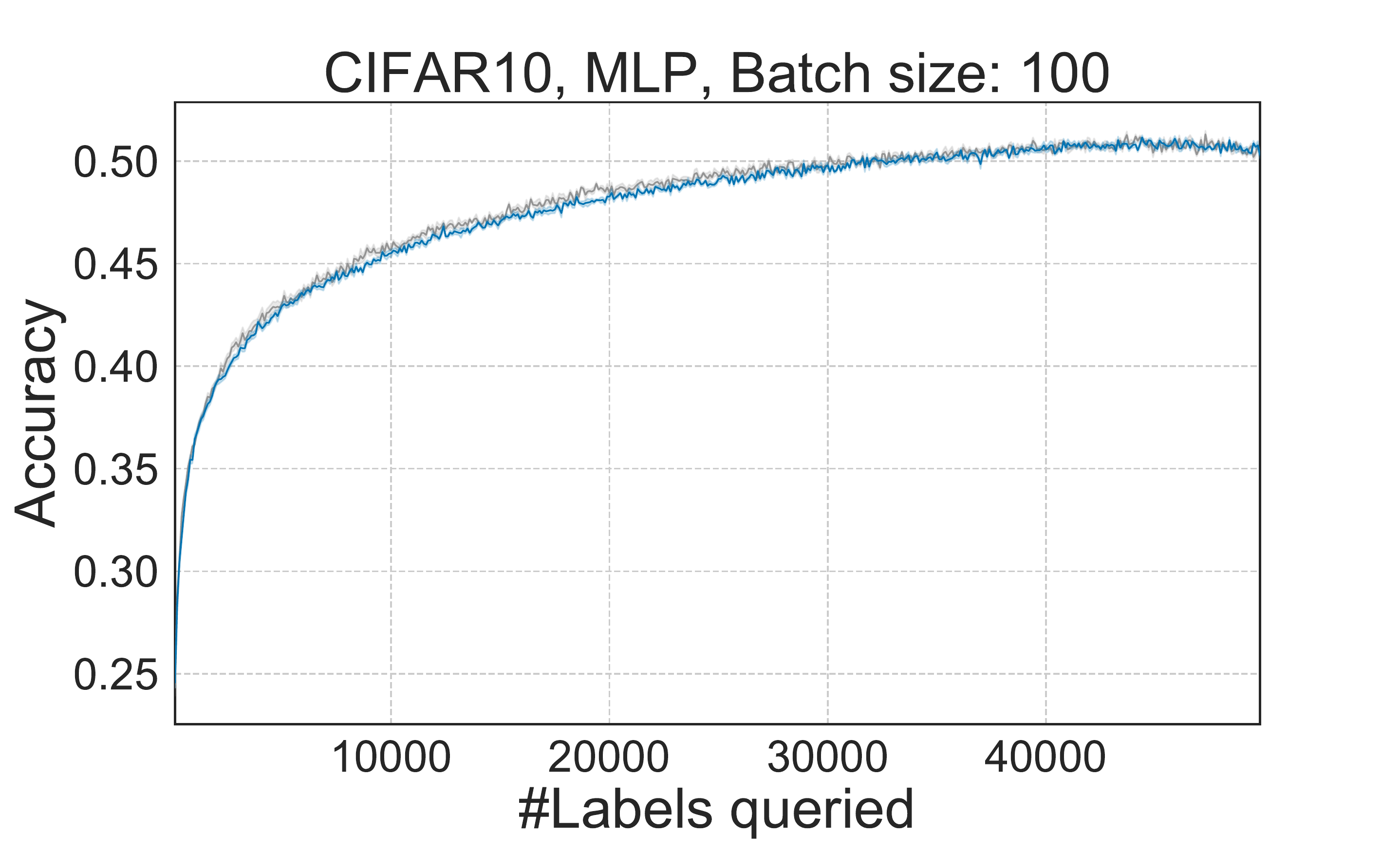}}
  \includegraphics[trim={0.3cm 0cm 2.5cm 0cm}, clip, width=0.24\textwidth]{{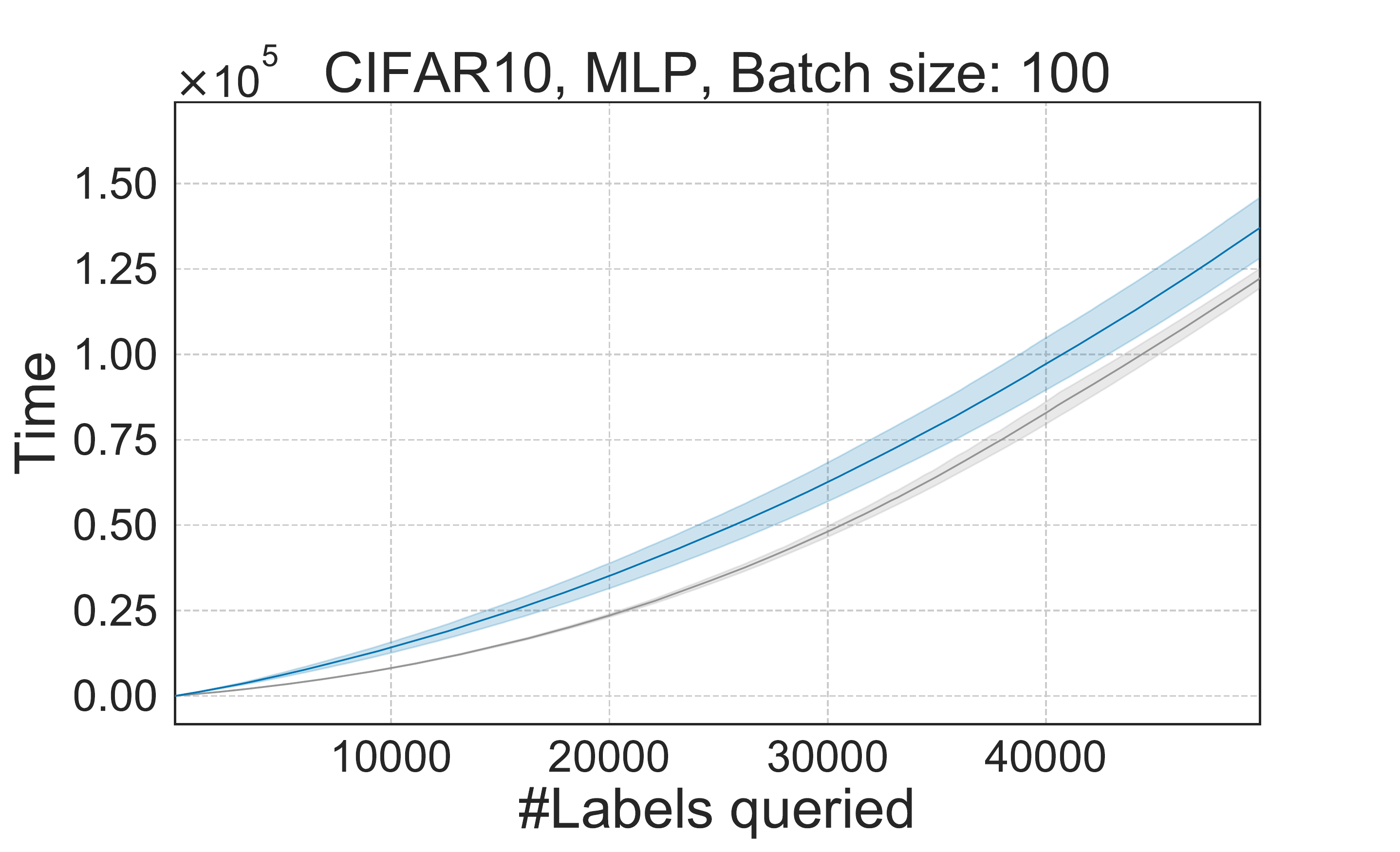}}
  \includegraphics[trim={0.3cm 0cm 2.5cm 0cm}, clip, width=0.24\textwidth]{{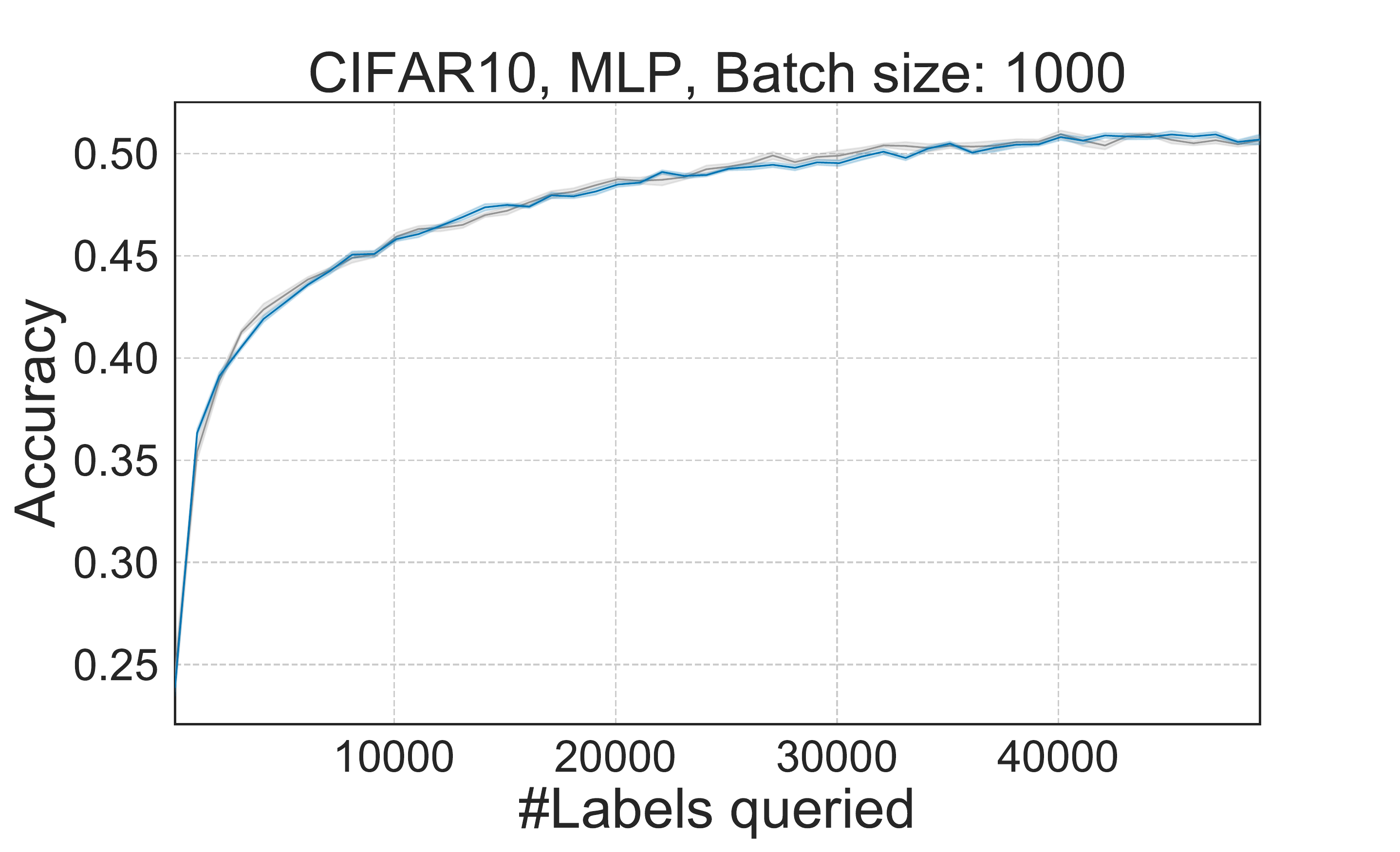}}
  \includegraphics[trim={0.3cm 0cm 2.5cm 0cm}, clip, width=0.24\textwidth]{{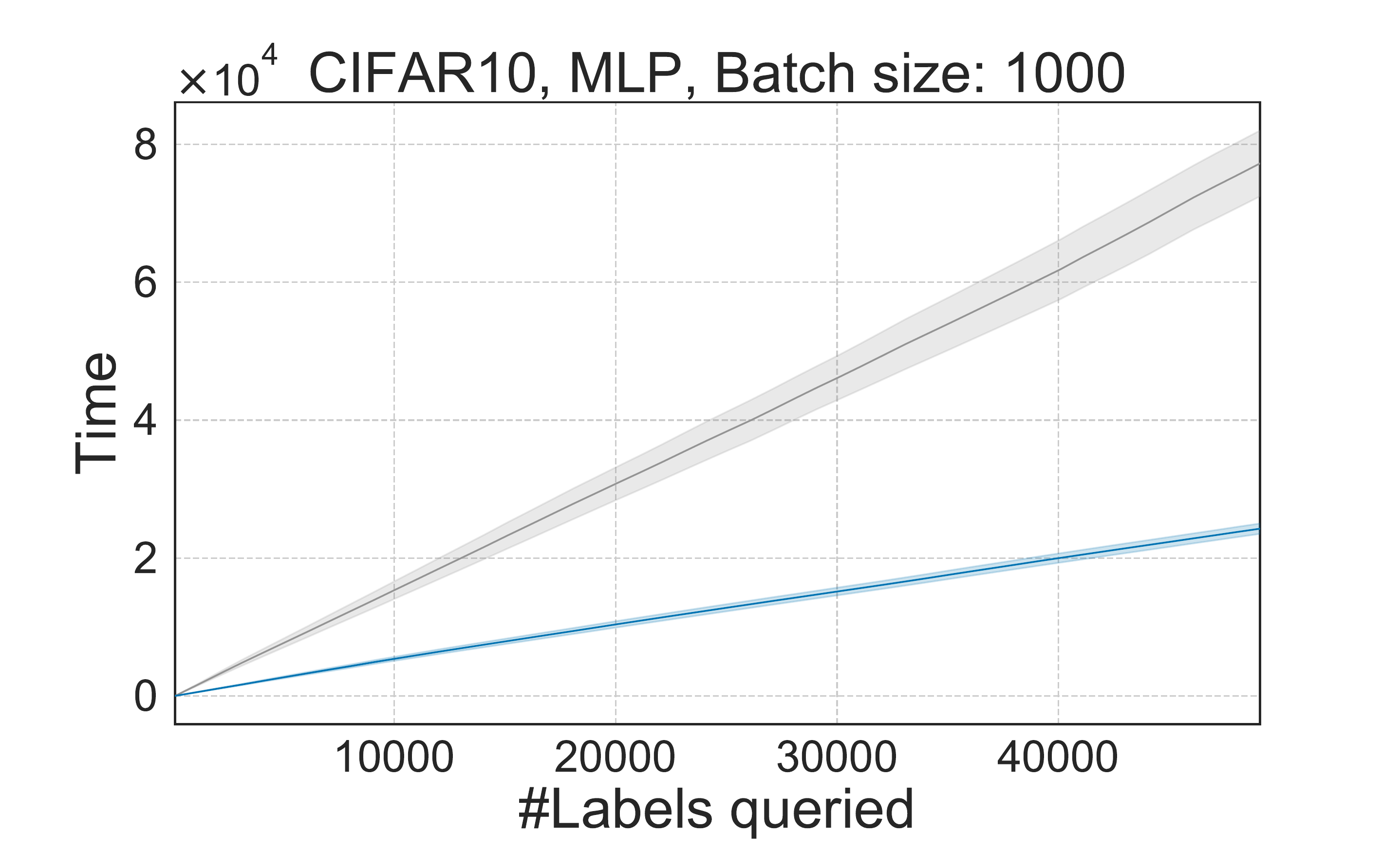}}
  \\
  \centering
 \includegraphics[trim={0.3cm 0cm 2.5cm 0cm}, clip, width=0.24\textwidth]{{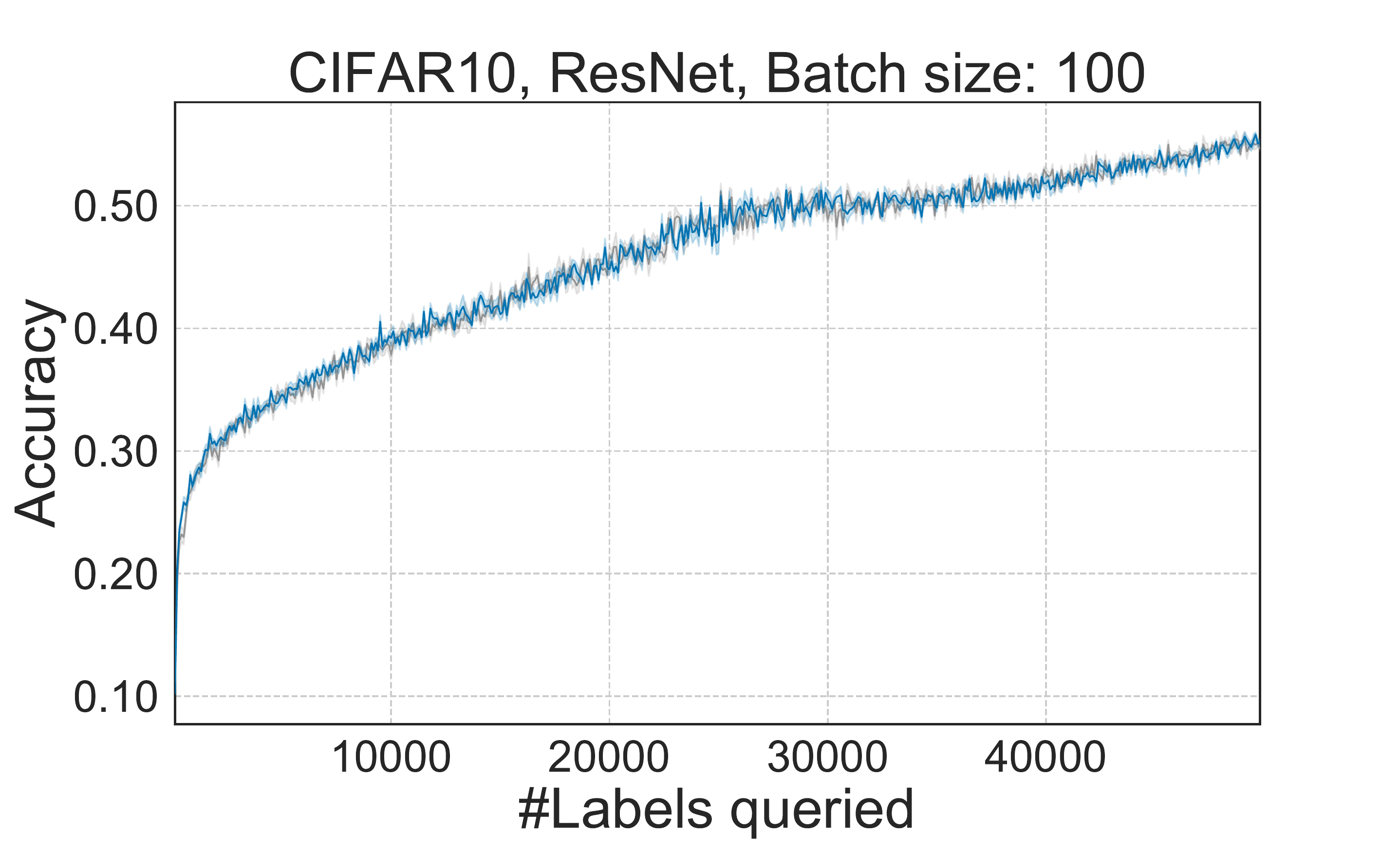}}
  \includegraphics[trim={0.3cm 0cm 2.5cm 0cm}, clip, width=0.24\textwidth]{{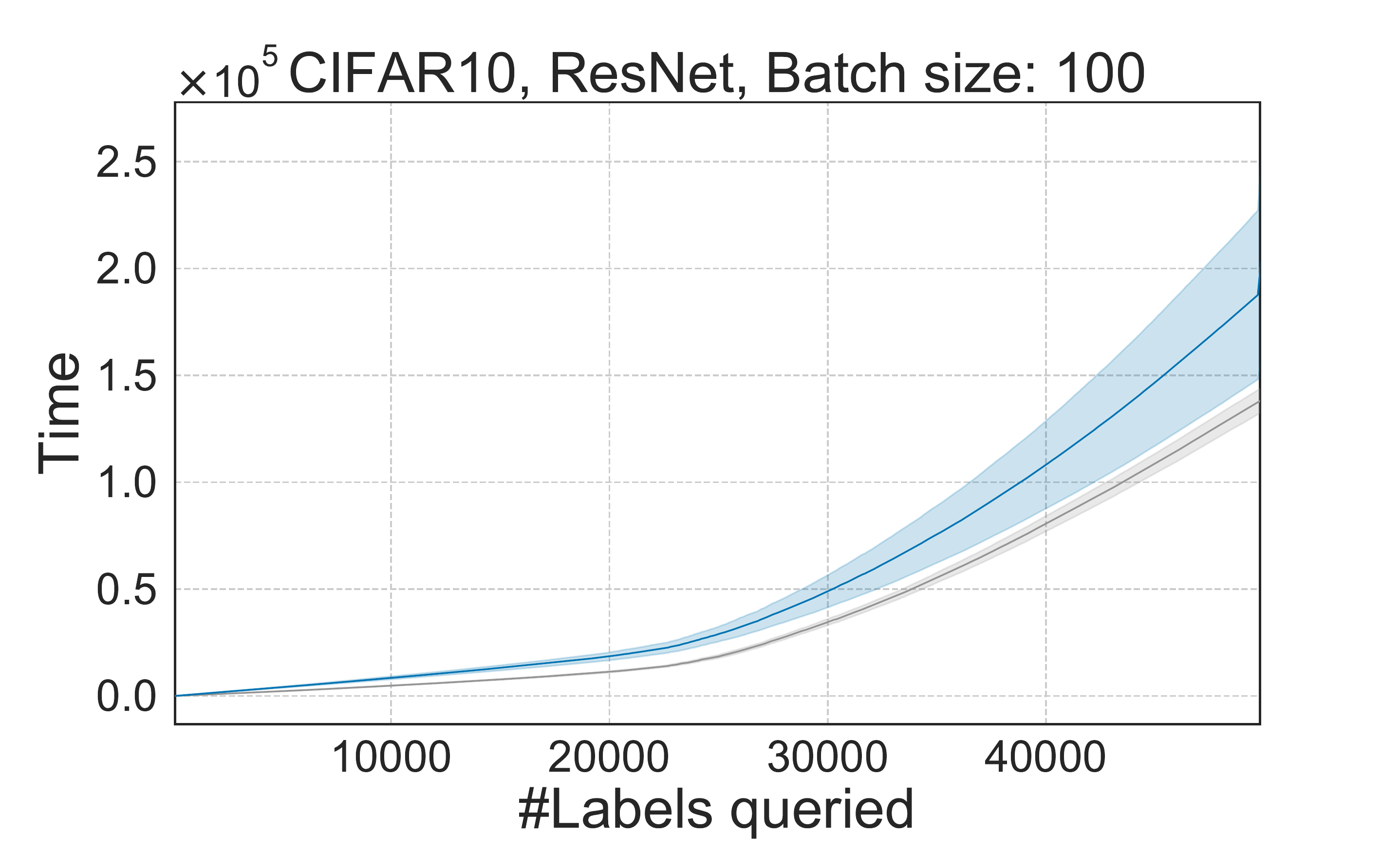}}
  \includegraphics[trim={0.3cm 0cm 2.5cm 0cm}, clip, width=0.24\textwidth]{{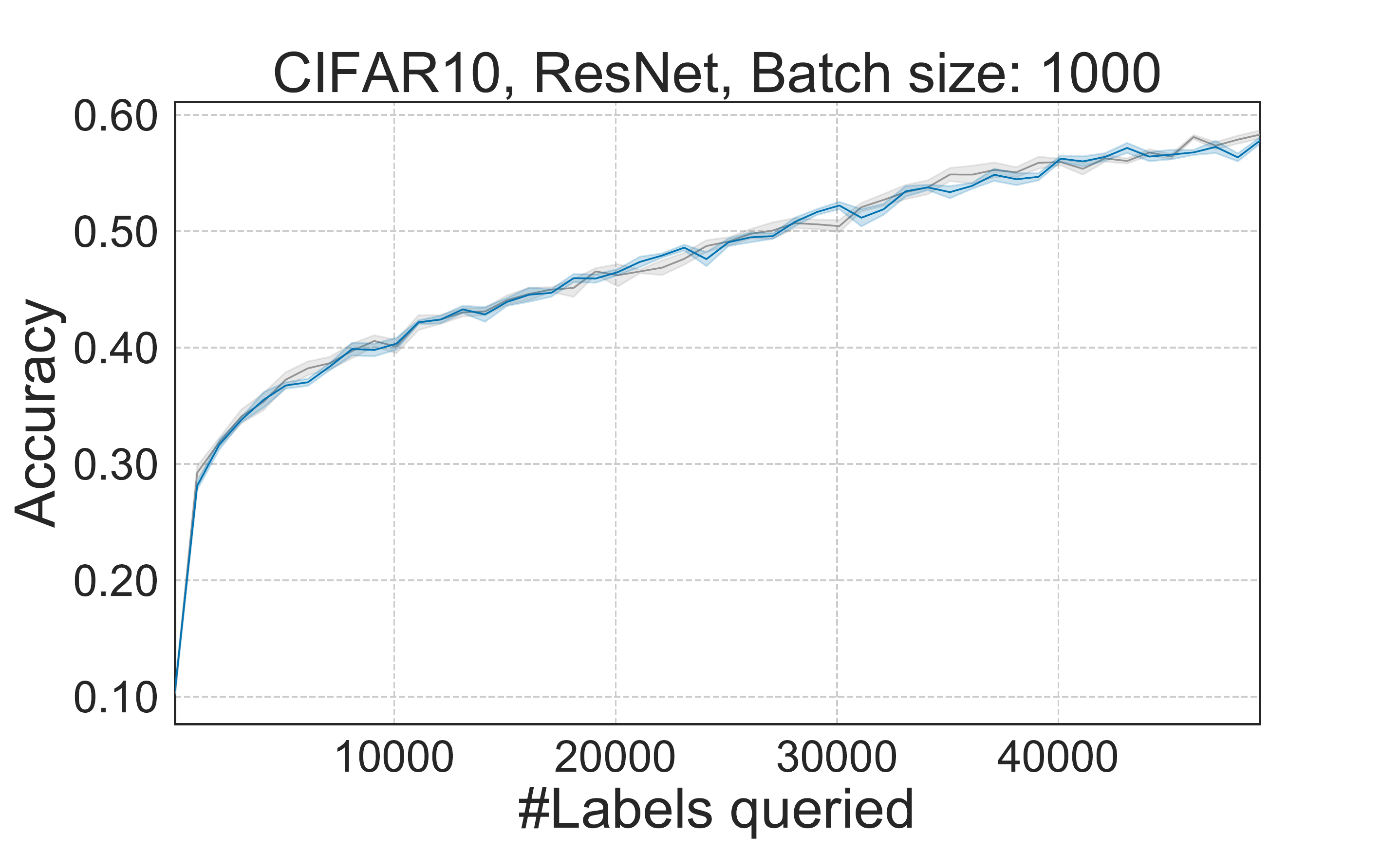}}
  \includegraphics[trim={0.3cm 0cm 2.5cm 0cm}, clip, width=0.24\textwidth]{{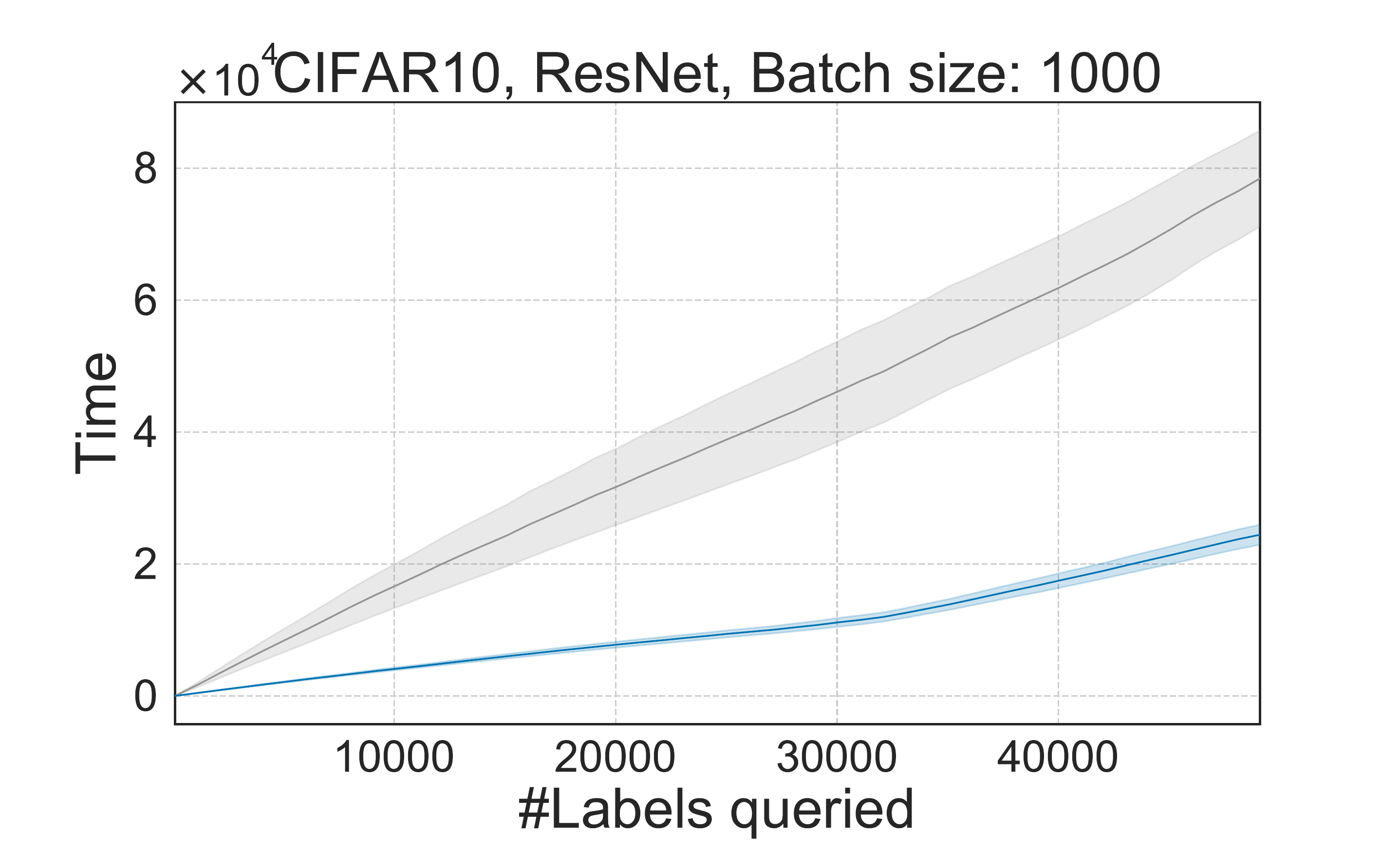}}
  \\
  \centering
  \begin{subfigure}[b]{0.32\linewidth}
   \includegraphics[trim={0cm 0cm 0cm 0cm}, clip, width=\textwidth]{figs/dpp_learning_curves/legend.pdf}
  \end{subfigure}
\caption{Learning curves and running times for CIFAR10 with MLP and ResNet.}
\label{fig:comp-cifar10}
\end{figure}

\end{document}